\newcommand*{\NEURIPS}{}
\newcommand*{\CAMREADY}{}

\ifdefined\ARXIV

\documentclass{article} %
\usepackage[table]{xcolor}
\usepackage{arxiv}
\usepackage{libertine} % Different font style
\usepackage{pifont}
\usepackage{footmisc}
\usepackage{hyperref}
\definecolor{myblue}{rgb}{0.15,0.35,0.8}
\hypersetup{ %
	pdftitle={},
	pdfauthor={},
	pdfsubject={},
	pdfkeywords={},
	colorlinks=true,
	linkcolor=myblue,
	citecolor=myblue,
	filecolor=myblue,
	urlcolor=myblue
}
\usepackage{url}

\def\papertitle{When Errors Can Be Beneficial: \\ A Categorization of Imperfect Rewards for Policy Gradient} % Used for defining main title
\def\runningtitle{When Errors Can Be Beneficial: A Categorization of Imperfect Rewards for Policy Gradient} % Defines the running title on the top of each page

\title{\papertitle}

\author{
	Author 1, Author 2, Author 3, Author 4, Author 5 \\[1.75mm]
	~\hspace{0.05mm}Princeton Language and Intelligence, Princeton University
}

\fi

\ifdefined\NEURIPS
	\PassOptionsToPackage{numbers}{natbib}
	\documentclass{article}
	\ifdefined\CAMREADY
		\usepackage[final,main]{neurips_2025_arxiv}
	\else
		\usepackage{neurips_2025_arxiv}
	\fi
	\usepackage{footmisc}
	\usepackage[utf8]{inputenc} 
	\usepackage[T1]{fontenc}    
	 \usepackage{xcolor}
	\usepackage{hyperref}
	\definecolor{myblue}{rgb}{0.176,0.372,0.85}
	\hypersetup{ %
		pdftitle={},
		pdfauthor={},
		pdfsubject={},
		pdfkeywords={},
		colorlinks=true,
		linkcolor=myblue,
		citecolor=myblue,
		filecolor=myblue,
		urlcolor=myblue}
	\usepackage{url}            	
	\usepackage{booktabs}       
	\usepackage{amsfonts}       
	\usepackage{nicefrac}       	
	\usepackage{microtype}      
\fi

\ifdefined\CVPR
	\documentclass[10pt,twocolumn,letterpaper]{article}
	\usepackage{footmisc}
	\usepackage{cvpr}
	\usepackage{times}
	\usepackage{epsfig}
	\usepackage{graphicx}
	\usepackage{amsmath}
	\usepackage{amssymb}
	\usepackage[square,comma,numbers]{natbib}
\fi
\ifdefined\AISTATS
	\documentclass[twoside]{article}
	\ifdefined\CAMREADY
		\usepackage[accepted]{aistats2015}
	\else
		\usepackage{aistats2015}
	\fi
	\usepackage{url}
	\usepackage[round,authoryear]{natbib}
\fi
\ifdefined\ICML
	\documentclass{article}
	\usepackage{footmisc}
	\usepackage{microtype}
	\usepackage{graphicx}
	\usepackage{booktabs}
	\usepackage{xcolor}
	\usepackage{hyperref}
	\definecolor{myblue}{rgb}{0.176,0.372,0.85}
	\hypersetup{ %
		pdftitle={},
		pdfauthor={},
		pdfsubject={},
		pdfkeywords={},
		colorlinks=true,
		linkcolor=myblue,
		citecolor=myblue,
		filecolor=myblue,
		urlcolor=myblue}

	\ifdefined\CAMREADY
		\usepackage[accepted]{icml2021}
	\else
		\usepackage{icml2021}
	\fi
\fi
\ifdefined\ICLR
	\documentclass{article}
	\usepackage{iclr2019_conference,times}
	\usepackage{footmisc}
	\usepackage{xcolor}
	\usepackage{hyperref}
	\definecolor{myblue}{rgb}{0.176,0.372,0.85}
	\hypersetup{ %
		pdftitle={},
		pdfauthor={},
		pdfsubject={},
		pdfkeywords={},
		colorlinks=true,
		linkcolor=myblue,
		citecolor=myblue,
		filecolor=myblue,
		urlcolor=myblue}
	\usepackage{url}

	\ifdefined\CAMREADY
		\iclrfinalcopy
	\fi
\fi
\ifdefined\COLT
	\ifdefined\CAMREADY
		\documentclass{colt2017}
	\else
		\documentclass[anon,12pt]{colt2017}
	\fi
	\usepackage{times}
\fi

\usepackage{dsfont}
\usepackage{booktabs,multirow}       % professional-quality tables
\usepackage{color}
\usepackage{amsfonts}
\usepackage{algorithm}
\usepackage{algorithmic}
\usepackage{graphicx}
\usepackage{nicefrac}
\usepackage{bbm}
\usepackage{wrapfig}
\usepackage{xcolor}
\usepackage{tikz}
\usepackage[skins,theorems]{tcolorbox}
\usepackage[titletoc,title]{appendix}
\usepackage{stmaryrd}
\usepackage{comment}
\usepackage{endnotes}
\usepackage{hyperendnotes}
\usepackage{enumerate}
\usepackage{enumitem}
\usepackage[normalem]{ulem}
	
\usepackage{titlesec}
\ifdefined\COLT
\else
	\usepackage{amsmath,amsthm,amssymb,mathtools}
\fi
\usepackage[small]{caption}
\usepackage[caption = false]{subfig}
\usepackage[extdef=true]{delimset}

\ifdefined\COLT
	\newtheorem{claim}[theorem]{Claim}
	\newtheorem{fact}[theorem]{Fact}
	\newtheorem{procedure}{Procedure}
	\newtheorem{conjecture}{Conjecture}	
	\newtheorem{hypothesis}{Hypothesis}	
	\newcommand{\qed}{\hfill\ensuremath{\blacksquare}}
\else
	\newtheorem{lemma}{Lemma}
	\newtheorem{corollary}{Corollary}
	\newtheorem{theorem}{Theorem}
	\newtheorem{proposition}{Proposition}
	\newtheorem{assumption}{Assumption}

	\theoremstyle{definition}
	\newtheorem{definition}{Definition}
	\newtheorem{remark}{Remark}
\fi

\usepackage{zref-clever}
\zcsetup{
  abbrev = false,
  cap = true,
  nameinlink = false,
  hyperref = true 
}
\zcRefTypeSetup{assumption}{
	Name-sg = Assumption,
	name-sg = assumption,
	Name-pl = Assumptions,
	name-pl = assumptions
}
\zcsetup{
  lastsep  = {{, and\nobreakspace}},
  tlastsep = {{, and\nobreakspace}},
  comp=false
}

\definecolor{green}{rgb}{0.0, 0.5, 0.0}

\definecolor{xcolor-gray}{gray}{0.95}

\definecolor{harmful}{rgb}{0.9529411764705882,0.3137254901960784,0.3137254901960784}
\definecolor{benign}{rgb}{1,0.5647058823529412,0}
\definecolor{beneficial}{rgb}{0.12941176470588237,0.6941176470588235,0.42745098039215684}
\colorlet{harmfulbg}{harmful!10}
\colorlet{benignbg}{benign!10}
\colorlet{beneficialbg}{beneficial!10}

\newtcolorbox{harmfulbox}{
	colback=harmfulbg,
	colframe=harmful,
	boxrule=0.75pt,
	enhanced,
	top=6pt,
	bottom=5pt,
	left=6pt,
	right=6pt
}
\newtcolorbox{benignbox}{
	colback=benignbg,
	colframe=benign,
	boxrule=0.75pt,
	enhanced,
	top=6pt,
	bottom=5pt,
	left=6pt,
	right=6pt
}
\newtcolorbox{beneficialbox}{
	colback=beneficialbg,
	colframe=beneficial,
	boxrule=0.75pt,
	enhanced,
	top=6pt,
	bottom=5pt,
	left=6pt,
	right=6pt
}

\newcommand\harmful[1]{{\color{harmful}#1}}
\newcommand\benign[1]{{\color{benign}#1}}
\newcommand\beneficial[1]{{\color{beneficial}#1}}

\definecolor{bg_lightblue}{rgb}{0.9529,0.9725,1}
\definecolor{darkolive}{rgb}{0.1607,0.1803,0.1176}
\tcbset{
	takeawaybox/.style={
		top=10pt,
		bottom=5pt,
		colback=bg_lightblue,
		colframe=darkolive,
		colbacktitle=darkolive,
		boxrule=1pt,
		enhanced,
		center,
		attach boxed title to top left={yshift=-0.1in,xshift=0.08in},
		boxed title style={boxrule=0pt,colframe=white}
	}
}
\newtcolorbox[auto counter]{takeawaybox}[2][]{
	takeawaybox,
	title=Takeaway~\thetcbcounter: #2,#1
}

\def\be{\begin{equation}}
	\def\ee{\end{equation}}
\def\beas{\begin{eqnarray*}}
	\def\eeas{\end{eqnarray*}}
\def\bea{\begin{eqnarray}}
	\def\eea{\end{eqnarray}}

\newcommand{\ebf}{{\mathbf e}}

\newcommand{\Z}{{\mathcal Z}}

\newcommand{\D}{{\mathcal D}}

\renewcommand{\S}{{\mathcal S}}

\newcommand{\X}{{\mathcal X}}
\newcommand{\Y}{{\mathcal Y}}
\newcommand{\OO}{{\mathcal O}}

\newcommand{\EE}{\mathop{\mathbb E}} % expectation operator
\newcommand{\R}{{\mathbb R}}

\newcommand{\indc}[1]{\mathds{1} \brk[s]*{ #1 }}

\newcommand{\inprod}[2]  {\left\langle{#1},{#2}\right\rangle}

\newcommand{\inprodBig}[2]  {\Big \langle{#1},{#2} \Big \rangle}

\DeclareMathOperator*{\argmax}{argmax}

\newcommand{\sign}{\mathrm{sign}}

\DeclareMathOperator{\var}{Var}

\newcommand{\proxyreward}{r_{\mathrm{P}}}
\newcommand{\Vproxy}{V_{\mathrm{P}}}
\newcommand{\Vproxybar}{\widebar{V}_{\mathrm{P}}}
\newcommand{\gtreward}{r_{\mathrm{G}}}
\newcommand{\Vgt}{V_{\mathrm{G}}}
\newcommand{\ystar}{y_{\star}}
\newcommand{\ymed}{y_{\mathrm{med}}}
\newcommand{\Ybad}{\Y_{\mathrm{bad}}}
\newcommand{\ybad}{y_{\mathrm{bad}}}
\newcommand{\advproxy}{A_{\mathrm{P}}}
\newcommand{\advgt}{A_{\mathrm{G}}}
\newcommand{\adv}{A}
\newcommand{\yw}{y^{+}}
\newcommand{\yl}{y^{-}}

\newcommand{\acc}{\mathrm{Acc}}
\newcommand{\hacc}{\mathrm{HAcc}}
\newcommand{\TV}{\mathrm{TV}}
\newcommand{\teps}{t_{\star}}
\newcommand{\tepsnomed}{t^{\mathrm{no-med}}_{\star}}
\newcommand{\tepsnomedP}{t^{\mathrm{no-med}}_{\mathrm{P}}}
\newcommand{\Deltaone}{\Delta_1}
\newcommand{\Deltatwo}{\Delta_2}
\newcommand{\thetagt}{\theta^{\mathrm{G}}}

\DeclareFontFamily{U}{mathx}{\hyphenchar\font45}
\DeclareFontShape{U}{mathx}{m}{n}{<-> mathx10}{}
\DeclareSymbolFont{mathx}{U}{mathx}{m}{n}
\DeclareMathAccent{\widebar}{0}{mathx}{"73}

\definecolor{darkspringgreen}{rgb}{0.09, 0.45, 0.27}
\usetikzlibrary{decorations.pathreplacing,calligraphy}
\usetikzlibrary{patterns}

\ifdefined\ENABLEENDNOTES
	\renewcommand{\theendnote}{\alph{endnote}} 
\else
	\renewcommand{\endnote}[1]{\null} 
\fi

\ifdefined\CVPR
	\usepackage[breaklinks=true,bookmarks=false]{hyperref}
	\ifdefined\CAMREADY
		\cvprfinalcopy 
	\fi
\fi

\ifdefined\ARXIV
	\newcommand*{\ABBR}{}
\fi
\ifdefined\ICML
	\newcommand*{\ABBR}{}
\fi
\ifdefined\NEURIPS
	\newcommand*{\ABBR}{}
\fi
\ifdefined\ICLR
	\newcommand*{\ABBR}{}
\fi
\ifdefined\COLT
	\newcommand*{\ABBR}{}
\fi
\ifdefined\ABBR
	\newcommand{\eg}{{\it e.g.}}
	\newcommand{\ie}{{\it i.e.}}
	\newcommand{\cf}{{\it cf.}}

\fi

\begin{document}
	
	% SPACING (only use when absolutely necessary!)
	% Use to control size below and above equations
	%	\setlength{\abovedisplayskip}{1.2pt}
	%	\setlength{\belowdisplayskip}{1.2pt}
	%	\setlength{\abovedisplayshortskip}{1.2pt}
	%	\setlength{\belowdisplayshortskip}{1.2pt}
	% Use to control size of itemize and enumerate environments
	%	\setenumerate{itemsep=0pt}
	%	\setitemize{itemsep=1pt}
	
	% TITLE AND AUTHORS
	\ifdefined\ARXIV
		\maketitle
	\fi
	\ifdefined\NEURIPS
	\title{When Errors Can Be Beneficial: A Categorization of Imperfect Rewards for Policy Gradient}
	\author{
		\textbf{Shuning Shang\thanks{Equal contribution.} \,, Hubert Strauss\footnotemark[1] \,, Stanley Wei, Sanjeev Arora, Noam Razin}\\[1.75mm]
		Princeton Language and Intelligence, Princeton University
		}
		\maketitle
	\fi
	\ifdefined\CVPR
		\title{Paper Title}
		\author{
			Author 1 \\
			Author 1 Institution \\	
			\texttt{author1@email} \\
			\and
			Author 2 \\
			Author 2 Institution \\
			\texttt{author2@email} \\	
			\and
			Author 3 \\
			Author 3 Institution \\
			\texttt{author3@email} \\
		}
		\maketitle
	\fi
	\ifdefined\AISTATS
		\twocolumn[
		\aistatstitle{Paper Title}
		\ifdefined\CAMREADY
			\aistatsauthor{Author 1 \And Author 2 \And Author 3}
			\aistatsaddress{Author 1 Institution \And Author 2 Institution \And Author 3 Institution}
		\else
			\aistatsauthor{Anonymous Author 1 \And Anonymous Author 2 \And Anonymous Author 3}
			\aistatsaddress{Unknown Institution 1 \And Unknown Institution 2 \And Unknown Institution 3}
		\fi
		]	
	\fi
	\ifdefined\ICML
		\icmltitlerunning{Paper Title}
		\twocolumn[
		\icmltitle{Paper Title} 
		\icmlsetsymbol{equal}{*}
		\begin{icmlauthorlist}
			\icmlauthor{Author 1}{inst} % Add ''equal'' next to institution identifier if appropriate
			\icmlauthor{Author 2}{inst}
		\end{icmlauthorlist}
		\icmlaffiliation{inst}{Some Institute}
		\icmlcorrespondingauthor{Author 1}{author1@email}
		\icmlkeywords{}
		\vskip 0.3in
		]
		\printAffiliationsAndNotice{} % Add \icmlEqualContribution inside {} if appropriate
	\fi
	\ifdefined\ICLR
		\title{Paper Title}
		\author{
			Author 1 \\
			Author 1 Institution \\
			\texttt{author1@email}
			\And
			Author 2 \\
			Author 2 Institution \\
			\texttt{author2@email}
			\And
			Author 3 \\ 
			Author 3 Institution \\
			\texttt{author3@email}
		}
		\maketitle
	\fi
	\ifdefined\COLT
		\title{Paper Title}
		\coltauthor{
			\Name{Author 1} \Email{author1@email} \\
			\addr Author 1 Institution
			\And
			\Name{Author 2} \Email{author2@email} \\
			\addr Author 2 Institution
			\And
			\Name{Author 3} \Email{author3@email} \\
			\addr Author 3 Institution}
		\maketitle
	\fi

	% ABSTRACT
	\vspace{-3mm}
\begin{abstract}
\vspace{-1.5mm}
Training language models via reinforcement learning often relies on imperfect \emph{proxy} rewards, since \emph{ground truth} rewards that precisely define the intended behavior are rarely available.
Standard metrics for assessing the quality of proxy rewards, such as ranking accuracy, treat incorrect rewards as strictly harmful.
In this work, however, we highlight that not all deviations from the ground truth are equal.
By theoretically analyzing which outputs attract probability during policy gradient optimization, we categorize reward errors according to their effect on the increase in ground truth reward.
The analysis establishes that reward errors, though conventionally viewed as harmful, can also be benign or even beneficial by preventing the policy from stalling around outputs with mediocre ground truth reward.
We then present two practical implications of our theory.
First, for reinforcement learning from human feedback (RLHF), we develop reward model evaluation metrics that account for the harmfulness of reward errors.
Compared to standard ranking accuracy, these metrics typically correlate better with the performance of a language model after RLHF, yet gaps remain in robustly evaluating reward models.
Second, we provide insights for reward design in settings with verifiable rewards.
A key theme underlying our results is that the effectiveness of a proxy reward function depends heavily on its interaction with the initial policy and learning algorithm.
\end{abstract}

	% KEYWORDS
	\ifdefined\COLT
		\medskip
	\fi

	% Add main paper sections here
	
	% INTRODUCTION
	\section{Introduction}
\label{sec:intro}
\vspace{-0.5mm}

Training language models via reinforcement learning commonly relies on imperfect \emph{proxy} rewards.
This reliance is largely unavoidable since specifying \emph{ground truth} rewards that exactly capture the intended behavior is rarely feasible.
For example, in reinforcement learning from human feedback (RLHF) \citep{ouyang2022training}, learned reward models serve as proxies for a ground truth reward that is assumed to govern human preferences.
Moreover, so-called verifiable rewards \citep{lambert2024tulu,guo2025deepseek} are often also imperfect proxies: rewards for mathematical reasoning usually do not verify intermediate steps and can suffer from incorrect parsing \citep{huang2025pitfalls,tao2025hybrid}, and
rewards for code generation are based on unit tests, which are inherently incomplete \citep{liu2023rltf,guo2025deepseek,olmo2025olmo}.

Despite the widespread use of proxy rewards, there is limited understanding of how discrepancies between the proxy and ground truth rewards affect the reinforcement learning process.
Arguably the most well-known potential consequence of such discrepancies is \emph{reward hacking}, where maximizing the proxy reward results in poor ground truth performance \citep{amodei2016concrete,skalse2022defining,pang2023reward,gao2023scaling,karwowski2024goodhart,fluri2025perils,laidlaw2025correlated}.
Reflecting this concern, standard metrics for assessing the quality of a proxy reward function, including ranking accuracy \citep{lambert2025rewardbench} and mean squared error \citep{huang2025best}, consider deviations from the ground truth reward as strictly harmful.
However, this coarse treatment overlooks the possibility that different erroneous proxy rewards may influence the learning process in distinct ways.

In this work, we highlight that not all reward errors are equal, or even necessarily harmful.
Formally, \emph{reward error} refers to cases where $\proxyreward(x,y) \neq \gtreward(x,y)$, for a proxy reward function~$\proxyreward$, ground truth reward function $\gtreward$, and one or more input-output pairs $(x, y)$.
Focusing on \emph{policy gradient}---the predominant approach for training language models via reinforcement learning---we theoretically characterize the effect of reward errors on the increase in ground truth reward (\zcref{sec:categorization}).
In particular, we categorize reward errors into \emph{harmful}, \emph{benign}, or \emph{beneficial}.
Harmful errors, aside from causing reward hacking, can slow down the rate at which the ground truth reward increases.
Benign errors have negligible impact on optimization.
Lastly, and perhaps most strikingly, we prove that some errors are beneficial, accelerating the increase in ground truth reward beyond what is achieved when optimizing the ground truth reward directly.
See \zcref{fig:rm_errors_categorization} for an overview of our categorization.

\begin{figure*}[t]
	\vspace{-7mm}
	\begin{center}
        \includegraphics[width=1\textwidth]{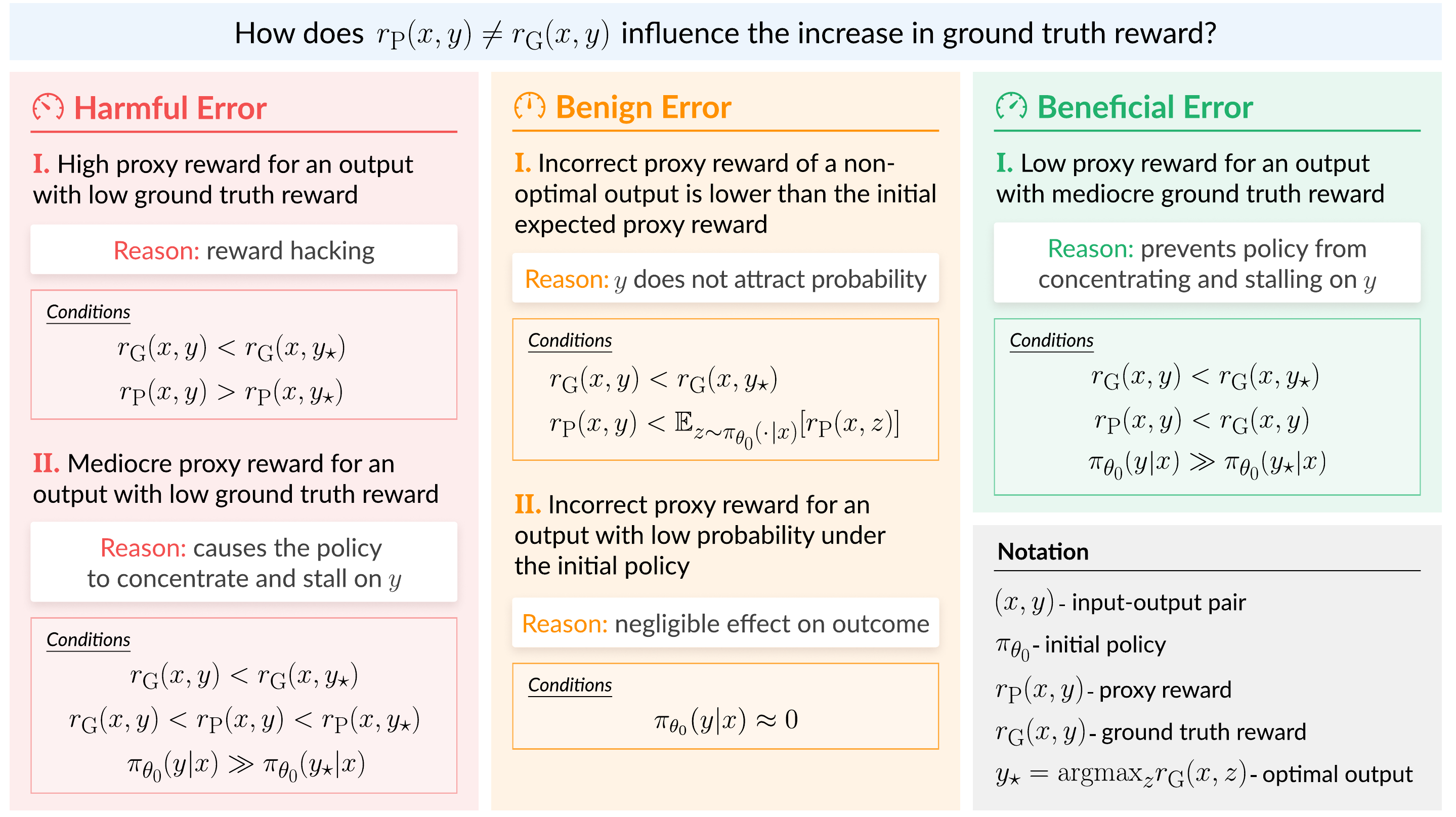}
	\end{center}
	\vspace{-1.5mm}
	\caption{
		\textbf{Reward error categorization overview.}
		We categorize \emph{reward errors}---cases where the proxy and ground truth rewards disagree on one or more input-output pairs---according to their effect on the increase in ground truth reward under policy gradient optimization (\zcref{sec:categorization}).
		Aside from being harmful, we prove that reward errors can also be benign or even beneficial.
		This categorization depends on the interplay between the proxy reward function, initial policy, and policy gradient updates.
		\zcref{sec:modified_acc,sec:reward_design} explore practical implications of our theory for reward model evaluation in RLHF and reward design in settings with verifiable rewards.
	}
	\label{fig:rm_errors_categorization}
	\vspace{-1mm}
\end{figure*}

We derive this categorization by analyzing which outputs attract probability during policy gradient optimization for linear softmax policies.
The main theoretical result, which underlies the existence of beneficial reward errors, shows that outputs with mediocre reward can attract probability and stall policy gradient in their vicinity for an arbitrarily long time.
Thus, assigning low proxy reward to outputs with mediocre ground truth reward can accelerate the increase in ground truth reward by steering the policy away from such outputs.
The analysis is corroborated by controlled experiments.

As a practical application of our theory, we consider the problem of reward model evaluation for RLHF (\zcref{sec:modified_acc}).
Reward models are primarily evaluated through ranking accuracy \citep{lambert2025rewardbench,zhou2025rmb,liu2025rm,frick2025evaluate,malik2025rewardbench}, which treats all incorrect rankings of outputs as equally harmful.
Based on our categorization of reward errors, we develop ranking accuracy variants that account for the harmfulness of an incorrect ranking to policy gradient optimization.
Experiments show that these harm-aware metrics typically correlate better with the performance of a language model after RLHF, across datasets and model families (Llama \cite{grattafiori2024llama}, OLMo \cite{olmo2025olmo}, and Qwen \cite{yang2025qwen3}).
Yet, despite these gains, the correlation can still be weak, showcasing challenges in robustly evaluating reward models.

Beyond reward model evaluation, we explore implications of our theory for reward design in settings with verifiable rule-based rewards (\zcref{sec:reward_design}).
Namely, we demonstrate that rewarding partially correct outputs (\cf~\cite{pyatkin2025generalizing,sun2026rl}) can be detrimental for optimization if the initial policy is noticeably more likely to produce partially correct outputs than fully correct ones.

A central theme of this work is that the effectiveness of proxy rewards cannot be assessed solely by how much they deviate from the ground truth.
It is essential to consider the interplay between the proxy reward function, initial policy, and learning algorithm.
Our results take a step towards characterizing principles for proxy reward evaluation and design that account for this interplay.
We hope that they will inspire further research in this direction, given the growing usage of reinforcement learning for training language models in complex environments.

	% PRELIMINARIES
	\section{Preliminaries}
\label{sec:prelim}
\vspace{-0.5mm}

We use $\pi_\theta$ to denote a language model, parameterized by $\theta$, that maps an input (\ie, prompt) $x \in \X$ to a distribution over outputs $y \in \Y$, where $\X$ and $\Y$ are the input and output spaces, respectively.
We will also refer to $\pi_\theta$ as a \emph{policy}.

\subsection{Reinforcement Learning for Language Models}
\label{sec:prelim:rl_for_lm_post_training}

Reinforcement learning is a key component of language model post-training pipelines \citep{ouyang2022training,olmo2025olmo,guo2025deepseek,yang2025qwen3}.
In many cases, it is difficult to specify a \emph{ground truth} reward function $\gtreward : \X \times \Y \to [-1, 1]$ that exactly defines the intended objective.
Thus, a \emph{proxy} reward function $\proxyreward : \X \times \Y \to [-1, 1]$ is used instead.\footnote{
    As typical in language model applications, we consider rewards defined over complete outputs.
}
A prominent example is reinforcement learning from human feedback (RLHF) \citep{ouyang2022training}, where $\proxyreward$ is a learned reward model and $\gtreward$ is a reward function assumed to capture human preferences.
Perhaps less obvious examples arise in environments with so-called verifiable rewards, such as mathematical reasoning and code generation.
An ideal reward function would distinguish fully correct outputs from partially correct or incorrect ones.
Yet, rewards in math environments typically do not verify intermediate steps and can suffer from incorrect parsing \citep{lambert2024tulu,huang2025pitfalls,tao2025hybrid}, and rewards for code generation rely on unit tests, which are fundamentally limited \citep{guo2025deepseek,olmo2025olmo}.

Given a proxy reward function $\proxyreward$, the language model $\pi_\theta$ is usually optimized via \emph{policy gradient} methods  (\eg, PPO \citep{schulman2017proximal}, RLOO \citep{ahmadian2024back}, and GRPO \citep{shao2024deepseekmath}).
This amounts to maximizing through gradient updates the expected proxy reward $\EE\nolimits_{x \sim \D, y \sim \pi_\theta (\cdot | x)} \brk[s]*{ \proxyreward (x, y) }$, where $\D$ is a distribution over inputs.
Optionally, a KL regularization term may be included to penalize divergence from some reference policy (\cf~\cite{ziegler2019fine,stiennon2020summarize,ouyang2022training}).
The premise is that, if $\proxyreward$ is a good proxy for $\gtreward$, increasing expected proxy reward should also increase the expected ground truth reward $\EE\nolimits_{x \sim \D, y \sim \pi_\theta (\cdot | x)} \brk[s]*{ \gtreward (x, y) }$.

\subsection{Evaluating the Quality of Proxy Reward Functions}
\label{sec:prelim:rm_eval}

Assessing the quality of a proxy reward function $\proxyreward$ is a problem of both practical and theoretical interest \citep{gleave2021quantifying,wulfe2022dynamics,skalse2024starc,karwowski2024goodhart,lambert2025rewardbench}.
This problem is particularly salient in the context of RLHF, where \emph{ranking accuracy} (\zcref{def:ranking_accuracy}) is the primary metric used to evaluate reward models \citep{lambert2025rewardbench,liu2025rm,frick2025evaluate,zhou2025rmb,malik2025rewardbench}.
A main reason for its adoption in practice is that it does not require direct access to the ground truth $\gtreward$, relying instead on preference labels that can be obtained through human annotation.
\begin{definition}
\label{def:ranking_accuracy}
For a set $\S$ containing preference examples $(x, \yw, \yl_1, \ldots, \yl_K)$, where $x \in \X$ is an input and $\yw, \yl_1, \ldots, \yl_K \in \Y$ are outputs satisfying $\gtreward (x, \yw) > \max\nolimits_{k \in [K]} \gtreward (x, \yl_k)$, the \emph{ranking accuracy} of $\proxyreward$ is given by:
\[
\acc (\proxyreward ; \S) := \frac{1}{|\S|} \sum\nolimits_{(x, \yw, \{\yl_k\}_{k \in [K]}) \in \S} \indc{ \max\nolimits_{k \in [K]} \proxyreward (x, \yl_k) < \proxyreward (x, \yw)}
\text{\,,}
\]
where $\indc{\cdot}$ is an indicator function and $[K] := \{1, \ldots, K \}$.
\end{definition}
Ranking accuracy, along with other metrics for evaluating proxy rewards in the literature \citep{gleave2021quantifying,wulfe2022dynamics,skalse2024starc,huang2025best}, rests on a common assumption: deviations of $\proxyreward$ from $\gtreward$ are strictly harmful.
Moreover, it treats all incorrect rankings of outputs as equal.
Our aim is to show that this view is too coarse, as it disregards the fact that different \emph{reward errors} (\ie, cases where the proxy and ground truth rewards disagree on one or more input-output pairs) affect the policy optimization process in distinct ways.

	% CATEGORIZATION OF IMPERFECT REWARDS
	\section{Categorization of Reward Errors: Harmful, Benign, and Beneficial}
\label{sec:categorization}
\vspace{-0.5mm}

Reward errors are conventionally viewed as \emph{harmful} due to the risk of reward hacking---if an output with low ground truth reward is assigned a high proxy reward, maximizing the proxy reward may lead to poor ground truth performance \citep{amodei2016concrete,skalse2022defining,pang2023reward,gao2023scaling,karwowski2024goodhart,fluri2025perils}.
However, we prove that aside from being harmful, reward errors can also be \emph{benign} or even \emph{beneficial}.
Specifically, by theoretically analyzing which outputs attract probability during policy gradient optimization, we categorize reward errors according to their effect on the increase in ground truth reward; see \zcref{fig:rm_errors_categorization} for an overview.

\zcref{sec:categorization:setting} presents the technical setting, after which \zcref{sec:categorization:benign} characterizes two types of benign reward errors by showing that outputs with low proxy reward or low probability under the initial policy negligibly affect the outcome of policy gradient.
\zcref{sec:categorization:beneficial} then delivers our main theoretical result, establishing the existence of beneficial reward errors.
It identifies that outputs with mediocre proxy reward can attract probability and trap policy gradient for an arbitrarily long time.
Thus, assigning low proxy reward to outputs with mediocre ground truth reward can accelerate the increase in ground truth reward by steering the policy away from such outputs.
Lastly, the attraction to mediocre outputs reveals an additional type of harmful reward error (\zcref{sec:categorization:harmful}).
Namely, suppose that an output with low ground truth reward is incorrectly assigned a higher proxy reward.
Even if the proxy reward is not high enough to induce reward hacking, this error may still slow down the increase in ground truth reward.
Controlled experiments support our theoretical findings.

\subsection{Technical Setting}
\label{sec:categorization:setting}

\textbf{Policy parameterization.}
We consider bandit environments with a single input (\ie, context) $x \in \X$, finite set of outputs $\Y$, and \emph{linear softmax policies} \citep{agarwal2021theory,mei2023ordering,razin2024vanishing,lin2025rethinking,foster2025good}.
Each output $y \in \Y$ is associated with a feature vector $\phi (y) \in \R^D \setminus \{ 0 \}$ and the policy $\pi_\theta$ is defined by:
\[
\pi_\theta (y) := \frac{ \exp \brk1{ \inprod{\phi (y)}{ \theta } } }{ \sum\nolimits_{z \in \Y} \exp \brk1{ \inprod{\phi (z)}{ \theta } } }
\text{\,,}
\]
where $\theta \in \R^D$.
For conciseness, in this section we omit the input $x \in \X$ from our notation.
Despite the apparent simplicity of linear softmax policies, they induce non-concave proxy and ground truth objectives for $\proxyreward, \gtreward : \Y \to [-1, 1]$:
\[
\Vproxy (\theta) := \EE\nolimits_{y \sim \pi_{\theta}} \brk[s]{ \proxyreward (y) } ~~ , ~~ \Vgt (\theta) := \EE\nolimits_{y \sim \pi_{\theta}} \brk[s]{ \gtreward (y) }
\text{\,,}
\]
whose optimization properties are not well understood \citep{mei2020global,agarwal2021theory,li2021softmax,mei2023ordering,lin2025rethinking}.

\textbf{Optimization.}
Language model post-training is often done with small learning rates (\cf~\cite{xu2024dpo,lambert2024tulu,olmo2025olmo}).
Accordingly, and following prior theoretical analyses of gradient-based optimization with small learning rates (\eg, \cite{saxe2014exact,woodworth2020kernel,razin2020implicit,vardi2021implicit,razin2021implicit,razin2022implicit,safran2022effective,kothapalli2023neural,tirer2023perturbation,chou2023more,slutzky2025implicit,bietti2025learning,ran2026outcome}), we analyze \emph{gradient flow}:
\be
\frac{d}{dt} \theta_t = \nabla \Vproxy (\theta_t) ~~ , ~~ t \geq 0
\text{\,,}
\label{eq:gf}
\ee
where $\theta_t$ denotes the parameters at time $t$ of training and $\pi_{\theta_0}$ is the initial policy.\footnote{
    Our theoretical results can be translated to gradient ascent with a small learning rate via Corollary 4 in \cite{elkabetz2021continuous}.
}

\textbf{Limitations.}
The technical setting involves three main simplifications: \emph{(i)} it considers linear softmax policies; \emph{(ii)} it assumes access to exact gradients of the expected proxy reward $\Vproxy$; and \emph{(iii)} it focuses on bandit environments with a single input.
As \zcref{sec:modified_acc,sec:reward_design} demonstrate empirically, although these simplifications abstract away complexities that may stem from using a neural network architecture, sample-based estimates of $\nabla \Vproxy (\theta)$, or the interaction between inputs, they allow deriving insights that apply to practical settings.

\subsection{Benign Errors}
\label{sec:categorization:benign}

\subsubsection{Proxy Reward Is Lower Than the Initial Expected Proxy Reward}
\label{sec:categorization:benign:low_reward_outputs}

We begin by making a straightforward, yet important observation: outputs with proxy reward lower than the initial expected proxy reward, $\Vproxy (\theta_0)$, generally do not attract probability.
This becomes apparent by examining the optimization dynamics for the logit $\inprod{\phi (y)}{\theta_t}$ of an output $y \in \Y$.
\begin{proposition}[Proof deferred to \zcref{app:proofs:logit_dynamics}]
\label{prop:logit_dynamics}
Suppose that we maximize via gradient flow the expected proxy reward $\Vproxy$ (\zcref{eq:gf}).
For any output $y \in \Y$ and time $t \geq 0$ it holds that:
\[
\frac{d}{dt} \inprod{\phi (y)}{\theta_t} = \underbrace{ \pi_{\theta_t} (y) \advproxy (y ; \theta_t) \cdot \norm*{ \phi(y) }^2 }_{\text{contribution due to $y$}} + \sum\nolimits_{z \in \Y \setminus \brk[c]{y} } \underbrace{ \pi_{\theta_t} (z) \advproxy (z ; \theta_t) \cdot \inprod{ \phi(z) }{ \phi(y) } }_{ \text{contribution due to other output} }
\text{\,,}
\]
where $\advproxy (z ; \theta_t) := \proxyreward (z) - \Vproxy (\theta_t)$ is the advantage of $z \in \Y$ under $\proxyreward$ and $\pi_{\theta_t}$.
\end{proposition}
Since $\Vproxy (\theta_t)$ is non-decreasing,\footnote{
    The fact that $\Vproxy (\theta_t)$ is monotonically non-decreasing follows from $\frac{d}{dt} \Vproxy (\theta_t) = \norm{ \nabla \Vproxy (\theta_t) }^2 \geq 0$.
}
if $\proxyreward (y) < \Vproxy (\theta_0)$, then the contribution of $y$ to its own logit dynamics is negative for all $t \geq 0$, and is more negative the lower $\proxyreward (y)$ is compared to $\Vproxy (\theta_t)$.
The contribution due to other outputs counteracts this negative push only if the proxy reward function and output features are misaligned, in the sense that there exist outputs with reward higher than $\Vproxy (\theta_0)$ whose features have a large inner product with $\phi (y)$.
Thus, as long as the feature geometry reasonably accords with the proxy rewards, an output $y$ with $\proxyreward (y) < \Vproxy (\theta_0)$ will not attract probability.
This implies that the following reward error type is benign.

\begin{benignbox}
\phantomsection\label{err:benign_1}
\benign{\textbf{Benign Error I:}} \textbf{Incorrect proxy reward is lower than the initial expected proxy reward.}\\
If an output $y$ is assigned an incorrect proxy reward $\proxyreward(y) \neq \gtreward (y)$ satisfying $\proxyreward (y) < \Vproxy (\theta_0)$, then it will generally not attract probability.
As a result, unless $y$ achieves maximal ground truth reward, this reward error usually does not harm the increase in ground truth reward.
\end{benignbox}

\subsubsection{Incorrect Proxy Reward for Improbable Outputs}
\label{sec:categorization:benign:low_prob_outputs}

The logit dynamics in \zcref{prop:logit_dynamics} further reveals that the contribution of an output $y$ is dampened by its probability $\pi_{\theta_t} (y)$.
This suggests that the optimization trajectory is not significantly affected by incorrect rewards assigned to outputs with low probability under the initial policy $\pi_{\theta_0}$.

\zcref{prop:low_prob_outputs} formalizes this prospect.
Namely, suppose that $\proxyreward$ and $\gtreward$ differ only on a set of outputs with low initial probability.
\zcref{prop:low_prob_outputs} shows that policies learned by maximizing these reward functions remain close in total variation distance throughout training.
This is in line with common intuition, by which low probability outputs have little influence since they are unlikely to be sampled when constructing sample-based gradient estimates.
However, we prove a stronger claim: even in an idealized setting with access to exact gradients---where all outputs factor into the update---the contribution of low probability outputs is still negligible.
This is a consequence of the optimization dynamics induced by policies that produce a distribution over outputs via the softmax function.

\begin{proposition}
\label{prop:low_prob_outputs}
Suppose that the proxy reward function $\proxyreward$ and ground truth reward function $\gtreward$ differ only on a set of outputs $\Z \subseteq \Y$, \ie, $\proxyreward (y) = \gtreward (y)$ for all $y \in \Y \setminus \Z$.
Denote by $\theta_t$ and $\thetagt_t$ the parameters at time $t \geq 0$ when maximizing via gradient flow (\zcref{eq:gf}) the expected reward with respect to $\proxyreward$ and $\gtreward$, respectively, starting from the same initial parameters $\theta_0 = \thetagt_0$.
For any time $T \geq 0$ and approximation level $\epsilon > 0$, if $\pi_{\theta_0} (\Z) \leq \frac{2}{\Delta_\Z \exp \brk{10 B^2 T}} \cdot \epsilon$, with $B := \max_{y \in \Y} \norm{\phi (y)}$ and $\Delta_\Z := \max_{y \in \Z} \abs{ \proxyreward (y) - \gtreward (y)}$, then for all $t \in [0, T]$:
\[
\TV \brk1{ \pi_{\theta_t}, \pi_{\thetagt_t} } \leq \epsilon
\text{\,,}
\]
where $\TV (\pi_{\theta_t}, \pi_{\thetagt_t}) := \frac{1}{2} \norm{ \pi_{\theta_t} - \pi_{\thetagt_t} }_1$ is the total variation distance between $\pi_{\theta_t}$ and $\pi_{\thetagt_t}$.
\end{proposition}

\begin{proof}[Proof sketch (full proof in \zcref{app:proofs:low_prob_outputs})]
We prove that $\pi_{\theta_t} (\Z)$ remains low for all $t \in [0, T]$, under both $\proxyreward$ and $\gtreward$, by showing that the rate at which $\pi_{\theta_t} (\Z)$ grows is small whenever $\pi_{\theta_t} (\Z)$ is low.
Since $\proxyreward$ and $\gtreward$ differ only on $\Z$ and, for softmax policies, the contribution of outputs in $\Z$ to the gradient is scaled by their own probability, this implies that the distance between $\nabla \Vproxy (\theta)$ and $\nabla \Vgt (\theta)$ is small at any $\theta$ visited during training (under either $\proxyreward$ or $\gtreward$).
The fact that $\pi_{\theta_t}$ and \smash{$\pi_{\thetagt_t}$} stay close in total variation distance then follows through standard arguments on the stability of gradient flow over smooth objectives to gradient perturbations.
\end{proof}

\zcref{prop:low_prob_outputs} yields an additional type of benign reward error.

\begin{benignbox}
\phantomsection\label{err:benign_2}
\benign{\textbf{Benign Error II:}} \textbf{Incorrect proxy reward for an output with low probability under $\pi_{\theta_0}$.}\\
If an output $y$ is improbable under $\pi_{\theta_0}$, then assigning it an incorrect proxy reward $\proxyreward (y) \neq \gtreward (y)$ negligibly affects the outcome of policy gradient.
\end{benignbox}

\subsection{Beneficial Errors That Prevent Attraction to Mediocre Outputs}
\label{sec:categorization:beneficial}

\zcref{sec:categorization:benign} showed that outputs with low proxy reward or low probability under the initial policy do not attract probability.
Where then does probability mass go during policy gradient?
It is helpful to first consider the case of orthonormal feature vectors $\brk[c]{ \phi (y) }_{y \in \Y}$.
In this setting, the logit dynamics of an output $y \in \Y$ simplifies to (\cf~\zcref{prop:logit_dynamics}):
\[
\frac{d}{dt} \inprod{\phi (y)}{\theta_t} = \pi_{\theta_t} (y) \advproxy (y ; \theta_t)
\text{\,.}
\]
As evident from the equation above, the logit of $y$ grows whenever its advantage $\advproxy (y ; \theta_t) = \proxyreward (y) - \Vproxy (\theta_t)$ is positive.
Though, since the advantage is scaled by the output probability $\pi_{\theta_t} (y)$, policy gradient can be drawn to outputs with mediocre proxy reward at the expense of outputs with higher proxy reward, if the former are more likely under the policy.
This phenomenon was empirically observed in~\cite{mei2020escaping}.
However, the analysis of \cite{mei2020escaping} only showed that if the policy already assigns high probability to a suboptimal output, escaping that region takes a long time, leaving open the question of when and to what extent the policy becomes highly concentrated on such outputs to begin with.

We prove that the attraction to mediocre outputs can be severe, stalling optimization for an arbitrarily long time (\zcref{thm:attraction_to_mediocre_outputs_beneficial_error}).
As we establish below, this implies that the following reward error type is beneficial for improving the rate at which the ground truth reward increases.

\begin{beneficialbox}
\phantomsection\label{err:beneficial}
\textbf{\beneficial{Beneficial Error I:} Low proxy reward for an output with mediocre ground truth reward.}\\
Let $\ymed$ be an output with mediocre ground truth reward, \ie, $\Vgt (\theta_0) < \gtreward (\ymed) < \gtreward (\ystar)$ where $\ystar := \argmax\nolimits_{y \in \Y} \gtreward (y)$ and $\Vgt (\theta_0) :=  \EE\nolimits_{y \sim \pi_{\theta_0}} \brk[s]{ \gtreward (y) }$.
If the probability of $\ymed$ under $\pi_{\theta_0}$ is higher than that of $\ystar$, then assigning low proxy reward to $\ymed$ can accelerate the increase in ground truth reward by steering the policy away from it.
\end{beneficialbox}

\smallskip

Formally, we compare the optimization trajectory across two settings.
In the first setting, we maximize directly a ground truth reward function $\gtreward$ that assigns mediocre reward to an output $\ymed$.
In the second setting, we maximize a proxy reward function $\proxyreward$ that is identical to $\gtreward$, except that it assigns low reward to $\ymed$.
\zcref{thm:attraction_to_mediocre_outputs_beneficial_error} establishes that, if $\pi_{\theta_0} (\ystar)$ is sufficiently small relative to $\pi_{\theta_0} (\ymed)$, the time required to achieve high ground truth reward when maximizing $\gtreward$ directly can be arbitrarily larger than the time required when maximizing $\proxyreward$.
This gap in optimization time arises because, when maximizing $\gtreward$, the policy initially concentrates its probability mass on $\ymed$ and stalls in its vicinity.
By contrast, when maximizing $\proxyreward$, the policy directly shifts its probability mass towards~$\ystar$.

For conciseness, \zcref{thm:attraction_to_mediocre_outputs_beneficial_error} provides an abridged version of the result; see \zcref{app:formal_statements} for the detailed theorem statement.
\zcref{fig:loglin_exps} corroborates \zcref{thm:attraction_to_mediocre_outputs_beneficial_error} by demonstrating empirically the effect of mediocre outputs: the smaller $\pi_{\theta_0}(\ystar)$ is relative to $\pi_{\theta_0} (\ymed)$, the longer policy gradient is delayed by an attraction to $\ymed$ (when $\ymed$ is assigned mediocre proxy reward).

\begin{figure*}[t]
	\vspace{-4mm}
	\begin{center}
        \includegraphics[width=1\textwidth]{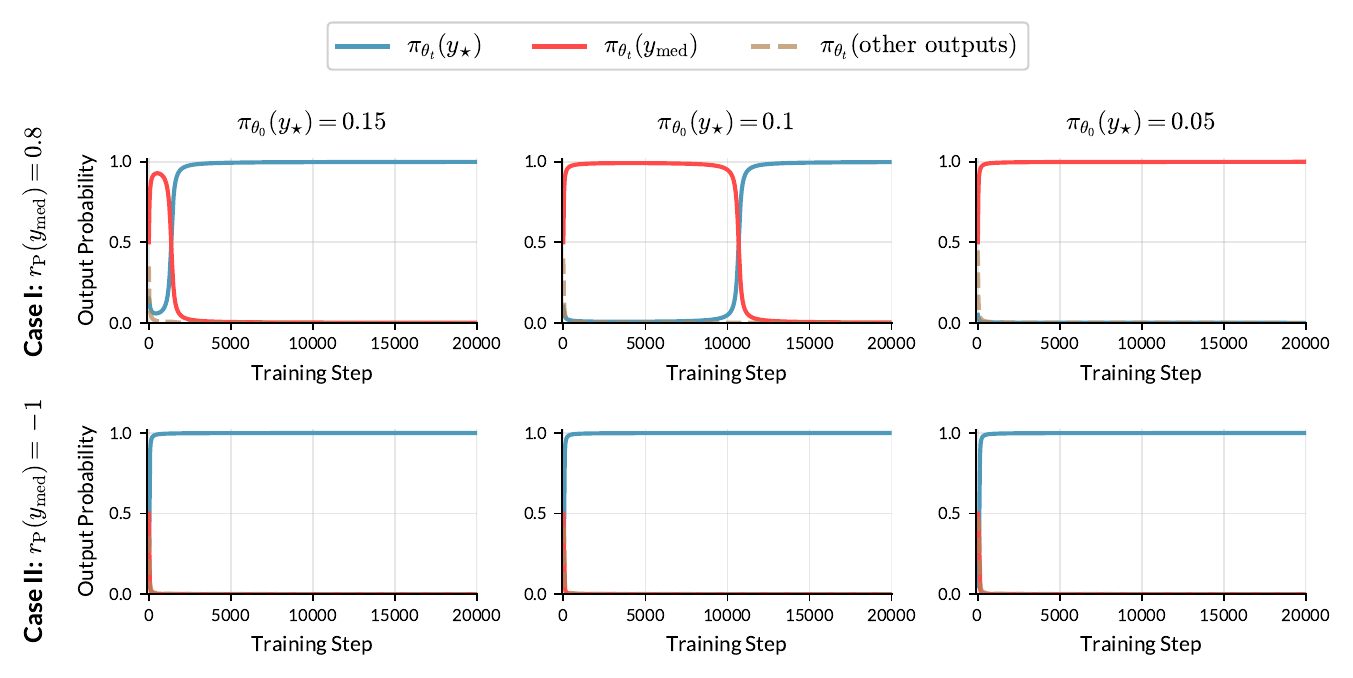}
	\end{center}
	\vspace{-2.25mm}
	\caption{
		\textbf{Attraction to mediocre outputs can impede policy gradient optimization.}
		Plotted is the evolution of output probabilities during policy gradient in settings corresponding to \zcref{thm:attraction_to_mediocre_outputs_beneficial_error}: a linear softmax policy with orthonormal output features, trained using exact gradients of the expected proxy reward.
        The ground truth reward $\gtreward$ assigns a maximal reward of $1$ to $\ystar$, a mediocre reward of $0.8$ to $\ymed$, and a low reward of $-1$ to all remaining outputs.
		In the top row, we use $\proxyreward = \gtreward$.
		Bottom row shows experiments in identical settings, except that $\proxyreward (\ymed) = -1$.
		In line with \zcref{thm:attraction_to_mediocre_outputs_beneficial_error}, when $\ymed$ receives mediocre proxy reward, the policy initially concentrates its probability mass on $\ymed$ and stagnates.
		By contrast, when $\ymed$ is assigned low proxy reward, the policy directly increases the probability of~$\ystar$.
		The gap in optimization time between the two cases becomes larger as $\pi_{\theta_0} (\ystar)$ decreases (left to right).
		\zcref{app:experiments:further} provides experiments demonstrating similar trends with sample-based gradients instead of exact gradients.
		See \zcref{app:experiments:details} for additional implementation details.
	}
	\label{fig:loglin_exps}
\end{figure*}

\begin{theorem}[Abridged version of \zcref{thm:attraction_to_mediocre_outputs_beneficial_error_formal}]
\label{thm:attraction_to_mediocre_outputs_beneficial_error}
Assume that the feature vectors $\brk[c]{ \phi (y) }_{y \in \Y}$ are orthonormal and that $\Y = \brk[c]{ \ystar, \ymed } \cup \Ybad$, where $\ystar = \argmax\nolimits_{y \in \Y} \gtreward (y)$ is an output with maximal ground truth reward, $\ymed$ has mediocre ground truth reward satisfying $\Vgt (\theta_0) = \EE\nolimits_{y \sim \pi_{\theta_0}} \brk[s]{ \gtreward (y) } < \gtreward (\ymed) < \gtreward (\ystar)$, and every output in $\Ybad \subset \Y$ has ground truth reward lower than $\Vgt (\theta_0)$.
Let $\epsilon \in \brk1{0, \gtreward (\ystar) - \gtreward (\ymed)}$ and denote by $\teps$ the initial time at which $\Vgt (\theta_t) \geq \gtreward (\ystar) - \epsilon$ when maximizing the expected ground truth reward directly via gradient flow (\ie, \zcref{eq:gf} with $\Vgt$ in place of $\Vproxy$).
Furthermore, let $\tepsnomed$ be the analogous time when maximizing a proxy reward function $\proxyreward$ that is identical to $\gtreward$, except that it assigns to $\ymed$ minimal reward, starting from the same initial parameters $\theta_0$.
For any $T \geq 0$, if $\pi_{\theta_0} (\ystar)$ is sufficiently small, then:
\[
\teps - \tepsnomed \geq T
\text{\,,}
\]
\ie, the gap between $\teps$ and $\tepsnomed$ can be arbitrarily large.
\end{theorem}

\begin{proof}[Proof sketch (full proof of \zcref{thm:attraction_to_mediocre_outputs_beneficial_error_formal} is in \zcref{app:proofs:attraction_to_mediocre_outputs})]
Consider first the case of maximizing $\gtreward$ directly.
The logit dynamics of an output $y$ in this case is given by $\frac{d}{dt} \inprod{\phi (y)}{\theta_t} = \pi_{\theta_t} (y) \advgt (y ; \theta_t)$, where $\advgt (y ; \theta_t) = \gtreward (y) - \Vgt (\theta_t)$ is the advantage of $y$ under $\gtreward$ and $\pi_{\theta_t}$ (\cf~\zcref{prop:logit_dynamics}).
If $\pi_{\theta_0} (\ymed)$ is substantially higher than $\pi_{\theta_0} (\ystar)$, the logit of the mediocre output $\ymed$ initially grows much faster than that of the optimal output $\ystar$, despite $\ymed$ having a smaller advantage than $\ystar$.
This causes the probability of $\ystar$ to decrease.
By carefully analyzing the optimization trajectory, we show that $\pi_{\theta_t} (\ystar)$ continues to decrease at least until $\pi_{\theta_t} (\ymed)$ is close to one, in the sense that \smash{$\pi_{\theta_t} (\ymed) \geq 1 - \Theta (\pi_{\theta_0} (\ystar)^{14/13})$}.
We then use the fact that the gradient vanishes whenever the policy is highly concentrated on a single output (\cf~\cite{mei2020escaping,agarwal2021theory,razin2024vanishing,razin2025what}) to prove that escaping the vicinity of $\ymed$ requires at least \smash{$\Omega (\pi_{\theta_0} (\ystar)^{-14/13})$} time, and so $\teps$ is at least of that order.
By contrast, when maximizing $\proxyreward$, the advantage of all outputs except $\ystar$ is negative throughout training since $\proxyreward$ assigns them a low reward.
Thus, the policy directly shifts its probability mass towards $\ystar$ and achieves $\epsilon$-optimal ground truth reward at time \smash{$\tepsnomed$}, which we prove is at most $\OO(\pi_{\theta_0} (\ystar)^{-1})$.
This implies that the gap between $\teps$ and \smash{$\tepsnomed$} becomes arbitrarily large as $\pi_{\theta_0} (\ystar) \to 0$.
\end{proof}

\textbf{Influence of feature similarity.}
Under the assumption of orthonormal feature vectors, \zcref{thm:attraction_to_mediocre_outputs_beneficial_error} shows that outputs with mediocre reward can prolong the time required to reach higher-quality outputs.
In \zcref{app:formal_statements}, we analyze how the feature geometry governs the extent to which mediocre outputs impede optimization.
Specifically, we prove that when $\inprod{\phi (\ystar)}{\phi (\ymed)} < 0$, an output $\ymed$ with mediocre reward delays optimization even further and may prevent the policy from ever reaching the optimal output $\ystar$.
Conversely, when $\inprod{\phi (\ystar)}{\phi (\ymed)} > 0$, attraction to $\ymed$ does not necessarily obstruct convergence to $\ystar$.
This highlights that whether outputs with mediocre reward are problematic for optimization also depends on their similarity to higher-quality outputs.

\subsection{Harmful Errors Beyond Reward Hacking}
\label{sec:categorization:harmful}

To conclude the categorization we revisit the most well-known type of reward error (\cf~\cite{amodei2016concrete,skalse2022defining,fluri2025perils}).

\begin{harmfulbox}
\phantomsection\label{err:harmful_1}
\textbf{\harmful{Harmful Error I:} High proxy reward for an output with low ground truth reward.}\\
Let $\ybad$ be an output with low ground truth reward and denote $\ystar := \argmax\nolimits_{y \in \Y} \gtreward (y)$.
If $\proxyreward (\ybad) > \proxyreward(\ystar)$, then maximizing $\proxyreward$ can lead to reward hacking.
\end{harmfulbox}

While assigning high proxy reward to an output $\ybad$ with low ground truth reward is obviously undesirable, the attraction to mediocre outputs characterized in \zcref{sec:categorization:beneficial} reveals a more subtle harmful reward error.
Even if the proxy reward assigned to $\ybad$ is not high enough to cause reward hacking, it may still impede optimization.
This follows as a corollary of \zcref{thm:attraction_to_mediocre_outputs_beneficial_error}; see \zcref{cor:attraction_to_mediocre_outputs_harmful_error} in \zcref{app:formal_statements} for the formal statement.

\begin{harmfulbox}
\phantomsection\label{err:harmful_2}
\textbf{\harmful{Harmful Error II:} Mediocre proxy reward for an output with low ground truth reward.}\\
Let $\ybad$ be an output with low ground truth reward whose probability under $\pi_{\theta_0}$ is higher than the probability of $\ystar := \argmax\nolimits_{y \in \Y} \gtreward (y)$.
Then, assigning a mediocre proxy reward to $\ybad$, satisfying $\Vproxy (\theta_0) < \proxyreward (\ybad) < \proxyreward (\ystar)$ where $\Vproxy (\theta_0) := \EE\nolimits_{y \sim \pi_{\theta_0}} \brk[s]{ \proxyreward (y) }$, can slow down the increase in ground truth reward by causing the policy to stall in the vicinity of $\ybad$.
\end{harmfulbox}

	% APPLICATION: MODIFIED ACC MEASURES
	\section{Application I: Harm-Aware Accuracy for Reward Model Evaluation}
\label{sec:modified_acc}

To demonstrate the potential utility of our theory, we consider the problem of reward model evaluation for RLHF.
Existing benchmarks primarily evaluate reward models via \emph{ranking accuracy} (see \zcref{sec:prelim}).
A limitation of ranking accuracy, uncovered in \zcref{sec:categorization}, is that it treats all incorrect rankings of outputs as equally harmful, while some deviations from the ground truth reward are benign or even beneficial for policy gradient optimization.
Indeed, recent work showed that ranking accuracy often does not correlate well with the performance of a language model after RLHF \citep{chen2024accuracy,wen2025rethinking,razin2025what}.

Motivated by this observation, we develop ranking accuracy variants that account for the harmfulness of an incorrect ranking to policy gradient.
Compared to standard ranking accuracy, our harm-aware variants usually correlate better with language model performance, across datasets and model families (Llama \cite{grattafiori2024llama}, OLMo \cite{olmo2025olmo}, and Qwen \cite{yang2025qwen3}).
However, our experiments also show that the correlation can still be weak, highlighting persistent challenges in robustly evaluating reward models.
For brevity, we defer some experiments and implementation details to \zcref{app:experiments:further,app:experiments:details}, respectively.

\subsection{Harm-Aware Ranking Accuracy}
\label{sec:modified_acc:metrics}

Let $\S$ be a dataset of preference examples \smash{$(x, \yw, \yl_1, \ldots, \yl_K)$}, where $x$ is an input and the output $\yw$ is preferred over $\yl_1, \ldots, \yl_K$ according to a ground truth reward function $\gtreward$.
In practice, we typically do not have access to $\gtreward$.
It is therefore impossible to determine for every incorrect ranking whether it is harmful or not based on the categorization in \zcref{sec:categorization}.
Nevertheless, the categorization allows us to identify certain types of incorrect rankings that are unlikely to be harmful.

For a reward model $\proxyreward$, let \smash{$\Vproxy (x; \theta) := \EE\nolimits_{y \sim \pi_{\theta} (\cdot | x)} \brk[s]{ \proxyreward (x,y)}$} be the expected proxy reward achieved by a language model $\pi_\theta$ given an input $x$. 
If $\proxyreward$ ranks a dispreferred output $\yl_k$ above $\yw$, but \smash{$\proxyreward (x, \yl_k) < \Vproxy (x ; \theta)$}, then this error may be benign since \smash{$\yl_k$} does not attract probability (\hyperref[err:benign_1]{\benign{Benign Error~I}}); it may even be beneficial when $\yw$ is of mediocre quality (\hyperref[err:beneficial]{\beneficial{Beneficial Error I}}).
Accordingly, we define the \emph{harm-aware ranking accuracy (HAcc)} by making a simple modification to ranking accuracy: HAcc does not penalize incorrect rankings whenever all dispreferred outputs receive proxy reward below an empirical estimate of $\Vproxy (x ; \theta)$, denoted $\Vproxybar (x ; \theta)$:
\[
% \vspace{1.8mm}
\hacc (\proxyreward; \S) := \frac{1}{ \abs{\S} }\sum\nolimits_{(x, \yw, \{\yl_k\}_{k \in [K]}) \in \S} \mathds{1} \Big [  \max_{k \in [K] }\proxyreward (x, \yl_k) < \max \brk[c]*{  \proxyreward (x, \yw) , \Vproxybar (x ; \theta)  }  \Big ]
\text{\,.}
\]
Furthermore, reward errors involving outputs that are improbable under $\pi_\theta$ are also benign (\hyperref[err:benign_2]{\benign{Benign Error~II}}).
Thus, in addition to ranking accuracy (Acc) and HAcc, we consider weighted variants (Acc-W and HAcc-W) that multiply the contribution of an example by \smash{$\widebar{\pi}_{\theta} (\yw | x) \prod_{k = 1}^K \widebar{\pi}_\theta (\yl_k | x) / Z (\S ; \pi_\theta)$}, where \smash{$\widebar{\pi}_\theta (y | x) := \pi_{\theta} (y | x)^{1 / \abs{y}}$} and $Z (\S; \pi_\theta)$ is the normalization constant ensuring weights sum to one.
The usage of $\widebar{\pi}_{\theta} (y | x)$ is equivalent to length normalizing the log probability of $y$.
We apply this transformation for numerical stability and to avoid biasing against long sequences.

\textbf{Computational cost.}
Unlike standard ranking accuracy, the proposed harm-aware variants depend on the language model (\ie, policy) $\pi_\theta$.
This dependence is desirable: our analysis, together with prior work \citep{razin2025what}, shows that the effectiveness of a reward model is inherently policy-dependent.
The drawback is a modest computational overhead.
Namely, these variants require sampling from $\pi_\theta$ to estimate $\Vproxybar(x;\theta)$ and computing the probabilities of outputs in $\S$.
In practice, this overhead is small relative to the cost of running policy gradient and is further reduced by reusing samples and probabilities when evaluating multiple reward models.

\subsection{Comparison of Ranking Accuracy Variants}
\label{sec:modified_acc:experiments}

\textbf{Main setting.}
We use four language models of different families and types: Llama-3.2-3B-Instruct, Llama-3.2-1B-Instruct \citep{grattafiori2024llama}, OLMo-2-1B-SFT \citep{olmo20242}, and Qwen3-1.7B-Base \citep{yang2025qwen3}.
Each language model is trained via the RLOO \citep{kool2019buy,ahmadian2024back} policy gradient method over prompts from the UltraFeedback dataset \citep{cui2024ultrafeedback}, using 13 reward models chosen to span a diverse range of RewardBench2 scores \citep{malik2025rewardbench} (see \zcref{tab:rms} in \zcref{app:experiments:details} for the full list).\footnote{
    We adopt RLOO because it is more resource efficient than PPO \citep{schulman2017proximal} and has shown competitive results \citep{ahmadian2024back}.
    We also experimented with GRPO \citep{shao2024deepseekmath} and observed analogous results.
}
We then compare ranking accuracy with its harm-aware variants (computed for the initial language models) based on their predictiveness of which reward model leads to better language model performance.
Following prior work \cite{gao2023scaling,coste2024reward,tang2024understanding,balashankar2024infalign,chen2024accuracy,wen2025rethinking}, we simulate ground truth rewards with a reward model (ArmoRM \citep{wang2024interpretable}), excluded from the ones used for training.
This ground truth model is used for evaluating both language model performance and reward model accuracy.
Unless stated otherwise, rewards and accuracy values are measured on examples from the policy gradient training set.\footnote{
        We mainly consider the ground truth reward increase on training examples as it directly reflects policy gradient optimization, which is the focus of our work, without conflating it with generalization.
        Nonetheless, as reported in \zcref{app:experiments:further}, we found the reward increase on test examples to be nearly identical.
}

\textbf{Additional settings.}
\zcref{app:experiments:further} includes experiments showing similar trends: \emph{(i)} when computing reward model accuracy values on RewardBench2 instead of the policy gradient training set; \emph{(ii)} over the WildChat-IF dataset \citep{lambert2024tulu} instead of UltraFeedback; \emph{(iii)} using a different ground truth reward model; and \emph{(iv)} when evaluating language model performance by win-rates according to a frontier (GPT) judge model, as opposed to by ground truth reward increase.

\begin{figure*}[t]
	\vspace{-3.5mm}
	\begin{center}
        \includegraphics[width=1\textwidth]{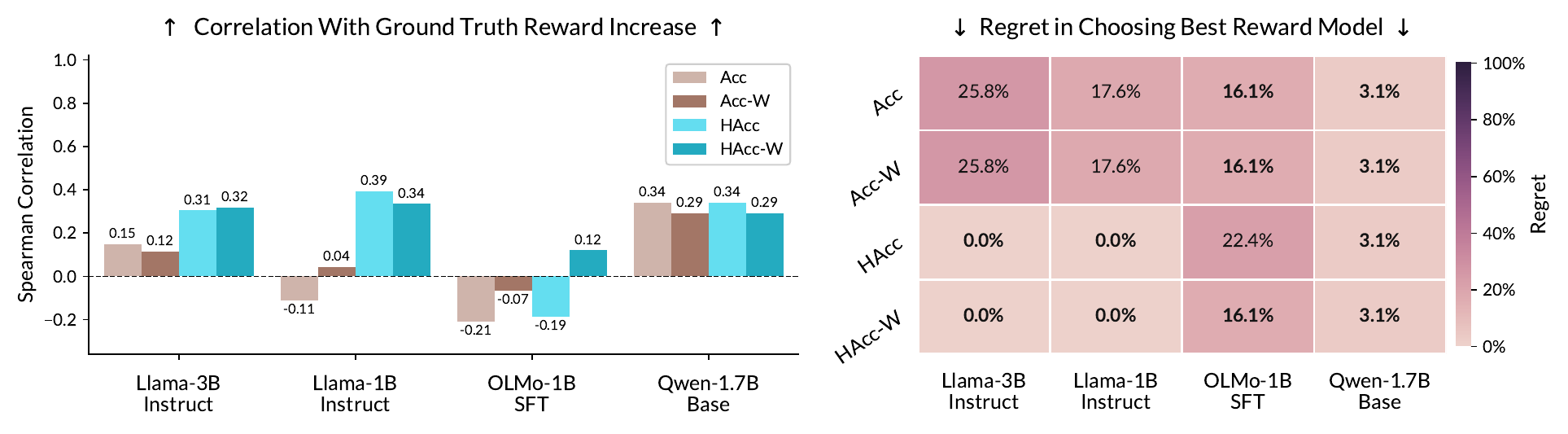}
	\end{center}
	\vspace{-2.25mm}
	\caption{
        \textbf{Harm-aware ranking accuracy variants are more predictive of which reward model leads to better language model performance.}
        For each language model, we run policy gradient (specifically, RLOO) using 13 different reward models on prompts from the UltraFeedback dataset, and compute (per language and reward model) the mean ground truth reward increase based on three separate runs.
        Compared to standard ranking accuracy and its weighted counterpart (Acc and Acc-W), the harm-aware variants (HAcc and HAcc-W) correlate better with ground truth reward increase and lead to lower regret in selecting the best reward model (see \zcref{sec:modified_acc:experiments} for the definition of regret).
        Yet, the correlation remains below $0.4$ and can even be negative, highlighting that robust reward model evaluation remains an open challenge.
        See \zcref{sec:modified_acc} for further details.
	}
	\label{fig:rm_accuracy_with_gt_exps}
\end{figure*}

\textbf{Result I: Harm-aware ranking accuracy variants are more predictive of which reward model leads to better language model performance.}
For each language model and ranking accuracy variant, \zcref{fig:rm_accuracy_with_gt_exps} presents the Spearman correlation between reward model accuracy and ground truth reward increase.
We also report the regret incurred when choosing the most accurate reward model, defined as $100 \brk{R_{\mathrm{best}} - R_{\mathrm{chosen}} } / R_{\mathrm{best}}$, where $R_{\mathrm{best}}$ is the largest ground truth reward increase due to policy gradient, across all reward models, and $R_{\mathrm{chosen}}$ is the increase under the chosen reward model.
Compared to standard ranking accuracy, we find that HAcc typically correlates better with ground truth reward increase and leads to lower regret.
On the other hand, we observe that weighting examples by output probabilities is not consistently helpful.
This suggests that the contribution of low probability outputs to standard ranking accuracy is not a primary cause for its limited correlation with language model performance (for the considered dataset and models).

\textbf{Result II: Robust reward model evaluation remains an open challenge.}
Despite the improvement provided by harm-aware ranking accuracy, the correlation between a ranking accuracy variant and ground truth reward increase remains below $0.4$, and in some settings is even negative.
As we discuss in \zcref{sec:conclusion}, this indicates that robustly evaluating reward models may require looking beyond adjustments to ranking accuracy.

	% APPLICATION: REWARD DESIGN
	\section{Application II: When Should Partially Correct Outputs Be Rewarded?}
\label{sec:reward_design}

Aside from reward model evaluation for RLHF (\zcref{sec:modified_acc}), we explore implications of our theory for reward design in settings with verifiable rule-based rewards.
In tasks such as instruction following or code generation, an output must satisfy a set of constraints or unit tests to be considered correct \citep{pyatkin2025generalizing,guo2025deepseek,olmo2025olmo}.
This raises the question of whether partially correct outputs should be rewarded.
For example, if an output satisfies one out of two desired constraints, should it receive some partial reward (\eg, 0.5) or should all non-fully correct outputs receive zero reward?

Since the goal is to produce fully correct outputs, it is natural to view the ground truth reward function as giving a reward of 1 to fully correct outputs and 0 to all others.
The question above then becomes a matter of proxy reward design.
Our analysis (\zcref{sec:categorization}) suggests that when the language model (\ie, initial policy) is noticeably more likely to produce partially correct outputs than fully correct ones, rewarding partial correctness can cause it to stall on these mediocre outputs (\hyperref[err:harmful_2]{\harmful{Harmful Error~II}}).
Thus, as long as there is a non-zero probability of producing fully correct outputs at initialization, in such cases it may be more effective to only reward full correctness.
We demonstrate this prospect empirically using instruction following tasks from IFBench \citep{pyatkin2025generalizing}.

\begin{figure*}[t]
	\vspace{0mm}
	\begin{center}
        \includegraphics[width=1\textwidth]{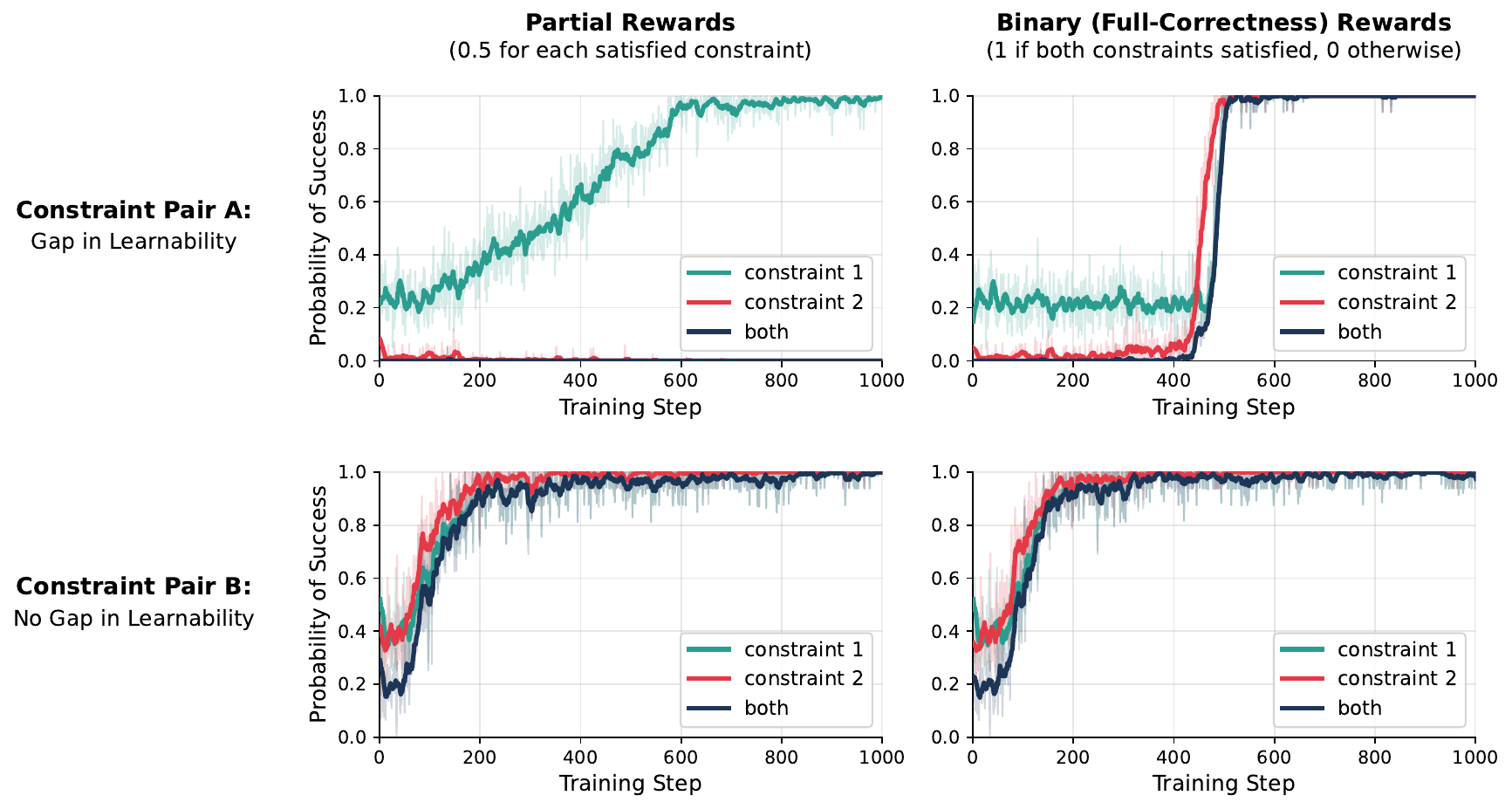}
	\end{center}
	\vspace{-2.25mm}
	\caption{
        \textbf{Rewarding partially correct outputs can impede policy gradient optimization.}
        We train Qwen3-1.7B using GRPO on two instruction following datasets, where the prompts in each dataset include a pair of constraints from IFBench that an output must satisfy to be considered correct (all prompts within a dataset share the same constraint pair).
        Plotted are probabilities of satisfying the constraints, averaged over ten training steps for clarity; unsmoothed values are shown in semi-transparent lines.
        As suggested by our theory (\zcref{sec:categorization}), when the initial probability of satisfying one constraint is noticeably higher than the probability of satisfying the other, rewarding partial correctness can cause the policy to learn only the easier constraint.
        In this case, rewarding only fully correct outputs can lead to faster learning of both constraints.
        However, when both constraints are initially satisfied with similar probability, both reward designs tend to work well.
        See \zcref{sec:reward_design} for further details.
	}
	\label{fig:rlvr_partial_rewards_qwen}
\end{figure*}

\textbf{Setting.}
We construct instruction following datasets by taking prompts from UltraFeedback and appending to them two instructions, each specifying a constraint from IFBench (\eg, “the response must start with a verb'').
All prompts within a dataset share the same constraints, and the datasets differ in which constraints they contain.
We then train Qwen3-1.7B on these datasets using GRPO \citep{shao2024deepseekmath} with two proxy reward designs.
In the first, we reward partial correctness by giving a reward of 0.5 for each satisfied constraint.
In the second, the reward is binary: fully correct outputs receive a reward of 1, and all other outputs receive a reward of 0.
See \zcref{app:experiments:further} for experiments with additional language models and \zcref{app:experiments:details} for further implementation details.

\textbf{Results.}
\zcref{fig:rlvr_partial_rewards_qwen} shows that, if the initial probability of satisfying one constraint noticeably exceeds the probability of satisfying the other, rewarding partial correctness may cause the policy to learn only the easier constraint, while binary (full-correctness) rewards can enable the policy to learn both constraints.\footnote{
        We note, however, that the probability of initially satisfying a constraint is not the sole factor determining its ease of learnability.
        As \zcref{app:experiments:further:reward_design} shows, in cases where one constraint is initially satisfied with substantially higher probability than the other, it is also possible for both constraints to be learned quickly under partial rewards (even when learning fails under binary rewards).
}
If instead the constraints are initially satisfied with similar probability, both reward designs tend to work well.
Taken together, these findings yield a practical insight: rewarding partially correct outputs is not universally beneficial or harmful.
Choosing suitable proxy rewards in verifiable settings requires considering both the task structure and language model capabilities.

	% RELATED WORK
	\section{Related Work}
\label{sec:related}

\textbf{Reinforcement learning with proxy rewards.}
Manually handcrafted and automatically learned proxy rewards have a long history across reinforcement learning applications (\eg, \cite{mataric1994reward,randlov1999learning,singh2010intrinsically,sorg2010reward,zheng2018learning,berner2019dota,chiang2019learning,peng2021amp,devidze2022exploration}), far too broad to fully cover here.
In the context of language models, which are the focus of this work, proxy rewards are widely used both for alignment with human preferences (\ie, RLHF; \citep{ziegler2019fine,ouyang2022training,grattafiori2024llama,razin2026why,swamy2026all}) and for verifiable tasks such as mathematical problem solving, code generation, and instruction following \citep{lambert2024tulu,guo2025deepseek,olmo2025olmo,huang2025pitfalls,tao2025hybrid,pyatkin2025generalizing,masud2026reward}.
Yet, despite the ubiquity of imperfect proxy rewards in practice, there is limited theoretical understanding of what makes a good proxy reward function.
We take a step towards addressing this gap.

\textbf{What makes a good proxy reward function?}
This question has roots in the reward shaping literature \citep{ng1999policy}. 
There, prior theoretical work characterized reward transformations that preserve the order of policies \citep{ng1999policy,asmuth2008potential,grzes2017reward,skalse2024starc} and demonstrated that reward shaping can improve the efficiency of stylized learning algorithms or value-based methods in long-horizon environments \citep{laud2003influence,dai2019maximum,cheng2021heuristic,devidze2022exploration,gupta2022unpacking}.
Furthermore, \cite{gleave2021quantifying,wulfe2022dynamics,skalse2024starc} derived measures for comparing two reward functions (\eg, a proxy and ground truth reward function) that account for said reward transformations, which leave policy ordering unchanged.
Similarly to our work, these lines of research can be interpreted as identifying cases where deviating from a ground truth reward function is in some sense benign or beneficial.
However, they disregard the impact of such deviations on policy gradient optimization.
As \zcref{sec:categorization} proves, reward errors that modify policy ordering can still be benign or beneficial.
Conversely, reward errors that preserve policy ordering may be harmful and slow down optimization.

Perhaps most closely related to our work is \cite{razin2025what}.
While reward models for RLHF are mainly evaluated via ranking accuracy (\cf~\zcref{sec:prelim:rm_eval}), \cite{razin2025what} established that other aspects can determine how good a reward model is.
In particular, they showed that when the reward model induces low reward variance, policy gradient suffers from a flat objective landscape that hinders optimization.
Our technical approach resembles that of \cite{razin2025what}, which also analyzes the rate of ground truth reward increase under proxy reward functions.
Though, unlike \cite{razin2025what}, we consider how different types of reward errors influence optimization beyond their effect on reward variance (see \zcref{remark:relation_to_reward_variance} in \zcref{app:formal_statements} for further discussion on this point).
This allows us to develop harm-aware ranking accuracy variants and provide insights for reward design in settings with verifiable rewards.

\textbf{Theoretical analyses of policy gradient optimization.}
For softmax policies, \ie, policies that produce a distribution over outputs via the softmax function, the reward maximization objective is non-concave even under tabular or linear parameterizations \citep{agarwal2021theory}.
Characterizing the complex optimization dynamics of policy gradient methods in such settings has therefore remained an active area of research.
Existing analyses showed that optimization can fail due to the gradient vanishing in long-horizon environments with sparse rewards \citep{agarwal2021theory,li2021softmax}, when the policy is nearly deterministic \citep{ahmed2019understanding,hennes2019neural,schaul2019ray,mei2020escaping,mei2020global,agarwal2021theory,garg2022alternate}, or, more generally, when the reward variance is low \citep{razin2024vanishing,razin2025what}.
Particularly relevant to our work, \cite{mei2020escaping} observed that policy gradient can be attracted to mediocre outputs at the expense of higher-quality ones, but left open the question of when and to what extent this phenomenon can impede optimization.
\zcref{thm:attraction_to_mediocre_outputs_beneficial_error} shows that the attraction to mediocre outputs can be severe, stalling optimization for an arbitrarily long time.
Our results also complement existing convergence guarantees \citep{mei2020global,mei2021leveraging,mei2023stochastic,klein2024beyond,liu2024elementary,lin2025rethinking}.
These guarantees depend on $\inf_{t \geq 0} \pi_{\theta_t}(\ystar)$---the infimal probability assigned to the optimal output during training.
We identify that the attraction to mediocre outputs can drive this quantity arbitrarily low.

	% CONCLUSION
	\section{Conclusion}
\label{sec:conclusion}

Imperfect proxy rewards are widely used for language model post-training, and reinforcement learning more broadly.
Yet, the understanding of how discrepancies between the proxy and ground truth rewards affect the learning process remains limited.
Such reward errors are conventionally viewed as harmful due to the risk of reward hacking.
In this work, however, we proved that aside from being harmful, reward errors can also be benign or even beneficial for policy gradient optimization.
Specifically, our theoretical analysis categorizes prominent reward error types into harmful, benign, or beneficial, and characterizes the mechanisms behind their distinct effects on optimization.

We then demonstrated that this categorization yields practical insights.
For RLHF, it motivated harm-aware reward model evaluation metrics that are more predictive of language model performance than standard ranking accuracy.
For environments with verifiable rewards, it revealed that rewarding partially correct outputs can impede learning rather than help it.
A key theme underlying these results is that the effectiveness of proxy rewards depends heavily on their interplay with the initial policy and learning algorithm.
As elaborated below, we hope that our work will inspire further research on proxy reward evaluation and design that accounts for this interplay.

\subsection{Limitations and Future Work}
\label{sec:conclusion:limitations_future_work}

\textbf{Theoretical analysis: beyond bandit environments.}
Our theory (\zcref{sec:categorization}) focuses on bandit environments since, in language model applications, rewards are often available only for complete outputs.
Extending the analysis to longer-horizon settings, corresponding to multi-turn dialogue or environments with process rewards, may reveal additional ways in which reward errors shape learning.
It would also be valuable to understand how reward errors interact across inputs, especially given recent evidence that learnability via policy gradient can depend on the data composition \citep{ran2026outcome,huang2026learning}.

\textbf{Towards robust reward model evaluation.}
Although the harm-aware ranking accuracy variants proposed in \zcref{sec:modified_acc} improve upon standard ranking accuracy, our results suggest that even these metrics capture only part of what makes an effective reward model.
We outline two potential barriers towards robust reward model evaluation, which future work can attempt to address.
First, output rankings provide only coarse information about the ground truth reward function.
This seems to place an inherent ceiling on metrics that solely rely on such data.
Second, reward model evaluation benchmarks often lack coverage of relevant inputs and outputs.
This limitation is reflected both in recent efforts to revisit existing benchmarks \citep{zhou2025rmb,wen2025rethinking,malik2025rewardbench} and in our experiments, where evaluating reward models on RewardBench2 is markedly less predictive of language model performance than evaluating them on examples drawn from the policy gradient training set (\zcref{fig:rm_accuracy_with_gt_exps_rb2} in \zcref{app:experiments:further}).

\textbf{Reward design in verifiable settings.}
We identified a drawback of rewarding partially correct outputs: it can hinder optimization by causing the policy to stall on such suboptimal outputs (\zcref{sec:reward_design}).
However, there are cases where rewarding partial correctness is beneficial.
In particular, \cite{sun2026rl} showed that when the initial policy produces fully correct outputs with near-zero probability, by first using partial rewards and then switching to binary (full-correctness) rewards, one can achieve better performance than using binary rewards from the start.
We therefore view analyzing adaptive proxy reward schemes that vary across inputs and training steps as a promising direction for future work.

	% ACKNOWLEDGMENTS
	\ifdefined\NEURIPS
		\begin{ack}
			We thank Eshbal Hezroni for aid in preparing illustrative figures.
SW is supported by an NSF Graduate Research Fellowship.
SA acknowledges funding from ONR, Schmidt Science,
and OpenAI.
NR is supported in part by the Zuckerman STEM Leadership Program.

		\end{ack}
	\else
		\newcommand{\ack}{We thank Eshbal Hezroni for aid in preparing illustrative figures.
SW is supported by an NSF Graduate Research Fellowship.
SA acknowledges funding from ONR, Schmidt Science,
and OpenAI.
NR is supported in part by the Zuckerman STEM Leadership Program.
}
	\fi
	\ifdefined\ARXIV
		\section*{Acknowledgements}
		\ack
	\else
		\ifdefined\COLT
			\acks{\ack}
		\else
			\ifdefined\CAMREADY
				\ifdefined\ICLR
					\newcommand*{\subsuback}{}
				\fi
				\ifdefined\NEURIPS
					% DO NOTHING - HANDLED EARLIER
				\else
					\section*{Acknowledgements}
					\ack
				\fi
			\fi
		\fi
	\fi

	% REFERENCES
	\section*{References}
	{\small
		\ifdefined\ICML
			\bibliographystyle{icml2021}
		\else
			\bibliographystyle{plainnat}
		\fi
		\bibliography{refs}
	}

	% APPENDIXES
	\clearpage
	\appendix
	
% Uncomment below for ICML appendix with two columns
%	\onecolumn
	
	% ENDNOTES
	\ifdefined\ENABLEENDNOTES
		\theendnotes
	\fi
	
	% TABLE OF CONTENTS FOR APPENDIX
%	{
%		\hypersetup{hidelinks}
%		\tableofcontents
%	}

	% Add Appendix sections here using \input

	% DEFERRED FORMAL STATEMENTS
	\section{Detailed Theorem Statements and Influence of Feature Similarity on Attraction to Mediocre Outputs}
\label{app:formal_statements}

In this appendix, we present \zcref{thm:attraction_to_mediocre_outputs_beneficial_error_formal}---the detailed statement of \zcref{thm:attraction_to_mediocre_outputs_beneficial_error}---which establishes the existence of beneficial reward errors.
In particular, \zcref{thm:attraction_to_mediocre_outputs_beneficial_error_formal} shows that mediocre outputs can attract probability and stall policy gradient optimization for an arbitrarily long time.
Thus, assigning low proxy reward to outputs with mediocre ground truth reward can be beneficial for accelerating the increase in ground truth reward.
As a corollary (\zcref{cor:attraction_to_mediocre_outputs_harmful_error}), we also obtain that assigning mediocre proxy reward to outputs with low ground truth reward can impede optimization, even if it does not induce reward hacking.
This formalizes \hyperref[err:harmful_2]{\harmful{Harmful Error~II}} from \zcref{sec:categorization:harmful}.
Both \zcref{thm:attraction_to_mediocre_outputs_beneficial_error_formal} and \zcref{cor:attraction_to_mediocre_outputs_harmful_error} consider orthonormal output feature vectors.
In \zcref{app:formal_statements:feature_similarity}, we relax this assumption and analyze how the similarity between output features can exacerbate or attenuate the attraction to mediocre outputs.
We refer the reader to \zcref{sec:categorization:setting} for the technical setting of the analysis.

Our aim is to characterize the phenomenon where mediocre outputs attract probability at the expense of higher-quality outputs.
Thus, we consider a setting in which $\Y$ contains an output $\ymed$ with ground truth reward lower than $\gtreward (\ystar)$, where $\ystar = \argmax_{y \in \Y} \gtreward (y)$ is the optimal output, but higher than the ground truth reward of all remaining outputs $\Ybad \subset \Y$.
For simplicity, we assume all outputs in $\Ybad$ have the same ground truth reward.
Our analysis extends straightforwardly to the case where outputs in $\Ybad$ have different rewards, as long as they are sufficiently lower than~$\gtreward(\ymed)$.

\begin{assumption}[Reward structure]
\label{assum:reward_structure_orthonormal}
The set of outputs can be decomposed as $\Y = \brk[c]{ \ystar, \ymed } \cup \Ybad$ such that $\gtreward (\ystar) > \gtreward (\ymed) > 0 \geq \max\nolimits_{y \in \Ybad} \gtreward (y)$.
Furthermore, all outputs in $\Ybad \subset \Y$ have the same ground truth reward, \ie, $\gtreward(y) = \gtreward(y')$ for all $y, y' \in \Ybad$. 
\end{assumption}

Suppose that we can maximize the ground truth reward directly (\ie, $\proxyreward = \gtreward$).
In this case, as discussed in \zcref{sec:categorization:beneficial}, an output $\ymed$ with mediocre (ground truth and proxy) reward can attract probability if \smash{$\gtreward (\ymed) > \Vgt (\theta_0) = \EE\nolimits_{y \sim \pi_{\theta_0}} \brk[s]{ \gtreward (y) }$} and the probability of $\ymed$ under $\pi_{\theta_0}$ is higher than the probability of $\ystar$.
This condition is formalized by assuming that $\pi_{\theta_0} (\ystar)$ is not too high and that $\pi_{\theta_0} (\Ybad)$ is neither too low (otherwise, $\Vgt (\theta_0)$ will be above $\gtreward (\ymed)$) nor too high (otherwise, $\pi_{\theta_0} (\ymed)$ may not be sufficiently higher than $\pi_{\theta_0} (\ystar)$).
See \zcref{lem:mediocre_output_reward_larger_than_initial_expected_reward} in \zcref{app:proofs} for a proof that \zcref{assum:reward_structure_orthonormal,assum:initial_probs_orthonormal} ensure $\gtreward (\ymed) > \Vgt (\theta_0)$.

\begin{assumption}[Initial policy]
\label{assum:initial_probs_orthonormal}
The initial policy $\pi_{\theta_0}$ satisfies the following conditions.
\begin{itemize}[leftmargin=6mm]
    \item The initial probability of $\ystar$ is not too high: 
    \[
        \pi_{\theta_0} (\ystar) < M \cdot \min \brk[c]*{ 0.01, \bigg(\frac{21}{1100(\Delta_1+\Delta_2)}\bigg)^{7/3}, \Delta_2^2 }^{13/14}
        \text{\,,}
    \] 
    where $\Deltaone := \gtreward (\ystar) - \gtreward (\ymed)$, $\Deltatwo := \gtreward (\ymed) - \gtreward (\ybad)$ for some $\ybad \in \Ybad$, and $M := \min \brk[c]1{ 1 ,  0.05 \cdot \gtreward (\ymed)^{2/7} (\Delta_1+\Delta_2)^{-1} }$.
    
    \item The initial probability of $\Ybad$ is neither too low nor too high:
    \[
        \max \brk[c]*{ \frac{2\pi_{\theta_0}(\ystar)^{7/13}}{\Deltatwo M^{7/13}}, 0.05 \cdot \gtreward (\ymed) } \leq \pi_{\theta_0} (\Ybad) \leq 0.1 \cdot \gtreward (\ymed)
        \text{\,.}
    \]
\end{itemize}
\end{assumption}

Under \zcref{assum:reward_structure_orthonormal,assum:initial_probs_orthonormal}, we prove that incorrectly assigning low reward to $\ymed$ is beneficial for increasing the ground truth reward.
Specifically, we compare the optimization trajectory when maximizing $\gtreward$ directly with the trajectory obtained when maximizing a proxy reward function $\proxyreward$ that is identical to $\gtreward$, except that it assigns low reward to the mediocre output $\ymed$.
\zcref{thm:attraction_to_mediocre_outputs_beneficial_error_formal} establishes that the time required to achieve high ground truth reward when maximizing $\gtreward$ directly can be arbitrarily larger than when maximizing $\proxyreward$.
This gap in optimization time arises because, when maximizing $\gtreward$, the policy first concentrates its probability mass on $\ymed$ and stalls in its vicinity.
By contrast, when maximizing $\proxyreward$, the policy directly shifts its probability mass towards $\ystar$.

\begin{theorem}[Detailed version of \zcref{thm:attraction_to_mediocre_outputs_beneficial_error}]
\label{thm:attraction_to_mediocre_outputs_beneficial_error_formal}
Suppose that the feature vectors \smash{$\brk[c]{ \phi (y) }_{y \in \Y}$} are orthonormal and that the ground truth reward function $\gtreward$ and initial policy parameters $\theta_0$ uphold \zcref{assum:reward_structure_orthonormal,assum:initial_probs_orthonormal}.
Furthermore, let $\proxyreward$ be a proxy reward function identical to $\gtreward$, except that it assigns $\ymed$ a low reward $\proxyreward (\ymed) = \min\nolimits_{y \in \Ybad} \gtreward (y)$.
For any $\epsilon \in (0, \Deltaone)$, where $\Deltaone := \gtreward (\ystar) - \gtreward (\ymed)$, the following hold.
\begin{itemize}[leftmargin=6mm]
    \item \textbf{Case I: Maximizing $\gtreward$.}
    If gradient flow is used to maximize the expected reward with respect to $\gtreward$ (\zcref{eq:gf} with $\Vgt$ in place of $\Vproxy$), then the initial time at which $\Vgt (\theta_t) \geq \gtreward( \ystar ) - \epsilon$, denoted $\teps$, is lower bounded as follows:
    \[
        \teps \geq \frac{2\sqrt{2}M^{14/13}\Delta_1 \gtreward (\ymed) \brk{ 1 - e^{-\sqrt{2}(\Delta_1 - \epsilon)} } }{(\Delta_1 + \Delta_2)(1 + 8\Delta_1)} \cdot \pi_{\theta_0} (\ystar)^{- 14 / 13} = \Omega \brk*{ \pi_{\theta_0} (\ystar)^{- 14 / 13} }
        \text{\,,}
    \]
    where $\Deltatwo := \gtreward (\ymed) - \gtreward (\ybad)$ for some $\ybad \in \Ybad$ and $M$ is as defined in \zcref{assum:initial_probs_orthonormal}.
    
    \item \textbf{Case II: Maximizing $\proxyreward$.}    
    In contrast, denote by $\tepsnomed$ the initial time at which $\Vgt (\theta_t) \geq \gtreward( \ystar ) - \epsilon$ when gradient flow is used to maximize the expected reward with respect to $\proxyreward$.
    Then, $\tepsnomed$ is upper bounded by a quantity that grows asymptotically slower as $\pi_{\theta_0} (\ystar) \to 0$ compared to the lower bound on $\teps$ in Case I:
    \[
        \tepsnomed \leq \frac{ \brk*{ \gtreward (\ystar) + 1 }^2 }{\epsilon^2  \brk*{ \proxyreward (\ystar) - \Vproxy (\theta_0)  } }  \cdot \pi_{\theta_0} (\ystar)^{-1} = \OO \brk*{ \pi_{\theta_0} (\ystar)^{-1} }
        \text{\,,}
    \]
    Furthermore, at any time $t$ greater than or equal to the right-hand side in the inequality above, we have $\pi_{\theta_t} (\ystar) \geq 1 - \epsilon / (\Deltaone + \Deltatwo)$.
\end{itemize}
Thus, for any arbitrarily large $T \geq 0$, if $\pi_{\theta_0} (\ystar)$ is sufficiently small, then $\teps - \tepsnomed \geq T$.
\end{theorem}

A proof sketch is provided after the statement of \zcref{thm:attraction_to_mediocre_outputs_beneficial_error}; see \zcref{app:proofs:attraction_to_mediocre_outputs} for the full proof.

\smallskip

\begin{remark}
\label{remark:relation_to_reward_variance}
Recently, \cite{razin2024vanishing,razin2025what} identified that if the reward function $r$ used for training induces low reward variance for the initial policy, \ie, if $\var_{y \sim \pi_{\theta_0}} \brk[s]{ r (y) } \approx 0$, then policy gradient suffers from slow optimization (see Theorem 1 in \cite{razin2025what}).
It is therefore natural to wonder whether $\proxyreward$ facilitates faster optimization than $\gtreward$ in \zcref{thm:attraction_to_mediocre_outputs_beneficial_error_formal} due to inducing a higher initial reward variance.
Interestingly, the opposite is true: the initial reward variance is actually higher under $\gtreward$.
The slow optimization under $\gtreward$ is not caused by low initial reward variance, but rather by a collapse in reward variance during training, as the policy becomes concentrated on the mediocre output $\ymed$.
Although $\proxyreward$ initially induces lower reward variance than $\gtreward$, it guides the policy directly towards $\ystar$, avoiding this intermediate variance collapse.
\end{remark}

\medskip

\textbf{Formal statement behind Harmful Error II.} 
As a corollary of \zcref{thm:attraction_to_mediocre_outputs_beneficial_error_formal}, by swapping the proxy and ground truth reward functions, we obtain that assigning mediocre proxy reward to an output with low ground truth reward can be harmful due to impeding optimization, even if it does not induce reward hacking.
That is, \zcref{cor:attraction_to_mediocre_outputs_harmful_error} formalizes \hyperref[err:harmful_2]{\harmful{Harmful Error~II}} from \zcref{sec:categorization:harmful}.

\begin{corollary}
\label{cor:attraction_to_mediocre_outputs_harmful_error}
Consider the setting and assumptions of \zcref{thm:attraction_to_mediocre_outputs_beneficial_error_formal}, with \zcref{assum:reward_structure_orthonormal,assum:initial_probs_orthonormal} interpreted with $\proxyreward$ in place of $\gtreward$.
Furthermore, let $\gtreward$ be identical to $\proxyreward$, except that it assigns $\ymed$ a low reward $\gtreward (\ymed) = \min\nolimits_{y \in \Ybad} \proxyreward (y)$.
For any $\epsilon \in (0, \Deltaone)$, where $\Deltaone := \proxyreward (\ystar) - \proxyreward (\ymed)$, the following hold.
\begin{itemize}[leftmargin=6mm]
    \item \textbf{Case I: Maximizing $\proxyreward$.}
    If gradient flow is used to maximize the expected reward with respect to $\proxyreward$ (\zcref{eq:gf}), then the initial time at which $\Vgt (\theta_t) \geq \gtreward( \ystar ) - \epsilon$, denoted $\teps$, is lower bounded as follows:
    \[
        \teps \geq \frac{2\sqrt{2}M^{14/13}\Delta_1 \proxyreward(\ymed)(1 - e^{-\sqrt{2}(\Delta_1 - \epsilon)})}{(\Delta_1 + \Delta_2)(1 + 8\Delta_1)} \cdot \pi_{\theta_0} (\ystar)^{- 14 / 13} = \Omega \brk*{ \pi_{\theta_0} (\ystar)^{- 14 / 13} }
        \text{\,,}
    \]
    where $\Deltatwo := \proxyreward (\ymed) - \proxyreward (\ybad)$ for some $\ybad \in \Ybad$ and $M$ is as defined in \zcref{assum:initial_probs_orthonormal}.
    
    \item \textbf{Case II: Maximizing $\gtreward$.}    
    In contrast, denote by $\tepsnomed$ the initial time at which $\Vgt (\theta_t) \geq \gtreward( \ystar ) - \epsilon$ when gradient flow is used to maximize the expected reward with respect to $\gtreward$ (\zcref{eq:gf} with $\Vgt$ in place of $\Vproxy$).
    Then, $\tepsnomed$ is upper bounded by a quantity that grows asymptotically slower as $\pi_{\theta_0} (\ystar) \to 0$ compared to the lower bound on $\teps$ in Case I:
    \[
        \tepsnomed \leq \frac{ \brk*{ \gtreward (\ystar) + 1 }^2 }{\epsilon^2 \brk*{\gtreward (\ystar) - \Vgt (\theta_0)}  } \cdot \pi_{\theta_0} (\ystar)^{-1} = \OO \brk*{ \pi_{\theta_0} (\ystar)^{-1} }
        \text{\,,}
    \]
    Furthermore, at any time $t$ greater than or equal to the right-hand side in the inequality above, we have $\pi_{\theta_t} (\ystar) \geq 1 - \epsilon / (\Deltaone+\Deltatwo)$.
\end{itemize}
Thus, for any arbitrarily large $T \geq 0$, if $\pi_{\theta_0} (\ystar)$ is sufficiently small, then $\teps - \tepsnomed \geq T$.
\end{corollary}

\begin{proof}
The result follows by applying \zcref{thm:attraction_to_mediocre_outputs_beneficial_error_formal} while swapping the roles of $\gtreward$ and $\proxyreward$.
For Case I (\ie, when maximizing $\proxyreward$), \zcref{thm:attraction_to_mediocre_outputs_beneficial_error_formal} yields a lower bound on the time until $\Vproxy(\theta_t) \geq \proxyreward(\ystar)-\epsilon$.
Since $\gtreward(\ystar) = \proxyreward(\ystar)$ and $\Vgt(\theta) \leq \Vproxy(\theta)$ for all $\theta$ (recall, $\gtreward$ is identical to $\proxyreward$, except that it assigns a lower reward to $\ymed$), this immediately implies that the same lower bound holds for $\teps$---the time until $\Vgt(\theta_t) \geq \gtreward(\ystar)-\epsilon$.
For Case II (\ie, when maximizing $\gtreward$), \zcref{thm:attraction_to_mediocre_outputs_beneficial_error_formal} implies that $\pi_{\theta_t}(\ystar)\geq 1 - \epsilon / (\Deltaone + \Deltatwo)$ at any time $t$ greater than or equal to the desired upper bound on $\tepsnomed$.
Hence, because ground truth rewards lie within the interval $[\gtreward (\ybad), \gtreward (\ystar)]$ and $\Deltaone + \Deltatwo = \gtreward (\ystar) - \gtreward (\ybad)$, where $\ybad \in \Ybad$, at any such time $\Vgt(\theta_t) \geq \gtreward(\ystar)-\epsilon$.
This establishes the upper bound on $\tepsnomed$.
\end{proof}

\subsection{Influence of Feature Similarity}
\label{app:formal_statements:feature_similarity}

We now analyze how the similarity (or dissimilarity) of output features affects the extent to which mediocre outputs impede optimization.
Namely, under conditions analogous to those of \zcref{thm:attraction_to_mediocre_outputs_beneficial_error_formal}, we characterize how the conclusion of \zcref{thm:attraction_to_mediocre_outputs_beneficial_error_formal} changes when $\phi (\ystar)$ and $\phi (\ymed)$ are not orthogonal (recall that $\ystar$ is the output with maximal ground truth reward and $\ymed$ is an output with mediocre ground truth reward).
The analysis reveals that when $\inprod{\phi (\ystar) }{ \phi (\ymed) } < 0$, the mediocre output $\ymed$ delays optimization even further and can prevent the policy from ever reaching the optimal output $\ystar$.
Conversely, when $\inprod{\phi (\ystar) }{ \phi (\ymed) } > 0$, attraction to $\ymed$ does not necessarily impede optimization and, if $\phi (\ystar)$ and $\phi (\ymed)$ are extremely similar, assigning $\ymed$ a mediocre proxy reward may be required for enabling convergence to $\ystar$.
\zcref{app:formal_statements:feature_similarity:neg_in_prod,app:formal_statements:feature_similarity:pos_in_prod} deliver the formal results for the cases of negative and positive inner products, respectively.

\subsubsection{Negative Inner Product Between $\phi (\ystar)$ and $\phi (\ymed)$}
\label{app:formal_statements:feature_similarity:neg_in_prod}

We consider the ground truth reward function $\gtreward$ described in \zcref{assum:reward_structure_orthonormal} (\ie, the ground truth from the orthonormal features setting considered in \zcref{thm:attraction_to_mediocre_outputs_beneficial_error_formal}) and make the following \zcref{assum:initial_probs_neg_inner_product} on the initial policy probabilities.
\zcref{assum:initial_probs_neg_inner_product} is identical to \zcref{assum:initial_probs_orthonormal} up to constants that depend on the norms of feature vectors.
Similar to \zcref{assum:initial_probs_orthonormal}, \zcref{assum:initial_probs_neg_inner_product} guarantees that $\gtreward (\ymed) > \Vgt (\theta_0)$ (see \zcref{lem:mediocre_output_reward_larger_than_initial_expected_reward_neg_inner_product}) and that $\pi_{\theta_0} (\ymed)$ is substantially higher than $\pi_{\theta_0} (\ystar)$.

\begin{assumption}[Initial policy]
\label{assum:initial_probs_neg_inner_product}
The initial policy $\pi_{\theta_0}$ satisfies the following conditions.
\begin{itemize}[leftmargin=6mm]
    \item The initial probability of $\ystar$ is not too high: 
    \[
        \pi_{\theta_0} (\ystar) < \min \brk[c]3{ M' \! \cdot  \min \brk[c]3{ 0.01, \bigg(\frac{21 \norm{ \phi(\ymed) }^2}{1100B^2(\Delta_1+\Delta_2)}\bigg)^{7/3} \! }^{13/14}  , \frac{ \norm{ \phi(\ymed) } }{ \norm{\phi(\ystar)} } \pi_{\theta_0}(\ymed) }
        \text{\,,}
    \] 
    where $\Deltaone := \gtreward (\ystar) - \gtreward (\ymed)$, $\Deltatwo := \gtreward (\ymed) - \gtreward (\ybad)$ for some $\ybad \in \Ybad$, $B := \max\nolimits_{y \in \Y} \norm{\phi (y)}$, and $M' := \min \brk[c]1{ 1 ,  \norm{ \phi(\ymed) }^2 \gtreward (\ymed)^{2/7} ( 40B^2 (\Delta_1+\Delta_2) )^{-1} }$.
    
    \item The initial probability of $\Ybad$ is neither too low nor too high:
    \[
        \max \brk[c]*{ \frac{2\pi_{\theta_0}(\ystar)^{7/13}}{\Deltatwo M'^{7/13}}, 0.05 \cdot \gtreward (\ymed) } \leq \pi_{\theta_0} (\Ybad) \leq 0.1 \cdot \gtreward (\ymed)
        \text{\,.}
    \]
\end{itemize}
\end{assumption}

We focus here on the case of $\inprod{\phi (\ystar)}{\phi (\ymed)} < 0$.
To isolate the effect of $\inprod{\phi (\ystar)}{\phi (\ymed)}$ on the optimization dynamics from interactions between other feature vectors, we require that, for any output in $\Ybad$, its feature vector is orthogonal to the feature vectors of all other outputs.
We also make the mild assumption that $\norm{\phi (\ymed)} \geq \norm{ \phi (\ybad) }$ for all $\ybad \in \Ybad$.

\begin{assumption}[Output feature vectors]
\label{assum:feature_structure_neg_inn_prod}
The feature vectors of $\ystar$ and $\ymed$ have negative inner product, \ie, $\inprod{\phi (\ystar)}{\phi (\ymed)} < 0$, and $\inprod{ \phi (\ybad) }{ \phi (y) } = 0$ for all $\ybad \in \Ybad, y \in \Y \setminus \{ \ybad \}$.
Furthermore, $\norm{\phi (\ymed)} \geq \norm{ \phi (\ybad) }$ for all $\ybad \in \Ybad$.
\end{assumption}

If we maximize $\gtreward$ directly, $\ymed$ attracts probability at the expense of $\ystar$ even more strongly than in the case of orthonormal feature vectors.
The reason for this stronger attraction is that, as long as the advantage $\advgt (\ymed ; \theta_t) := \gtreward (\ymed) - \Vgt (\theta_t)$ is positive, $\inprod{\phi(\ystar)}{\phi(\ymed)}$ pushes the logit of $\ystar$ downwards (see \zcref{prop:logit_dynamics}).
As \zcref{thm:attraction_to_mediocre_outputs_beneficial_error_neg_inner_product} shows, this yields a larger separation between $\teps$---the time needed to achieve high ground truth reward when maximizing $\gtreward$ directly---and the corresponding time $\tepsnomed$ when maximizing a proxy reward function $\proxyreward$ that assigns low reward to $\ymed$.
In particular, the lower bound in \zcref{thm:attraction_to_mediocre_outputs_beneficial_error_neg_inner_product} on $\teps$ grows asymptotically faster as $\pi_{\theta_0}(\ystar) \to 0$ compared to the lower bound in \zcref{thm:attraction_to_mediocre_outputs_beneficial_error_formal}.
On the other hand, the upper bound on $\tepsnomed$ can remain roughly the same since $\advproxy (\ymed ; \theta_t) < 0$ throughout optimization when $\ymed$ is assigned a low proxy reward, in which case $\inprod{\phi(\ystar)}{\phi(\ymed)}$ only pushes the logit of $\ystar$ further upwards.

\begin{theorem}
\label{thm:attraction_to_mediocre_outputs_beneficial_error_neg_inner_product}
Suppose that the ground truth reward function $\gtreward$, initial policy parameters $\theta_0$, and output feature vectors $\{ \phi (y) \}_{y \in \Y}$ uphold \zcref{assum:reward_structure_orthonormal,assum:initial_probs_neg_inner_product,assum:feature_structure_neg_inn_prod}. 
Let $\proxyreward$ be a proxy reward function identical to $\gtreward$, except that it assigns $\ymed$ a low reward $\proxyreward (\ymed) = \min\nolimits_{y \in \Ybad} \gtreward (y)$.
For any $\epsilon \in (0, \Deltaone)$, where $\Deltaone := \gtreward (\ystar) - \gtreward (\ymed)$, the following hold.
\begin{itemize}[leftmargin=6mm]
    \item \textbf{Case I: Maximizing $\gtreward$.}
    If gradient flow is used to maximize the expected reward with respect to $\gtreward$ (\zcref{eq:gf} with $\Vgt$ in place of $\Vproxy$), then the initial time at which $\Vgt (\theta_t) \geq \gtreward( \ystar ) - \epsilon$, denoted $\teps$, is lower bounded as follows:
    \[
        \teps \geq \frac{\gtreward(\ymed)M'^{14/13 + \alpha} (1 - e^{-2(\Delta_1- \epsilon)})}{4B^2(\Delta_1 + \Delta_2) K} \cdot \pi_{\theta_0} (\ystar)^{- 14 / 13 - \alpha } = \Omega \brk*{ \pi_{\theta_0} (\ystar)^{- 14 / 13 - \alpha } }
        \text{\,,}
    \]
    where $\Deltatwo := \gtreward (\ymed) - \gtreward (\ybad)$ for some $\ybad \in \Ybad$.
    The constants $B$ and $M'$ are as defined in \zcref{assum:initial_probs_neg_inner_product}, $\alpha := - \inprod{\phi (\ystar) }{ \phi (\ymed) } /  \big ( 26 \brk{ \norm{ \phi (\ymed) }^2 - 0.5 \inprod{\phi (\ystar) }{ \phi (\ymed) } } \big )  \in \brk*{ 0 , 1 / 13 }$, and 
    \[
        K := \brk*{\frac{26}{\gtreward(\ymed)^2} }^{13\alpha/7} \brk*{1 + \frac{\|\phi(\ymed)\|^2\pi_{\theta_0}(\ymed)^2}{8(\|\phi(\ystar)\|^2-\inprod{\phi(\ystar)}{\phi(\ymed)})\advgt(\ystar;\theta_0)}} 
        \text{\,.}
    \]
    
    \item \textbf{Case II: Maximizing $\proxyreward$.}    
    In contrast, denote by $\tepsnomed$ the initial time at which $\Vgt (\theta_t) \geq \gtreward( \ystar ) - \epsilon$ when gradient flow is used to maximize the expected reward with respect to $\proxyreward$.
    Then, $\tepsnomed$ is upper bounded by a quantity that grows asymptotically slower as $\pi_{\theta_0} (\ystar) \to 0$ compared to the lower bound on $\teps$ in Case I:
    \[
        \tepsnomed \leq \frac{ \brk*{ \gtreward (\ystar) + 1 }^2 }{\epsilon^2  \brk1{ \proxyreward (\ystar) - \Vproxy (\theta_0)  } \norm{ \phi(\ystar) }^2 } \cdot \pi_{\theta_0} (\ystar)^{-1} = \OO \brk*{ \pi_{\theta_0} (\ystar)^{-1} }
        \text{\,,}
    \]
\end{itemize}
Thus, for any arbitrarily large $T > 0$, if $\pi_{\theta_0} (\ystar)$ is sufficiently small, then $\teps - \tepsnomed \geq T$.
\end{theorem}

\begin{proof}[Proof sketch (full proof in \zcref{app:proofs:attraction_to_mediocre_outputs_beneficial_error_neg_inner_product})]
The proof follows a line similar to that of \zcref{thm:attraction_to_mediocre_outputs_beneficial_error_formal}.
The key difference is that, due to the negative inner product between $\phi (\ystar)$ and $\phi (\ymed)$, the logit dynamics of $\ystar$ now contains an additional term that depends on $\inprod{\phi (\ystar)}{\phi (\ymed)}$.
Specifically, when maximizing $\gtreward$ directly (Case I), the logit dynamics of $\ystar$ is given by (\cf~\zcref{prop:logit_dynamics}):
\[
    \frac{d}{dt} \inprod{\phi (\ystar)}{\theta_t} = \pi_{\theta_t} (\ystar) \advgt (\ystar ; \theta_t) \cdot \norm*{ \phi(\ystar) }^2  + \pi_{\theta_t} (\ymed) \advgt (\ymed ; \theta_t) \cdot \inprod{ \phi(\ystar) }{ \phi(\ymed) }
    \text{\,.}
\]
Thus, as long as $\advgt (\ymed ; \theta_t) > 0$, the second term on the right-hand side is negative and pushes the probability of $\ystar$ further downwards compared to the case of orthonormal feature vectors.
We show that this implies that the policy will become even more concentrated on $\ymed$, and so will take a longer time to escape the vicinity of $\ymed$.
On the other hand, when maximizing $\proxyreward$ (Case II), the advantage $\advproxy (\ymed; \theta_t)$ is negative throughout optimization since $\proxyreward$ assigns $\ymed$ a low reward.
Hence, the contribution of $\ymed$ to the logit dynamics of $\ystar$ can only push the logit of $\ystar$ further upwards compared to the orthonormal feature vectors case.
This straightforwardly allows deriving an upper bound on $\tepsnomed$ that is roughly the same as the one in \zcref{thm:attraction_to_mediocre_outputs_beneficial_error_formal}.
\end{proof}

\medskip

\textbf{When $\inprod{\phi (\ystar)}{\phi (\ymed)} < 0$ the policy may fail to ever assign high probability to $\ystar$.}
\zcref{thm:attraction_to_mediocre_outputs_beneficial_error_neg_inner_product} indicates that the lower the inner product between $\phi (\ystar)$ and $\phi (\ymed)$ is, the longer the policy stalls in the vicinity of the mediocre output $\ymed$ when maximizing $\gtreward$.
One may wonder if the policy always escapes the attraction of $\ymed$ and reaches $\ystar$ as optimization progresses.
We show that the answer is no.
For certain initial policies and output feature vectors, \zcref{prop:neg_inner_product_fail} proves that the policy never achieves a ground truth reward above $\gtreward (\ymed)$, and thus never assigns high probability to $\ystar$.
Notably, this failure occurs despite the existence of policy parameters that assign $\ystar$ a probability arbitrarily close to one.
\zcref{prop:neg_inner_product_fail} is obtained by adapting Proposition~13 from \cite{lin2025rethinking} to our setting.

\begin{proposition}
\label{prop:neg_inner_product_fail}
Suppose that the ground truth reward function $\gtreward$ upholds \zcref{assum:reward_structure_orthonormal}.
Furthermore, let $\ybad \in \Ybad$, define $\Deltaone := \gtreward(\ystar) - \gtreward(\ymed)$ and $\Deltatwo := \gtreward(\ymed) - \gtreward(\ybad)$, and assume that the initial policy parameters $\theta_0$ and output feature vectors $\{ \phi (y) \}_{y \in \Y}$ satisfy the following conditions, alongside \zcref{assum:feature_structure_neg_inn_prod}.
\begin{itemize}[leftmargin=6mm]
    \item $\pi_{\theta_0} (\ymed) > \pi_{\theta_0} (y)$ for all $y \in \Y \setminus \{ \ymed \}$;

    \item $\pi_{\theta_0} (\ystar) \leq \frac{ \Deltatwo \pi_{\theta_0} (\Ybad) }{ \Deltaone }$;
    
    \item and $\theta_0 = C \cdot \brk*{ \phi (\ybad) - \phi (\ystar) }$ for some $C > \max \brk[c]1{ 0, - \ln (\zeta) / \norm{ \phi (\ybad) - \phi (\ystar) }^2 }$, where 
    \[
        \zeta := \frac{\ln \pi_{\theta_0}(\ymed) - \ln \pi_{\theta_0}(\ybad)}{\ln \pi_{\theta_0}(\ymed) - \ln \pi_{\theta_0}(\ystar)} \cdot \frac{\Vgt(\theta_0) - \gtreward(\ybad)}{\gtreward(\ystar)-\Vgt(\theta_0)} > 0
        \text{\,.}
    \]
\end{itemize}
If gradient flow is used to maximize the expected reward with respect to $\gtreward$ (\zcref{eq:gf} with $\Vgt$ in place of $\Vproxy$), then $\Vgt (\theta_t) < \gtreward (\ymed)$ for all $t \geq 0$.
\end{proposition}

\begin{proof}[Proof sketch (full proof in \zcref{app:proofs:neg_inner_product_fail})]
At initialization, the policy favors $\ybad$ over $\ystar$ in the sense that $\pi_{\theta_0} (\ybad) > \pi_{\theta_0} (\ystar)$ and $\pi_{\theta_0} (\ystar) / \pi_{\theta_0} (\ybad) < \zeta$.
The proof follows by showing that the suppressive effect of $\ymed$ on $\ystar$, which stems from $\ymed$ having a positive initial advantage and the inner product of $\phi (\ystar)$ and $\phi (\ymed)$ being negative, is strong enough to ensure that the policy maintains this relationship between $\ybad$ and $\ystar$ throughout optimization.
In other words, the policy never assigns substantially more probability to $\ystar$ than to $\ybad$.
This in turn implies that the expected ground truth reward remains below $\gtreward (\ymed)$ for all $t \geq 0$.
\end{proof}

\begin{remark}
Consider the policy parameters $\theta (\beta) = \beta \cdot \phi (\ystar)$, for $\beta \in \R$.
Under the conditions of \zcref{prop:neg_inner_product_fail} on $\{ \phi (y) \}_{y \in \Y}$, taking $\beta \to \infty$ implies $\pi_{\theta (\beta)} (\ystar) \to 1$ since $\inprod{\phi (\ystar)}{\phi (\ymed)} < 0$ and $\inprod{\phi (\ystar)}{\phi (y)} = 0$ for all $y \in \Ybad$.
Thus, as mentioned in the text preceding \zcref{prop:neg_inner_product_fail}, this result shows that even though there exist policy parameters that assign $\ystar$ a probability arbitrarily close to one, policy gradient fails to reach them.
\end{remark}

\begin{remark}
For completeness, we provide an example of output feature vectors satisfying \zcref{assum:feature_structure_neg_inn_prod} for which the conditions in \zcref{prop:neg_inner_product_fail} hold.
 Consider for simplicity a case with only three outputs, $\ystar$, $\ymed$, and $\ybad$. 
 Let
\[
\phi(\ystar)=\ebf_1 \quad , \quad
\phi(\ymed)= -\rho \cdot \ebf_1 + \sqrt{1-\rho^2} \cdot \ebf_2 \quad , \quad
\phi(\ybad)=\lambda \cdot \ebf_3
\text{\,,}
\]
for $0 < \lambda^2 < \rho < 1$ and orthonormal vectors $\ebf_1, \ebf_2, \ebf_3 \in \R^D$.
In this case, the inner products between the feature vectors are given by:
\[
\langle \phi(\ystar),\phi(\ymed)\rangle=-\rho<0 \quad , \quad
\langle \phi(\ybad),\phi(\ymed)\rangle
=
\langle \phi(\ybad),\phi(\ystar)\rangle
=0
\text{\,.}
\]
Thus, taking $\theta_0 = C \cdot \brk*{ \phi (\ybad) - \phi (\ystar) }$ for some $C > 0$, the initial logits assigned to $\ystar$, $\ymed$, and $\ybad$ by $\theta_0$ are:
\[
\inprod{\phi (\ystar)}{ \theta_0 }  = -C \quad , \quad
\inprod{\phi (\ymed) }{ \theta_0} = C\rho \quad , \quad
\inprod{\phi (\ybad) }{ \theta_0 } = C \lambda^2
\text{\,.}
\]
This implies that $\pi_{\theta_0}(\ymed)>\pi_{\theta_0}(\ybad)>\pi_{\theta_0}(\ystar)$.
In particular, $\ymed$ has the highest initial probability, as required by the first condition in \zcref{prop:neg_inner_product_fail}.
For the second and third conditions, notice that:
\[
\zeta := \frac{\ln \pi_{\theta_0}(\ymed) - \ln \pi_{\theta_0}(\ybad)}{\ln \pi_{\theta_0}(\ymed) - \ln \pi_{\theta_0}(\ystar)} \cdot \frac{\Vgt(\theta_0) - \gtreward(\ybad)}{\gtreward(\ystar)-\Vgt(\theta_0)} =
\frac{\rho-\lambda^2}{\rho+1}\cdot
\frac{\Vgt(\theta_0)-\gtreward(\ybad)}{\gtreward(\ystar)-\Vgt(\theta_0)}
> 0
\]
and
\[
\frac{\pi_{\theta_0}(\ystar)}{\pi_{\theta_0}(\ybad)}
=
\exp \brk*{ -C (1 + \lambda^2) }
\text{\,.}
\]
Hence, for any $C > \max \brk[c]1{ 0, - \ln (\zeta) / \norm{ \phi (\ybad) - \phi (\ystar) }^2, \ln(\Delta_1/\Delta_2) }$ we get that:
\[
\pi_{\theta_0}(\ystar) \leq \frac{\Deltatwo \pi_{\theta_0}(\ybad)}{\Deltaone}
\text{\,,}
\]
meaning all conditions in \zcref{prop:neg_inner_product_fail} are satisfied.
\end{remark}

\subsubsection{Positive Inner Product Between $\phi (\ystar)$ and $\phi (\ymed)$}
\label{app:formal_statements:feature_similarity:pos_in_prod}

Moving to the case where $\inprod{\phi (\ystar)}{\phi (\ymed)} > 0$, we again consider the ground truth reward function described in \zcref{assum:reward_structure_orthonormal} and impose the same assumptions on the initial policy as in the negative inner product case (\ie, \zcref{assum:initial_probs_neg_inner_product}).
Furthermore, we make the following assumption on the structure of output features, analogous to \zcref{assum:feature_structure_neg_inn_prod}.

\begin{assumption}[Output feature vectors]
\label{assum:feature_structure_pos_inn_prod}
The feature vectors of $\ystar$ and $\ymed$ have positive inner product, \ie, $\inprod{\phi (\ystar)}{\phi (\ymed)} > 0$, and $\inprod{ \phi (\ybad) }{ \phi (y) } = 0$ for all $\ybad \in \Ybad, y \in \Y \setminus \{ \ybad \}$. 
\end{assumption}

Suppose that we maximize $\gtreward$ directly.
In contrast to the case of $\inprod{\phi (\ystar)}{\phi (\ymed)} < 0$, when the advantage $\advgt (\ymed ; \theta_t) = \gtreward (\ymed) - \Vgt (\theta_t)$ is positive, the mediocre output $\ymed$ contributes a positive term $\pi_{\theta_t} (\ymed) \advgt (\ymed ; \theta_t) \cdot \inprod{ \phi (\ystar) }{ \phi (\ymed) }$ to the logit dynamics of $\ystar$ (see \zcref{prop:logit_dynamics}).
Intuitively, this implies that the attraction to $\ymed$ need not impede optimization: even if probability mass initially moves towards $\ymed$, this movement can simultaneously facilitate the growth of $\pi_{\theta_t}(\ystar)$.
However, once $\Vgt (\theta_t)$ exceeds $\gtreward(\ymed)$, the advantage of $\ymed$ becomes negative.
As a result, the contribution of $\ymed$ to the logit of $\ystar$ reverses sign and starts pushing it downwards.
This creates an intricate dynamics, where to enable the growth of $\pi_{\theta_t} (\ystar)$ towards one, the contribution of $\ystar$ to its own logit dynamics, \smash{$\pi_{\theta_t} (\ystar) \advgt (\ystar ; \theta_t) \cdot \norm{\phi (\ystar)}^2$}, needs to overcome the negative contribution of $\ymed$.

To enable a clean characterization, we focus on a regime where \smash{$\norm{\phi (\ystar)}^2 > \inprod{\phi (\ystar)}{\phi (\ymed)}$} and \smash{$\inprod{\phi (\ystar)}{\phi (\ymed)} > \norm{\phi (\ymed)}^2$}, leaving treatment of other regimes to future work.
As established in \zcref{thm:attraction_to_mediocre_outputs_beneficial_error_pos_inner_product}, under this condition the mediocre output $\ymed$ does not harm optimization.
Specifically, the time needed to achieve high ground truth reward when maximizing $\gtreward$ directly has the same asymptotic dependence on $\pi_{\theta_0}(\ystar)$ as the upper bounds in \zcref{thm:attraction_to_mediocre_outputs_beneficial_error_formal,thm:attraction_to_mediocre_outputs_beneficial_error_neg_inner_product}, which apply to cases where $\ymed$ is assigned low proxy reward.
Furthermore, if $\phi (\ystar)$ and $\phi (\ymed)$ are highly similar, in the sense that their inner product is above a certain threshold, then assigning $\ymed$ a mediocre reward can be necessary for optimization to succeed.
In such cases, we prove that if one maximizes a proxy reward function $\proxyreward$ that assigns low reward to $\ymed$, the policy fails to achieve near-optimal ground truth reward, and thus never assigns high probability to~$\ystar$.

\begin{remark}
\zcref{thm:attraction_to_mediocre_outputs_beneficial_error_pos_inner_product} implies that, when the feature vectors of $\ystar$ and $\ymed$ are highly similar, the type of reward error caused by assigning $\ymed$ a low proxy reward changes from beneficial to harmful.
By the same reasoning, in the case where $\ymed$ has low ground truth reward, assigning it a mediocre proxy reward is beneficial rather than harmful.
\end{remark}

\begin{theorem}
\label{thm:attraction_to_mediocre_outputs_beneficial_error_pos_inner_product}
Suppose that the ground truth reward function $\gtreward$, initial policy parameters $\theta_0$, and output feature vectors $\{ \phi (y) \}_{y \in \Y}$ uphold \zcref{assum:reward_structure_orthonormal,assum:initial_probs_neg_inner_product,assum:feature_structure_pos_inn_prod}.
Furthermore, assume that \smash{$ \norm{\phi (\ymed)}^2 < \inprod{\phi (\ystar)}{\phi (\ymed)} < \norm{\phi (\ystar)}^2 < 100 \norm{ \phi (\ymed) }^2$}, and let $\proxyreward$ be a proxy reward function identical to $\gtreward$, except that it assigns $\ymed$ a low reward $\proxyreward (\ymed) \le \min\nolimits_{y \in \Ybad} \gtreward (y)$.
For any $\epsilon \in (0, \Deltaone)$, where $\Deltaone := \gtreward (\ystar) - \gtreward (\ymed)$, the following hold.
\begin{itemize}[leftmargin=6mm]
    \item \textbf{Case I: Maximizing $\gtreward$.}
    If gradient flow is used to maximize the expected reward with respect to $\gtreward$ (\zcref{eq:gf} with $\Vgt$ in place of $\Vproxy$), then the initial time at which $\Vgt (\theta_t) \geq \gtreward( \ystar ) - \epsilon$, denoted $\teps$, is upper bounded as follows:
    \[
        \teps \leq \frac{ \brk*{ \gtreward (\ystar) + 1 }^2 }{\epsilon^2  \Delta_1 \norm{ \phi (\ystar) - \phi(\ymed) }^2 } \cdot \pi_{\theta_0} (\ystar)^{-1} = \OO \brk*{ \pi_{\theta_0} (\ystar)^{-1} }
        \text{\,.}
    \]

    \item \textbf{Case II: Maximizing $\proxyreward$.}    
    In contrast, if gradient flow is used to maximize the expected reward with respect to $\proxyreward$ and $\inprod{ \phi (\ystar) }{ \phi (\ymed) }$ is sufficiently high, in the sense that
    \[
        \inprod{ \phi (\ystar) }{ \phi (\ymed) } > \frac{ \pi_{\theta_0} (\ystar) \advproxy (\ystar ; \theta_0) \norm{ \phi (\ystar) }^2 - \pi_{\theta_0} (\ymed) \advproxy (\ymed ; \theta_0) \norm{ \phi (\ymed )}^2 }{ \pi_{\theta_0} (\ystar) \advproxy (\ystar ; \theta_0) - \pi_{\theta_0} (\ymed) \advproxy (\ymed ; \theta_0) }
        \text{\,,}
    \]
    then $\Vgt (\theta_t) < \gtreward (\ystar) - \brk{ \gtreward (\ystar) - \gtreward (\ymed) } \pi_{\theta_0} (\ymed)$ for all $t \geq 0$.
\end{itemize}
\end{theorem}

\begin{proof}[Proof sketch (full proof in \zcref{app:proofs:attraction_to_mediocre_outputs_beneficial_error_pos_inner_product})]
When maximizing $\gtreward$ directly (Case I), the logit dynamics of $\ystar$ is given by (\cf~\zcref{prop:logit_dynamics}):
\[
    \frac{d}{dt} \inprod{\phi (\ystar)}{\theta_t} = \pi_{\theta_t} (\ystar) \advgt (\ystar ; \theta_t) \cdot \norm*{ \phi(\ystar) }^2  + \pi_{\theta_t} (\ymed) \advgt (\ymed ; \theta_t) \cdot \inprod{ \phi(\ystar) }{ \phi(\ymed) }
    \text{\,.}
\]
The advantage of the mediocre output $\ymed$ is initially positive, \ie, $\advgt (\ymed ; \theta_0) > 0$.
Thus, since $\inprod{ \phi (\ystar) }{ \phi (\ymed) } > 0$, whenever $\ymed$ attracts probability it also pushes the logit of $\ystar$ upwards.
In particular, the conditions on the output feature vectors ensure that $\pi_{\theta_t} (\ystar)$ is monotonically non-decreasing, as opposed to the orthonormal feature vectors case (\zcref{thm:attraction_to_mediocre_outputs_beneficial_error_formal}) where $\pi_{\theta_t} (\ystar)$ initially decreases due to the increase in the logit of $\ymed$.
As a result, the policy does not become highly concentrated on $\ymed$ and $\teps$ can be upper bounded by a quantity that has the same asymptotic dependence on $\pi_{\theta_0} (\ystar)$ as the upper bound in \zcref{thm:attraction_to_mediocre_outputs_beneficial_error_formal}, which applies to cases where $\ymed$ is assigned low proxy reward.
On the other hand, when maximizing $\proxyreward$ (Case II), the advantage of $\ymed$ is negative throughout optimization since $\proxyreward$ assigns $\ymed$ a low reward.
Hence, $\ymed$ contributes negatively to the logit dynamics of $\ystar$.
We show that this suppressive effect of $\ymed$ on $\ystar$ prevents the policy from assigning $\ystar$ a higher probability than $\ymed$.
Specifically, $\pi_{\theta_t} (\ystar) / \pi_{\theta_t} (\ymed)$ is monotonically non-increasing for all $t \geq 0$, and initially $\pi_{\theta_0} (\ystar) < \pi_{\theta_0} (\ymed)$.
This in turn implies that the expected ground truth reward is bounded away from $\gtreward( \ystar)$ even as $t \to \infty$.
\end{proof}

% For completeness, \zcref{lemma:pos_inner_product_fail_conditions_nonvacuous} establishes that the conditions for Case II of \zcref{thm:attraction_to_mediocre_outputs_beneficial_error_pos_inner_product} are non-vacuous.
% That is, there exist settings in which these conditions jointly hold.

% \begin{lemma}
% \label{lemma:pos_inner_product_fail_conditions_nonvacuous}
% Under the setting of \zcref{thm:attraction_to_mediocre_outputs_beneficial_error_pos_inner_product}, suppose that $D\ge |\Y|$ and $1+\gtreward(\ybad)>0.05\gtreward(\ymed)$ for some $\ybad\in\Ybad$.
% Then there exist feature vectors $\{ \phi (y) \}_{y \in \Y}$, initial policy parameters $\theta_0$, and a value for $\proxyreward (\ymed) \in [-1, 1]$ such that the conditions for Case II of \zcref{thm:attraction_to_mediocre_outputs_beneficial_error_pos_inner_product} hold.
% Specifically, \zcref{assum:initial_probs_neg_inner_product,assum:feature_structure_pos_inn_prod}, \smash{$ \norm{\phi (\ymed)}^2 < \inprod{\phi (\ystar)}{\phi (\ymed)} < \norm{\phi (\ystar)}^2 < 100 \norm{ \phi (\ymed) }^2$}, and the lower bound on $\inprod{ \phi (\ystar) }{ \phi (\ymed) }$ are all satisfied.
% \end{lemma}

% The proof of \zcref{lemma:pos_inner_product_fail_conditions_nonvacuous} is deferred to \zcref{app:proofs:pos_inner_product_fail_conditions_nonvacuous}.

\subsubsection{Empirical Demonstration}
\label{app:formal_statements:feature_similarity:empirical_demonstration}

\zcref{fig:loglin_exps_features_sim} empirically demonstrates how the similarity between output features affects the extent to which mediocre outputs impede optimization. Consistent with our theoretical analysis (\zcref{thm:attraction_to_mediocre_outputs_beneficial_error_formal,thm:attraction_to_mediocre_outputs_beneficial_error_neg_inner_product,thm:attraction_to_mediocre_outputs_beneficial_error_pos_inner_product}), when $\phi (\ystar)$ and $\phi (\ymed)$ have a negative inner product, attraction to $\ymed$ is more severe than in the orthogonal case, causing the policy to stall near $\ymed$ for longer. Conversely, when $\phi(\ystar)$ and $\phi(\ymed)$ have a positive inner product, the attraction to $\ymed$ is milder and does not obstruct the policy from assigning high probability to $\ystar$.

\begin{figure*}[t]
	\vspace{-4mm}
	\begin{center}
        \includegraphics[width=1\textwidth]{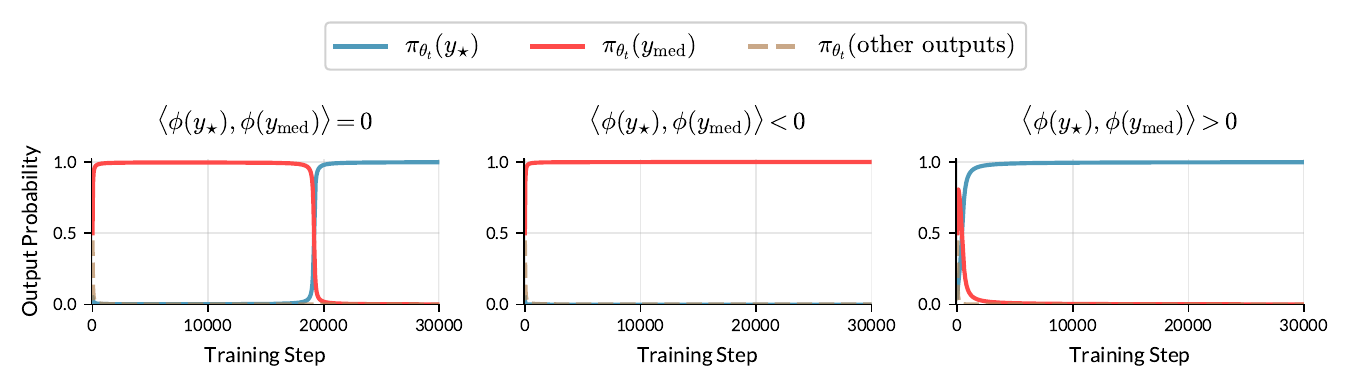}
	\end{center}
	\vspace{-2.5mm}
	\caption{
		\textbf{Feature similarity affects whether mediocre outputs impede policy gradient optimization.}
		Plotted is the evolution of output probabilities during policy gradient in settings corresponding to \zcref{thm:attraction_to_mediocre_outputs_beneficial_error_formal,thm:attraction_to_mediocre_outputs_beneficial_error_neg_inner_product,thm:attraction_to_mediocre_outputs_beneficial_error_pos_inner_product} (left to right plots).
        We train linear softmax policies using exact gradients of the expected ground truth reward $\gtreward$, which assigns a maximal reward of $1$ to $\ystar$, a mediocre reward of $0.8$ to $\ymed$, and a low reward of $-1$ to the remaining outputs.
        In all experiments, initial policy probabilities and feature vector norms are identical, with $\pi_{\theta_0} (\ystar) = 0.05$, $\pi_{\theta_0} (\ymed) = 0.5$, $\| \phi(\ystar) \| = 1.5$, and $\| \phi (\ymed) \| = 1$.
        The feature vectors of other outputs (\ie, not $\ystar$ and $\ymed$) have unit norm and are mutually orthogonal, as well as orthogonal to $\phi (\ystar)$ and $\phi (\ymed)$.
        What varies across the experiments is the value of \smash{$\inprod{\phi (\ystar)}{\phi (\ymed)}$}, which is equal to $0$ in the leftmost plot, \smash{$- 3 / 2\sqrt{2}$} in the middle plot, and \smash{$3 / 2 \sqrt{2}$} in the rightmost plot.
		In line with our theory, when $\phi (\ystar)$ and $\phi (\ymed)$ are orthogonal or have a negative inner product, the policy initially concentrates its probability mass on $\ymed$ and stagnates, with the attraction to $\ymed$ being stronger in the negative inner product case.
		By contrast, when $\phi (\ystar)$ and $\phi (\ymed)$ have a positive inner product, the attraction to $\ymed$ is milder and does not obstruct the policy from assigning high probability to $\ystar$.
		See \zcref{app:experiments:details} for additional implementation details.
	}
	\label{fig:loglin_exps_features_sim}
\end{figure*}

	% DEFERRED PROOFS
	\section{Deferred Proofs}
\label{app:proofs}

\subsection{Notation}

We use $\advgt(y;\theta) := \gtreward(y) - \Vgt(\theta)$ to denote the advantage of an output $y \in \Y$ under the ground truth reward function $\gtreward$ and the policy $\pi_\theta$.
Similarly, we let $\advproxy(y;\theta) := \proxyreward(y) - \Vproxy(\theta)$ denote the advantage of $y$ under $\proxyreward$ and $\pi_\theta$.
Furthermore, $\norm{\cdot}$ denotes the Euclidean norm and $\norm{\cdot}_1$ the $\ell_1$ norm.
For matrices, $\norm{\cdot}_2$ stands for the spectral norm.

\subsection{Proof of \zcref{prop:logit_dynamics}}
\label{app:proofs:logit_dynamics}

Fix an output $y \in \Y$.
By \zcref{eq:gf} and the chain rule we have that
\[
\frac{d}{dt} \inprod{\phi (y)}{\theta_t} 
=
\left\langle \phi(y), \tfrac{d}{dt}\theta_t \right\rangle
=
\langle \phi(y), \nabla \Vproxy(\theta_t)\rangle .
\]
By \zcref{lem:gradient_expression} we know that $\nabla \Vproxy(\theta) = \sum\nolimits_{z\in \Y}\pi_\theta(z)\advproxy(z;\theta) \cdot \phi(z)$.
Plugging this gradient expression into the logit dynamics above yields:
\[
\frac{d}{dt}\langle \phi(y),\theta_t\rangle
=
\left\langle
\phi(y),
\sum\nolimits_{z\in \Y}\pi_{\theta_t}(z)\advproxy(z;\theta_t)\cdot \phi(z)
\right\rangle
=
\sum\nolimits_{z\in \Y}\pi_{\theta_t}(z)\advproxy(z;\theta_t)\langle \phi(z),\phi(y)\rangle.
\]
Separating the term corresponding to $z=y$ from the rest concludes the proof:
\[
\frac{d}{dt}\langle \phi(y),\theta_t\rangle
=
\pi_{\theta_t}(y)\advproxy(y;\theta_t) \cdot \norm{ \phi (y)}^2
+
\sum\nolimits_{z\in \Y\setminus\{y\}}
\pi_{\theta_t}(z)\advproxy(z;\theta_t)\cdot \langle \phi(z),\phi(y)\rangle.
\]
\qed

\subsection{Proof of \zcref{prop:low_prob_outputs}}
\label{app:proofs:low_prob_outputs}

If $\Delta_\Z = 0$, then $\proxyreward = \gtreward$, and so trivially $\theta_t = \thetagt_t$ and $\TV \brk1{\pi_{\theta_t}, \pi_{\thetagt_t}} = 0$ for all $t \ge 0$.
Hence, in the remainder of the proof we assume $\Delta_\Z > 0$.
We begin by showing that $\pi_{\theta_T}(\Z)$ remains low until time $T$, regardless of which reward function is used for optimization.

\begin{lemma}
\label{lemma: prob_A_upper_loglin}
Suppose that gradient flow is used to maximize the expected reward with respect to $r: \Y \to [-1, 1]$, starting from $\theta_0$.
Then, for all $t \in [0, T]$ it holds that:
\begin{align*}
    \pi_{\theta_t}(\Z) \le \pi_{\theta_0}(\Z) \cdot \exp \brk*{4B^2T} \leq \frac{2}{\Delta_\Z \exp \brk{6 B^2 T}} \cdot \epsilon
    \text{\,.}
\end{align*}
\end{lemma}

\begin{proof}
For simplicity of notation, let $f_y(\theta) := \inprod{ \phi (y) }{ \theta }$ denote the logit of $y \in \Y$ under the policy parameters $\theta$.
By \zcref{prop:logit_dynamics}, the fact that rewards are bounded within $[-1, 1]$, and the Cauchy-Schwarz inequality, we can upper bound $\frac{d}{dt} f_y (\theta_t)$, for any $y \in \Y$ and $t \in [0, T]$, as follows:
\be
\begin{split}
    \abs*{ \frac{d}{dt} f_y (\theta_t) } &= \abs*{ \inprod{ \phi(y) }{ \tfrac{d}{dt} \theta_t } } \\
    & = \abs*{ \sum\nolimits_{z \in \Y} \pi_{\theta_t}(z) (r (z) - V (\theta_t)) \langle \phi(y), \phi(z) \rangle } \\
    & \le 2\sum\nolimits_{z \in \Y} \pi_{\theta_t}(z)   \abs*{ \langle \phi(y), \phi(z) \rangle } \\
    & \le 2 \sum\nolimits_{z \in \Y} \pi_{\theta_t}(z) \norm{ \phi (y) } \norm{ \phi (z) }  \\
    & \le 2B^2 \sum\nolimits_{z \in \Y} \pi_{\theta_t}(z)  \\
    & = 2B^2 ,
\end{split}
\label{eq: upper_f}
\ee
where $V (\theta) := \EE\nolimits_{z \sim \pi_{\theta}} \brk[s]{ r(z) }$ is the expected reward with respect to $r$ and $\pi_{\theta}$, and recall that $B = \max_{y \in \Y} \norm{ \phi(y) }$.
This leads to an upper bound on $\frac{d}{dt} \pi_{\theta_t}(y)$:
\begin{align*}
    \frac{d}{dt} \pi_{\theta_t}(y) & = \pi_{\theta_t} (y) \frac{d}{dt} \ln \pi_{\theta_t} (y) \\
    & = \pi_{\theta_t}(y) \brk*{ \tfrac{d}{dt} f_y (\theta_t) - \tfrac{d}{dt} \ln \brk*{ \sum\nolimits_{u \in \Y} \exp \brk*{ f_u (\theta_t) } } } \\
    & = \pi_{\theta_t}(y) \brk*{ \tfrac{d}{dt} f_y (\theta_t) - \EE\nolimits_{u \sim \pi_{\theta_t}}\brk[s]*{ \tfrac{d}{dt} f_u(\theta_t) } } \\
    &\le \pi_{\theta_t}(y) \brk*{ \abs*{ \tfrac{d}{dt} f_y(\theta_t) } + \abs*{ \EE\nolimits_{u \sim \pi_{\theta_t}} \brk[s]*{ \tfrac{d}{dt} f_u(\theta_t) } } }\\
    &\le \pi_{\theta_t}(y) \brk*{ \abs*{ \tfrac{d}{dt} f_y(\theta_t) } + \EE\nolimits_{u \sim \pi_{\theta_t}} \brk[s]*{ \abs*{ \tfrac{d}{dt} f_u(\theta_t) } }  } \\
    &\le 4B^2\pi_{\theta_t}(y).
\end{align*}
Here, in the first and second inequalities we apply the triangle inequality, and the last inequality is from \zcref{eq: upper_f}.
Summing over all $y \in \Z$, we therefore have that for all $t \in [0, T]$:
\[
    \frac{d}{dt}\pi_{\theta_t}(\Z) \le 4B^2 \pi_{\theta_t}(\Z) .
\]
By Grönwall's inequality, this implies that for all $t \in [0, T]$:
\begin{align*}
    \pi_{\theta_t}(\Z) \le \pi_{\theta_0}(\Z) \cdot \exp \brk{4B^2 t} \leq \pi_{\theta_0}(\Z) \cdot \exp \brk{4B^2 T} \leq \frac{2}{\Delta_\Z \exp(6 B^2 T)} \cdot \epsilon ,
\end{align*}
where the last transition is by the assumption that $\pi_{\theta_0} (\Z) \leq \frac{2}{\Delta_\Z \exp \brk{10 B^2 T}} \cdot \epsilon$
\end{proof}

\medskip

With \zcref{lemma: prob_A_upper_loglin} in place, we turn to bound the deviation between $\theta_t$ and $\thetagt_t$ at time $t \in [0, T]$.
Examining their time derivatives, by the triangle inequality we get that
\[
    \norm*{ \tfrac{d}{dt} \thetagt_t -  \tfrac{d}{dt} \theta_t} = \norm{ \nabla \Vgt(\thetagt_t) - \nabla \Vproxy(\theta_t) } \le \underbrace{ \norm{ \nabla \Vgt(\thetagt_t) - \nabla \Vgt(\theta_t) } }_{I_1} + \underbrace{ \norm{ \nabla \Vgt(\theta_t) - \nabla \Vproxy(\theta_t) } }_{I_2} .
\]
For $I_1$, \zcref{lem:tabular_grad_loglin} shows that $\nabla \Vgt$ is $6B^2$-Lipschitz, for $B = \max_{y \in \Y} \norm{ \phi(y) }$.
Thus,
\begin{align*}
    I_1 \le 6B^2 \norm*{ \thetagt_t - \theta_t }.
\end{align*}
For $I_2$, by examining the gradient expression derived in \zcref{lem:gradient_expression} we can see that
\[
\begin{split}
\nabla \Vgt(\theta_t) - \nabla \Vproxy(\theta_t) & =  \EE\nolimits_{y \sim \pi_{\theta_t}} \brk[s]*{ \brk*{ \gtreward (y) - \Vgt (\theta_t) - \proxyreward(y) + \Vproxy (\theta_t) } \cdot \phi (y) } \\
& = \EE\nolimits_{y \sim \pi_{\theta_t}} \brk[s]*{ \brk*{ \gtreward (y) - \proxyreward (y) } \cdot \phi (y) } - \brk*{ \Vgt (\theta_t) - \Vproxy (\theta_t) } \cdot \bar{\phi}_{\theta_t} \\
& = \EE\nolimits_{y \sim \pi_{\theta_t}} \brk[s]*{ \brk*{ \gtreward (y) - \proxyreward (y) } \cdot \phi (y) } - \EE\nolimits_{y \sim \pi_{\theta_t}} \brk[s]*{ \brk*{ \gtreward (y) - \proxyreward (y)} \cdot \bar{\phi}_{\theta_t} } \\
& = \EE\nolimits_{y \sim \pi_{\theta_t}} \brk[s]*{ \brk*{ \gtreward (y) - \proxyreward (y) } \cdot \brk*{ \phi (y) - \bar{\phi}_{\theta_t} } } ,
\end{split}
\]
where $\bar{\phi}_{\theta_t} := \EE\nolimits_{z \sim \pi_{\theta_t} } \brk[s]*{ \phi (z) }$.
By the triangle inequality and the fact that $\gtreward (y) = \proxyreward (y)$ for all $y \in \Y \setminus \Z$, this leads to the following upper bound on $I_2$:
\[
\begin{split}
  I_2 & = \norm*{ \nabla \Vgt(\theta_t) - \nabla \Vproxy(\theta_t) } \\
  & \le \EE\nolimits_{y \sim \pi_{\theta_t}} \brk[s]*{ \abs*{ \gtreward(y)-\proxyreward(y) } \cdot \norm*{ \phi (y) - \bar{\phi}_{\theta_t} } } \\
    &\le 2B\Delta_{\Z} \pi_{\theta_t}(\Z) ,
\end{split}    
\]
where $\Delta_{\Z} = \max_{y \in \Z} \abs{ \gtreward(y) - \proxyreward(y) }$.
Thus, applying the bound from \zcref{lemma: prob_A_upper_loglin} on $\pi_{\theta_t}(\Z)$, we arrive at:
\begin{align*}
    I_2 \leq \frac{4B \epsilon }{\exp(6B^2T)} .
\end{align*}
Combining the upper bounds on $I_1$ and $I_2$, we get:
\begin{align}
    \norm*{ \tfrac{d}{dt} \thetagt_t - \tfrac{d}{dt} \theta_t } \leq 6B^2 \norm*{ \thetagt_t-\theta_t } + \frac{4B \epsilon}{\exp(6B^2T)} .
    \label{eq: low_prob_ode}
\end{align}
Since $h(t) := \thetagt_t-\theta_t$ is absolutely continuous with respect to $t$ over $[0, T]$, so is $\norm{ h (t) }$.
As a result, for almost every $t \in [0,T]$,
\[
\frac{d}{dt} \norm*{ \thetagt_t-\theta_t } \leq \norm*{ \tfrac{d}{dt} \thetagt_t - \tfrac{d}{dt} \theta_t } .
\]
Thus, by \zcref{eq: low_prob_ode}, for almost every $t \in [0, T]$ it holds that
\[
\frac{d}{dt}\|\thetagt_t-\theta_t\| \le 6B^2\|\thetagt_t-\theta_t\| + \frac{4B \epsilon}{\exp(6B^2T)} .
\]
Grönwall's inequality, along with the fact that $\thetagt_0 = \theta_0$, then yields for all $t \in [0, T]$:
$$
\|\thetagt_t - \theta_t \| \le \frac{4B\epsilon}{\exp(6B^2T)} \cdot \frac{\exp(6B^2T)-1}{6B^2} \le \frac{2\epsilon}{3B} \le \frac{\epsilon}{B}. 
$$
Finally, since $\TV (\pi_{\theta} , \pi_{\theta'}) \leq B \norm{ \theta - \theta'}$ (\zcref{lem:tv_lipschitz_loglin}) for any $\theta, \theta' \in \R^D$, we may conclude that for all $t \in [0, T]$:
\begin{align*}
    \TV \brk1{ \pi_{\theta_t}, \pi_{\thetagt_t} } \le B\|\thetagt_t - \theta_t \| \le \epsilon.
\end{align*}
\qed

\subsubsection{Auxiliary Lemmas for the Proof of \zcref{prop:low_prob_outputs}}

\begin{lemma}
\label{lem:tv_lipschitz_loglin}
For any $\theta,\theta'\in\R^{D}$ it holds that $\TV \brk*{ \pi_\theta,\pi_{\theta'} } \leq B \norm{\theta - \theta' }$, where $B = \max_{y \in \Y} \norm{ \phi(y) }$.
\end{lemma}

\begin{proof}
Define the curve $\theta (s) := \theta'+s \cdot (\theta-\theta')$ for $s \in[0,1]$.
By the fundamental theorem of calculus it holds that
\[
\pi_\theta - \pi_{\theta^\prime}
= \int_0^1 J( \theta (s) ) (\theta-\theta') \, ds ,
\]
where $J(\xi)$ is the Jacobian of $\pi_{\xi}$ at $\xi \in \R^D$, with rows
$J(\xi)_{y, :} = \brk{ \nabla \pi_{\xi} (y) }^\top$.
Taking $\ell_1$ norms and using the induced operator norm $\|\cdot\|_{2\to 1}$ gives
\begin{align}
\norm*{ \pi_\theta-\pi_{\theta'} }_1
\le \int_0^1 \|J(\theta(s))(\theta-\theta')\|_1 \, ds
\le \|\theta-\theta'\| \cdot \int_0^1 \|J(\theta(s))\|_{2\to 1}\,ds .
\label{eq: small_prob_int_upper}
\end{align}
It therefore suffices to show that $\|J(\xi)\|_{2\to 1}\le 2B$ for all $\xi \in \R^D$.
For linear softmax policies,
\begin{align*}
    \nabla \pi_{\xi}(y) = \pi_\xi (y) (\phi(y) - \bar \phi_\xi), 
\end{align*}
where $\bar \phi_\xi := \sum_{y \in \Y} \pi_{\xi}(y) \phi(y)$. 
Furthermore, from the definition of norm $\| \cdot \|_{2\to1}$, we have:
\begin{align*}
    \norm*{ J(\xi) }_{2\to1} &= \sup\nolimits_{\|v\|=1}\|J(\xi)v\|_1 = \sup\nolimits_{\|v\|=1}\sum\nolimits_{y \in \Y} \abs*{\inprod{\nabla \pi_{\xi}(y)}{v} } ,
\end{align*}
and we can upper bound $\abs{ \inprod{\nabla \pi_{\xi}(y)}{v} }$, for $y \in \Y$ and $v \in \R^D$ with $\norm{v} = 1$, by:
\begin{align*}
    \abs*{ \inprod{\nabla \pi_{\xi}(y)}{v} } &= \abs*{ \pi_{\xi}(y) \inprod{\phi(y) - \bar \phi_\xi}{v} } \\
    &\le \pi_{\xi}(y) \cdot \norm*{ \phi(y) - \bar \phi_\xi } \cdot \norm*{ v }\\
    &\le \pi_{\xi}(y) \cdot 2B .
\end{align*}
The first inequality is from the Cauchy-Schwarz inequality and the second is due to $\|\phi(y)\| \le B$ for all $y \in \Y$. Summing over all $y \in \Y$ then yields:
\begin{align*}
    \sum\nolimits_{y \in \Y} \abs*{ \inprod{\nabla \pi_{\xi}(y)}{v} } \le \sum\nolimits_{y \in \Y} 2B\pi_{\xi}(y) = 2 B .
\end{align*}
Since this holds for all $v \in \R^D$ with $\norm{v} = 1$, we get the desired bound on $\norm{ J(\xi) }_{2\to1}$:
\begin{align*}
    \norm{ J(\xi) }_{2\to1} \le 2 B .
\end{align*}
Plugging this into \zcref{eq: small_prob_int_upper} concludes the proof:
\[
\TV \brk1{ \pi_{\theta}, \pi_{\theta'} } = \frac{1}{2} \norm*{ \pi_\theta-\pi_{\theta'} }_1 \le \frac{1}{2} \norm*{ \theta-\theta' } \cdot \int_0^1 2B \, ds  = B \norm*{ \theta-\theta' } .
\]
\end{proof}

\begin{lemma}
\label{lem:tabular_grad_loglin}
For any $\theta,\theta' \in \R^D$ it holds that $\|\nabla \Vgt(\theta)-\nabla \Vgt(\theta')\| \leq 6B^2\|\theta-\theta'\|$, where $B = \max_{y \in \Y} \norm{ \phi(y) }$.
\end{lemma}

\begin{proof}
By \zcref{lem:gradient_expression}, we may write $\nabla \Vgt(\theta)$ as:
\begin{align*}
    \nabla \Vgt(\theta) = \EE\nolimits_{y \sim \pi_\theta}[\advgt(y; \theta) \phi(y)] = \EE\nolimits_{y \sim \pi_\theta}[\gtreward(y) \phi(y)] - \Vgt(\theta) \cdot \bar \phi_{\theta} ,
\end{align*}
where $\bar \phi_\theta := \EE\nolimits_{y \sim \pi_{\theta}}[\phi(y)]$.
Furthermore, notice that for linear softmax policies and $y \in \Y$,
\begin{align*}
    \nabla \pi_{\theta}(y) = \pi_{\theta}(y)(\phi(y) - \bar \phi_\theta).
\end{align*}
Now, for convenience of notation denote $a(\theta) := \EE\nolimits_{y \sim \pi_\theta}[\gtreward(y) \phi(y)]$ and $b(\theta) := \bar \phi_{\theta}$, and let $J_a(\theta), J_b (\theta) \in \R^{D \times D}$ be the Jacobians of $a (\theta)$ and $b (\theta)$ at $\theta$, respectively.
We have:
\begin{align*}
    J_a(\theta) &= \sum\nolimits_{y \in \Y} \gtreward(y) \phi(y) (\nabla \pi_{\theta}(y))^{\top}\\
    &= \sum\nolimits_{y \in \Y} \gtreward(y) \phi(y) \pi_\theta (y) (\phi(y) - \bar \phi_\theta)^\top\\
    &= \EE\nolimits_{y \sim \pi_\theta} \brk[s]*{ \gtreward(y) \phi(y) (\phi(y) - \bar \phi_\theta)^\top } .
\end{align*}
Similarly,
\begin{align*}
    J_b (\theta) &= \sum\nolimits_{y \in \Y} \phi(y) \nabla \pi_{\theta}(y)^\top = \EE\nolimits_{y \sim \pi_{\theta}} \brk[s]*{ \phi(y) (\phi(y) - \bar \phi_\theta)^\top } .
\end{align*}
Thus, the Hessian $\nabla^2 \Vgt (\theta)$ is given by
\begin{align*}
    \nabla^2 \Vgt (\theta) &= J_a(\theta) - b (\theta) \nabla \Vgt (\theta)^\top - \Vgt (\theta) \cdot J_b (\theta) \\
    &= \EE\nolimits_{y \sim \pi_\theta} \brk[s]*{ \gtreward(y) \phi(y) (\phi(y) - \bar \phi_\theta)^\top } - \bar \phi_{\theta} \nabla \Vgt (\theta)^\top -  \EE\nolimits_{y \sim \pi_{\theta}} \brk[s]*{ \Vgt(\theta) \phi(y) (\phi(y) - \bar \phi_\theta)^\top } \\
    &= \EE\nolimits_{y \sim \pi_\theta} \brk[s]*{ \advgt(y; \theta) \phi(y) (\phi(y) - \bar \phi_\theta)^\top } - \bar \phi_{\theta} \nabla \Vgt (\theta)^\top .
\end{align*}
To upper bound the spectral norm of the Hessian, \ie, $\norm{ \nabla^2 \Vgt (\theta) }_2$, we bound the two terms separately.
For the first term, $\EE\nolimits_{y \sim \pi_\theta} \brk[s]{ \advgt(y; \theta) \phi(y) (\phi(y) - \bar \phi_\theta)^\top }$, we have:
\be
\begin{split}
    \norm*{ \EE\nolimits_{y \sim \pi_\theta} \brk[s]*{ \advgt(y; \theta) \phi(y) (\phi(y) - \bar \phi_\theta)^\top } }_2 & \leq \EE\nolimits_{y \sim \pi_\theta} \brk[s]*{ \norm*{ \advgt(y; \theta) \phi(y)(\phi(y) - \bar \phi_\theta)^\top }_{2} } \\
    & \le \EE\nolimits_{y \sim \pi_\theta} \brk[s]*{ 2 \norm{ \phi (y) } \norm{ \phi (y) - \bar \phi_\theta } } \\
    & \le 4B^2 .
    \label{eq: upper_op1}
\end{split}
\ee
The first inequality is from $\abs{\advgt(y; \theta)} \le 2$ and the second inequality is from $B = \max\nolimits_{y \in \Y} \norm {\phi (y)}$. 
For the second term, $\bar \phi_{\theta} \nabla \Vgt (\theta)^\top$, we have:
\begin{align}
    \norm*{ \bar \phi_{\theta} \nabla \Vgt (\theta)^\top }_{2} = \norm*{ \bar \phi_{\theta} } \norm{ \nabla \Vgt (\theta) } \le B \cdot \norm*{ \EE\nolimits_{y \sim \pi_\theta}\brk[s]*{ \advgt(y; \theta) \phi(y) } } \le 2B^2 .
    \label{eq: upper_op2}
\end{align}
Combining \zcref{eq: upper_op1} and \zcref{eq: upper_op2}, we obtain:
\begin{align*}
    \norm*{ \nabla^2 \Vgt (\theta) }_2 \le 4B^2 + 2B^2 = 6B^2 ,
\end{align*}
and so $\nabla \Vgt$ is $6B^2$-Lipschitz.
\end{proof}

\subsection{Proof of \zcref{thm:attraction_to_mediocre_outputs_beneficial_error_formal}}
\label{app:proofs:attraction_to_mediocre_outputs}

For convenience, we start by introducing the following helper variable:
\[
    \gamma := \left(\frac{\pi_{\theta_0}(\ystar)}{M}\right)^{14/13} ,
\]
where recall that $M := \min \brk[c]1{ 1 ,  0.05 \cdot \gtreward (\ymed)^{2/7} (\Delta_1+\Delta_2)^{-1} }$, with $\Deltaone := \gtreward (\ystar) - \gtreward (\ymed)$ and $\Deltatwo := \gtreward (\ymed) - \gtreward (\ybad)$ for some $\ybad \in \Ybad$.
We can then straightforwardly state the conditions of \zcref{assum:initial_probs_orthonormal} in terms of $\gamma$.

\begin{lemma}
\label{assump: multi_s}
Under \zcref{assum:reward_structure_orthonormal,assum:initial_probs_orthonormal}, the following properties hold:
\begin{enumerate}
    \item $\gamma \le \min\bigg\{0.01, \bigg(\frac{21}{1100(\Delta_1+\Delta_2)}\bigg)^{7/3}, \Delta_2^2\bigg\}$;
    \label{assumpt: multi_gamma}
    
    \item $\pi_{\theta_0}(\ystar) \leq \min\Bigg\{ \gamma^{13 / 14} ,  \frac{\gtreward(\ymed)^{2/7} }{20(\Delta_1+\Delta_2)} \cdot \gamma^{13/14} \Bigg\}$; and
    \label{assumpt: multi_ystar}

    \item $\max\Big\{ \frac{2\sqrt{\gamma}}{\Delta_2}, 0.05 \cdot  \gtreward(\ymed) \Big\} \le \pi_{\theta_0}(\Ybad) \le 0.1 \cdot \gtreward(\ymed)$. 
    \label{assumpt: multi_ybad}
\end{enumerate}
\end{lemma}

\begin{proof}
All three conditions are obtained directly from \zcref{assum:initial_probs_orthonormal} by substituting the definition of $M$ and noticing that $\pi_{\theta_0}(\ystar) = M\gamma^{13/14}$.
Note that \zcref{assum:reward_structure_orthonormal} on the ground truth reward structure ensures that $\Deltaone$ and $\Deltatwo$ are positive, so all quantities in the conditions above are well-defined.
\end{proof}

\subsubsection{Proof of Case I}

Suppose that gradient flow is used to maximize the expected reward with respect to $\gtreward$.
We define $t_\gamma$ as the initial time at which $\Vgt(\theta_t)$ reaches $\gtreward(\ymed) - \gamma$, \ie:
\[
t_\gamma := \min \brk[c]*{ t \geq 0 : \Vgt (\theta_t) \geq \gtreward(\ymed) - \gamma} ,
\]
and analyze the time interval $[0, t_\gamma]$.
In \zcref{lemma: decrease_y_bad_multi_s} and \zcref{prop: monotonical_multi_s}, we show that throughout this time interval, \(\pi_{\theta_t}(\ymed)\) monotonically increases while $\pi_{\theta_t}(\Ybad)$ and $\pi_{\theta_t} (\ystar)$ monotonically decrease. 
Moreover, we prove that until $t_\gamma$ the policy becomes highly concentrated on $\ymed$, which in turn implies through \zcref{lemma:l2_grad_bound} that $\Vgt(\theta_t)$ remains stuck near $\ymed$ for a long time.
This will yield the desired lower bound on $\teps$---the initial time at which $\Vgt (\theta_t)$ is $\epsilon$-optimal.

Note that if $\Vgt (\theta_t) < \gtreward(\ymed)$ for all $t \geq 0$, then the lower bound on $\teps$ trivially holds as $\Vgt (\theta_t)$ never reaches $\gtreward(\ystar) - \epsilon$.
Thus, throughout the proof of this case we can assume that there exists a time at which $\Vgt (\theta_t) \geq \gtreward (\ymed)$.

Towards showing that $\pi_{\theta_t} (\ymed)$ increases over $[0, t_\gamma]$ while $\pi_{\theta_t} (\ystar)$ decreases, the following lemma proves that $\frac{d}{dt} \pi_{\theta_t}(\Ybad) < 0$ when $\Vgt (\theta_t) < \gtreward(\ymed)$.

\begin{lemma}
\label{lemma: decrease_y_bad_multi_s}
At any time $t \geq 0$, if $\Vgt (\theta_t) < \gtreward(\ymed)$, then $\frac{d}{dt}\pi_{\theta_t}(\Ybad) < 0$.
\end{lemma}

\begin{proof}
    For any $z \in \Ybad$, by \zcref{lem:probability_derivative_expression}:
    \begin{align*}
        \frac{d}{dt}\pi_{\theta_t}(z) &= \pi_{\theta_t}(z) \Big(\pi_{\theta_t}(z) \advgt(z; \theta_t) - \sum\nolimits_{y \in \Y} \pi_{\theta_t}(y)^2 \advgt(y; \theta_t)\Big) ,
    \end{align*}
    Summing the equation above over all $z \in \Ybad$, we arrive at:
    \begin{align*}
        \sum_{z \in \Ybad} \frac{d}{dt}\pi_{\theta_t}(z) &= \sum_{z \in \Ybad} \pi_{\theta_t}(z)^2 \advgt(z; \theta_t) - \brk2{ \sum_{z \in \Ybad} \pi_{\theta_t}(z) } \brk2{ \sum_{y \in \Y} \pi_{\theta_t}(y)^2 \advgt(y; \theta_t) } .
    \end{align*}
    Thus,
    \begin{align*}
        &\frac{d}{dt}\pi_{\theta_t}(\Ybad) \\
        & = \sum_{z \in \Ybad} \pi_{\theta_t}(z)^2 \advgt(z; \theta_t)  \\
        & \hspace{4mm} - \Big(\!\sum_{z \in \Ybad} \!\! \pi_{\theta_t}(z)\Big) \Big(\! \sum_{z \in \Ybad} \!\! \pi_{\theta_t}(z)^2 \advgt(z; \theta_t) + \pi_{\theta_t}(\ymed)^2 \advgt(\ymed; \theta_t) + \pi_{\theta_t}(\ystar)^2 \advgt(\ystar; \theta_t) \Big) \\
        & = \underbrace{\Big(\sum\nolimits_{z \in \Ybad} \pi_{\theta_t}(z)^2 \advgt(z; \theta_t)\Big)\Big(1 - \sum\nolimits_{z \in \Ybad} \pi_{\theta_t}(z)\Big)}_{I_1} \\
        & \hspace{4mm} - \underbrace{\Big(\sum\nolimits_{z \in \Ybad} \pi_{\theta_t}(z)\Big) \Big(\pi_{\theta_t}(\ymed)^2 \advgt(\ymed; \theta_t) + \pi_{\theta_t}(\ystar)^2 \advgt(\ystar; \theta_t) \Big)}_{I_2} .
    \end{align*}
    When maximizing $\gtreward$ via gradient flow, $\Vgt (\theta_t)$ is monotonically increasing since 
    \[
        \frac{d}{dt} \Vgt (\theta_t) = \norm{ \nabla \Vgt (\theta_t) }^2 \geq 0 ,
    \]
    for all $t \geq 0$.
    Hence, $\advgt(z; \theta_t) \le 0$ for all $z \in \Ybad$ and $t \geq 0$, because outputs in $\Ybad$ have minimal ground truth reward. 
    Furthermore, as long as $\Vgt (\theta_t) < \gtreward (\ymed)$, it holds that $\advgt(\ymed; \theta_t) > 0$ and $\advgt(\ystar; \theta_t) > 0$.
    This implies that $I_1 < 0$ and $I_2 > 0$, so $\frac{d}{dt} \pi_{\theta_t}(\Ybad) = I_1 - I_2 < 0$.
\end{proof}

\medskip

Using \zcref{lemma: decrease_y_bad_multi_s}, we can now prove that until $t_\gamma$ the probability of $\ymed$ monotonically increases while the probability of $\ystar$ monotonically decreases, and characterize how small $\pi_{\theta_t} (\ystar)$ becomes.
This is a core component in the proof of \zcref{thm:attraction_to_mediocre_outputs_beneficial_error_formal}.

\begin{proposition}
\label{prop: monotonical_multi_s}
For all $t \in [0, t_{\gamma}]$ it holds that:
\begin{itemize}
    \item $\frac{d}{dt}\pi_{\theta_t}(\ystar) < 0$; and
    \item $\frac{d}{dt}\pi_{\theta_t}(\ymed) > 0$.
\end{itemize}
Furthermore,
\[
\min\nolimits_{t \in [0, t_\gamma]} \pi_{\theta_t}(\ystar) = \pi_{\theta_{t_\gamma}}(\ystar) \le \frac{\pi_{\theta_0}(\ymed)^2 \gamma}{8\advgt(\ystar; \theta_0)} .
\]
\end{proposition}

\begin{proof}
Notice that when $\frac{d}{dt} \pi_{\theta_t}(\ystar) < 0$, we must have $\frac{d}{dt} \pi_{\theta_t}(\ymed) > 0$.
This is because $\frac{d}{dt} \pi_{\theta_t} (\ystar) + \frac{d}{dt} \pi_{\theta_t} (\ymed) + \frac{d}{dt} \pi_{\theta_t} (\Ybad) = \frac{d}{dt} \pi_{\theta_t} (\Y) = \frac{d}{dt} 1 = 0$, and so when $\frac{d}{dt} \pi_{\theta_t}(\ystar) < 0$:
\[
\frac{d}{dt} \pi_{\theta_t}(\ymed) = - \frac{d}{dt} \pi_{\theta_t}(\ystar) - \frac{d}{dt} \pi_{\theta_t}(\Ybad) > - \frac{d}{dt} \pi_{\theta_t}(\Ybad) > 0 ,
\]
where the second inequality follows from \zcref{lemma: decrease_y_bad_multi_s}. 
Hence, for proving monotonicity of $\pi_{\theta_t} (\ystar)$ and $\pi_{\theta_t} (\ymed)$, it suffices to show that $\frac{d}{dt} \pi_{\theta_t}(\ystar) < 0$ for all $t \in [0, t_\gamma]$.

To do this, we let $t_{\sqrt{\gamma}}$ be the initial time at which $\Vgt (\theta_t) \geq \gtreward(\ymed) - \sqrt{\gamma}$, \ie:
\[
t_{\sqrt{\gamma}} := \min \brk[c]*{ t \geq 0 : \Vgt (\theta_t) \geq \gtreward(\ymed) - \sqrt{\gamma} } .
\]
From \zcref{lem:mediocre_output_reward_larger_than_initial_expected_reward}, $\Vgt (\theta_0) < \gtreward (\ymed) - \sqrt{\gamma}$.
Hence, $t_{\sqrt{\gamma}} > 0$.
Furthermore, since $\Vgt (\theta_t)$ is continuous in $t$, at $t_{\sqrt{\gamma}}$ it holds that $\Vgt (\theta_{t_{\sqrt{\gamma}}}) = \gtreward(\ymed) - \sqrt{\gamma}$.
The proof proceeds by showing that $\frac{d}{dt} \pi_{\theta_t}(\ystar) < 0$ up to time $t_{\sqrt{\gamma}}$ and upper bounding $\pi_{\theta_{t_{\sqrt{\gamma}}}}(\ystar)$ by $\pi_{\theta_0}(\ymed)^2 \gamma / 8\advgt(\ystar; \theta_0)$.
We will then establish that $\pi_{\theta_t} (\ystar)$ continues decreasing until time $t_\gamma$, from which it follows that the upper bound on $\pi_{\theta_{t_{\sqrt{\gamma}}}}(\ystar)$ transfers to $\pi_{\theta_{t_\gamma}}(\ystar)$.

\medskip

\textbf{Step 1: monotonicity in $[0, t_{\sqrt{\gamma}}]$.}
 We show that $\frac{d}{dt} \pi_{\theta_t}(\ystar) < 0$ and hence $\frac{d}{dt} \pi_{\theta_t}(\ymed) > 0$ for all $t \in [0, t_{\sqrt{\gamma}}]$.
Assume by way of contradiction that there exists a time $t' \le t_{\sqrt{\gamma}}$ at which $\left.\frac{d}{dt} \pi_{\theta_t}(\ystar)\right|_{t=t'} \ge 0$ and denote by $\tau$ the initial such time, \ie:
\[
\tau := \min \brk[c]*{ t \in [0, t_{\sqrt{\gamma}}] : \tfrac{d}{dt} \pi_{\theta_t}(\ystar) \geq 0 } .
\]
For all $t \in [0, \tau)$, we therefore have that $\frac{d}{dt} \pi_{\theta_t}(\ystar) < 0$ and $\frac{d}{dt} \pi_{\theta_t} (\ymed) > 0$, and so $\pi_{\theta_\tau}(\ystar) \le \pi_{\theta_0}(\ystar)$ and $\pi_{\theta_\tau}(\ymed) \ge \pi_{\theta_0}(\ymed)$. 
We now upper bound $\pi_{\theta_{\tau}} (\Ybad)^2$ under this assumption.
Notice that
\begin{align*}
        \advgt(\ymed; \theta_\tau) &= (1 - \pi_{\theta_\tau}(\ymed)) \gtreward(\ymed) - \pi_{\theta_\tau}(\ystar) \gtreward(\ystar) - \sum\nolimits_{z \in \Ybad} \pi_{\theta_\tau}(z) \gtreward(z)\\
        &\ge -\Delta_1 \pi_{\theta_\tau}(\ystar) + \Deltatwo \pi_{\theta_\tau}(\Ybad).
\end{align*}
Thus,
\begin{align*}
    \pi_{\theta_\tau}(\Ybad) \le \frac{\advgt(\ymed; \theta_\tau) + \Delta_1 \pi_{\theta_\tau}(\ystar)}{\Deltatwo} \le \frac{1.3 \advgt(\ymed; \theta_\tau)}{\Deltatwo}.
\end{align*}
The second inequality is from $\Delta_1 \pi_{\theta_\tau}(\ystar) \le 2\pi_{\theta_0}(\ystar) \le 2\gamma^{13/14} \le 0.3\sqrt{\gamma} \le 0.3\advgt(\ymed; \theta_\tau)$ for $\gamma \le 0.01$ (\zcref{assump: multi_s}).
Then, we have:
\begin{align*}
    \pi_{\theta_\tau}(\Ybad)^2 &\le \frac{1.69 \advgt(\ymed; \theta_\tau)^2}{\Deltatwo^2}\\
    &< \frac{7\pi_{\theta_0}(\ymed)^2 \advgt(\ymed; \theta_\tau)}{8 \Delta_2} \cdot \frac{2 \advgt(\ymed; \theta_\tau)}{\Deltatwo \pi_{\theta_0}(\ymed)^2}\\
    &\le \frac{7\pi_{\theta_0}(\ymed)^2 \advgt(\ymed; \theta_\tau)}{8 \Delta_2} \cdot \frac{2\advgt(\ymed; \theta_0)}{\Deltatwo \pi_{\theta_0}(\ymed)^2} ,
\end{align*}
where the second inequality is due to $2 \times 7 > 1.69 \times 8$ and the third inequality is from $\advgt(\ymed; \theta_\tau) \leq \advgt(\ymed; \theta_0)$ (under gradient flow $\Vgt (\theta_t)$ is monotonically non-decreasing). 
Focusing on the second term on the right-hand side, we prove that $\frac{2\advgt(\ymed; \theta_0)}{\Deltatwo \pi_{\theta_0}(\ymed)^2} \le 1$. 
This follows from:
\begin{align*}
    \advgt(\ymed; \theta_0) &= (1 - \pi_{\theta_0}(\ymed))\gtreward(\ymed) - \pi_{\theta_0}(\ystar) \gtreward(\ystar) - \sum\nolimits_{z \in \Ybad}\pi_{\theta_0}(z) \gtreward(z)\\
    &\le (1 - \pi_{\theta_0}(\ymed))\gtreward(\ymed) + \pi_{\theta_0}(\Ybad)\\
    &\le 0.2 \gtreward(\ymed) + 0.1 \gtreward(\ymed)\\
    &= 0.3 \gtreward(\ymed)\\
    &< \frac{0.8^2 \gtreward(\ymed)}{2}\\
    &\le \frac{\Deltatwo\pi_{\theta_0}(\ymed)^2}{2} .
\end{align*}
Here, the first inequality is from $\gtreward(z) \ge -1$ for $z \in \Ybad$ and $\gtreward(\ystar) \geq 0$; the second inequality is from \zcref{assumpt: multi_ystar,assumpt: multi_ybad} in \zcref{assump: multi_s}, which imply $\pi_{\theta_0}(\ymed) \ge 0.8$; and the last inequality is from $\gtreward(\ybad) \le 0$ and $\pi_{\theta_0}(\ymed) \ge 0.8$. 
Therefore, $\frac{2\advgt(\ymed; \theta_0)}{\Deltatwo \pi_{\theta_0}(\ymed)^2} \le 1$, which yields
\begin{equation}
\label{eq: ybad2_upper}
    \pi_{\theta_\tau}(\Ybad)^2 \le \frac{7\pi_{\theta_0}(\ymed)^2 \advgt(\ymed; \theta_\tau)}{8\Delta_2} .
\end{equation}

We now upper bound $\left.\frac{d}{dt} \pi_{\theta_t}(\ystar)\right|_{t=\tau}$, based on the expression derived in \zcref{lem:probability_derivative_expression}:
\begin{align*}
    \left.\frac{d}{dt} \pi_{\theta_t}(\ystar)\right|_{t=\tau}
    &= \pi_{\theta_\tau}(\ystar) \Big(\pi_{\theta_\tau}(\ystar) \advgt(\ystar; \theta_\tau) - \sum\nolimits_{y \in \Y} \pi_{\theta_\tau}(y)^2 \advgt(y; \theta_\tau)\Big)\\
    &= \pi_{\theta_\tau}(\ystar) \Big(\big(\pi_{\theta_\tau}(\ystar) - \pi_{\theta_\tau}(\ystar)^2\big) \advgt(\ystar; \theta_\tau) - \pi_{\theta_\tau}(\ymed)^2 \advgt(\ymed; \theta_\tau) \\
    & \hspace{18mm} - \sum\nolimits_{z \in \Ybad} \pi_{\theta_\tau}(z)^2 \advgt(z; \theta_\tau)\Big)\\
    &\le \pi_{\theta_\tau}(\ystar) \Big(\big(\pi_{\theta_\tau}(\ystar) - \pi_{\theta_\tau}(\ystar)^2\big) \advgt(\ystar; \theta_\tau) - \pi_{\theta_\tau}(\ymed)^2 \advgt(\ymed; \theta_\tau) \\
    & \hspace{18mm} + (\Delta_2 - \sqrt{\gamma}) \sum\nolimits_{z \in \Ybad} \pi_{\theta_\tau}(z)^2\Big),
\end{align*}
where the inequality is by $-\advgt(z; \theta_\tau) \le -\advgt(z;\theta_{t_{\sqrt{\gamma}}}) \le \gtreward(\ymed)-\sqrt{\gamma} - \gtreward(\ybad) = \Delta_2 - \sqrt{\gamma}$.
We continue upper bounding $\left.\frac{d}{dt} \pi_{\theta_t}(\ystar)\right|_{t=\tau}$ as follows:
\begin{align*}
    \left.\frac{d}{dt} \pi_{\theta_t}(\ystar)\right|_{t=\tau}
    & \le \pi_{\theta_\tau}(\ystar) \Big(\big(\pi_{\theta_\tau}(\ystar) - \pi_{\theta_\tau}(\ystar)^2\big) \advgt(\ystar; \theta_\tau) - \pi_{\theta_\tau}(\ymed)^2 \advgt(\ymed; \theta_\tau) \\
    & \hspace{18mm} + (\Delta_2 - \sqrt{\gamma}) \sum\nolimits_{z \in \Ybad} \pi_{\theta_\tau}(z)^2\Big)\\
    &\le \pi_{\theta_\tau}(\ystar) \Big(\big(\pi_{\theta_0}(\ystar) - \pi_{\theta_0}(\ystar)^2\big) \advgt(\ystar; \theta_\tau) - \pi_{\theta_0}(\ymed)^2 \advgt(\ymed; \theta_\tau) \\
    & \hspace{18mm} + (\Delta_2 - \sqrt{\gamma}) \pi_{\theta_\tau}(\Ybad)^2\Big)\\
    &\le \pi_{\theta_\tau}(\ystar) \Big(\big(\pi_{\theta_0}(\ystar) - \pi_{\theta_0}(\ystar)^2\big) \advgt(\ystar; \theta_\tau) - (1/8)\pi_{\theta_0}(\ymed)^2 \advgt(\ymed; \theta_\tau)\Big)\\
    &\le \pi_{\theta_\tau}(\ystar) \Big(\big(\pi_{\theta_0}(\ystar) - \pi_{\theta_0}(\ystar)^2\big) \advgt(\ystar; \theta_0) - (1/8)\pi_{\theta_0}(\ymed)^2 \sqrt{\gamma} \Big) .
\end{align*}
The second inequality is from $\pi_{\theta_\tau}(\ystar) - \pi_{\theta_\tau}(\ystar)^2 \leq \pi_{\theta_0}(\ystar) - \pi_{\theta_0}(\ystar)^2$ (this holds since $\pi_{\theta_\tau} (\ystar) \leq \pi_{\theta_0} (\ystar) < 0.5$), $\pi_{\theta_\tau}(\ymed) \ge \pi_{\theta_0}(\ymed)$, and $\sum\nolimits_{z \in \Ybad} \pi_{\theta_t}(z)^2 \le \big(\sum\nolimits_{z \in \Ybad} \pi_{\theta_t}(z)\big)^2$.
The third inequality is from $(\Delta_2-\sqrt{\gamma}) \pi_{\theta_\tau}(\Ybad)^2 \le \Delta_2 \pi_{\theta_\tau}(\Ybad)^2 \le \frac{7}{8} \cdot \pi_{\theta_0}(\ymed)^2 \advgt(\ymed; \theta_\tau)$ (\zcref{eq: ybad2_upper}).
The last inequality is from the fact that $\Vgt (\theta_t)$ is monotonically increasing, so $\advgt(\ystar; \theta_\tau) \le \advgt(\ystar; \theta_0)$ and $\advgt(\ymed; \theta_\tau) \ge \advgt(\ymed;\theta_{t_{\sqrt{\gamma}}}) = \sqrt{\gamma}$.

From \zcref{assump: multi_s}, we further obtain that:
\begin{align*}
    \pi_{\theta_0}(\ystar) \le \frac{\gtreward(\ymed)^{2/7} \gamma^{13/14}}{20(\Delta_1+\Delta_2)} \le \frac{\sqrt{\gamma} \cdot \gamma^{3/7}}{20 \advgt(\ystar; \theta_0)} < \frac{0.8^2\sqrt{\gamma}}{8\advgt(\ystar; \theta_0)} \le \frac{\pi_{\theta_0}(\ymed)^2\sqrt{\gamma}}{8 \advgt(\ystar; \theta_0)}.
\end{align*}
Here, the first inequality is from \zcref{assumpt: multi_ystar} in \zcref{assump: multi_s}; the second inequality is from $\gtreward(\ymed) \le 1$ and $\advgt(\ystar; \theta_0) \le \Delta_1 + \Delta_2$; the third inequality is from $\gamma^{3/7} < 1.6$; and the last inequality is from $\pi_{\theta_0}(\ymed) = 1 - \pi_{\theta_0}(\ystar) - \pi_{\theta_0}(\Ybad) \ge 1 - \gamma^{13/14} - 0.1\gtreward(\ymed) \ge 0.8$. 
Therefore, 
\begin{align*}
    \big(\pi_{\theta_0}(\ystar) - \pi_{\theta_0}(\ystar)^2\big) \advgt(\ystar; \theta_0) - (1/8)\pi_{\theta_0}(\ymed)^2 \sqrt{\gamma} < 0,
\end{align*}
which implies that $\left.\frac{d}{dt} \pi_{\theta_t}(\ystar)\right|_{t=\tau} < 0$.
However, this contradicts the definition of $\tau$, according to which $\left.\frac{d}{dt} \pi_{\theta_t}(\ystar)\right|_{t=\tau} \geq 0$. 
Hence, $\frac{d}{dt} \pi_{\theta_t}(\ystar) < 0$ and $\frac{d}{dt} \pi_{\theta_t} (\ymed) > 0$ for all $t \in [0, t_{\sqrt{\gamma}}]$.

\medskip

\textbf{Step 2: upper bound on $\pi_{\theta_{t_{\sqrt{\gamma}}}}(\ystar)$ and monotonicity in $[ t_{\sqrt{\gamma}}, t_\gamma]$.}
Next, we prove that $\pi_{\theta_t} (\ystar)$ keeps decreasing and $\pi_{\theta_t} (\ymed)$ keeps increasing on the time interval $[t_{\sqrt{\gamma}}, t_\gamma]$.
This requires deriving upper bounds on $\frac{d}{dt} \pi_{\theta_t}(\ystar)$ and $\frac{d}{dt} \pi_{\theta_t}(\ymed)$ for $t \in [0, t_{\sqrt{\gamma}}]$ that allow quantifying how small the probability of $\ystar$ becomes until time $t_{\sqrt{\gamma}}$.

\medskip

\textbf{Step 2.1: upper bound on $\frac{d}{dt} \pi_{\theta_t}(\ystar)$ in $[0, t_{\sqrt{\gamma}}]$.}
In this step, we prove that for all $t \in [0, t_{\sqrt{\gamma}}]$:
\[
    \frac{d}{dt} \pi_{\theta_t}(\ystar) \le -0.3 \cdot \pi_{\theta_t}(\ystar)\pi_{\theta_t}(\ymed)^2 \advgt(\ymed; \theta_t) .
\]

First, by \zcref{lem:probability_derivative_expression}, for $t \in [0, t_{\sqrt{\gamma}}]$ it holds that:
\be
\begin{split}
    & \frac{d}{dt} \pi_{\theta_t}(\ystar) \\
    & = \pi_{\theta_t}(\ystar) \Big(\pi_{\theta_t}(\ystar) \advgt(\ystar; \theta_t) - \sum\nolimits_{y \in \Y} \pi_{\theta_t}(y)^2 \advgt(y; \theta_t)\Big)\\
        &= \pi_{\theta_t}(\ystar) \brk2{\! \big(\pi_{\theta_t}(\ystar) - \pi_{\theta_t}(\ystar)^2\big) \advgt(\ystar; \theta_t) - \pi_{\theta_t}(\ymed)^2 \advgt(\ymed; \theta_t) - \!\!\! \sum_{z \in \Ybad} \!\!\! \pi_{\theta_t}(z)^2 \advgt(z; \theta_t) } \\
        &\le \pi_{\theta_t}(\ystar) \brk2{ \! \underbrace{ \big(\pi_{\theta_0}(\ystar) - \pi_{\theta_0}(\ystar)^2\big) \advgt(\ystar; \theta_t) - \pi_{\theta_t}(\ymed)^2 \advgt(\ymed; \theta_t) - \!\!\! \sum_{z \in \Ybad} \!\!\! \pi_{\theta_t}(z)^2 \advgt(z; \theta_t)}_{G} } ,
\end{split}
\label{eq:bound_ystar_time_derive_sqrt_gamma}
\ee
where the inequality is due to $\pi_{\theta_t}(\ystar)$ being monotonically decreasing in $[0, t_{\sqrt{\gamma}}]$ and $\pi_{\theta_0} (\ystar) < 0.5$ (the function $g(p) = p - p^2$ is monotonically increasing over $[0, 0.5]$). 
We proceed by establishing that $G \leq -0.3 \pi_{\theta_t}(\ymed)^2 \advgt(\ymed; \theta_t)$ through the following computations:
\begin{align*}
    & G + 0.3 \pi_{\theta_t}(\ymed)^2 \advgt(\ymed; \theta_t) \\
    & = -0.7\pi_{\theta_t}(\ymed)^2 \advgt(\ymed; \theta_t) + \big(\pi_{\theta_0}(\ystar) - \pi_{\theta_0}(\ystar)^2\big) \advgt(\ystar; \theta_t) - \sum_{z \in \Ybad} \pi_{\theta_t}(z)^2 \advgt(z; \theta_t)\\
    & \le -0.7\pi_{\theta_t}(\ymed)^2 \advgt(\ymed; \theta_t) + \pi_{\theta_0}(\ystar)\advgt(\ystar; \theta_t) - \sum_{z \in \Ybad} \pi_{\theta_t}(z)^2 \advgt(z; \theta_t)\\
    & = -0.7\pi_{\theta_t}(\ymed)^2 (\gtreward(\ymed) - \Vgt (\theta_t)) + \pi_{\theta_0}(\ystar)\advgt(\ystar; \theta_t) + \sum_{z \in \Ybad} \pi_{\theta_t}(z)^2 (\Vgt (\theta_t) - \gtreward(z))\\
    & = -0.7 \pi_{\theta_t}(\ymed)^2 \Big( (1-\pi_{\theta_t}(\ymed))\gtreward(\ymed) - \pi_{\theta_t}(\ystar)\gtreward(\ystar) - \sum\nolimits_{z \in \Ybad} \pi_{\theta_t}(z) \gtreward(\ybad) \Big ) \\
    & \hspace{4mm} + \pi_{\theta_0}(\ystar)\advgt(\ystar; \theta_t) \\
    & \hspace{4mm} + \sum\nolimits_{z \in \Ybad} \pi_{\theta_t}(z)^2\big((1-\pi_{\theta_t}(z))(-\gtreward(z)) + \pi_{\theta_t}(\ystar) \gtreward(\ystar) + \pi_{\theta_t}(\ymed) \gtreward(\ymed) \big ) \\
    & \hspace{4mm} + \sum\nolimits_{z \in \Ybad} \pi_{\theta_t}(z)^2 \brk*{ \sum\nolimits_{z^\prime \in \Ybad \setminus \{z\}} \pi_{\theta_t}(z^\prime)\gtreward(z^\prime) } \\
     & = -0.7\pi_{\theta_t}(\ymed)^2\big(\pi_{\theta_t}(\Ybad) \Deltatwo - \pi_{\theta_t}(\ystar)\Delta_1\big) + \pi_{\theta_0}(\ystar)\advgt(\ystar; \theta_t) \\
     & \hspace{4mm} + \sum\nolimits_{z \in \Ybad} \pi_{\theta_t}(z)^2\big(\pi_{\theta_t}(\ymed)\Delta_2 + \pi_{\theta_t}(\ystar)(\Delta_1+\Delta_2) \big)\\
     & \le -0.7\pi_{\theta_t}(\ymed)^2\big(\pi_{\theta_t}(\Ybad) \Deltatwo - \pi_{\theta_t}(\ystar)\Delta_1\big) + \pi_{\theta_0}(\ystar)\advgt(\ystar; \theta_t) \\
     & \hspace{4mm} + \pi_{\theta_t}(\Ybad)^2\big(\pi_{\theta_t}(\ymed)\Delta_2 + \pi_{\theta_t}(\ystar)(\Delta_1+\Delta_2) \big)\\
     & \le \underbrace{\pi_{\theta_t}(\Ybad) \Deltatwo \brk[s]*{ -0.7 \pi_{\theta_t}(\ymed)^2 + \pi_{\theta_t}(\Ybad) \pi_{\theta_t}(\ymed) } }_{I_1}\\
     & \hspace{4mm} + \underbrace{\pi_{\theta_0}(\ystar)\big(0.7 \pi_{\theta_t}(\ymed)^2 \Delta_1 + \pi_{\theta_t}(\Ybad)^2 (\Delta_1 + \Delta_2) + \advgt(\ystar; \theta_t)\big) }_{I_2}.
\end{align*}
The first inequality is from $1 - \pi_{\theta_0}(\ystar) \le 1$; the second inequality is due to $\sum_{z \in \Ybad} \pi_{\theta_t}(z)^2 \le \pi_{\theta_t} (\Ybad)^2$; and the last inequality is from $\pi_{\theta_t}(\ystar) \le \pi_{\theta_0}(\ystar)$.
We now upper bound $I_1$ and $I_2$ separately.
Specifically, we will prove the following two inequalities:
\[
\begin{split}
I_1 & \le -0.2 \pi_{\theta_t}(\Ybad) \pi_{\theta_t}(\ymed)^2\Deltatwo , \\
I_2 & \le 0.2 \pi_{\theta_t}(\Ybad) \pi_{\theta_t}(\ymed)^2 \Deltatwo .
\end{split}
\]
For $I_1$, note that it suffices to show that $0.5 \pi_{\theta_t}(\ymed) \ge  \pi_{\theta_t}(\Ybad)$. 
Indeed, from \zcref{assump: multi_s} we know that $\pi_{\theta_0}(\Ybad) < 0.2$ and $\gamma < 0.05$, so $\pi_{\theta_0}(\ystar) \le \gamma^{13/14} < 0.13$ and $\pi_{\theta_0}(\ymed) = 1 - \pi_{\theta_0}(\ystar) - \pi_{\theta_0}(\Ybad) > 0.67$.
In turn, this implies that $\pi_{\theta_t}(\Ybad) \le \pi_{\theta_0}(\Ybad) < 0.5 \pi_{\theta_0}(\ymed) \le 0.5 \pi_{\theta_t}(\ymed)$. 
Hence, $I_1\leq -0.2 \pi_{\theta_t}(\Ybad) \pi_{\theta_t}(\ymed)^2\Deltatwo$.

As for $I_2$, we have:
\begin{align*}
    0.7 \pi_{\theta_t}(\ymed)^2 \Delta_1 + \pi_{\theta_t}(\Ybad)^2 (\Delta_1 + \Delta_2) + \advgt(\ystar; \theta_t) & \le 0.7 \Delta_1 + 0.1(\Delta_1 + \Delta_2) + \advgt(\ystar; \theta_0)\\
    & \le 2(\Delta_1 + \Delta_2) .
\end{align*}
Here, the first inequality is from $\pi_{\theta_t}(\Ybad)^2 \le \pi_{\theta_0}(\Ybad)^2 \le \pi_{\theta_0}(\Ybad) \le 0.1$ and $\Vgt (\theta_t) \ge \Vgt (\theta_0)$, and the second inequality is from $\advgt(\ystar; \theta_0) \le \Delta_1 + \Delta_2$. 
Meanwhile, since $\advgt(\ymed; \theta_t) \ge \sqrt{\gamma}$, it follows that:
\begin{align*}
        \advgt(\ymed; \theta_t) &= (1 - \pi_{\theta_t}(\ymed)) \gtreward(\ymed) - \pi_{\theta_t}(\ystar) \gtreward(\ystar) - \sum\nolimits_{z \in \Ybad} \pi_{\theta_t}(z) \gtreward(z)\\
        &= -\Delta_1 \pi_{\theta_t}(\ystar) + \Delta_2 \pi_{\theta_t}(\Ybad).
    \end{align*}
Thus, $-\Delta_1 \pi_{\theta_t}(\ystar) + \Delta_2 \pi_{\theta_t}(\Ybad) \ge \sqrt{\gamma}$, which implies that
\be
\pi_{\theta_t}(\Ybad) \ge \frac{\sqrt{\gamma}+ \Delta_1 \pi_{\theta_t}(\ystar)}{\Delta_2} \ge \frac{\sqrt{\gamma}}{\Delta_2} .
\label{eq:y_bad_prob_lower_bound_t_sqrt_gamma_orth}
\ee
Since $\pi_{\theta_0} (\ymed) > 0.67$ (as we showed above when bounding $I_1$), from \zcref{assump: multi_s} we know that $\gamma \le \pi_{\theta_0}(\ymed)^{14 / 3} / \brk{ 10(\Delta_1+\Delta_2) }^{7/3}$.
This leads to:
    $$
    \pi_{\theta_0}(\ystar) \le \gamma^{13/14} \le \frac{\sqrt{\gamma} \pi_{\theta_0}(\ymed)^2}{10(\Delta_1 + \Delta_2)},
    $$
    which yields:
    \begin{align*}
        I_2 \le 2(\Delta_1 + \Delta_2) \pi_{\theta_0}(\ystar) \le 0.2 \pi_{\theta_t}(\Ybad) \pi_{\theta_t}(\ymed)^2 \Deltatwo.
    \end{align*}
Therefore, $I_1 + I_2\leq 0$ and we can conclude that $G \leq -0.3 \pi_{\theta_t}(\ymed)^2 \advgt(\ymed; \theta_t)$. 
Going back to \zcref{eq:bound_ystar_time_derive_sqrt_gamma}, we obtain the desired upper bound on $\frac{d}{dt} \pi_{\theta_t}(\ystar)$ for $t \in [0, t_{\sqrt{\gamma}}]$:
\begin{align*}
    \frac{d}{dt} \pi_{\theta_t}(\ystar) &\le -0.3 \cdot \pi_{\theta_t}(\ystar)\pi_{\theta_t}(\ymed)^2 \advgt(\ymed; \theta_t).
\end{align*}

\medskip

\textbf{Step 2.2: upper bound on $\frac{d}{dt} \pi_{\theta_t}(\ymed)$ in $[0, t_{\sqrt{\gamma}}]$.}
Step 1 showed that $\frac{d}{dt}\pi_{\theta_t}(\ymed) > 0$ for all $t\in[0,t_{\sqrt{\gamma}}]$.
Here, we prove that $\pi_{\theta_t} (\ymed)$ does not grow too fast by establishing that for any such $t$:
\[
    \frac{d}{dt}\pi_{\theta_t}(\ymed) \le 2.1 \pi_{\theta_t}(\ymed)^2 (1-\pi_{\theta_t}(\ymed))\advgt(\ymed; \theta_t) .
\]

We first upper bound $\frac{d}{dt}\pi_{\theta_t}(\ymed)$, based on the expression derived in \zcref{lem:probability_derivative_expression}:
\be
\begin{split}
    & \frac{d}{dt}\pi_{\theta_t}(\ymed) \\
    & = \pi_{\theta_t}(\ymed) \brk2{ \pi_{\theta_t}(\ymed)\advgt(\ymed; \theta_t) -\sum\nolimits_{y \in \Y}\pi_{\theta_t}(y)^2\advgt(y; \theta_t) }\\
    & = \pi_{\theta_t}(\ymed)\Big ( \brk1{ \pi_{\theta_t}(\ymed) - \pi_{\theta_t}(\ymed)^2 }\advgt(\ymed; \theta_t) - \pi_{\theta_t}(\ystar)^2 \advgt(\ystar; \theta_t) \\
    & \hspace{22mm} - \sum\nolimits_{z \in \Ybad} \pi_{\theta_t}(z)^2 \advgt(z; \theta_t) \Big ) \\
    & \le \pi_{\theta_t}(\ymed) \brk2{ \big(\pi_{\theta_t}(\ymed) - \pi_{\theta_t}(\ymed)^2\big)\advgt(\ymed; \theta_t) - \sum\nolimits_{z \in \Ybad} \pi_{\theta_t}(z)^2 \advgt(z; \theta_t) }\\
    & = \pi_{\theta_t}(\ymed)\Big(\big(\pi_{\theta_t}(\ymed) - \pi_{\theta_t}(\ymed)^2\big)\advgt(\ymed; \theta_t) + (\Delta_2 - \advgt(\ymed; \theta_t)) \sum_{z \in \Ybad} \pi_{\theta_t}(z)^2 \Big)\\
    & \le \pi_{\theta_t}(\ymed)\Big(\big(\pi_{\theta_t}(\ymed) - \pi_{\theta_t}(\ymed)^2\big)\advgt(\ymed; \theta_t) + (\Delta_2 - \advgt(\ymed; \theta_t)) \pi_{\theta_t}(\Ybad)^2 \Big)\\
    & \le \pi_{\theta_t}(\ymed)\Big(\big(1 - \pi_{\theta_t}(\ymed)\big)\big(\pi_{\theta_t}(\ymed) \advgt(\ymed; \theta_t) + \underbrace{(\Delta_2 - \advgt(\ymed; \theta_t)) \pi_{\theta_t}(\Ybad)}_{G} \big)\Big).
\end{split}
\label{eq:y_med_deriv_upper_bound_sqrt_gamma_orth}
\ee
The first inequality is from $\advgt(\ystar; \theta_t) \ge 0$; 
the second inequality is due to $\sum_{z \in \Ybad} \pi_{\theta_t}(z)^2 \le \pi_{\theta_t} (\Ybad)^2$; and the last inequality is by $\pi_{\theta_t}(\Ybad) \le 1 - \pi_{\theta_t}(\ymed)$. 
Next, we upper bound the term $G$ from the inequality above, proving that $G \leq 1.1  \pi_{\theta_t}(\ymed) \advgt(\ymed; \theta_t)$.
Equivalently, we show that $1.1 \pi_{\theta_t}(\ymed) \advgt(\ymed; \theta_t) - G \ge 0$. 
This follows from the following computations:
\begin{align*}
    & 1.1 \pi_{\theta_t}(\ymed) \advgt(\ymed; \theta_t) - G \\
    & = 1.1 \pi_{\theta_t}(\ymed) \advgt(\ymed; \theta_t) - (\Delta_2 - \advgt(\ymed; \theta_t)) \pi_{\theta_t}(\Ybad) \\
    & = 1.1 \pi_{\theta_t}(\ymed)\big(\gtreward(\ymed)-\Vgt (\theta_t)\big) - \pi_{\theta_t}(\Ybad)\big(\Vgt (\theta_t) - \gtreward(\ybad)\big)\\
    & = 1.1 \pi_{\theta_t}(\ymed)\Big((1-\pi_{\theta_t}(\ymed))\gtreward(\ymed) - \pi_{\theta_t}(\ystar)\gtreward(\ystar)- \sum_{z \in \Ybad} \pi_{\theta_t}(z)\gtreward(z)\Big)\\
    &\hspace{4mm} - \pi_{\theta_t}(\Ybad)\Big((1-\pi_{\theta_t}(\Ybad))(-\gtreward(\ybad)) + \pi_{\theta_t}(\ystar) \gtreward(\ystar) + \pi_{\theta_t}(\ymed)\gtreward(\ymed) \Big)\\
    & = 1.1 \pi_{\theta_t}(\ymed)\big(\pi_{\theta_t}(\Ybad)\Deltatwo - \pi_{\theta_t}(\ystar) \Delta_1 \big) - \pi_{\theta_t}(\Ybad)\Big(\pi_{\theta_t}(\ymed)\Delta_2 + \pi_{\theta_t}(\ystar)(\Delta_1 + \Delta_2) \Big) \\
    & = \pi_{\theta_t}(\Ybad) \big(1.1 \pi_{\theta_t}(\ymed) \Deltatwo - \pi_{\theta_t}(\ymed) \Delta_2 \big) \\
    & \hspace{4mm} - \pi_{\theta_t}(\ystar) \big(1.1 \pi_{\theta_t}(\ymed) \Delta_1 + \pi_{\theta_t}(\Ybad)  (\Delta_1 + \Delta_2) \big)\\
    & = \underbrace{0.1 \pi_{\theta_t}(\Ybad)\Deltatwo \pi_{\theta_t}(\ymed)}_{I_1} - \underbrace{\pi_{\theta_t}(\ystar) \big(1.1 \pi_{\theta_t}(\ymed) \Delta_1 + \pi_{\theta_t}(\Ybad)  (\Delta_1 + \Delta_2) \big)}_{I_2}.
\end{align*}
It remains to show that $I_2 \le 0.1 \pi_{\theta_t}(\ymed) \pi_{\theta_t}(\Ybad) \Deltatwo = I_1$.
Focusing on $I_2$, we have:
\begin{align*}
    1.1 \pi_{\theta_t}(\ymed) \Delta_1 + \pi_{\theta_t}(\Ybad)  (\Delta_1 + \Delta_2)
    &< 1.1 \pi_{\theta_t}(\ymed) \Delta_1 + 0.4 \pi_{\theta_t}(\ymed)(\Delta_1 + \Delta_2)\\
    &\le 1.5 \pi_{\theta_t}(\ymed)(\Delta_1 + \Delta_2) ,
\end{align*}
where the first inequality follows from \zcref{assump: multi_s}, which along with the monotonicity of $\pi_{\theta_t} (\Ybad)$ (\zcref{lemma: decrease_y_bad_multi_s}) and $\pi_{\theta_t} (\ymed)$, implies that $\pi_{\theta_t}(\Ybad) \le \pi_{\theta_0}(\Ybad) < 0.4 \pi_{\theta_0}(\ymed) \le 0.4 \pi_{\theta_t}(\ymed)$. 
Furthermore, because $\gamma \le \big(\frac{1}{50 (\Delta_1 + \Delta_2)}\big)^{7/3}$ (\zcref{assump: multi_s}) and $\pi_{\theta_t} (\ystar)$ monotonically decreases over $[0, t_{\sqrt{\gamma}}]$, we can bound $\pi_{\theta_t}(\ystar)$ as:
$$
\pi_{\theta_t}(\ystar) \le \pi_{\theta_0}(\ystar) \le \gamma^{13/14} \le \frac{\sqrt{\gamma} }{50 (\Delta_1 + \Delta_2)} .
$$
By \zcref{eq:y_bad_prob_lower_bound_t_sqrt_gamma_orth}, $\pi_{\theta_t}(\Ybad) \ge \frac{\sqrt{\gamma}}{\Delta_2}$ whenever $\advgt(\ymed; \theta_t) \ge \sqrt{\gamma}$. 
Using this fact, we obtain:
$$
I_2 \le 1.5 \pi_{\theta_t}(\ymed) \pi_{\theta_t}(\ystar)(\Delta_1 + \Delta_2) \le 0.1 \pi_{\theta_t}(\ymed) \pi_{\theta_t}(\Ybad)\Deltatwo = I_1 .
$$
Thus, $G \leq 1.1  \pi_{\theta_t}(\ymed) \advgt(\ymed; \theta_t)$ for $G$ defined in \zcref{eq:y_med_deriv_upper_bound_sqrt_gamma_orth} and we can conclude that:
\begin{align*}
    \frac{d}{dt}\pi_{\theta_t}(\ymed) &\le 2.1 \pi_{\theta_t}(\ymed)^2 (1-\pi_{\theta_t}(\ymed))\advgt(\ymed; \theta_t).
\end{align*}

\medskip

\textbf{Step 2.3: upper bound on $\pi_{\theta_{t_{\sqrt{\gamma}}}}(\ystar)$.}
To summarize, we showed that for all $t \in [0, t_{\sqrt{\gamma}}]$:
\begin{align*}
    \frac{d}{dt} \pi_{\theta_t}(\ystar) &\le -0.3 \cdot \pi_{\theta_t}(\ystar)\pi_{\theta_t}(\ymed)^2 \advgt(\ymed; \theta_t),\\
    \frac{d}{dt}\pi_{\theta_t}(\ymed) &\le 2.1 \cdot \pi_{\theta_t}(\ymed)^2(1-\pi_{\theta_t}(\ymed))\advgt(\ymed; \theta_t).
\end{align*}
Since both sides of the first inequality are negative and both sides of the second inequality are positive, we can divide them to get:
\begin{align*}
    \frac{\frac{d}{dt} \pi_{\theta_t}(\ystar)}{\frac{d}{dt}\pi_{\theta_t}(\ymed)} \le - \frac{1}{7} \cdot \frac{\pi_{\theta_t}(\ystar)}{1-\pi_{\theta_t}(\ymed)}.
\end{align*}
We can rewrite this inequality as:
\begin{align*}
    \frac{\frac{d}{dt} \pi_{\theta_t}(\ystar)}{\pi_{\theta_t}(\ystar)} &\le \frac{1}{7} \cdot \frac{-\frac{d}{dt}\pi_{\theta_t}(\ymed)}{1-\pi_{\theta_t}(\ymed)} = \frac{1}{7} \cdot \frac{\frac{d}{dt}(1-\pi_{\theta_t}(\ymed))}{1-\pi_{\theta_t}(\ymed)},
\end{align*}
from which it follows that:
\begin{align*}
    \frac{d}{dt}\ln \pi_{\theta_t}(\ystar) \le \frac{1}{7} \cdot \frac{d}{dt} \ln (1-\pi_{\theta_t}(\ymed) )
    .
\end{align*}
Integrating both sides then leads to:
\begin{align*}
    \ln \pi_{\theta_t}(\ystar) \Big|_{0}^{t_{\sqrt{\gamma}}} \le \frac{1}{7} \cdot \ln (1 - \pi_{\theta_t}(\ymed))\Big|_{0}^{t_{\sqrt{\gamma}}} .
\end{align*}
Thus,
\begin{align*}
    \frac{\pi_{\theta_{t_{\sqrt{\gamma}}}}(\ystar)}{\pi_{\theta_0}(\ystar)} &\le \brk*{ \frac{1-\pi_{\theta_{t_{\sqrt{\gamma}}}}(\ymed)}{1-\pi_{\theta_0}(\ymed)} }^{1/7}, 
\end{align*}
and so, 
\begin{align}
\label{eq: ystar_upper_tgamma}
    \pi_{\theta_{t_{\sqrt{\gamma}}}}(\ystar) \le \frac{\pi_{\theta_0}(\ystar) (1-\pi_{\theta_{t_{\sqrt{\gamma}}}}(\ymed))^{1/7}}{(1-\pi_{\theta_0}(\ymed))^{1/7}}.
\end{align}
Next, we prove that $1 - \pi_{\theta_{t_{\sqrt{\gamma}}}}(\ymed) \le \frac{\sqrt{\gamma}+ \pi_{\theta_0}(\ystar)}{\gtreward(\ymed)}$. 
Assume by way of contradiction that this is not the case.
We would therefore have that:
\begin{align*}
    \Vgt(\theta_{t_{\sqrt{\gamma}}}) &= \pi_{\theta_{t_{\sqrt{\gamma}}}}(\ymed) \gtreward(\ymed) + \pi_{\theta_{t_{\sqrt{\gamma}}}}(\ystar) \gtreward(\ystar) + \sum\nolimits_{z \in \Ybad} \pi_{\theta_{t_{\sqrt{\gamma}}}}(z) \gtreward(z)\\
    &< \bigg(1 - \frac{\sqrt{\gamma}+ \pi_{\theta_0}(\ystar)}{\gtreward(\ymed)}\bigg)\gtreward(\ymed) + \pi_{\theta_0}(\ystar) \le \gtreward(\ymed) - \sqrt{\gamma} ,
\end{align*}
where the inequality is by $\gtreward(z) \le 0$ for all $z \in \Ybad$ and $\gtreward(\ystar) \le 1$.
However, this contradicts the definition of $t_{\sqrt{\gamma}}$, which is the initial time at which $\Vgt (\theta_t) \geq \gtreward(\ymed) - \sqrt{\gamma}$.
Hence, it must be that:
\begin{align*}
    1 - \pi_{\theta_{t_{\sqrt{\gamma}}}}(\ymed) \le \frac{\sqrt{\gamma}+ \pi_{\theta_0}(\ystar)}{\gtreward(\ymed)} .
\end{align*}

Plugging this upper bound into \zcref{eq: ystar_upper_tgamma}, we get:
\be
\begin{split}
    \pi_{\theta_{t_{\sqrt{\gamma}}}}(\ystar) & \le \frac{\pi_{\theta_0}(\ystar)\big(\sqrt{\gamma}+\pi_{\theta_0}(\ystar)\big)^{1/7}}{\big[\gtreward(\ymed)(1-\pi_{\theta_0}(\ymed))\big]^{1/7}} \\
    & \le \frac{\pi_{\theta_0}(\ystar)(1.3 \sqrt{\gamma})^{1/7}}{(0.05 \cdot \gtreward(\ymed)^2)^{1/7}} \\
    & = \pi_{\theta_0}(\ystar) \cdot \brk*{ \frac{ 26\sqrt{\gamma} }{ \gtreward(\ymed)^2} }^{1/7},
\end{split}
\label{eq: ystar_upper_sqrtgamma}
\ee
where the second inequality is from $1 - \pi_{\theta_0}(\ymed) \ge \pi_{\theta_0}(\Ybad) \ge 0.05 \cdot \gtreward(\ymed)$ and $\pi_{\theta_0}(\ystar) \le \gamma^{13/14} \le 0.3 \sqrt{\gamma}$ (see \zcref{assump: multi_s}).

\medskip

\textbf{Step 2.4: completing the proof of monotonicity in $[t_{\sqrt{\gamma}}, t_\gamma]$}
Finally, we conclude the proof of \zcref{prop: monotonical_multi_s} by considering the optimization dynamics over $[t_{\sqrt{\gamma}}, t_\gamma]$ and showing that $\pi_{\theta_t}(\ystar)$ continues to decrease during this time interval.

Assume by way of contradiction that there exists $t^\prime \in [0 , t_{\gamma}]$ such that $\left.\frac{d}{dt} \pi_{\theta_t}(\ystar)\right|_{t=t^\prime} \ge 0$. 
Let $\tau$ be the initial time in $[0, t_\gamma]$ at which $\frac{d}{dt} \pi_{\theta_t}(\ystar) \geq 0$, \ie:
\[
\tau := \min \brk[c]*{ t \in [ 0 , t_\gamma] : \tfrac{d}{dt} \pi_{\theta_t}(\ystar) \geq 0 }
.
\]
By definition, $\frac{d}{dt}\pi_{\theta_t}(\ystar) < 0$, and so $\frac{d}{dt}\pi_{\theta_t}(\ymed) > 0$, for all $t \in [0, \tau)$.
We also know that if $\tau$ exists, then $\tau \ge t_{\sqrt{\gamma}}$ since $\frac{d}{dt} \pi_{\theta_t}(\ystar) < 0$ for all $t \in [0, t_{\sqrt{\gamma}}]$.
This implies that $\pi_{\theta_\tau}(\ystar) \le \pi_{\theta_{t_{\sqrt{\gamma}}}}(\ystar)$ and $\pi_{\theta_\tau}(\ymed) \ge \pi_{\theta_0}(\ymed)$.
Next, we establish that under the assumption above 
\[
\pi_{\theta_\tau}(\Ybad)^2 \le \frac{7 \pi_{\theta_0}(\ymed)^2 \advgt(\ymed; \theta_\tau)}{8\Delta_2} .
\]
To see this, notice that:
\begin{align*}
        \advgt(\ymed; \theta_\tau) &= (1 - \pi_{\theta_\tau}(\ymed)) \gtreward(\ymed) - \pi_{\theta_\tau}(\ystar) \gtreward(\ystar) - \sum\nolimits_{z \in \Ybad} \pi_{\theta_\tau}(z) \gtreward(z)\\
        &= -\Delta_1 \pi_{\theta_\tau}(\ystar) + \Deltatwo \pi_{\theta_\tau}(\Ybad) .
\end{align*}
Thus,
\begin{align*}
    \pi_{\theta_\tau}(\Ybad) = \frac{\advgt(\ymed; \theta_\tau) + \Delta_1 \pi_{\theta_\tau}(\ystar)}{\Deltatwo} \le \frac{1.3 \advgt(\ymed; \theta_\tau)}{\Deltatwo} .
\end{align*}
The inequality above follows from \zcref{assumpt: multi_ystar} in \zcref{assump: multi_s} since
\begin{align*}
    \pi_{\theta_0}(\ystar) \le \frac{\gtreward(\ymed)^{2/7}\gamma^{13/14}}{20(\Delta_1 + \Delta_2)} \le \frac{\gtreward(\ymed)^{2/7}\gamma^{13/14}}{20\Delta_1} \le \frac{0.3 \gamma^{13/14}}{\Delta_1 (26/\gtreward(\ymed)^2)^{1/7}}, 
\end{align*}
which combined with \zcref{eq: ystar_upper_sqrtgamma} yields:
\begin{align*}
    \Delta_1 \pi_{\theta_\tau}(\ystar) \le \Delta_1 \pi_{\theta_{t_{\sqrt{\gamma}}}}(\ystar) \le \Delta_1 \pi_{\theta_0}(\ystar) \cdot (26 \sqrt{\gamma}/\gtreward(\ymed)^2)^{1/7}\le 0.3 \gamma \le 0.3\advgt(\ymed; \theta_\tau).    
\end{align*}
It therefore holds that:
\begin{align*}
    \pi_{\theta_\tau}(\Ybad)^2 &\le \frac{1.69 \advgt(\ymed; \theta_\tau)^2}{\Deltatwo^2}\\
    &< \frac{7\pi_{\theta_0}(\ymed)^2 \advgt(\ymed; \theta_\tau)}{8 \Delta_2} \cdot \frac{2 \advgt(\ymed; \theta_\tau)}{\Deltatwo \pi_{\theta_0}(\ymed)^2}\\
    &\le \frac{7\pi_{\theta_0}(\ymed)^2 \advgt(\ymed; \theta_\tau)}{8 \Delta_2} \cdot \frac{2\advgt(\ymed; \theta_0)}{\Deltatwo \pi_{\theta_0}(\ymed)^2},
\end{align*}
where the third inequality is from $\advgt(\ymed; \theta_\tau) \leq \advgt(\ymed; \theta_0)$ (the advantage of an output is monotonically non-increasing since $\Vgt (\theta_t)$ is non-decreasing under gradient flow). 
We now upper bound the second term on the right-hand side in the inequality above by one.
Examining $\advgt(\ymed; \theta_0)$:
\begin{align*}
    \advgt(\ymed; \theta_0) &= (1 - \pi_{\theta_0}(\ymed))\gtreward(\ymed) - \pi_{\theta_0}(\ystar) \gtreward(\ystar) - \sum\nolimits_{z \in \Ybad}\pi_{\theta_0}(z) \gtreward(z)\\
    &\le (1 - \pi_{\theta_0}(\ymed))\gtreward(\ymed) + \pi_{\theta_0}(\Ybad)\\
    &\le 0.2 \gtreward(\ymed) + 0.1 \gtreward(\ymed)\\
    &= 0.3 \gtreward(\ymed)\\
    &< \frac{ 0.8^2 \gtreward(\ymed)}{2}\\
    &\le \frac{21 \Deltatwo \pi_{\theta_0}(\ymed)^2}{44}.
\end{align*}
Here, the first inequality is from $\gtreward(z) \ge -1$; the second inequality is from \zcref{assump: multi_s} since $\pi_{\theta_0}(\Ybad) \le 0.1\gtreward(\ymed) \le 0.1$ and $\pi_{\theta_0}(\ymed) \ge 0.8$; and the last inequality is from $\gtreward(\ybad) \le 0$ and $\pi_{\theta_0}(\ymed) \ge 0.8$. 
This implies that $\frac{2\advgt(\ymed; \theta_0)}{\Deltatwo \pi_{\theta_0}(\ymed)^2} \le 1$, and so: 
\begin{equation}
\label{eq: ybad2_upper2}
    \pi_{\theta_\tau}(\Ybad)^2 \le \frac{7\pi_{\theta_0}(\ymed)^2 \advgt(\ymed; \theta_\tau)}{8\Delta_2}.
\end{equation}

Let us now consider $\left.\frac{d}{dt} \pi_{\theta_t}(\ystar)\right|_{t=\tau}$ and obtain a contradiction by showing that it is negative.
We start with the following computations:
\begin{align*}
        & \left.\frac{d}{dt} \pi_{\theta_t}(\ystar)\right|_{t=\tau} \\
        & = \pi_{\theta_\tau}(\ystar) \Big(\pi_{\theta_\tau}(\ystar) \advgt(\ystar; \theta_\tau) - \sum_y \pi_{\theta_\tau}(y)^2 \advgt(y; \theta_\tau)\Big)\\
        &= \pi_{\theta_\tau}(\ystar) \Big( \big(\pi_{\theta_\tau}(\ystar) - \pi_{\theta_\tau}(\ystar)^2\big) \advgt(\ystar; \theta_\tau) - \pi_{\theta_\tau}(\ymed)^2 \advgt(\ymed; \theta_\tau) \\
        & \hspace{18mm} - \sum\nolimits_{z \in \Ybad}  \pi_{\theta_\tau}(z)^2 \advgt(z; \theta_\tau)\Big)\\
        &\le \pi_{\theta_\tau}(\ystar) \Big(\! \big(\pi_{\theta_{t_{\sqrt{\gamma}}}}(\ystar) - \pi_{\theta_{t_{\sqrt{\gamma}}}}(\ystar)^2\big) \advgt(\ystar; \theta_\tau) - \pi_{\theta_0}(\ymed)^2 \advgt(\ymed; \theta_\tau) \\
        & \hspace{18mm} - \sum\nolimits_{z \in \Ybad} \pi_{\theta_\tau}(z)^2 \advgt(z; \theta_\tau)\Big)\\
        & \le \pi_{\theta_\tau}(\ystar) \Big( \big(\pi_{\theta_{t_{\sqrt{\gamma}}}}(\ystar) - \pi_{\theta_{t_{\sqrt{\gamma}}}}(\ystar)^2\big) \advgt(\ystar; \theta_0) - \pi_{\theta_0}(\ymed)^2 \advgt(\ymed; \theta_\tau) \\
        & \hspace{18mm} + (\Delta_2 - \gamma) \sum\nolimits_{z \in \Ybad} \pi_{\theta_\tau}(z)^2 \Big)\\
        &\le \pi_{\theta_\tau}(\ystar) \Big( \big(\pi_{\theta_{t_{\sqrt{\gamma}}}}(\ystar) - \pi_{\theta_{t_{\sqrt{\gamma}}}}(\ystar)^2\big) \advgt(\ystar; \theta_0) - \pi_{\theta_0}(\ymed)^2 \advgt(\ymed; \theta_\tau) \\
        & \hspace{18mm} + (\Delta_2 - \gamma) \pi_{\theta_\tau}(\Ybad)^2 \Big)\\
        &\le \pi_{\theta_\tau}(\ystar) \Big(\big(\pi_{\theta_{t_{\sqrt{\gamma}}}}(\ystar) - \pi_{\theta_{t_{\sqrt{\gamma}}}}(\ystar)^2\big) \advgt(\ystar; \theta_0) - (1/8) \pi_{\theta_0}(\ymed)^2 \advgt(\ymed; \theta_\tau) \Big)\\
        &\le \pi_{\theta_\tau}(\ystar) \Big(\big(\pi_{\theta_{t_{\sqrt{\gamma}}}}(\ystar) - \pi_{\theta_{t_{\sqrt{\gamma}}}}(\ystar)^2\big) \advgt(\ystar; \theta_0) - (1/8)\pi_{\theta_0}(\ymed)^2 \gamma \Big) \\
        &\le \pi_{\theta_\tau}(\ystar) \Big( \pi_{\theta_{t_{\sqrt{\gamma}}}}(\ystar)  \advgt(\ystar; \theta_0) - (1/8)\pi_{\theta_0}(\ymed)^2 \gamma \Big)
    \end{align*}
The first inequality is from $\pi_{\theta_\tau}(\ystar) - \pi_{\theta_\tau}(\ystar)^2 \leq \pi_{\theta_{t_{\sqrt{\gamma}}}}(\ystar) - \pi_{\theta_{t_{\sqrt{\gamma}}}}(\ystar)^2$ (this holds since $\pi_{\theta_\tau} (\ystar) \leq \pi_{\theta_{t_{\sqrt{\gamma}}}} (\ystar) \leq \pi_{\theta_{0}} (\ystar) < 0.5$); the second inequality is from $\advgt(\ystar; \theta_\tau) \le \advgt(\ystar; \theta_0)$ and $- \advgt (z ; \theta_\tau) \leq \Delta_2 - \gamma$ for all $z \in \Ybad$; the third inequality is due to $\sum_{z \in \Ybad} \pi_{\theta_\tau}(z)^2 \leq \pi_{\theta_\tau} (\Ybad)^2$; the fourth inequality is from \zcref{eq: ybad2_upper2}; the fifth inequality is from $\advgt(\ymed; \theta_\tau) \ge \advgt(\ymed; \theta_{t_\gamma}) = \gamma$; and the last inequality is by $1 - \pi_{\theta_{t_{\sqrt{\gamma}}}}(\ystar) \le 1$.

Now, by \zcref{eq: ystar_upper_sqrtgamma} and \zcref{assumpt: multi_ystar} in \zcref{assump: multi_s}, we have 
\begin{align*}
    \pi_{\theta_{t_{\sqrt{\gamma}}}}(\ystar) &\le \pi_{\theta_0}(\ystar) \cdot (26\sqrt{\gamma}/\gtreward(\ymed)^2)^{1/7}\\
    &\le \frac{\gamma^{13/14} \gtreward(\ymed)^{2/7}}{20(\Delta_1+\Delta_2)} \cdot \bigg(\frac{26\sqrt{\gamma}}{\gtreward(\ymed)^2}\bigg)^{1/7}\\
    &< \frac{0.8^2\gamma^{13/14} \gtreward(\ymed)^{2/7}}{8 \cdot 26^{1/7}(\Delta_1+\Delta_2)} \cdot \bigg(\frac{26\sqrt{\gamma}}{\gtreward(\ymed)^2}\bigg)^{1/7}\\
    &\le \frac{\pi_{\theta_0}(\ymed)^2 \gamma}{8\advgt(\ystar; \theta_0)}. 
\end{align*}
Therefore:
\[
\begin{split}
\left.\frac{d}{dt} \pi_{\theta_t}(\ystar)\right|_{t=\tau} & \le \pi_{\theta_\tau}(\ystar) \Big( \pi_{\theta_{t_{\sqrt{\gamma}}}}(\ystar)  \advgt(\ystar; \theta_0) - (1/8)\pi_{\theta_0}(\ymed)^2 \gamma \Big) \\
& < \pi_{\theta_\tau}(\ystar) \brk*{ \frac{\pi_{\theta_0}(\ymed)^2 \gamma}{8\advgt(\ystar; \theta_0)} \cdot \advgt(\ystar; \theta_0) - (1/8)\pi_{\theta_0}(\ymed)^2 \gamma }\\
& = 0,
\end{split}
\]
which contradicts the fact that $\tau$ is the initial time at which $\left.\frac{d}{dt} \pi_{\theta_t}(\ystar)\right|_{t=\tau} \geq 0$. 
Hence, it must be that $\frac{d}{dt} \pi_{\theta_t}(\ystar) < 0$ for all $t \in [ 0,  t_\gamma]$. 
Additionally, by \zcref{lemma: decrease_y_bad_multi_s} we know that $\pi_{\theta_t}(\Ybad)$ is also decreasing in the time interval $[0, t_\gamma]$, and so $\frac{d}{dt}\pi_{\theta_t}(\ymed) > 0$ over $[0, t_\gamma]$. 

To conclude, we showed that $\frac{d}{dt} \pi_{\theta_t}(\ystar) < 0$ and $\frac{d}{dt} \pi_{\theta_t}(\ymed) > 0$ for all $t \in [0, t_\gamma]$.
Thus, $\min\nolimits_{t \in [0, t_\gamma]} \pi_{\theta_t}(\ystar) = \pi_{\theta_{t_\gamma}}(\ystar)$.
Furthermore, since
\[
\pi_{\theta_{t_{\sqrt{\gamma}}}} (\ystar) \le \frac{\pi_{\theta_0}(\ymed)^2 \gamma}{8\advgt(\ystar; \theta_0)}
\]
and $t_{\sqrt{\gamma}} < t_{\gamma}$, it holds that
\[
\pi_{\theta_{t_\gamma}} (\ystar) < \pi_{\theta_{t_{\sqrt{\gamma}}}} (\ystar) \le \frac{\pi_{\theta_0}(\ymed)^2 \gamma}{8\advgt(\ystar; \theta_0)} ,
\]
completing the proof of \zcref{prop: monotonical_multi_s}.
\end{proof}

\medskip

Based on \zcref{prop: monotonical_multi_s}, we can now complete the proof of Case I.

\begin{proof}[Proof of \zcref{thm:attraction_to_mediocre_outputs_beneficial_error_formal}, Case I]
    By \zcref{prop: monotonical_multi_s}, as long as $\Vgt (\theta_t) \le \gtreward(\ymed) - \gamma$, we know that $\pi_{\theta_t}(\ystar)$ is monotonically decreasing and $\pi_{\theta_t}(\ymed)$ is monotonically increasing. 
    Recall that $t_\gamma$ denotes the initial time at which $\Vgt (\theta_t) \ge \gtreward(\ymed) - \gamma$.
    Since $\Vgt (\theta_t)$ is continuous in $t$, at time $t_\gamma$ it must be that $\Vgt (\theta_{t_\gamma}) = \gtreward(\ymed) - \gamma$ and $t_\gamma$ is the initial time at which this equality holds.
    
    We first prove that at $t_\gamma$ the policy is highly concentrated on $\ymed$ by upper bounding $1 - \pi_{\theta_{t_\gamma}}(\ymed)$.
    This follows from
    \begin{align*}
        \Vgt (\theta_{t_\gamma}) = \pi_{\theta_{t_\gamma}}(\ymed) \gtreward(\ymed) + \pi_{\theta_{t_\gamma}}(\ystar) \gtreward(\ystar) + \sum\nolimits_{z \in \Ybad}\pi_{\theta_{t_\gamma}}(z) \gtreward(z) = \gtreward(\ymed) - \gamma ,
    \end{align*}
    which implies:
    \be
        \begin{split}
        1 - \pi_{\theta_{t_\gamma}}(\ymed) &= \frac{1}{\gtreward(\ymed)} \brk*{ \gamma + \pi_{\theta_{t_\gamma}}(\ystar)\gtreward(\ystar) + \sum\nolimits_{z \in \Ybad}\pi_{\theta_{t_\gamma}}(z) \gtreward(z) } \\
        & \le \frac{1}{\gtreward(\ymed)}\bigg(\gamma + \frac{\gamma}{8 \Delta_1}\bigg) \\
        & = \frac{(1+8\Delta_1)\gamma}{8\Delta_1 \gtreward(\ymed)} .
        \end{split}
        \label{eq:upper_bound_1_minus_pi_med}
    \ee
    The inequality is due to $\pi_{\theta_{t_\gamma}}(\ystar) \le \frac{\pi_{\theta_0}(\ymed)^2 \gamma}{8 \advgt(\ystar; \theta_0)} \le \frac{\gamma}{8 \Delta_1}$ (\zcref{prop: monotonical_multi_s}).

    Now, the fact that the policy assigns high probability to $\ymed$ at $t_{\gamma}$ allows us to lower bound $\teps$---the initial time at which $\Vgt (\theta_t) \ge \gtreward(\ystar) - \epsilon$.
    For any time $T \ge t_\gamma$, we have:
    \begin{align}
        \|\theta_T - \theta_{t_\gamma}\| &= \bigg\| \int_{t_\gamma}^T \frac{d}{dt} \theta_t dt \bigg\|\nonumber\\
        &\le \int_{t_\gamma}^T \left\| \frac{d}{dt} \theta_t \right\| dt\nonumber\\
        &= \int_{t_\gamma}^T\|\nabla \Vgt(\theta_t)\| dt\nonumber\\
        &\le \int_{t_\gamma}^T 2(\Delta_1+\Delta_2)(1 - \pi_{\theta_t}(\ymed)) dt, 
        \label{eq: norm_theta_upper}
    \end{align}
    where the last inequality is by \zcref{lemma:l2_grad_bound}, with $B = 1$ since the feature vectors are orthonormal.
    We proceed by upper bounding the rate at which $1 - \pi_{\theta_t}(\ymed)$ increases (\ie, the rate at which $\pi_{\theta_t} (\ymed)$ decreases). 
    From the chain rule we get:
    \begin{align*}
        \frac{d}{dt}(1 - \pi_{\theta_t}(\ymed)) &= \inprod{\nabla (1 - \pi_{\theta_t}(\ymed))}{\frac{d}{dt} \theta_t}\\
        &= \inprod{\nabla (1 - \pi_{\theta_t}(\ymed))}{\nabla \Vgt(\theta_t)}\\
        &= \inprod{\nabla(\pi_{\theta_t}(\ystar) + \sum\nolimits_{z \in \Ybad} \pi_{\theta_t}(z))}{\nabla \Vgt(\theta_t)}.
    \end{align*}
    For all $y \in \Y$, the gradient of $\pi_{\theta_t}(y)$ with respect to $\theta_t$ can be computed as follows:
    \begin{align*}
        \nabla \pi_{\theta_t}(y) = \pi_{\theta_t}(y) \brk*{\phi(y) - \sum\nolimits_{z \in \Y}\pi_{\theta_t}(z)\phi(z)} = \pi_{\theta_t}(y)\brk*{\phi(y) - \bar{\phi}_{\theta_t}} ,
    \end{align*}
    where $\bar{\phi}_{\theta_t} := \sum\nolimits_{z \in \Y}\pi_{\theta_t}(z)\phi(z)$.
    Thus:
    \begin{align}
        \frac{d}{dt}(1 - \pi_{\theta_t}(\ymed)) &= \inprod{\nabla\pi_{\theta_t}(\ystar)}{\nabla \Vgt(\theta_t) } + \sum\nolimits_{z \in \Ybad} \inprod{\nabla\pi_{\theta_t}(z)}{\nabla \Vgt(\theta_t)}  \nonumber\\
        &= \pi_{\theta_t}(\ystar) \inprod{\phi(\ystar) - \bar{\phi}_{\theta_t}}{\nabla \Vgt(\theta_t)} + \sum\nolimits_{z \in \Ybad} \pi_{\theta_t}(z) \inprod{\phi(z) -  \bar{\phi}_{\theta_t}}{\nabla \Vgt(\theta_t)} \nonumber\\
        &\le (1-\pi_{\theta_t}(\ymed)) \sup\nolimits_{y \in \Y}\| \phi(y) -  \bar{\phi}_{\theta_t}\| \|\nabla \Vgt(\theta_t) \|.\label{eq: frac_yprime_upper}
    \end{align}
    Notice that since the feature vectors are orthonormal, for any $y \in \Y$:
    \[
    \begin{split}
       \|\phi(y) -  \bar{\phi}_{\theta_t}\|^2 & = 1 - 2\pi_{\theta_t}(y) + \sum\nolimits_{z \in \Y}\pi_{\theta_t}(z)^2 \\
        & = (1 - \pi_{\theta_t}(y))^2 + \sum\nolimits_{z \in \Y \setminus \{ y \}}\pi_{\theta_t}(z)^2 \\
        & \le (1 -\pi_{\theta_t}(y))^2 + \brk*{ \sum\nolimits_{z \in \Y \setminus \{ y \}}\pi_{\theta_t}(z)}^2 \\
        & = 2(1 - \pi_{\theta_t}(y))^2 .
    \end{split}
    \]
    Thus, $\|\phi(y) -  \bar{\phi}_{\theta_t}\| \le \sqrt{2} (1 - \pi_{\theta_t}(y)) \le \sqrt{2}$.
    Plugging this bound into \zcref{eq: frac_yprime_upper} and applying \zcref{lemma:l2_grad_bound} with $B=1$ to upper bound $\| \nabla \Vgt (\theta_t) \|$ then yields:
    \begin{align*}
        \frac{d}{dt}(1 - \pi_{\theta_t}(\ymed)) &\le 2\sqrt{2}(\Delta_1 + \Delta_2) (1 - \pi_{\theta_t}(\ymed))^2
        .
    \end{align*}
This implies that
\begin{align*}
\frac{d}{dt} \bigg(\frac{1}{1-\pi_{\theta_t}(\ymed)}\bigg) &= - \frac{ \frac{d}{dt} (1-\pi_{\theta_t}(\ymed)) }{(1-\pi_{\theta_t}(\ymed))^2} \ge -2\sqrt{2}(\Delta_1 + \Delta_2).
\end{align*}
	Integrating both sides from $t_\gamma$ to $t$, for any $t \geq t_\gamma$, leads to:
	\begin{align*}
	    \frac{1}{1-\pi_{\theta_t}(\ymed)} - \frac{1}{1-\pi_{\theta_{t_\gamma}}(\ymed)} \ge -2\sqrt{2}(\Delta_1 + \Delta_2)(t-t_\gamma) .
	\end{align*}
	Hence, for any $t \ge t_\gamma$ such that
	\[
	2\sqrt{2}(\Delta_1 + \Delta_2)(1 - \pi_{\theta_{t_\gamma}}(\ymed)) (t - t_\gamma) < 1,
	\]
	it holds that
	$$
	1 - \pi_{\theta_t}(\ymed) \le \frac{1-\pi_{\theta_{t_\gamma}}(\ymed)}{1 - 2\sqrt{2}(\Delta_1 + \Delta_2)(1 - \pi_{\theta_{t_\gamma}}(\ymed)) (t - t_\gamma)} .
	$$
	Consequently, for any $T \ge t_\gamma$ satisfying
	\[
	2\sqrt{2}(\Delta_1 + \Delta_2)(1 - \pi_{\theta_{t_\gamma}}(\ymed)) (T - t_\gamma) < 1,
	\]
	we can plug this upper bound into \zcref{eq: norm_theta_upper} and obtain:
	\[
    \begin{split}
	    \|\theta_T - \theta_{t_\gamma}\| &\le 2(\Delta_1 + \Delta_2) \int_{t_\gamma}^T (1 - \pi_{\theta_t}(\ymed))dt\\
	    &= 2(\Delta_1 + \Delta_2) \int_0^{T-t_\gamma} \frac{1 - \pi_{\theta_{t_\gamma}}(\ymed)}{1 - 2\sqrt{2}(\Delta_1 + \Delta_2)(1 - \pi_{\theta_{t_\gamma}}(\ymed)) s} ds\\
	    &= \frac{1}{\sqrt{2}} \ln\bigg(\frac{1}{1-2\sqrt{2}(\Delta_1 + \Delta_2)(1 - \pi_{\theta_{t_\gamma}}(\ymed)) (T - t_\gamma)}\bigg) .
    \end{split}
    \]

    By \zcref{lemma:vgt_lipschitzness}, $\Vgt$ is $1$-Lipschitz with respect to $\theta$, and so for any $T \ge t_\gamma$:
    \be
        |\Vgt (\theta_T) - \Vgt(\theta_{t_\gamma})| \le \|\theta_T - \theta_{t_\gamma}\| \le \frac{1}{\sqrt{2}} \ln\bigg(\frac{1}{1-2\sqrt{2}(\Delta_1 + \Delta_2)(1 - \pi_{\theta_{t_\gamma}}(\ymed)) (T - t_\gamma)}\bigg) .
        \label{eq:bound_dist_V_T_V_t_gamma}
    \ee
    If $\teps = \infty$, then the desired lower bound holds trivially.
    Hence, assume $\teps < \infty$.
    By the definition of $\teps$ and continuity of $\Vgt(\theta_t)$, $\Vgt(\theta_{\teps}) = \gtreward(\ystar) - \epsilon$.
    In particular, $\teps > t_\gamma$ since $\Vgt(\theta_{\teps}) = \gtreward (\ystar) - \epsilon > \gtreward (\ymed) > \Vgt (\theta_{t_\gamma})$.
    Therefore,
	    \[
	    \Vgt(\theta_{\teps}) - \Vgt(\theta_{t_\gamma}) = \gtreward (\ystar) - \epsilon - \gtreward (\ymed) + \gamma = \Delta_1 - \epsilon + \gamma \ge \Delta_1 - \epsilon .
	    \]
	    If $2\sqrt{2}(\Delta_1 + \Delta_2)(1 - \pi_{\theta_{t_\gamma}}(\ymed)) (\teps - t_\gamma) \ge 1$, then directly
	    \begin{align*}
	        \teps - t_\gamma \ge \frac{1}{2\sqrt{2}(\Delta_1 + \Delta_2)(1 - \pi_{\theta_{t_\gamma}}(\ymed))} \ge \frac{1 - e^{-\sqrt{2}(\Delta_1 - \epsilon)}}{2\sqrt{2}(\Delta_1 + \Delta_2)(1 - \pi_{\theta_{t_\gamma}}(\ymed))},
	    \end{align*}
	    since $1 - e^{-\sqrt{2}(\Delta_1 - \epsilon)} \le 1$.
	    Otherwise, we apply the bound in \zcref{eq:bound_dist_V_T_V_t_gamma} with $T = \teps$ to obtain:
	    \begin{align*}
	       \Delta_1 - \epsilon \le \Vgt (\theta_{\teps}) - \Vgt (\theta_{t_\gamma}) \leq \frac{1}{\sqrt{2}} \ln\bigg(\frac{1}{1-2\sqrt{2}(\Delta_1 + \Delta_2)(1 - \pi_{\theta_{t_\gamma}}(\ymed)) (\teps - t_\gamma)}\bigg),
	    \end{align*}
	    which implies
	    \begin{align*}
	        \teps - t_\gamma \ge \frac{1 - e^{-\sqrt{2}(\Delta_1 - \epsilon)}}{2\sqrt{2}(\Delta_1 + \Delta_2)(1 - \pi_{\theta_{t_\gamma}}(\ymed))}.
	    \end{align*}
	    Therefore, in either case:
	    \begin{align*}
	        \teps & \geq \teps - t_\gamma \\
            &\ge \frac{1 - e^{-\sqrt{2}(\Delta_1 - \epsilon)}}{2\sqrt{2}(\Delta_1 + \Delta_2)(1 - \pi_{\theta_{t_\gamma}}(\ymed))} \\
	        &\ge \frac{2\sqrt{2}\Delta_1 \gtreward(\ymed)(1 - e^{-\sqrt{2}(\Delta_1 - \epsilon)})}{(\Delta_1 + \Delta_2)(1 + 8\Delta_1) \gamma} \\
	        &\ge \frac{2\sqrt{2}M^{14/13}\Delta_1 \gtreward (\ymed) \brk{ 1 - e^{-\sqrt{2}(\Delta_1 - \epsilon)} } }{(\Delta_1 + \Delta_2)(1 + 8\Delta_1)} \cdot \pi_{\theta_0} (\ystar)^{- 14 / 13} \\
	        &= \Omega \brk*{ \pi_{\theta_0} (\ystar)^{- 14 / 13} } ,
    \end{align*}
    where the second inequality follows from \zcref{eq:upper_bound_1_minus_pi_med}.
    This completes the proof for Case I of \zcref{thm:attraction_to_mediocre_outputs_beneficial_error_formal}.
\end{proof}

\subsubsection{Proof of Case II}

Suppose that gradient flow is used to maximize the expected reward with respect to $\proxyreward$, which assigns $\ymed$ a low reward $\proxyreward(\ymed) = \min\nolimits_{y \in \Ybad} \gtreward (y)$.
In this case, we show that $\pi_{\theta_t}(\ystar)$ increases already from $t = 0$ and that $\tepsnomed = \OO (\pi_{\theta_0}(\ystar)^{-1})$, where recall that $\tepsnomed$ is the initial time at which $\Vgt (\theta_t) \ge \gtreward(\ystar) - \epsilon$ when maximizing $\proxyreward$.

\begin{proof}[Proof of \zcref{thm:attraction_to_mediocre_outputs_beneficial_error_formal}, Case II]
Notice that $\Vgt(\theta) \ge \Vproxy(\theta)$ for all $\theta \in \R^D$ since $\gtreward(y) \ge \proxyreward(y)$ for all $y \in \Y$.
Thus, when $\Vproxy(\theta_t) \ge \gtreward(\ystar) - \epsilon$ it also holds that $\Vgt(\theta_t) \ge \gtreward(\ystar) - \epsilon$. 
We can therefore focus on $\Vproxy$ and consider the time it takes until it reaches a value of $\gtreward(\ystar) - \epsilon$.
In particular, denote by $\tepsnomedP$ the initial time at which $\Vproxy(\theta_t) \ge \gtreward(\ystar) - \epsilon$, \ie:
\[
\tepsnomedP := \min \brk[c]*{ t \geq 0 :\Vproxy (\theta_t) \ge \gtreward (\ystar) - \epsilon }.
\]
The discussion above implies that $\tepsnomed \le \tepsnomedP$, and so it suffices to upper bound $\tepsnomedP$.
For simplicity of notation, denote $\rho := \frac{\gtreward(\ystar) + 1 - \epsilon}{\gtreward(\ystar) + 1}$. 
Our first claim is that when $\pi_{\theta_t}(\ystar) \ge \rho$, we must have $\Vproxy (\theta_t) \ge \gtreward (\ystar) - \epsilon$. 
This can be seen via the following calculation:
\begin{align*}
    \Vproxy (\theta_t) &\ge \pi_{\theta_t}(\ystar) \proxyreward(\ystar) - (1 - \pi_{\theta_t}(\ystar))\\
    &= \pi_{\theta_t}(\ystar) (\gtreward(\ystar) + 1) -1\\
    &\ge \gtreward(\ystar) - \epsilon .
\end{align*}
This implies that $\pi_{\theta_t}(\ystar) \leq \rho$ for all $t \in [0, \tepsnomedP]$.

We prove that 
\begin{align*}
    T:= \frac{1}{(1-\rho)^2\advproxy(\ystar; \theta_0)}\bigg(\frac{1}{\pi_{\theta_0}(\ystar)} - \frac{1}{\rho}\bigg) \geq \tepsnomedP ,
\end{align*}
from which the desired upper bound on $\tepsnomed$ immediately follows by $\tepsnomed \leq \tepsnomedP$ and substituting the value of $\rho$.

Assume by way of contradiction that $\tepsnomedP > T$.
Let us lower bound the rate at which $\pi_{\theta_t}(\ystar)$ increases over $[0, \tepsnomedP]$.
Notice first that $\Vproxy (\theta_0) > \proxyreward (y)$ for all $y \in \Y \setminus \{ \ystar \}$ since all such $y$ have the same proxy reward value, which is lower than $\proxyreward( \ystar ) = \gtreward (\ystar)$. 
Now, consider the time derivative of $\pi_{\theta_t}(\ystar)$, starting from the expression derived in \zcref{lem:probability_derivative_expression}:
\begin{align*}
    & \frac{d}{dt} \pi_{\theta_t}(\ystar) \\
    & = \pi_{\theta_t}(\ystar) \Big(\pi_{\theta_t}(\ystar) \advproxy(\ystar; \theta_t) - \sum\nolimits_{y \in \Y} \pi_{\theta_t}(y)^2 \advproxy(y; \theta_t)\Big)\\
        &= \pi_{\theta_t}(\ystar) \Big(\big(\pi_{\theta_t}(\ystar) - \pi_{\theta_t}(\ystar)^2\big) \advproxy(\ystar; \theta_t) - \pi_{\theta_t}(\ymed)^2 \advproxy(\ymed; \theta_t) - \!\!\! \sum_{z \in \Ybad} \!\! \pi_{\theta_t}(z)^2 \advproxy(z; \theta_t)\Big)\\
        &\ge \pi_{\theta_t}(\ystar)^2 (1 - \pi_{\theta_t}(\ystar)) \advproxy(\ystar; \theta_t).
\end{align*}
Here, the inequality is from $\advproxy(y; \theta_t)  \le \advproxy(y; \theta_0) \le 0$ for all $y \in \Y \setminus\{\ystar\}$. 
As we proved above, $\pi_{\theta_t}(\ystar) \leq \rho$ for all $t \in [0, \tepsnomedP]$.
Thus for such $t$ it holds that:
    \[
\frac{d}{dt} \pi_{\theta_t}(\ystar) \geq (1-\rho)\advproxy(\ystar; \theta_t) \pi_{\theta_t}(\ystar)^2.
\]
Next, we show that $\advproxy(\ystar; \theta_t) \ge (1-\rho)\advproxy(\ystar; \theta_0)$. 
To prove this, observe that since $\proxyreward(\ystar) = \gtreward(\ystar)$ and $\Vproxy(\theta_0)\ge -1$,
\[
\frac{\gtreward(\ystar) - \epsilon - \Vproxy (\theta_0)}{\proxyreward(\ystar) - \Vproxy (\theta_0)}
= 1 - \frac{\epsilon}{\proxyreward(\ystar) - \Vproxy (\theta_0)}
\le 1 - \frac{\epsilon}{\proxyreward(\ystar) + 1}
= \frac{\gtreward(\ystar) + 1 - \epsilon}{\gtreward(\ystar) + 1}
= \rho.
\]
Using this fact, we have that:
\begin{align*}
    \frac{\advproxy(\ystar; \theta_t)}{\advproxy(\ystar; \theta_0)} = \frac{\proxyreward(\ystar) - \Vproxy (\theta_t)}{\proxyreward(\ystar) - \Vproxy (\theta_0)}=1 - \frac{\Vproxy (\theta_t) - \Vproxy (\theta_0)}{\proxyreward(\ystar) - \Vproxy (\theta_0)} \ge 1 - \frac{\gtreward(\ystar) - \epsilon -\Vproxy (\theta_0)}{\proxyreward(\ystar) - \Vproxy (\theta_0)} \geq 1-\rho .
\end{align*}
Therefore, we can plug in this bound to get:
\[
    \frac{d}{dt}\pi_{\theta_t}(\ystar) \ge (1-\rho)^2 \advproxy(\ystar; \theta_0) \pi_{\theta_t}(\ystar)^2 .
\]
Dividing both sides by $\pi_{\theta_t}(\ystar)^2$, integrating from $0$ to $t$, and rearranging terms results in the following lower bound:
$$
\pi_{\theta_t}(\ystar) \ge \frac{\pi_{\theta_0}(\ystar)}{1 - (1-\rho)^2\advproxy(\ystar; \theta_0)\pi_{\theta_0}(\ystar) \cdot t} .
$$
At time $t = T$ for
\[
T = \frac{1}{(1-\rho)^2\advproxy(\ystar; \theta_0)}\bigg(\frac{1}{\pi_{\theta_0}(\ystar)} - \frac{1}{\rho}\bigg) ,
\]
this lower bound on $\pi_{\theta_T} (\ystar)$ implies  that $\pi_{\theta_T}(\ystar) \ge \rho$, and so $\Vproxy (\theta_T) \geq \gtreward( \ystar) - \epsilon$, in contradiction to our assumption that $\tepsnomedP > T$. 
Therefore,
\[
\tepsnomed \le \tepsnomedP \leq T \leq \frac{ \brk*{ \gtreward (\ystar) + 1 }^2 }{\epsilon^2  \brk*{ \proxyreward (\ystar) - \Vproxy (\theta_0)  } }  \cdot \pi_{\theta_0} (\ystar)^{-1} = \OO \brk*{ \pi_{\theta_0} (\ystar)^{-1} } .
\]
Moreover, notice that $\Vproxy (\theta_t)$ is monotonically non-decreasing when maximizing $\proxyreward$ via gradient flow since $\frac{d}{dt} \Vproxy (\theta_t) = \norm{ \nabla \Vproxy (\theta_t) }^2 \geq 0$ for all $t \geq 0$.
Thus, for all $t \geq T \geq \tepsnomedP$:
\begin{align*}
    \epsilon \ge \gtreward(\ystar) - \Vproxy(\theta_{t}) = (1 - \pi_{\theta_{t}}(\ystar))(\Deltaone + \Deltatwo),
\end{align*}
from which we can conclude that for all $t \geq T$:
\[
   \pi_{\theta_{t}}(\ystar) \ge 1 - \frac{\epsilon}{\Deltaone + \Deltatwo}.
\]
\end{proof}

\subsection{Proof of \zcref{thm:attraction_to_mediocre_outputs_beneficial_error_neg_inner_product}}
\label{app:proofs:attraction_to_mediocre_outputs_beneficial_error_neg_inner_product}

For convenience of notation, we start by introducing the following helper variables:
\be
\gamma := \left(\frac{\pi_{\theta_0}(\ystar)}{M'}\right)^{14/13} \quad , \quad s := \inprod{\phi(\ystar)}{\phi(\ymed)}
    \quad , \quad
    h(s) := \frac{\norm{\phi(\ymed)}^2 - s}{7\brk1{\norm{\phi(\ymed)}^2 - s/2}} ,
    \label{eq:neg_def_helper_variables}
\ee
where recall that $M' := \min \brk[c]1{ 1 ,  \norm{ \phi(\ymed) }^2 \gtreward (\ymed)^{2/7} ( 40B^2 (\Delta_1+\Delta_2) )^{-1} }$, with $\Deltaone := \gtreward (\ystar) - \gtreward (\ymed)$, $\Deltatwo := \gtreward (\ymed) - \gtreward (\ybad)$ for some $\ybad \in \Ybad$, and $B := \max_{y \in \Y} \norm{ \phi(y) }$.
We can then straightforwardly state the conditions of \zcref{assum:initial_probs_neg_inner_product} in terms of $\gamma$.

\begin{lemma}
\label{assump: loglin_neg_minor}
Under \zcref{assum:reward_structure_orthonormal,assum:initial_probs_neg_inner_product}, the following properties hold:
\begin{enumerate}
    \item $\gamma \le \min\bigg\{0.01, \bigg(\frac{21\|\phi(\ymed)\|^2}{1100 B^2 (\Delta_1+\Delta_2)}\bigg)^{7/3} \bigg\}$;
    \label{assump: gamma_loglin_neg}
    \item $\pi_{\theta_0}(\ystar) \leq \min\Bigg\{ \gamma^{13 / 14} ,  \frac{\|\phi(\ymed)\|^2\gtreward(\ymed)^{2/7} }{40B^2(\Delta_1+\Delta_2)} \cdot \gamma^{13/14},  \frac{ \norm{\phi(\ymed)} }{ \norm{\phi(\ystar)} } \pi_{\theta_0}(\ymed) \Bigg\}$; and
    \label{assump: loglin_neg_ystar}
    \item $\max\Big\{ \frac{2\sqrt{\gamma}}{\Delta_2}, 0.05 \cdot \gtreward(\ymed) \Big\} \le \pi_{\theta_0}(\Ybad) \le 0.1 \cdot \gtreward(\ymed)$. 
    \label{assumpt: loglin_neg_ybad}
\end{enumerate}
\end{lemma}

\begin{proof}
All three conditions are obtained directly from \zcref{assum:initial_probs_neg_inner_product} by substituting the definition of $M'$ and noticing that $\pi_{\theta_0} (\ystar) = M' \gamma^{13/14}$.
Note that \zcref{assum:reward_structure_orthonormal} on the ground truth reward structure ensures $\Deltaone$ and $\Deltatwo$ are positive, so all quantities in the conditions above are well-defined.
\end{proof}

\subsubsection{Proof of Case I}

Suppose that gradient flow is used to maximize the expected reward with respect to $\gtreward$.
We define $t_\gamma$ as the initial time at which $\Vgt(\theta_t)$ reaches $\gtreward(\ymed) - C\gamma^{13/14 + h(s)/2}$, \ie:
\be
    t_\gamma := \min \brk[c]*{ t \geq 0 : \Vgt(\theta_t) \geq \gtreward(\ymed) - C\gamma^{13/14 + h(s)/2} },
    \label{eq:t_gamma_def}
\ee
where $C := \big( 26 / \gtreward(\ymed)^2 \big)^{13\alpha/7}$ for $\alpha := - s /  \big ( 26 \brk{ \norm{ \phi (\ymed) }^2 - 0.5 s } \big ) \in \brk*{ 0 , 1 / 13 }$.
As in the proof of \zcref{thm:attraction_to_mediocre_outputs_beneficial_error_formal}, the key step is to show that on $[0, t_\gamma]$ the probability of $\ystar$ decreases while the probability of $\ymed$ increases.
We begin with a monotonicity lemma showing that, whenever $\Vgt(\theta_t) < \gtreward(\ymed)$, a decrease in $\pi_{\theta_t}(\ystar)$ implies an increase in $\pi_{\theta_t}(\ymed)$.
We then prove that $\pi_{\theta_t}(\ystar)$ indeed decreases throughout $[0, t_\gamma]$, which in turn yields monotonic growth of $\pi_{\theta_t}(\ymed)$ on this interval.
The negative inner product between $\phi (\ystar)$ and $\phi (\ymed)$ strengthens this effect relative to the case of orthonormal features (\zcref{thm:attraction_to_mediocre_outputs_beneficial_error_formal}) by pulling probability away from $\ystar$ more aggressively and concentrating the policy more strongly around $\ymed$.
Combining this stronger concentration with the gradient-norm bound from \zcref{lemma:l2_grad_bound} will allow establishing the desired lower bound on $\teps$.

Note that if $\Vgt(\theta_t) < \gtreward(\ymed)$ for all $t \geq 0$, then the lower bound on $\teps$ holds trivially, since $\Vgt(\theta_t)$ never reaches $\gtreward(\ystar) - \epsilon$.
Thus, throughout the proof of this case, we may assume that there exists a time at which $\Vgt(\theta_t) \geq \gtreward(\ymed)$.

\begin{lemma}
\label{lemma: conditional_monotonicity_loglin_neg}
At any time $t \geq 0$ such that $\Vgt(\theta_t) < \gtreward(\ymed)$, if $\frac{d}{dt}\pi_{\theta_t}(\ystar) < 0$, then $\frac{d}{dt}\pi_{\theta_t}(\ymed) > 0$.
\end{lemma}

\begin{proof}
    Fix any time $t \geq 0$ such that $\Vgt(\theta_t) < \gtreward(\ymed)$.
    By \zcref{lem:probability_derivative_expression}, instantiated with $r = \gtreward$ and $y = \ystar$, we have
    \[
        \frac{d}{dt}\pi_{\theta_t}(\ystar)
        = \pi_{\theta_t}(\ystar)
        \inprod{
            \phi(\ystar) - \sum\nolimits_{y \in \Y} \pi_{\theta_t}(y)\phi(y)
        }{
            \sum\nolimits_{y \in \Y} \pi_{\theta_t}(y)\advgt(y;\theta_t)\phi(y)
        } ,
    \]
    Since $\inprod{ \phi (z) }{ \phi (y) } = 0$ for all $z \in \Ybad, y \in \Y \setminus \{ z \}$ (\zcref{assum:feature_structure_neg_inn_prod}), we may write:
    \begin{align*}
        \frac{d}{dt} \pi_{\theta_t}(\ystar) ={}& \pi_{\theta_t}(\ystar) \Big [ \pi_{\theta_t}(\ystar) \advgt(\ystar; \theta_t) \big(\|\phi(\ystar)\|^2 - \|\phi(\ystar)\|^2\pi_{\theta_t}(\ystar) - s \pi_{\theta_t}(\ymed)\big)\\
        & \hspace{12mm} + \pi_{\theta_t}(\ymed) \advgt(\ymed; \theta_t) \big(s - \|\phi(\ymed)\|^2\pi_{\theta_t}(\ymed) - s \pi_{\theta_t}(\ystar) \big) \\
        & \hspace{12mm} - \sum\nolimits_{z \in \Ybad} \pi_{\theta_t}(z)^2 \advgt(z; \theta_t) \|\phi(z)\|^2 \Big ]\\
        \ge {}& \pi_{\theta_t}(\ystar) \Big [ \pi_{\theta_t}(\ystar) \advgt(\ystar; \theta_t) \big(\|\phi(\ystar)\|^2 - \|\phi(\ystar)\|^2\pi_{\theta_t}(\ystar) - s \pi_{\theta_t}(\ymed)\big) \\
        & \hspace{12mm} + \pi_{\theta_t}(\ymed) \advgt(\ymed; \theta_t) \big(s - \|\phi(\ymed)\|^2\pi_{\theta_t}(\ymed) - s \pi_{\theta_t}(\ystar) \big) \Big].
    \end{align*}
    Here, the inequality is due to $\advgt(z; \theta_t) \le 0$ for all $z \in \Ybad$. 
    Thus, when $\frac{d}{dt}\pi_{\theta_t}(\ystar) < 0$, it must hold that:
    \begin{align*}
        &\pi_{\theta_t}(\ystar) \advgt (\ystar; \theta_t) \big(\|\phi(\ystar)\|^2 - \|\phi(\ystar)\|^2 \pi_{\theta_t}(\ystar) - s \pi_{\theta_t}(\ymed)\big)\\
        & < \pi_{\theta_t}(\ymed) \advgt(\ymed; \theta_t) \big(- s + \|\phi(\ymed)\|^2\pi_{\theta_t}(\ymed) + s \pi_{\theta_t}(\ystar) \big),
    \end{align*}
    or equivalently,
    \begin{equation}
    \label{eq: loglin_neg_minor_cond}
        \frac{\pi_{\theta_t}(\ystar) \advgt(\ystar; \theta_t)}{\pi_{\theta_t}(\ymed) \advgt(\ymed; \theta_t)} < \frac{- s + \|\phi(\ymed)\|^2\pi_{\theta_t}(\ymed) + s \pi_{\theta_t}(\ystar)}{\|\phi(\ystar)\|^2 - \|\phi(\ystar)\|^2\pi_{\theta_t}(\ystar) - s \pi_{\theta_t}(\ymed)},
    \end{equation}
    since $\|\phi(\ystar)\|^2 - \|\phi(\ystar)\|^2\pi_{\theta_t}(\ystar) - s \pi_{\theta_t}(\ymed) > 0$ and $\pi_{\theta_t}(\ymed) \advgt(\ymed; \theta_t) > 0$. 
    We continue by upper bounding the right-hand side in the inequality above via the following claim:
    \begin{align}
    \label{eq: loglin_neg_minor_cond_2}
        \frac{- s + \|\phi(\ymed)\|^2\pi_{\theta_t}(\ymed) + s \pi_{\theta_t}(\ystar)}{\|\phi(\ystar)\|^2 - \|\phi(\ystar)\|^2\pi_{\theta_t}(\ystar) - s \pi_{\theta_t}(\ymed)} \le \frac{\|\phi(\ymed)\|^2 - \|\phi(\ymed)\|^2\pi_{\theta_t}(\ymed) - s \pi_{\theta_t}(\ystar)}{\|\phi(\ystar)\|^2\pi_{\theta_t}(\ystar) - s + s\pi_{\theta_t}(\ymed)}.
    \end{align}
    Observe that the denominators of both sides are positive. Hence, upon multiplying both sides by the denominators and subtracting common terms, to show that \zcref{eq: loglin_neg_minor_cond_2} holds it suffices to prove:
    $$
    s^2(1-\pi_{\theta_t}(\ymed) - \pi_{\theta_t}(\ystar)) \le \|\phi(\ystar)\|^2\|\phi(\ymed)\|^2 (1-\pi_{\theta_t}(\ymed) - \pi_{\theta_t}(\ystar)).
    $$
    Indeed, this is true because $s^2 = \inprod{\phi(\ystar)}{\phi(\ymed)}^2 \le \norm{ \phi(\ymed) }^2 \norm{ \phi(\ystar) }^2$.
    Combining \zcref{eq: loglin_neg_minor_cond} and \zcref{eq: loglin_neg_minor_cond_2}, we thus get:
    \begin{align*}       
    \frac{ \pi_{\theta_t}(\ystar) \advgt(\ystar; \theta_t)}{\pi_{\theta_t}(\ymed) \advgt(\ymed; \theta_t)} & < \frac{- s + \|\phi(\ymed)\|^2\pi_{\theta_t}(\ymed) + s \pi_{\theta_t}(\ystar) }{ \|\phi(\ystar)\|^2 - \|\phi(\ystar)\|^2\pi_{\theta_t}(\ystar) - s \pi_{\theta_t}(\ymed)} \\
    &\le \frac{\|\phi(\ymed)\|^2 - \|\phi(\ymed)\|^2\pi_{\theta_t}(\ymed) - s \pi_{\theta_t}(\ystar)}{\|\phi(\ystar)\|^2\pi_{\theta_t}(\ystar) - s + s\pi_{\theta_t}(\ymed)},
    \end{align*}
    which implies:
    \begin{align}
        &\pi_{\theta_t}(\ymed) \advgt(\ymed; \theta_t)(\|\phi(\ymed)\|^2 - \|\phi(\ymed)\|^2\pi_{\theta_t}(\ymed) - s\pi_{\theta_t}(\ystar))\nonumber\\
        &- \pi_{\theta_t}(\ystar) \advgt(\ystar; \theta_t)(\|\phi(\ystar)\|^2 \pi_{\theta_t}(\ystar) - s +s\pi_{\theta_t}(\ymed)) > 0.\label{eq: loglin_neg_minor_cond_3}
    \end{align} 

    Next, we analyze the time derivative of $\pi_{\theta_t}(\ymed)$, which is (see~\zcref{lem:probability_derivative_expression}):
     \begin{align*}
         \frac{d}{dt}\pi_{\theta_t}(\ymed) ={}& \pi_{\theta_t}(\ymed) \Big[\pi_{\theta_t}(\ymed) \advgt(\ymed; \theta_t) \big(\|\phi(\ymed)\|^2 - \|\phi(\ymed)\|^2\pi_{\theta_t}(\ymed) - s \pi_{\theta_t}(\ystar)\big)\\
    &\hspace{16mm} + \pi_{\theta_t}(\ystar) \advgt(\ystar; \theta_t) \big(s - \|\phi(\ystar)\|^2\pi_{\theta_t}(\ystar) - s \pi_{\theta_t}(\ymed) \big) \\
    & \hspace{16mm} - \sum\nolimits_{z \in \Ybad} \pi_{\theta_t}(z)^2 \advgt(z; \theta_t) \|\phi(z)\|^2 \Big]\\
    \ge{}& \pi_{\theta_t}(\ymed) \Big[\pi_{\theta_t}(\ymed) \advgt(\ymed; \theta_t) \big(\|\phi(\ymed)\|^2 - \|\phi(\ymed)\|^2\pi_{\theta_t}(\ymed) - s \pi_{\theta_t}(\ystar) \big)\\
    &+ \pi_{\theta_t}(\ystar) \advgt(\ystar; \theta_t) \big(s - \|\phi(\ystar)\|^2\pi_{\theta_t}(\ystar) - s \pi_{\theta_t}(\ymed) \big) \Big]\\
    >{}& 0.
    \end{align*}
    Here, the first inequality is due to $\advgt(z; \theta_t) \le 0$ for all $ z \in \Ybad$, and the second inequality is from \zcref{eq: loglin_neg_minor_cond_3}. 
    Putting it all together, we have shown that if $\Vgt(\theta_t) < \gtreward(\ymed)$ and $\frac{d}{dt}\pi_{\theta_t}(\ystar) < 0$, then necessarily $\frac{d}{dt}\pi_{\theta_t}(\ymed) > 0$.
\end{proof}

Using \zcref{lemma: conditional_monotonicity_loglin_neg}, we can now prove that until $t_\gamma$ the probability of $\ymed$ monotonically increases while the probability of $\ystar$ monotonically decreases, and characterize how small $\pi_{\theta_t}(\ystar)$ becomes.
This is a core component in the proof of \zcref{thm:attraction_to_mediocre_outputs_beneficial_error_neg_inner_product}.

\begin{proposition}
\label{prop: monotonical_loglin_neg}
    For all $t \in [0, t_{\gamma}]$ it holds that:
    \begin{itemize}
        \item $\frac{d}{dt}\pi_{\theta_t}(\ystar) < 0$; and
        \item $\frac{d}{dt}\pi_{\theta_t}(\ymed) > 0$.
    \end{itemize}
    Furthermore,
    \[
    \min\nolimits_{t\in[0, t_\gamma]} \pi_{\theta_t}(\ystar) = \pi_{\theta_{t_{\gamma}}}(\ystar)  \le \frac{C\|\phi(\ymed)\|^2\pi_{\theta_0}(\ymed)^2\gamma^{\frac{13}{14}+\frac{h(s)}{2}}}{8(\|\phi(\ystar)\|^2-s) \advgt(\ystar; \theta_0)}  ,
    % \min\nolimits_{t\in[0, t_\gamma]} \pi_{\theta_t}(\ystar) = \pi_{\theta_{t_{\gamma}}}(\ystar)  \le \frac{\pi_{\theta_0}(\ystar)\big(\sqrt{\gamma}+\pi_{\theta_0}(\ystar)\big)^{h(s)}}{\big[\gtreward(\ymed)(1-\pi_{\theta_0}(\ymed))\big]^{h(s)}} ,
    \]
    where $C := \big( 26 / \gtreward(\ymed)^2 \big)^{13\alpha/7}$, as introduced in \zcref{eq:t_gamma_def}, and $s := \inprod{\phi(\ystar)}{\phi(\ymed)}$ and $h(s) := \frac{\norm{\phi(\ymed)}^2 - s}{7\brk1{\norm{\phi(\ymed)}^2 - s/2}}$ are as defined in \zcref{eq:neg_def_helper_variables}.
\end{proposition}

\begin{proof}
Since $t_\gamma$ is the initial time at which $\Vgt (\theta_t) \geq \gtreward (\ymed) - C \gamma^{13/14 + h(s) / 2}$, \zcref{lemma: conditional_monotonicity_loglin_neg} implies that for all $t \in [0, t_\gamma]$, if $\frac{d}{dt}\pi_{\theta_t}(\ystar) < 0$, then $\frac{d}{dt} \pi_{\theta_t}(\ymed) > 0$. 
Hence, for establishing that $\frac{d}{dt} \pi_{\theta_t} (\ystar) < 0$ and $\frac{d}{dt} \pi_{\theta_t} (\ymed) > 0$ over $[0, t_\gamma]$, it suffices to prove that $\frac{d}{dt} \pi_{\theta_t}(\ystar) < 0$.

To that end, we let $t_{\sqrt{\gamma}}$ be the initial time at which $\Vgt (\theta_t) \geq \gtreward(\ymed) - \sqrt{\gamma}$, \ie:
\[
    t_{\sqrt{\gamma}} := \min \brk[c]*{ t \geq 0 : \Vgt(\theta_t) \geq \gtreward(\ymed) - \sqrt{\gamma} } .
\]
From \zcref{lem:mediocre_output_reward_larger_than_initial_expected_reward_neg_inner_product}, $\Vgt (\theta_0) < \gtreward (\ymed) - \sqrt{\gamma}$, so $t_{\sqrt{\gamma}} > 0$.
Furthermore, since $\Vgt (\theta_t)$ is continuous in $t$, at $t_{\sqrt{\gamma}}$ it holds that $\Vgt (\theta_{t_{\sqrt{\gamma}}}) = \gtreward (\ymed) - \sqrt{\gamma}$.
The proof proceeds by showing that $\frac{d}{dt} \pi_{\theta_t}(\ystar) < 0$ up to time $t_{\sqrt{\gamma}}$ and upper bounding $\pi_{\theta_{t_{\sqrt{\gamma}}}}(\ystar)$. 
Lastly, we will establish that $\pi_{\theta_t} (\ystar)$ continues decreasing until time $t_\gamma$.
   
\medskip

\phantomsection\label{proof:monotonical_loglin_neg:step1}
\textbf{Step 1: monotonicity in $[0, t_{\sqrt{\gamma}}]$.}
We show that $\frac{d}{dt} \pi_{\theta_t} (\ystar) < 0$, and so $\frac{d}{dt} \pi_{\theta_t}(\ymed) > 0$, for all $t \in [0, t_{\sqrt{\gamma}}]$.
Assume by way of contradiction that there exists a time $t^\prime \le t_{\sqrt{\gamma}}$ at which $\left.\frac{d}{dt} \pi_{\theta_t}(\ystar)\right|_{t=t^\prime} \ge 0$ and denote by $\tau$ the initial such time, \ie:
\[
\tau := \inf \brk[c]*{ t \in [0, t_{\sqrt{\gamma}}] : \tfrac{d}{dt} \pi_{\theta_t}(\ystar) \geq 0 } .
\]
For all $t \in [0, \tau)$, we therefore have that $\frac{d}{dt} \pi_{\theta_t}(\ystar) < 0$ and $\frac{d}{dt} \pi_{\theta_t} (\ymed) > 0$, which imply $\pi_{\theta_\tau}(\ystar) \le \pi_{\theta_0}(\ystar)$ and $\pi_{\theta_\tau}(\ymed) \ge \pi_{\theta_0}(\ymed)$, respectively.

\medskip

We now show that $\frac{d}{dt} \pi_{\theta_t}(\Ybad) <0$ for all $ t \in [0, \tau)$. 
By \zcref{lem:probability_derivative_expression}, the time derivative of $\pi_{\theta_t}(z)$ for $z \in \Ybad$ can be written as:
\begin{align*}
    \frac{d}{dt}\pi_{\theta_t}(z) & = \pi_{\theta_t}(z)^2 \advgt(z; \theta_t) \|\phi(z)\|^2  -  \pi_{\theta_t}(z) \Big(\pi_{\theta_t}(z)^2 \advgt(z; \theta_t) \|\phi(z)\|^2  \\
    & \hspace{28mm} + \pi_{\theta_t}(\ymed)\advgt(\ymed; \theta_t) \big ( \|\phi(\ymed)\|^2\pi_{\theta_t}(\ymed) + s\pi_{\theta_t}(\ystar) \big )\\
    & \hspace{28mm} + \pi_{\theta_t}(\ystar) \advgt(\ystar; \theta_t)\big ( \|\phi(\ystar)\|^2\pi_{\theta_t}(\ystar) + s \pi_{\theta_t}(\ymed) \big) \Big) .
\end{align*}
Summing over all $z \in \Ybad$, we have:
\begin{align*}
    \frac{d}{dt} \pi_{\theta_t}(\Ybad) & = \underbrace{\sum\nolimits_{z \in \Ybad} \big( \pi_{\theta_t}(z)^2 \advgt(z; \theta_t) \|\phi(z)\|^2\big)\big(1 - \pi_{\theta_t}(z)\big)}_{I_1} \\
    & \hspace{4mm}- \pi_{\theta_t}(\Ybad) \Big(\pi_{\theta_t}(\ymed)\advgt(\ymed; \theta_t)(\|\phi(\ymed)\|^2\pi_{\theta_t}(\ymed) + s\pi_{\theta_t}(\ystar)) \\
    & \hspace{24mm} + \pi_{\theta_t}(\ystar) \advgt(\ystar; \theta_t)(\|\phi(\ystar)\|^2\pi_{\theta_t}(\ystar) + s \pi_{\theta_t}(\ymed)) \Big) ,
\end{align*}
where we denote by $I_2$ the second term in the expression above (excluding the minus sign).
Note that $I_1 \le 0$ from the fact that $\advgt(z; \theta_t) \le 0$ for all $z \in \Ybad$ (outputs in $\Ybad$ have the lowest ground truth reward among all outputs).
We now show that $I_2 > 0$.
It suffices to prove that:
\be
\begin{split}
0 & < \pi_{\theta_t}(\ymed)\advgt(\ymed; \theta_t)(\|\phi(\ymed)\|^2\pi_{\theta_t}(\ymed) + s\pi_{\theta_t}(\ystar)) 
\\ & \hspace{4mm} + \pi_{\theta_t}(\ystar) \advgt(\ystar; \theta_t)(\|\phi(\ystar)\|^2\pi_{\theta_t}(\ystar) + s \pi_{\theta_t}(\ymed)) .
\end{split}
\label{eq: loglin_neg_minor_cond_4}
\ee
If $\|\phi(\ystar)\|^2\pi_{\theta_t}(\ystar) + s \pi_{\theta_t}(\ymed) \ge 0$, which is equivalent to the second term on the right-hand side being non-negative, then:
\begin{align*}
    &\pi_{\theta_t}(\ymed) \advgt(\ymed; \theta_t)(\|\phi(\ymed)\|^2\pi_{\theta_t}(\ymed) + s\pi_{\theta_t}(\ystar)) \\
    & \hspace{4mm} + \pi_{\theta_t}(\ystar) \advgt(\ystar; \theta_t)(\|\phi(\ystar)\|^2\pi_{\theta_t}(\ystar) + s \pi_{\theta_t}(\ymed))\\
    & \geq \pi_{\theta_t}(\ymed)\advgt(\ymed; \theta_t)(\|\phi(\ymed)\|^2\pi_{\theta_t}(\ymed) + s\pi_{\theta_t}(\ystar))\\
    & \geq \pi_{\theta_t}(\ymed) \advgt(\ymed; \theta_t)(\|\phi(\ymed)\|^2\pi_{\theta_0}(\ymed) + s\pi_{\theta_0}(\ystar))\\
    & \geq \pi_{\theta_t}(\ymed) \advgt(\ymed; \theta_t)\|\phi(\ymed)\|\big(\|\phi(\ymed)\|\pi_{\theta_0}(\ymed) - \|\phi(\ystar)\|\pi_{\theta_0}(\ystar)\big)\\
    & > 0 ,
\end{align*}
where the second inequality is from $\pi_{\theta_t}(\ystar) \le \pi_{\theta_0}(\ystar)$ and $\pi_{\theta_t}(\ymed) \ge \pi_{\theta_0}(\ymed)$ for all $t \in [0, \tau)$, the third inequality is from Cauchy-Schwarz since $-s \le \|\phi(\ystar)\| \|\phi(\ymed)\|$, and the last inequality is from \zcref{assump: loglin_neg_ystar} in \zcref{assump: loglin_neg_minor}. 
Overall, this implies that if $\|\phi(\ystar)\|^2\pi_{\theta_t}(\ystar) + s \pi_{\theta_t}(\ymed) \ge 0$, then $I_1 - I_2 < 0$.

We now consider the case of $\|\phi(\ystar)\|^2\pi_{\theta_t}(\ystar) + s\pi_{\theta_t}(\ymed) < 0$, noting that this inequality is equivalent to the second term in \zcref{eq: loglin_neg_minor_cond_4} being negative.
We show that \zcref{eq: loglin_neg_minor_cond_4} holds by proving the following equivalent inequality:
$$
\frac{\pi_{\theta_t}(\ystar) \advgt(\ystar; \theta_t)}{\pi_{\theta_t}(\ymed) \advgt(\ymed; \theta_t)} \le \frac{\|\phi(\ymed)\|^2\pi_{\theta_t}(\ymed) + s\pi_{\theta_t}(\ystar)}{-\|\phi(\ystar)\|^2\pi_{\theta_t}(\ystar) - s \pi_{\theta_t}(\ymed)}.
$$
To see that this inequality holds, we lower bound the right-hand side by:
$$
\frac{\|\phi(\ymed)\|^2\pi_{\theta_t}(\ymed) + s\pi_{\theta_t}(\ystar)}{-\|\phi(\ystar)\|^2\pi_{\theta_t}(\ystar) - s \pi_{\theta_t}(\ymed)} \ge \frac{- s + \|\phi(\ymed)\|^2\pi_{\theta_t}(\ymed) + s \pi_{\theta_t}(\ystar)}{\|\phi(\ystar)\|^2 - \|\phi(\ystar)\|^2\pi_{\theta_t}(\ystar) - s \pi_{\theta_t}(\ymed)}.
$$
Since both denominators are positive in this case, multiplying both sides of the inequality by them and canceling common terms, we arrive at the following equivalent form of the inequality above:
$$
s^2 \pi_{\theta_t}(\ymed) \le \|\phi(\ystar)\|^2\|\phi(\ymed)\|^2 \pi_{\theta_t}(\ymed),
$$
which straightforwardly holds because $s^2 = \inprod{\phi(\ystar)}{\phi(\ymed)}^2 \le \norm{ \phi(\ymed) }^2 \norm{ \phi(\ystar) }^2$. 
Together with
\zcref{eq: loglin_neg_minor_cond}, this implies that:
\begin{align}
\frac{\pi_{\theta_t}(\ystar) \advgt(\ystar; \theta_t)}{\pi_{\theta_t}(\ymed) \advgt(\ymed; \theta_t)} & < \frac{- s + \|\phi(\ymed)\|^2\pi_{\theta_t}(\ymed) + s \pi_{\theta_t}(\ystar) }{ \|\phi(\ystar)\|^2 - \|\phi(\ystar)\|^2\pi_{\theta_t}(\ystar) - s \pi_{\theta_t}(\ymed)}\nonumber\\
&\le \frac{\|\phi(\ymed)\|^2\pi_{\theta_t}(\ymed) + s\pi_{\theta_t}(\ystar)}{-\|\phi(\ystar)\|^2\pi_{\theta_t}(\ystar) - s \pi_{\theta_t}(\ymed)} .
\label{eq: loglin_neg_cond}
\end{align}
Thus, \zcref{eq: loglin_neg_minor_cond_4} holds, and so $I_2 > 0$, also when $\|\phi(\ystar)\|^2\pi_{\theta_t}(\ystar) + s \pi_{\theta_t}(\ymed) < 0$.
Altogether, we may conclude that $\frac{d}{dt}\pi_{\theta_t}(\Ybad) =I_1-I_2 < 0$ for all $t \in [0, \tau)$.

\medskip

Next, we prove that $\pi_{\theta_\tau}(\Ybad)^2 \le \frac{7\|\phi(\ymed)\|^2\pi_{\theta_0}(\ymed)^2 \advgt(\ymed; \theta_\tau)}{8\max\nolimits_{z \in \Ybad} \|\phi(z)\|^2 \Delta_2}$. 
We can write $\advgt(\ymed; \theta_\tau)$ as: 
\begin{align*}
        \advgt(\ymed; \theta_\tau) &= (1 - \pi_{\theta_\tau}(\ymed)) \gtreward(\ymed) - \pi_{\theta_\tau}(\ystar) \gtreward(\ystar) - \sum\nolimits_{z \in \Ybad} \pi_{\theta_\tau}(z) \gtreward(z) \\
        &= -\Delta_1 \pi_{\theta_\tau}(\ystar) + \Deltatwo \pi_{\theta_\tau}(\Ybad).
\end{align*}
Thus,
\begin{align*}
    \pi_{\theta_\tau}(\Ybad) = \frac{\advgt(\ymed; \theta_\tau) + \Delta_1 \pi_{\theta_\tau}(\ystar)}{\Deltatwo} \le \frac{1.3 \advgt(\ymed; \theta_\tau)}{\Deltatwo}.
\end{align*}
The inequality is due to $\Delta_1 \pi_{\theta_\tau}(\ystar) \le 2\pi_{\theta_0}(\ystar) \le 2\gamma^{13/14} \le 0.3\sqrt{\gamma} \le 0.3\advgt(\ymed; \theta_\tau)$, where $2\gamma^{13 / 14} \leq 0.3 \sqrt{\gamma}$ since $\gamma \le 0.01$. 
Then,
\begin{align*}
    \pi_{\theta_\tau}(\Ybad)^2 &\le \frac{1.69 \advgt(\ymed; \theta_\tau)^2}{\Deltatwo^2}\\
    &< \frac{7\|\phi(\ymed)\|^2\pi_{\theta_0}(\ymed)^2 \advgt(\ymed; \theta_\tau)}{8\max\nolimits_{z \in \Ybad} \|\phi(z)\|^2 \Delta_2} \cdot \frac{2\max\nolimits_{z \in \Ybad} \|\phi(z)\|^2  \advgt(\ymed; \theta_\tau)}{\|\phi(\ymed)\|^2 \Deltatwo \pi_{\theta_0}(\ymed)^2}\\
    &\le \frac{7\|\phi(\ymed)\|^2\pi_{\theta_0}(\ymed)^2 \advgt(\ymed; \theta_\tau)}{8 \max\nolimits_{z \in \Ybad} \|\phi(z)\|^2 \Delta_2} \cdot \frac{2  \advgt(\ymed; \theta_0)}{\Deltatwo \pi_{\theta_0}(\ymed)^2},
\end{align*}
where the second inequality is by $1.69 \times 8 < 2 \times 7$ and the last inequality is by $\max\nolimits_{z \in \Ybad}\|\phi(z)\|^2 \le \|\phi(\ymed)\|^2$ (\zcref{assum:feature_structure_neg_inn_prod}) and $\advgt(\ymed; \theta_\tau) \le \advgt(\ymed; \theta_0)$.
Hence, it suffices to prove that $\frac{2 \advgt(\ymed; \theta_0)}{\Deltatwo \pi_{\theta_0}(\ymed)^2} \le 1$.
Indeed, it holds that:
\begin{align*}
    \advgt(\ymed; \theta_0) &= (1 - \pi_{\theta_0}(\ymed))\gtreward(\ymed) - \pi_{\theta_0}(\ystar) \gtreward(\ystar) - \sum\nolimits_{z \in \Ybad}\pi_{\theta_0}(z) \gtreward(z)\\
    &\le (1 - \pi_{\theta_0}(\ymed))\gtreward(\ymed) + \pi_{\theta_0}(\Ybad)\\
    &\le 0.2 \gtreward(\ymed) + 0.1 \gtreward(\ymed)\\
    &< \frac{ 0.8^2 \gtreward(\ymed)}{2}\\
    &\le \frac{ \Deltatwo \pi_{\theta_0}(\ymed)^2}{2},
\end{align*}
where the first inequality is from $\gtreward(z) \ge -1$; the second inequality is from \zcref{assump: loglin_neg_minor}, according to which $\pi_{\theta_0}(\Ybad) \le 0.1\gtreward(\ymed) \le 0.1$ and $\pi_{\theta_0}(\ymed) \ge 0.8$; and the last inequality is from $\gtreward(\ybad) \le 0$ and $\pi_{\theta_0}(\ymed) \ge 0.8$.
Thus, $\frac{2\advgt(\ymed; \theta_0)}{\Deltatwo \pi_{\theta_0}(\ymed)^2} \le 1$, from which it follows that:
\begin{equation}
\label{eq: ybad2_upper_neg}
    \pi_{\theta_\tau}(\Ybad)^2 \le \frac{7\|\phi(\ymed)\|^2\pi_{\theta_0}(\ymed)^2 \advgt(\ymed; \theta_\tau)}{8\max\nolimits_{z \in \Ybad} \|\phi(z)\|^2\Delta_2},
\end{equation}
as desired.

\medskip

We proceed by analyzing $\left.\frac{d}{dt} \pi_{\theta_t}(\ystar)\right|_{t=\tau}$, based on the expression derived in \zcref{lem:probability_derivative_expression}:
    \begin{align*}
        & \left.\frac{d}{dt} \pi_{\theta_t}(\ystar)\right|_{t=\tau} \\
        & = \pi_{\theta_\tau}(\ystar) \Big[ \pi_{\theta_\tau}(\ystar) \advgt(\ystar; \theta_\tau) \big(\|\phi(\ystar)\|^2 - \|\phi(\ystar)\|^2\pi_{\theta_\tau}(\ystar) - s \pi_{\theta_\tau}(\ymed)\big)\\
        & \hspace{16mm} + \pi_{\theta_\tau}(\ymed) \advgt(\ymed; \theta_\tau) \big(s - \|\phi(\ymed)\|^2\pi_{\theta_\tau}(\ymed) - s \pi_{\theta_\tau}(\ystar) \big) \\
        & \hspace{16mm} - \sum\nolimits_{z \in \Ybad} \pi_{\theta_\tau}(z)^2 \advgt(z; \theta_\tau) \|\phi(z)\|^2 \Big] \\
        & \le \pi_{\theta_\tau}(\ystar) \Big[  \pi_{\theta_\tau}(\ystar) \advgt(\ystar; \theta_\tau) \big(\|\phi(\ystar)\|^2 - \|\phi(\ystar)\|^2\pi_{\theta_\tau}(\ystar) - s \pi_{\theta_\tau}(\ymed)\big)\\
        &\hspace{16mm} + \pi_{\theta_\tau}(\ymed) \advgt(\ymed; \theta_\tau) \big(s - \|\phi(\ymed)\|^2\pi_{\theta_\tau}(\ymed) - s \pi_{\theta_\tau}(\ystar) \big ) \\
        & \hspace{16mm} + (\Delta_2 - \sqrt{\gamma}) \sum\nolimits_{z \in \Ybad} \pi_{\theta_\tau}(z)^2 \|\phi(z)\|^2 \Big]\\
        & \le \pi_{\theta_\tau}(\ystar) \Big[ (\|\phi(\ystar)\|^2-s)\pi_{\theta_\tau}(\ystar) \advgt(\ystar; \theta_\tau) \big(1 - \pi_{\theta_\tau}(\ystar)\big) \\
        & \hspace{16mm} - \|\phi(\ymed)\|^2\pi_{\theta_\tau}(\ymed)^2 \advgt(\ymed; \theta_\tau) \\
        & \hspace{16mm} + (\Delta_2 - \sqrt{\gamma}) \sum\nolimits_{z \in \Ybad} \pi_{\theta_\tau}(z)^2 \|\phi(z)\|^2 \Big] \\
        & \le \pi_{\theta_\tau}(\ystar) \Big[ (\|\phi(\ystar)\|^2-s)\pi_{\theta_0}(\ystar) \advgt(\ystar; \theta_0) \big(1 - \pi_{\theta_0}(\ystar)\big) \\
        & \hspace{16mm} - \|\phi(\ymed)\|^2\pi_{\theta_\tau}(\ymed)^2 \advgt(\ymed; \theta_\tau)\\
        &\hspace{16mm} + \max\nolimits_{z \in \Ybad} \|\phi(z)\|^2(\Delta_2 - \sqrt{\gamma})\pi_{\theta_\tau}(\Ybad)^2  \Big]\\
    \end{align*}
The first inequality is from $-\advgt(z; \theta_\tau) \le -\advgt(z; \theta_{t_{\sqrt{\gamma}}}) \le \Deltatwo -\sqrt{\gamma}$; the second inequality is from $s \le 0$ and $\pi_{\theta_\tau}(\ymed) \le 1 - \pi_{\theta_\tau}(\ystar)$; and the third inequality is from $\advgt(\ystar; \theta_\tau) \le \advgt(\ystar; \theta_0)$, $\pi_{\theta_\tau}(\ystar) \le \pi_{\theta_0}(\ystar)$, $\pi_{\theta_\tau}(\ymed) \ge \pi_{\theta_0}(\ymed)$ and $\sum\nolimits_{z \in \Ybad} \pi_{\theta_\tau}(z)^2 \le \pi_{\theta_\tau} (\Ybad)^2$.
We continue bounding $\left.\frac{d}{dt} \pi_{\theta_t}(\ystar)\right|_{t=\tau}$ as follows:
\be
\begin{split}
    & \left.\frac{d}{dt} \pi_{\theta_t}(\ystar)\right|_{t=\tau} \\
    & \le \pi_{\theta_\tau}(\ystar) \Big[ (\|\phi(\ystar)\|^2-s)\pi_{\theta_0}(\ystar) \advgt(\ystar; \theta_0) \big(1 - \pi_{\theta_0}(\ystar)\big) \\
    & \hspace{16mm} - \|\phi(\ymed)\|^2\pi_{\theta_\tau}(\ymed)^2 \advgt(\ymed; \theta_\tau) + \max\nolimits_{z \in \Ybad} \|\phi(z)\|^2(\Delta_2 - \sqrt{\gamma})\pi_{\theta_\tau}(\Ybad)^2  \Big]\\
    & \le \pi_{\theta_\tau}(\ystar) \Big[ (\|\phi(\ystar)\|^2-s)\pi_{\theta_0}(\ystar) \advgt(\ystar; \theta_0) \big(1 - \pi_{\theta_0}(\ystar)\big) \\
    & \hspace{16mm} - (\|\phi(\ymed)\|^2/8) \pi_{\theta_\tau}(\ymed)^2 \advgt(\ymed; \theta_\tau) \Big]\\
    & \le \pi_{\theta_\tau}(\ystar) \Big[(\|\phi(\ystar)\|^2-s)\pi_{\theta_0}(\ystar) \advgt(\ystar; \theta_0) \big(1 - \pi_{\theta_0}(\ystar)\big) - (\|\phi(\ymed)\|^2/8) \pi_{\theta_\tau}(\ymed)^2 \sqrt{\gamma} \Big] .
\end{split}
\label{eq:pi_ystar_derive_tau_upper_bound_neg}
\ee
Here, the second inequality is from \zcref{eq: ybad2_upper_neg} and $\pi_{\theta_0} (\ymed) \leq \pi_{\theta_\tau} (\ymed)$, according to which
\[
\begin{split}
    (\Delta_2-\sqrt{\gamma}) \pi_{\theta_\tau}(\Ybad)^2 & \le \Delta_2 \pi_{\theta_\tau}(\Ybad)^2 \\
    & \le \frac{7\|\phi(\ymed)\|^2\pi_{\theta_0}(\ymed)^2 \advgt(\ymed; \theta_\tau)}{8\max\nolimits_{z \in \Ybad} \|\phi(z)\|^2} \\
    & \le \frac{7\|\phi(\ymed)\|^2\pi_{\theta_\tau}(\ymed)^2 \advgt(\ymed; \theta_\tau)}{8\max\nolimits_{z \in \Ybad} \|\phi(z)\|^2} ,
\end{split}
\] 
and the third inequality is from the fact that $\Vgt(\theta_t)$ is monotonically non-decreasing when maximizing it via gradient flow, so $\advgt(\ymed; \theta_\tau) \ge \advgt(\ymed; \theta_{t_{\sqrt{\gamma}}})$.

Now, by \zcref{assump: loglin_neg_minor}, since $\|\phi(\ystar)\|^2 > 0 > s$ we get:
\[
\begin{split}
    \pi_{\theta_0}(\ystar) & \le \frac{\|\phi(\ymed)\|^2\gtreward(\ymed)^{2/7}\gamma^{13/14}}{40B^2(\Delta_1+\Delta_2)} \\
    & \le \frac{\|\phi(\ymed)\|^2\sqrt{\gamma} \cdot \gamma^{3/7}}{40B^2 \advgt(\ystar; \theta_0)} \\
    & < \frac{0.8^2\|\phi(\ymed)\|^2\sqrt{\gamma}}{16B^2 \advgt(\ystar; \theta_0)} \\
    & \le \frac{\|\phi(\ymed)\|^2\pi_{\theta_0}(\ymed)^2\sqrt{\gamma}}{8(\|\phi(\ystar)\|^2-s) \advgt(\ystar; \theta_0)} .
\end{split}
\]
The first inequality is from \zcref{assump: loglin_neg_ystar} in \zcref{assump: loglin_neg_minor}; the second inequality is from $\gtreward(\ymed) \le 1$ and $\advgt(\ystar; \theta_0) \le \Delta_1+\Delta_2$; the third inequality is from $\gamma^{3/7} \le 1$; and the last inequality is from $\pi_{\theta_0}(\ymed) = 1 - \pi_{\theta_0}(\ystar) - \pi_{\theta_0}(\Ybad) \ge 1 - \gamma^{13/14} - 0.1\gtreward(\ymed) \ge 0.8$ and $\|\phi(\ystar)\|^2-s \le 2B^2$.
Plugging this upper bound on $\pi_{\theta_0}(\ystar)$ into \zcref{eq:pi_ystar_derive_tau_upper_bound_neg}, and using $\pi_{\theta_\tau}(\ymed) \ge \pi_{\theta_0}(\ymed)$ together with $1-\pi_{\theta_0}(\ystar)<1$, yields $\left.\frac{d}{dt} \pi_{\theta_t}(\ystar)\right|_{t=\tau} < 0$.
This contradicts the fact that $\tau$ is defined as the initial time at which $\left.\frac{d}{dt} \pi_{\theta_t}(\ystar)\right|_{t=\tau} \ge 0$.
Thus, it must be that $\frac{d}{dt} \pi_{\theta_t}(\ystar) < 0$ for all $t \in [0, t_{\sqrt{\gamma}}]$.

\medskip

\textbf{Step 2: upper bound on $\pi_{\theta_{t_{\sqrt{\gamma}}}}(\ystar)$ and monotonicity in $[ t_{\sqrt{\gamma}}, t_\gamma]$.}
Next, we prove that $\pi_{\theta_t} (\ystar)$ keeps decreasing and $\pi_{\theta_t}(\ymed)$ keeps increasing for all $t \in [t_{\sqrt{\gamma}}, t_\gamma]$.
To do so, we first derive upper bounds on $\frac{d}{dt} \pi_{\theta_t}(\ystar)$ and $\frac{d}{dt} \pi_{\theta_t}(\ymed)$ for $t \in [0, t_{\sqrt{\gamma}}]$ that allow quantifying how small the probability of $\ystar$ becomes until $t_{\sqrt{\gamma}}$.

\medskip

\textbf{Step 2.1: upper bound on $\frac{d}{dt} \pi_{\theta_t}(\ystar)$ in $[0, t_{\sqrt{\gamma}}]$.}
In this step, we prove that for all $t \in [0, t_{\sqrt{\gamma}}]$:
\begin{align*}
    \frac{d}{dt} \pi_{\theta_t}(\ystar) &\le -0.3 (\|\phi(\ymed)\|^2-s) \cdot \pi_{\theta_t}(\ystar)\pi_{\theta_t}(\ymed)^2 \advgt(\ymed; \theta_t) .
\end{align*}

First, for $t \in [0, t_{\sqrt{\gamma}}]$, based on the expression derived in \zcref{lem:probability_derivative_expression} we upper bound $\frac{d}{dt} \pi_{\theta_t}(\ystar)$ as follows:
\be
\begin{split}
        &\frac{d}{dt} \pi_{\theta_t}(\ystar) \\
        & = \pi_{\theta_t}(\ystar) \Big[ \pi_{\theta_t}(\ystar) \advgt(\ystar; \theta_t) \big(\|\phi(\ystar)\|^2 - \|\phi(\ystar)\|^2\pi_{\theta_t}(\ystar) - s \pi_{\theta_t}(\ymed)\big)\\
        & \hspace{16mm} + \pi_{\theta_t}(\ymed) \advgt(\ymed; \theta_t) \big(s - \|\phi(\ymed)\|^2\pi_{\theta_t}(\ymed) - s \pi_{\theta_t}(\ystar) \big) \\
        & \hspace{16mm} - \sum\nolimits_{z \in \Ybad}\pi_{\theta_t}(z)^2 \advgt(z; \theta_t) \|\phi(z)\|^2\Big] \\
        & \le \pi_{\theta_t}(\ystar) \Big[ (\|\phi(\ystar)\|^2-s) \pi_{\theta_t}(\ystar) \advgt(\ystar; \theta_t) \big(1 - \pi_{\theta_t}(\ystar)\big) \\
        & \hspace{16mm} - (\|\phi(\ymed)\|^2-s) \pi_{\theta_t}(\ymed)^2 \advgt(\ymed; \theta_t)\\
        &\hspace{16mm} - \sum\nolimits_{z \in \Ybad}\pi_{\theta_t}(z)^2 \advgt(z; \theta_t) \|\phi(z)\|^2 \Big]\\
        & \le \pi_{\theta_t}(\ystar) \Big[ (\|\phi(\ystar)\|^2-s) \pi_{\theta_0}(\ystar) \advgt(\ystar; \theta_t) \big(1 - \pi_{\theta_0}(\ystar)\big) \\
        &\hspace{16mm} - (\|\phi(\ymed)\|^2-s) \pi_{\theta_t}(\ymed)^2 \advgt(\ymed; \theta_t)\\
        &\hspace{16mm} - \sum\nolimits_{z \in \Ybad}\pi_{\theta_t}(z)^2 \advgt(z; \theta_t) \|\phi(z)\|^2 \Big] .
    \end{split}
    \label{eq:pi_ystar_derivative_upper_bound_neg_quantitative}
    \ee
    The first inequality is from $\pi_{\theta_t}(\ymed) \le 1 - \pi_{\theta_t}(\ystar)$ and the last inequality is by $\pi_{\theta_t}(\ystar) \leq \pi_{\theta_0} (\ystar) < 0.5$, which implies $\pi_{\theta_t} (\ystar) (1 - \pi_{\theta_t} (\ystar)) \leq \pi_{\theta_0} (\ystar) ( 1 - \pi_{\theta_0} (\ystar))$. 
    Denote by $G$ the term inside the square brackets in the last inequality above.
    We now prove that $G \leq -0.3(\|\phi(\ymed)\|^2-s) \pi_{\theta_t}(\ymed)^2 \advgt(\ymed; \theta_t)$, which is equivalent to 
    \[
        G+0.3(\|\phi(\ymed)\|^2-s) \pi_{\theta_t}(\ymed)^2 \advgt(\ymed; \theta_t) \leq 0 .
    \]
    Towards doing so, we perform the following sequence of computations:
    \begin{align*}
        & G+0.3(\|\phi(\ymed)\|^2-s) \pi_{\theta_t}(\ymed)^2 \advgt(\ymed; \theta_t)  \\
        &= -0.7(\|\phi(\ymed)\|^2-s) \pi_{\theta_t}(\ymed)^2 \advgt(\ymed; \theta_t) \\
        & \hspace{4mm} + (\|\phi(\ystar)\|^2-s) \pi_{\theta_0}(\ystar) \advgt(\ystar; \theta_t) \big(1 - \pi_{\theta_0}(\ystar)\big) - \sum\nolimits_{z \in \Ybad} \pi_{\theta_t}(z)^2 \advgt(z; \theta_t) \|\phi(z)\|^2\\
        & \le -0.7(\|\phi(\ymed)\|^2-s) \pi_{\theta_t}(\ymed)^2 \advgt(\ymed; \theta_t) \\
        & \hspace{4mm} + (\|\phi(\ystar)\|^2-s) \pi_{\theta_0}(\ystar) \advgt(\ystar; \theta_t) - \sum\nolimits_{z \in \Ybad} \max\nolimits_{z \in \Ybad} \|\phi(z)\|^2\pi_{\theta_t}(z)^2 \advgt(z; \theta_t)\\
        &= -0.7(\|\phi(\ymed)\|^2-s) \pi_{\theta_t}(\ymed)^2 \advgt(\ymed; \theta_t) \\
        & \hspace{4mm} + (\|\phi(\ystar)\|^2-s) \pi_{\theta_0}(\ystar) \advgt(\ystar; \theta_t) + \sum\nolimits_{z \in \Ybad} \max\nolimits_{z \in \Ybad} \|\phi(z)\|^2\pi_{\theta_t}(z)^2 (\Vgt(\theta_t) - \gtreward(z))\\
        &= -0.7 (\|\phi(\ymed)\|^2-s) \pi_{\theta_t}(\ymed)^2 \Big( (1-\pi_{\theta_t}(\ymed))\gtreward(\ymed) - \pi_{\theta_t}(\ystar)\gtreward(\ystar) \\
        & \hspace{20mm} - \gtreward(\ybad) \sum\nolimits_{z \in \Ybad} \pi_{\theta_t}(z) \Big) + (\|\phi(\ystar)\|^2-s) \pi_{\theta_0}(\ystar) \advgt(\ystar; \theta_t) \\
        & \hspace{4mm} + \sum\nolimits_{z \in \Ybad} \max\nolimits_{z \in \Ybad} \|\phi(z)\|^2\pi_{\theta_t}(z)^2 \Big((1-\pi_{\theta_t}(z))(-\gtreward(z)) + \pi_{\theta_t}(\ystar) \gtreward(\ystar) \\
        & \hspace{20mm} + \pi_{\theta_t}(\ymed) \gtreward(\ymed) + \sum\nolimits_{z^\prime \in \Ybad \setminus \{z\}} \pi_{\theta_t}(z^\prime)\gtreward(z^\prime) \Big) \\
        &= -0.7 (\|\phi(\ymed)\|^2-s) \pi_{\theta_t}(\ymed)^2 \big(\pi_{\theta_t}(\Ybad)\Deltatwo - \pi_{\theta_t}(\ystar)\Delta_1 \big) \\
        & \hspace{4mm} + (\|\phi(\ystar)\|^2-s) \pi_{\theta_0}(\ystar) \advgt(\ystar; \theta_t)\\
        &\hspace{4mm} + \sum\nolimits_{z \in \Ybad} \max\nolimits_{z \in \Ybad} \|\phi(z)\|^2\pi_{\theta_t}(z)^2 \big(\pi_{\theta_t}(\ymed)\Delta_2 + \pi_{\theta_t}(\ystar)(\Delta_1 + \Delta_2) \big) ,
    \end{align*}
    where the inequality is from $1 - \pi_{\theta_0}(\ystar) \le 1$.

    Continuing from the last expression above, we get:
    \begin{align*}
        & G+0.3(\|\phi(\ymed)\|^2-s) \pi_{\theta_t}(\ymed)^2 \advgt(\ymed; \theta_t) \\
        & \le -0.7 (\|\phi(\ymed)\|^2-s) \pi_{\theta_t}(\ymed)^2 \big(\pi_{\theta_t}(\Ybad)\Deltatwo - \pi_{\theta_t}(\ystar)\Delta_1 \big) \\
        & \hspace{4mm} + (\|\phi(\ystar)\|^2-s) \pi_{\theta_0}(\ystar) \advgt(\ystar; \theta_t)\\
        &\hspace{4mm} + \max\nolimits_{z \in \Ybad} \|\phi(z)\|^2\pi_{\theta_t}(\Ybad)^2 \big(\pi_{\theta_t}(\ymed)\Delta_2 + \pi_{\theta_t}(\ystar)(\Delta_1 + \Delta_2) \big)\\
        &\le \pi_{\theta_t}(\Ybad)\Big[ -0.7(\|\phi(\ymed)\|^2-s)\pi_{\theta_t}(\ymed)^2\Deltatwo \\
        & \hspace{20mm} + \max\nolimits_{z \in \Ybad} \|\phi(z)\|^2 \big(\pi_{\theta_t}(\Ybad)\pi_{\theta_t}(\ymed)\Delta_2 \big) \Big]\\
        & \hspace{4mm} +\pi_{\theta_0}(\ystar) \Big[ 0.7(\|\phi(\ymed)\|^2-s)\pi_{\theta_t}(\ymed)^2\Delta_1 + (\|\phi(\ystar)\|^2-s)\advgt(\ystar; \theta_t) \\
        & \hspace{20mm} + \max\nolimits_{z \in \Ybad} \|\phi(z)\|^2(\Delta_1 + \Delta_2)\pi_{\theta_t}(\Ybad)^2\Big ]\\
        &\le \underbrace{\pi_{\theta_t}(\Ybad)\Deltatwo \Big[ -0.7(\|\phi(\ymed)\|^2-s)\pi_{\theta_t}(\ymed)^2 + \|\phi(\ymed)\|^2\pi_{\theta_t}(\Ybad)\pi_{\theta_t}(\ymed)\Big]}_{I_1}\\
        & \hspace{4mm} + \pi_{\theta_0}(\ystar)\Big [ 0.7(\|\phi(\ymed)\|^2-s)\pi_{\theta_t}(\ymed)^2\Delta_1 + \max\nolimits_{z \in \Ybad} \|\phi(z)\|^2\pi_{\theta_t}(\Ybad)^2(\Delta_1 + \Delta_2) \\
        & \hspace{20mm} + (\|\phi(\ystar)\|^2-s)\advgt(\ystar; \theta_t) \Big ].
    \end{align*}
    Here, the first inequality is due to $\sum\nolimits_{z \in \Ybad} \pi_{\theta_t}(z)^2 \le \pi_{\theta_t} (\Ybad)^2$; the second inequality is from $\pi_{\theta_t}(\ystar) \le \pi_{\theta_0}(\ystar)$; and the last inequality is by \zcref{assum:feature_structure_neg_inn_prod}.
    We denote the second term in the last inequality by $I_2$, \ie, $I_2$ is equal to $\pi_{\theta_0} (\ystar)$ multiplied by the term inside the last square brackets.
    
    Next, we show that:
    \[
    \begin{split}
    I_1 &\le -0.2 (\|\phi(\ymed)\|^2-s) \pi_{\theta_t}(\Ybad) \pi_{\theta_t}(\ymed)^2 \Deltatwo , \\
    I_2 & \le 0.2 (\|\phi(\ymed)\|^2-s) \pi_{\theta_t}(\Ybad) \pi_{\theta_t}(\ymed)^2 \Deltatwo ,
    \end{split}
    \]
    from which it follows that $G+0.3(\|\phi(\ymed)\|^2-s) \pi_{\theta_t}(\ymed)^2 \advgt(\ymed; \theta_t) \leq 0$, and so $G \le -0.3(\|\phi(\ymed)\|^2-s) \pi_{\theta_t}(\ymed)^2 \advgt(\ymed; \theta_t)$.
    Plugging this inequality into \zcref{eq:pi_ystar_derivative_upper_bound_neg_quantitative} will then yield the desired upper bound on $\frac{d}{dt} \pi_{\theta_t}(\ystar)$.

    For $I_1$, it suffices to show that $0.5(\|\phi(\ymed)\|^2-s) \pi_{\theta_t}(\ymed) \ge \|\phi(\ymed)\|^2 \pi_{\theta_t}(\Ybad)$.
    From \zcref{assump: loglin_neg_minor}, we have $\pi_{\theta_0}(\Ybad) < 0.2$ and $\gamma < 0.05$, so $\pi_{\theta_0}(\ystar) \le \gamma^{13/14} < 0.13$ and $\pi_{\theta_0}(\ymed) = 1 - \pi_{\theta_0}(\ystar) - \pi_{\theta_0}(\Ybad) > 0.67$.
    This implies that $\pi_{\theta_t}(\Ybad) \le \pi_{\theta_0}(\Ybad) < 0.4 \pi_{\theta_0}(\ymed) \le 0.4 \pi_{\theta_t}(\ymed)$. 
    Since $s \le 0$, we further obtain:
    \begin{align*}
        &\|\phi(\ymed)\|^2 \pi_{\theta_t}(\Ybad)  
        < 0.4\pi_{\theta_t}(\ymed) \|\phi(\ymed)\|^2 < 0.5 (\|\phi(\ymed)\|^2-s) \pi_{\theta_t}(\ymed),
    \end{align*}
    which implies that $I_1 \le -0.2 (\|\phi(\ymed)\|^2-s) \pi_{\theta_t}(\Ybad) \pi_{\theta_t}(\ymed)^2\Deltatwo$. 
    
For $I_2$, we have:
\begin{align*}
    &0.7 (\|\phi(\ymed)\|^2-s)\pi_{\theta_t}(\ymed)^2 \Delta_1 + \max\nolimits_{z \in \Ybad} \|\phi(z)\|^2\pi_{\theta_t}(\Ybad)^2 (\Delta_1 + \Delta_2) \\
    & \hspace{4mm} + (\|\phi(\ystar)\|^2-s) \advgt(\ystar; \theta_t)\\
    & \le  0.7(\|\phi(\ymed)\|^2-s) \Delta_1 + 0.1\max\nolimits_{z \in \Ybad} \|\phi(z)\|^2(\Delta_1 + \Delta_2) + (\|\phi(\ystar)\|^2-s)\advgt(\ystar; \theta_0)\\
    & \leq  2(B^2 - s)(\Delta_1 + \Delta_2),
\end{align*}
where the first inequality is from $\pi_{\theta_t}(\Ybad)^2 \le \pi_{\theta_0}(\Ybad)^2 \le \pi_{\theta_0}(\Ybad) \le 0.1$ and $\Vgt(\theta_t) \ge \Vgt(\theta_0)$, and the second inequality is from $\advgt(\ystar; \theta_0) \le \Delta_1 + \Delta_2$. 
Meanwhile, as $\advgt(\ymed; \theta_t) \ge \sqrt{\gamma}$ for $t \in [0, t_{\sqrt{\gamma}}]$, we get:
\begin{align*}
    \advgt(\ymed; \theta_t) &= (1 - \pi_{\theta_t}(\ymed)) \gtreward(\ymed) - \pi_{\theta_t}(\ystar) \gtreward(\ystar) - \sum\nolimits_{z \in \Ybad} \pi_{\theta_t}(z) \gtreward(z)\\
    &= -\Delta_1 \pi_{\theta_t}(\ystar) + \Delta_2 \pi_{\theta_t}(\Ybad).
\end{align*}
Thus, $-\Delta_1 \pi_{\theta_t}(\ystar) + \Delta_2 \pi_{\theta_t}(\Ybad) \ge \sqrt{\gamma}$, and rearranging this inequality gives
\be
\pi_{\theta_t}(\Ybad) \ge \frac{\sqrt{\gamma}+ \Delta_1 \pi_{\theta_t}(\ystar)}{\Delta_2} \ge \frac{\sqrt{\gamma}}{\Delta_2} .
\label{eq:y_bad_prob_lower_bound_t_sqrt_gamma_neg_quantitative}
\ee
Here, the second inequality uses $\Delta_1 \pi_{\theta_t}(\ystar) \ge 0$. 
By the bound on $\gamma$ in \zcref{assump: gamma_loglin_neg} of \zcref{assump: loglin_neg_minor}, together with $\pi_{\theta_0}(\ymed)>0.67$ (shown above when bounding $I_1$), it holds that
\[
\gamma^{3/7}
\le \frac{21\|\phi(\ymed)\|^2}{1100B^2(\Delta_1+\Delta_2)}
\le \frac{\|\phi(\ymed)\|^2\pi_{\theta_0}(\ymed)^2}{10B^2(\Delta_1+\Delta_2)} .
\]
Moreover, notice that since $s<0$ and $\|\phi(\ymed)\|^2\le B^2$, we have
\[
\frac{\|\phi(\ymed)\|^2}{B^2}
\le
\frac{\|\phi(\ymed)\|^2-s}{B^2-s}.
\]
The last two inequalities thus give:
\[
\pi_{\theta_0}(\ystar) \le \gamma^{13/14} = \sqrt{\gamma} \cdot \gamma^{3 / 7} \le \frac{(\|\phi(\ymed)\|^2-s)\sqrt{\gamma} \pi_{\theta_0}  (\ymed)^2}{10 (B^2-s)(\Delta_1 + \Delta_2)},
\]
which leads to:
\[
I_2 \le 2(B^2-s)(\Delta_1 + \Delta_2) \pi_{\theta_0}(\ystar) \le 0.2 (\|\phi(\ymed)\|^2-s) \pi_{\theta_t}(\Ybad) \pi_{\theta_t}(\ymed)^2 \Deltatwo .
\]

Combining the upper bounds on $I_1$ and $I_2$, we have shown that $I_1 + I_2 \leq 0$, and so $G \leq -0.3(\|\phi(\ymed)\|^2-s) \pi_{\theta_t}(\ymed)^2 \advgt(\ymed; \theta_t)$.
Plugging this upper bound on $G$ into \zcref{eq:pi_ystar_derivative_upper_bound_neg_quantitative} establishes that for all $t \in [0, t_{\sqrt{\gamma}}]$:
\begin{align*}
    \frac{d}{dt} \pi_{\theta_t}(\ystar) &\le -0.3 (\|\phi(\ymed)\|^2-s) \cdot \pi_{\theta_t}(\ystar)\pi_{\theta_t}(\ymed)^2 \advgt(\ymed; \theta_t) ,
\end{align*}
as desired.

\medskip

\textbf{Step 2.2: upper bound on $\frac{d}{dt} \pi_{\theta_t}(\ymed)$ in $[0, t_{\sqrt{\gamma}}]$.}
In Step 1, we showed that $\frac{d}{dt} \pi_{\theta_t} (\ymed) > 0$ for $t \in [0, t_{\sqrt{\gamma}}]$.
In this step, we prove that $\pi_{\theta_t} (\ymed)$ does not grow too fast by establishing that:
\begin{align*}
    \frac{d}{dt}\pi_{\theta_t}(\ymed) \le 2.1 (\|\phi(\ymed)\|^2-s/2) \pi_{\theta_t}(\ymed)^2 (1-\pi_{\theta_t}(\ymed))\advgt(\ymed; \theta_t) .
\end{align*}

We begin upper bounding $ \frac{d}{dt}\pi_{\theta_t}(\ymed)$ using the expression derived in \zcref{lem:probability_derivative_expression}:
\be
\begin{split}
    & \frac{d}{dt}\pi_{\theta_t}(\ymed) \\
    & = \pi_{\theta_t}(\ymed) \Big [\pi_{\theta_t}(\ymed) \advgt(\ymed; \theta_t) \big( \|\phi(\ymed)\|^2 - \|\phi(\ymed)\|^2\pi_{\theta_t}(\ymed) - s \pi_{\theta_t}(\ystar) \big)\\
    & \hspace{20mm} + \pi_{\theta_t}(\ystar) \advgt(\ystar; \theta_t) \big(s - \|\phi(\ystar)\|^2\pi_{\theta_t}(\ystar) - s \pi_{\theta_t}(\ymed) \big) \\
    & \hspace{20mm} - \sum\nolimits_{z \in \Ybad} \pi_{\theta_t}(z)^2 \advgt(z; \theta_t) \|\phi(z)\|^2 \Big] \\
    & \le \pi_{\theta_t}(\ymed) \Big [\pi_{\theta_t}(\ymed) \advgt(\ymed; \theta_t) \big(\|\phi(\ymed)\|^2 - \|\phi(\ymed)\|^2\pi_{\theta_t}(\ymed) - s \pi_{\theta_t}(\ystar)\big) \\
     &\hspace{20mm} - \sum\nolimits_{z \in \Ybad} \pi_{\theta_t}(z)^2 \advgt(z; \theta_t) \|\phi(z)\|^2 \Big ]\\
    & \le \pi_{\theta_t}(\ymed) \Big [\pi_{\theta_t}(\ymed) \advgt(\ymed; \theta_t) \big(\|\phi(\ymed)\|^2 - \|\phi(\ymed)\|^2\pi_{\theta_t}(\ymed) - s \pi_{\theta_t}(\ystar)\big) \\
    &\hspace{20mm} + \max\nolimits_{z \in \Ybad} \|\phi(z)\|^2(\Delta_2 - \advgt(\ymed; \theta_t) ) \sum\nolimits_{z \in \Ybad} \pi_{\theta_t}(z)^2 \Big ] \\
    & \le \pi_{\theta_t}(\ymed) \Big [\pi_{\theta_t}(\ymed) \advgt(\ymed; \theta_t) \big(\|\phi(\ymed)\|^2 - \|\phi(\ymed)\|^2\pi_{\theta_t}(\ymed) - s \pi_{\theta_t}(\ystar)\big) \\
    &\hspace{20mm} + \max\nolimits_{z \in \Ybad} \|\phi(z)\|^2(\Delta_2 - \advgt(\ymed; \theta_t) ) \pi_{\theta_t}(\Ybad)^2 \Big ] .
\end{split}
\label{eq:pi_ymed_derivative_upper_bound_neg_quantitative}
\ee
Here, the first inequality is from $s - \norm{\phi(\ystar)}^2 \pi_{\theta_t}(\ystar) - s \pi_{\theta_t}(\ymed) \le 0$ and the last inequality is from $\sum\nolimits_{z \in \Ybad} \pi_{\theta_t}(z)^2 \le \pi_{\theta_t} (\Ybad)^2$.
Our next claim is that $1 - \pi_{\theta_t}(\ymed) \ge 2 \pi_{\theta_t}(\ystar)$ for all $t \in [0, t_{\sqrt{\gamma}}]$. 
Indeed, by \zcref{lemma: max_prob_y_prime_s}, for $t \in [0, t_{\sqrt{\gamma}}]$ it holds that:
\[
1 - \pi_{\theta_t}(\ymed) \ge \frac{\sqrt{\gamma}}{1 + \gtreward(\ymed)} \ge \frac{\sqrt{\gamma}}{2} .
\]
Moreover, \zcref{assump: loglin_neg_ystar,assumpt: loglin_neg_ybad} in \zcref{assump: loglin_neg_minor} give $\pi_{\theta_t}(\ystar) \le \pi_{\theta_0}(\ystar) \le \gamma^{13/14}$. 
Under the conditions on $\gamma$ in \zcref{assump: loglin_neg_minor}, it holds that $\sqrt{\gamma} \ge 4 \gamma^{13/14}$.
Thus, $1-\pi_{\theta_t}(\ymed) \ge 2 \pi_{\theta_t}(\ystar)$, from which we conclude:
\[
\|\phi(\ymed)\|^2 - \|\phi(\ymed)\|^2\pi_{\theta_t}(\ymed) - s \pi_{\theta_t}(\ystar) \le (\|\phi(\ymed)\|^2 - s/2) (1 - \pi_{\theta_t}(\ymed)).
\]
Returning to \zcref{eq:pi_ymed_derivative_upper_bound_neg_quantitative}, this leads to:
\be
\begin{split}
    \frac{d}{dt}\pi_{\theta_t}(\ymed)
    \le{}& \pi_{\theta_t}(\ymed) \Big[(\|\phi(\ymed)\|^2-s/2)\pi_{\theta_t}(\ymed)(1 - \pi_{\theta_t}(\ymed)) \advgt(\ymed; \theta_t) \\
    &+ \max\nolimits_{z \in \Ybad} \|\phi(z)\|^2(\Delta_2 - \advgt(\ymed; \theta_t) ) \pi_{\theta_t}(\Ybad)^2 \Big]\\
    \le{}& \pi_{\theta_t}(\ymed) (1 - \pi_{\theta_t}(\ymed)) \Big[(\|\phi(\ymed)\|^2-s/2)\pi_{\theta_t}(\ymed)\advgt(\ymed; \theta_t) \\
    &+ \underbrace{\max\nolimits_{z \in \Ybad} \|\phi(z)\|^2 (\Delta_2 - \advgt(\ymed; \theta_t) ) \pi_{\theta_t}(\Ybad)}_{G} \Big]
\end{split}
\label{eq:pi_ymed_derivative_upper_bound_neg_quantitative_final}
\ee
where the second inequality is from $\pi_{\theta_t}(\Ybad) \le 1 - \pi_{\theta_t}(\ymed)$. 
We now show that $1.1 (\|\phi(\ymed)\|^2-s/2) \pi_{\theta_t}(\ymed) \advgt(\ymed; \theta_t) - G \geq 0$:
\begin{align*}
&1.1 (\|\phi(\ymed)\|^2-s/2) \pi_{\theta_t}(\ymed) \advgt(\ymed; \theta_t) - G \\
    & = 1.1 (\|\phi(\ymed)\|^2-s/2) \pi_{\theta_t}(\ymed) \advgt(\ymed; \theta_t) \\
    & \hspace{4mm} - \max\nolimits_{z \in \Ybad} \|\phi(z)\|^2 \pi_{\theta_t}(\Ybad) (\Delta_2 - \advgt(\ymed; \theta_t)) \\
    & \ge 1.1\|\phi(\ymed)\|^2 \pi_{\theta_t}(\ymed)\big(\gtreward(\ymed)-\Vgt(\theta_t)\big) \\
    & \hspace{4mm} - \max\nolimits_{z \in \Ybad} \|\phi(z)\|^2\pi_{\theta_t}(\Ybad)\big(\Vgt(\theta_t) - \gtreward(\ybad)\big)\\
    & = 1.1 \|\phi(\ymed)\|^2\pi_{\theta_t}(\ymed)\big(\pi_{\theta_t}(\Ybad) \Deltatwo - \pi_{\theta_t}(\ystar) \Delta_1 \big)\\
    &\hspace{4mm} - \max\nolimits_{z \in \Ybad} \|\phi(z)\|^2\pi_{\theta_t}(\Ybad)\big(\pi_{\theta_t}(\ymed)\Delta_2 + \pi_{\theta_t}(\ystar) (\Delta_1+\Delta_2) \big)\\
    & = \pi_{\theta_t}(\Ybad) \Big(1.1 \|\phi(\ymed)\|^2\pi_{\theta_t}(\ymed) \Deltatwo - \max\nolimits_{z \in \Ybad} \|\phi(z)\|^2\pi_{\theta_t}(\ymed) \Delta_2 \Big)\\
    &\hspace{4mm} - \pi_{\theta_t}(\ystar) \big(1.1 \|\phi(\ymed)\|^2\pi_{\theta_t}(\ymed) \Delta_1 + \max\nolimits_{z \in \Ybad} \|\phi(z)\|^2\pi_{\theta_t}(\Ybad)  (\Delta_1 + \Delta_2) \big)\\
    & \ge \underbrace{0.1 \pi_{\theta_t}(\Ybad) \pi_{\theta_t}(\ymed)\Deltatwo \|\phi(\ymed)\|^2 }_{I_1}\\
    &\hspace{4mm} - \underbrace{\pi_{\theta_t}(\ystar) \|\phi(\ymed)\|^2 \big(1.1 \pi_{\theta_t}(\ymed) \Delta_1 + \pi_{\theta_t}(\Ybad)  (\Delta_1 + \Delta_2) \big)}_{I_2}.
\end{align*}
Here, the first inequality is from $s \le 0$ and the last inequality is from \zcref{assum:feature_structure_neg_inn_prod}. 

We proceed by showing that $I_2 \le 0.1 \|\phi(\ymed)\|^2\pi_{\theta_t}(\ymed) \pi_{\theta_t}(\Ybad)\Deltatwo = I_1$.
From \zcref{assump: loglin_neg_minor}, we have $\pi_{\theta_t}(\Ybad) \le \pi_{\theta_0}(\Ybad) < 0.4 \pi_{\theta_0}(\ymed) \le 0.4 \pi_{\theta_t}(\ymed)$. 
Therefore,
\begin{align*}
    1.1 \pi_{\theta_t}(\ymed) \Delta_1 + \pi_{\theta_t}(\Ybad)  (\Delta_1 + \Delta_2) & < 1.1 \pi_{\theta_t}(\ymed) \Delta_1 + 0.4 \pi_{\theta_t}(\ymed)(\Delta_1 + \Delta_2) \\
    & \leq 1.5 \pi_{\theta_t}(\ymed)(\Delta_1 + \Delta_2) .
\end{align*}
As shown in \zcref{eq:y_bad_prob_lower_bound_t_sqrt_gamma_neg_quantitative}, $\pi_{\theta_t}(\Ybad) \ge \frac{\sqrt{\gamma}}{\Delta_2}$ for $t \in [0, t_{\sqrt{\gamma}}]$.
Furthermore, by \zcref{assump: loglin_neg_minor} we know that
\[
\gamma \le \bigg(\frac{\|\phi(\ymed)\|^2}{50 B^2(\Delta_1 + \Delta_2)}\bigg)^{7/3} ,
\]
and
\[
\pi_{\theta_0}(\ystar) \le \gamma^{13/14} \le \frac{\sqrt{\gamma} \|\phi(\ymed)\|^2}{50 B^2(\Delta_1 + \Delta_2)} .
\]
Thus, we arrive at:
\[
\begin{split}
I_2 & \le 1.5 B^2 \pi_{\theta_t}(\ymed) \pi_{\theta_0}(\ystar)(\Delta_1 + \Delta_2) \\
& \le 0.03 \|\phi(\ymed)\|^2\pi_{\theta_t}(\ymed) \pi_{\theta_t}(\Ybad)\Deltatwo \\
& \le 0.1 \|\phi(\ymed)\|^2\pi_{\theta_t}(\ymed) \pi_{\theta_t}(\Ybad)\Deltatwo \\
& = I_1 ,
\end{split}
\]
which implies that $G \leq 1.1 (\|\phi(\ymed)\|^2-s/2) \pi_{\theta_t}(\ymed) \advgt(\ymed; \theta_t)$.
Going back to \zcref{eq:pi_ymed_derivative_upper_bound_neg_quantitative_final}, the desired upper bound on $\frac{d}{dt} \pi_{\theta_t}(\ymed)$ follows:
\begin{align*}
    \frac{d}{dt}\pi_{\theta_t}(\ymed)
    &\le \pi_{\theta_t}(\ymed) (1-\pi_{\theta_t}(\ymed)) \Big[(\|\phi(\ymed)\|^2-s/2)\pi_{\theta_t}(\ymed)\advgt(\ymed; \theta_t) + G \Big] \\
    &\le \pi_{\theta_t}(\ymed) (1-\pi_{\theta_t}(\ymed)) \Big[(\|\phi(\ymed)\|^2-s/2)\pi_{\theta_t}(\ymed)\advgt(\ymed; \theta_t) \\
    &\hspace{44mm} + 1.1 (\|\phi(\ymed)\|^2-s/2)\pi_{\theta_t}(\ymed)\advgt(\ymed; \theta_t) \Big] \\
    &= 2.1 (\|\phi(\ymed)\|^2-s/2) \pi_{\theta_t}(\ymed)^2 (1-\pi_{\theta_t}(\ymed))\advgt(\ymed; \theta_t).
\end{align*}

\medskip

\textbf{Step 2.3: upper bound on $\pi_{\theta_{t_{\sqrt{\gamma}}}}(\ystar)$.}
To summarize, we showed that for all $t \in [0, t_{\sqrt{\gamma}}]$:
\begin{align*}
    \frac{d}{dt} \pi_{\theta_t}(\ystar) &\le -0.3(\|\phi(\ymed)\|^2-s) \cdot \pi_{\theta_t}(\ystar)\pi_{\theta_t}(\ymed)^2 \advgt(\ymed; \theta_t), \\
    \frac{d}{dt}\pi_{\theta_t}(\ymed) &\le 2.1(\|\phi(\ymed)\|^2-s/2) \cdot \pi_{\theta_t}(\ymed)^2(1-\pi_{\theta_t}(\ymed)) \advgt(\ymed; \theta_t).
\end{align*}
Since both sides of the first inequality are negative and both sides of the second inequality are positive, we can divide them to obtain:
\begin{align*}
    \frac{\frac{d}{dt} \pi_{\theta_t}(\ystar)}{\frac{d}{dt}\pi_{\theta_t}(\ymed)}
    &\le - \frac{\|\phi(\ymed)\|^2-s}{7(\|\phi(\ymed)\|^2-s/2)} \cdot \frac{\pi_{\theta_t}(\ystar)}{1-\pi_{\theta_t}(\ymed)} \\
    & = -h(s)\cdot \frac{\pi_{\theta_t}(\ystar)}{1-\pi_{\theta_t}(\ymed)},
\end{align*}
where recall that $h(s) = \frac{\|\phi(\ymed)\|^2-s}{7(\|\phi(\ymed)\|^2-s/2)}$ (\zcref{eq:neg_def_helper_variables}).
We can rewrite this inequality as:
\begin{align*}
    \frac{\frac{d}{dt} \pi_{\theta_t}(\ystar)}{\pi_{\theta_t}(\ystar)}
    &\le h(s)\cdot \frac{-\frac{d}{dt}\pi_{\theta_t}(\ymed)}{1-\pi_{\theta_t}(\ymed)}
    = h(s)\cdot \frac{\frac{d}{dt}(1-\pi_{\theta_t}(\ymed))}{1-\pi_{\theta_t}(\ymed)},
\end{align*}
from which it follows that:
\begin{align*}
    \frac{d}{dt}\ln \pi_{\theta_t}(\ystar) \le h(s)\cdot \frac{d}{dt} \ln (1-\pi_{\theta_t}(\ymed)).
\end{align*}
Integrating both sides over $[0, t_{\sqrt{\gamma}}]$ gives:
\begin{align*}
    \ln \pi_{\theta_t}(\ystar)\Big|_{0}^{t_{\sqrt{\gamma}}}
    \le h(s)\cdot \ln (1-\pi_{\theta_t}(\ymed))\Big|_{0}^{t_{\sqrt{\gamma}}}.
\end{align*}
Therefore,
\begin{align*}
    \frac{\pi_{\theta_{t_{\sqrt{\gamma}}}}(\ystar)}{\pi_{\theta_0}(\ystar)}
    \le \brk*{ \frac{1-\pi_{\theta_{t_{\sqrt{\gamma}}}}(\ymed)}{1-\pi_{\theta_0}(\ymed)} }^{h(s)},
\end{align*}
and hence
\begin{align}
\label{eq: ystar_upper_tsqrtgamma_loglin}
\pi_{\theta_{t_{\sqrt{\gamma}}}}(\ystar) \le \frac{\pi_{\theta_0}(\ystar) (1-\pi_{\theta_{t_{\sqrt{\gamma}}}}(\ymed))^{h(s)}}{(1-\pi_{\theta_0}(\ymed))^{h(s)}} .
\end{align}
Next, we prove that $1 - \pi_{\theta_{t_{\sqrt{\gamma}}}}(\ymed) \le \frac{\sqrt{\gamma}+ \pi_{\theta_0}(\ystar)}{\gtreward(\ymed)}$.
To see this, assume by way of contradiction that $1 - \pi_{\theta_{t_{\sqrt{\gamma}}}}(\ymed) > \frac{\sqrt{\gamma}+ \pi_{\theta_0}(\ystar)}{\gtreward(\ymed)}$.
Then, we would have:
\begin{align*}
    \Vgt(\theta_{t_{\sqrt{\gamma}}}) &= \pi_{\theta_{t_{\sqrt{\gamma}}}}(\ymed) \gtreward(\ymed) + \pi_{\theta_{t_{\sqrt{\gamma}}}}(\ystar) \gtreward(\ystar) + \pi_{\theta_{t_{\sqrt{\gamma}}}}(\Ybad) \gtreward(\ybad)\\
    &< \bigg(1 - \frac{\sqrt{\gamma}+ \pi_{\theta_0}(\ystar)}{\gtreward(\ymed)}\bigg)\gtreward(\ymed) + \pi_{\theta_0}(\ystar) \\
    & \le \gtreward(\ymed) - \sqrt{\gamma},
\end{align*}
where $\ybad \in \Ybad$, and so $\gtreward (\ybad) \le 0$.
This contradicts the definition of $t_{\sqrt{\gamma}}$ as the initial time at which $\Vgt (\theta_t) \ge \gtreward(\ymed) - \sqrt{\gamma}$. 

Plugging the upper bound on $1 - \pi_{\theta_{t_{\sqrt{\gamma}}}} (\ymed)$ into \zcref{eq: ystar_upper_tsqrtgamma_loglin} then gives:
\be
\begin{split}
    \pi_{\theta_{t_{\sqrt{\gamma}}}}(\ystar) & \le \frac{\pi_{\theta_0}(\ystar)\big(\sqrt{\gamma}+\pi_{\theta_0}(\ystar)\big)^{h(s)}}{\big[\gtreward(\ymed)(1-\pi_{\theta_0}(\ymed))\big]^{h(s)}} \\
    & \le \frac{\pi_{\theta_0}(\ystar)(1.3 \sqrt{\gamma})^{h(s)}}{(0.05 \gtreward(\ymed)^2)^{h(s)}} \\
    & = \pi_{\theta_0}(\ystar) \cdot \brk*{ \frac{ 26 \sqrt{\gamma} }{ \gtreward(\ymed)^2 }}^{h(s)} .
\end{split}
\label{eq: ystar_upper_sqrtgamma_loglin}
\ee
where the second inequality is from $1 - \pi_{\theta_0}(\ymed) \ge \pi_{\theta_0}(\Ybad) \ge 0.05 \cdot \gtreward(\ymed)$ and $\pi_{\theta_0}(\ystar) \le \gamma^{13/14} \le 0.3 \sqrt{\gamma}$ (see \zcref{assump: loglin_neg_minor}).
From \zcref{eq: ystar_upper_sqrtgamma_loglin} and \zcref{assump: loglin_neg_ystar} in \zcref{assump: loglin_neg_minor} we then get:
\begin{align}
    \pi_{\theta_{t_{\sqrt{\gamma}}}}(\ystar) &\le \pi_{\theta_0}(\ystar) \cdot (26\sqrt{\gamma}/\gtreward(\ymed)^2)^{h(s)}\nonumber\\
    &\le \frac{\|\phi(\ymed)\|^2\gamma^{13/14}\gtreward(\ymed)^{2/7}}{40B^2(\Delta_1 + \Delta_2)} \cdot \bigg(\frac{26\sqrt{\gamma}}{\gtreward(\ymed)^2}\bigg)^{h(s)}\nonumber\\
    &< \frac{0.8^2 \|\phi(\ymed)\|^2\gamma^{13/14}\gtreward(\ymed)^{2/7}}{16 \cdot 26^{1/7}B^2(\Delta_1 + \Delta_2)} \cdot \bigg(\frac{26\sqrt{\gamma}}{\gtreward(\ymed)^2}\bigg)^{h(s)}\nonumber\\
    &\le \frac{C\|\phi(\ymed)\|^2\pi_{\theta_0}(\ymed)^2\gamma^{\frac{13}{14}+\frac{h(s)}{2}}}{8(\|\phi(\ystar)\|^2-s) \advgt(\ystar; \theta_0)}.
    \label{eq: ystar_upper_sqrtgamma_final_loglin}
\end{align}
Here, the last inequality uses $h(s) = 1/7 + 13\alpha/7$, and hence $\gtreward(\ymed)^{2/7}(26/\gtreward(\ymed)^2)^{h(s)} = 26^{1/7}C$, together with $\pi_{\theta_0}(\ymed)\ge 0.8$, $\|\phi(\ystar)\|^2-s\le 2B^2$, and $\advgt(\ystar;\theta_0)\le \Delta_1+\Delta_2$.

\medskip

\textbf{Step 2.4: completing the proof of monotonicity in $[t_{\sqrt{\gamma}}, t_\gamma]$}
Finally, we conclude the proof of \zcref{prop: monotonical_loglin_neg}.
We first show that $t_{\sqrt{\gamma}} < t_\gamma$, \ie, we prove that the interval $[t_{\sqrt{\gamma}}, t_\gamma]$ is non-empty.
Recall that $h(s) = 1/7 + 13\alpha/7$ and $C = (26/\gtreward(\ymed)^2)^{13\alpha/7}$.
Therefore,
\[
\begin{split}
\frac{C\gamma^{13/14+h(s)/2}}{\sqrt{\gamma}}
= C\gamma^{3/7+h(s)/2} = \sqrt{\gamma}\brk*{\frac{26\sqrt{\gamma}}{\gtreward(\ymed)^2}}^{13\alpha/7}.
\end{split}
\]
Moreover, by \zcref{assumpt: loglin_neg_ybad} in \zcref{assump: loglin_neg_minor},
\[
    \frac{2\sqrt{\gamma}}{\Delta_2}
    \le \pi_{\theta_0}(\Ybad)
    \le 0.1\gtreward(\ymed),
\]
and since $\Delta_2 = \gtreward(\ymed)-\gtreward(\ybad) \le \gtreward(\ymed)+1 \le 2$, for $\ybad \in \Ybad$, we get $\sqrt{\gamma}\le 0.1\gtreward(\ymed)$.
Using also the fact that $13\alpha/7 \in (0,1/7)$ and $\gtreward(\ymed)\le 1$, this yields
\[
\begin{split}
    \frac{C\gamma^{13/14+h(s)/2}}{\sqrt{\gamma}}
    &\le 0.1\gtreward(\ymed)
    \brk*{\frac{2.6}{\gtreward(\ymed)}}^{13\alpha/7} \\
    &= 0.1\cdot 2.6^{13\alpha/7}
    \gtreward(\ymed)^{1-13\alpha/7}
    < 1 .
\end{split}
\]
Hence, $C\gamma^{13/14+h(s)/2} < \sqrt{\gamma}$.
Thus, the threshold $\gtreward(\ymed)-C\gamma^{13/14+h(s)/2}$ is strictly larger than $\gtreward(\ymed)-\sqrt{\gamma}$, and by continuity of $\Vgt(\theta_t)$ together with the definitions of $t_{\sqrt{\gamma}}$ and $t_\gamma$, it follows that $t_{\sqrt{\gamma}} < t_\gamma$.

Now, we can consider the optimization dynamics over $[t_{\sqrt{\gamma}}, t_\gamma]$ and show that $\pi_{\theta_t}(\ystar)$ continues to decrease during this time interval.
Assume by way of contradiction that there exists $t^\prime \in [0, t_\gamma]$ such that $\left.\frac{d}{dt} \pi_{\theta_t}(\ystar)\right|_{t=t^\prime} \ge 0$.
Let $\tau$ be the initial time in $[0, t_\gamma]$ at which $\frac{d}{dt} \pi_{\theta_t}(\ystar) \geq 0$, \ie:
\[
 \tau := \min \brk[c]*{ t \in [0, t_\gamma] : \tfrac{d}{dt} \pi_{\theta_t}(\ystar) \ge 0 } .
\]
By definition, $\frac{d}{dt} \pi_{\theta_t}(\ystar) < 0$, and so $\frac{d}{dt} \pi_{\theta_t} (\ymed) > 0$ (\zcref{lemma: conditional_monotonicity_loglin_neg}), for all $t \in [0, \tau)$.
We also know that if $\tau$ exists, then $\tau \ge t_{\sqrt{\gamma}}$ since $\frac{d}{dt} \pi_{\theta_t}(\ystar) < 0$ for all $t \in [0, t_{\sqrt{\gamma}}]$.
This implies that $\pi_{\theta_\tau} (\ystar) \leq \pi_{\theta_{t_{\sqrt{\gamma}}}}(\ystar)$ and $\pi_{\theta_\tau}(\ymed) \ge \pi_{\theta_0}(\ymed)$.

Now, by the same arguments made in \hyperref[proof:monotonical_loglin_neg:step1]{Step 1}, it follows that $\frac{d}{dt} \pi_{\theta_t}(\Ybad) \le 0$ for all $t \in [0, \tau)$.
We continue by showing that $\pi_{\theta_\tau}(\Ybad)^2 \le \frac{7\|\phi(\ymed)\|^2\pi_{\theta_0}(\ymed)^2 \advgt(\ymed; \theta_\tau)}{8\max\nolimits_{z \in \Ybad} \|\phi(z)\|^2\Delta_2}$.
To see this, notice that:
\begin{align*}
    \advgt(\ymed; \theta_\tau) &= (1 - \pi_{\theta_\tau}(\ymed)) \gtreward(\ymed) - \pi_{\theta_\tau}(\ystar) \gtreward(\ystar) - \sum\nolimits_{z \in \Ybad} \pi_{\theta_\tau}(z) \gtreward(z)\\
    &= -\Delta_1 \pi_{\theta_\tau}(\ystar) + \Deltatwo \pi_{\theta_\tau}(\Ybad) .
\end{align*}
Thus,
\begin{align*}
    \pi_{\theta_\tau}(\Ybad) = \frac{\advgt(\ymed; \theta_\tau) + \Delta_1 \pi_{\theta_\tau}(\ystar)}{\Deltatwo} \le \frac{1.3 \advgt(\ymed; \theta_\tau)}{\Deltatwo},
\end{align*}
where the inequality is from \zcref{assump: loglin_neg_ystar} in \zcref{assump: loglin_neg_minor} since $\pi_{\theta_0}(\ystar) \le \frac{\|\phi(\ymed)\|^2\gtreward(\ymed)^{2/7}\gamma^{13/14}}{40B^2(\Delta_1+\Delta_2)} \le \frac{\gtreward(\ymed)^{2/7}\gamma^{13/14}}{40(\Delta_1+\Delta_2)} \le \frac{0.3\gamma^{13/14}}{\Delta_1(26/\gtreward(\ymed)^2)^{1/7}}$, which combined with \zcref{eq: ystar_upper_sqrtgamma_loglin} yields $\Delta_1 \pi_{\theta_\tau}(\ystar) \le \Delta_1 \pi_{\theta_{t_{\sqrt{\gamma}}}}(\ystar) \le \Delta_1 \pi_{\theta_0}(\ystar) \cdot (26 \sqrt{\gamma}/\gtreward(\ymed)^2)^{h(s)}\le 0.3 C\gamma^{13/14+h(s)/2} \le 0.3 \advgt(\ymed; \theta_\tau)$.

Squaring both sides of the inequality, we then have:
\begin{align*}
    \pi_{\theta_\tau}(\Ybad)^2 &\le \frac{1.69 \advgt(\ymed; \theta_\tau)^2}{\Deltatwo^2}\\
    &< \frac{7\|\phi(\ymed)\|^2\pi_{\theta_0}(\ymed)^2 \advgt(\ymed; \theta_\tau)}{8\max\nolimits_{z \in \Ybad} \|\phi(z)\|^2 \Delta_2} \cdot \frac{2 \max\nolimits_{z \in \Ybad} \|\phi(z)\|^2 \advgt(\ymed; \theta_\tau)}{\|\phi(\ymed)\|^2\Deltatwo \pi_{\theta_0}(\ymed)^2}\\
    &\le \frac{7\|\phi(\ymed)\|^2\pi_{\theta_0}(\ymed)^2 \advgt(\ymed; \theta_\tau)}{8\max\nolimits_{z \in \Ybad} \|\phi(z)\|^2 \Delta_2} \cdot \frac{2 \advgt(\ymed; \theta_0)}{\Deltatwo\pi_{\theta_0}(\ymed)^2}.
\end{align*}
The second inequality is from $1.69 \times 8 < 2 \times 7$ and the last inequality is by $\max\nolimits_{z \in \Ybad}\|\phi(z)\|^2 \le \|\phi(\ymed)\|^2$ (\zcref{assum:feature_structure_neg_inn_prod}) and $\advgt(\ymed; \theta_\tau) \le \advgt(\ymed; \theta_0)$.
We now prove that the second term on the right-hand side of the inequality above is at most $1$.
To do so, we upper bound $\advgt(\ymed; \theta_0)$:
\begin{align*}
    \advgt(\ymed; \theta_0) &= (1 - \pi_{\theta_0}(\ymed))\gtreward(\ymed) - \pi_{\theta_0}(\ystar) \gtreward(\ystar) - \sum\nolimits_{z \in \Ybad}\pi_{\theta_0}(z) \gtreward(z)\\
    &\le (1 - \pi_{\theta_0}(\ymed))\gtreward(\ymed) + \pi_{\theta_0}(\Ybad)\\
    &\le 0.2 \gtreward(\ymed) + 0.1 \gtreward(\ymed)\\
    &= 0.3 \gtreward(\ymed)\\
    &< \frac{ 0.8^2 \gtreward(\ymed)}{2}\\
    &\le \frac{\Deltatwo \pi_{\theta_0}(\ymed)^2}{2}.
\end{align*}
The first inequality is due to $\gtreward(z) \ge -1$ for all $z \in \Ybad$; the second inequality is from \zcref{assump: loglin_neg_minor} since $\pi_{\theta_0}(\Ybad) \le 0.1\gtreward(\ymed) \le 0.1$ and $\pi_{\theta_0}(\ymed) \ge 0.8$; and the last inequality is from $\gtreward(\ybad) \le 0$ and $\pi_{\theta_0}(\ymed) \ge 0.8$. 
Therefore, $\frac{2 \advgt(\ymed; \theta_0)}{\Deltatwo \pi_{\theta_0}(\ymed)^2} \le 1$, which leads to 
\begin{equation}
\label{eq: ybad2_upper3}
    \pi_{\theta_\tau}(\Ybad)^2 \le \frac{7\|\phi(\ymed)\|^2\pi_{\theta_0}(\ymed)^2 \advgt(\ymed; \theta_\tau)}{8\max\nolimits_{z \in \Ybad} \|\phi(z)\|^2\Delta_2}.
\end{equation}

Let us now consider $\left.\frac{d}{dt} \pi_{\theta_t}(\ystar)\right|_{t=\tau}$ and obtain a contradiction by showing that it is negative.
We start with the following computations, based on the expression derived in \zcref{lem:probability_derivative_expression}:
\begin{align*}
        \left.\frac{d}{dt} \pi_{\theta_t}(\ystar)\right|_{t=\tau} & = \pi_{\theta_\tau}(\ystar) \Big[ \pi_{\theta_\tau}(\ystar) \advgt(\ystar; \theta_\tau) \big(\|\phi(\ystar)\|^2 - \|\phi(\ystar)\|^2\pi_{\theta_\tau}(\ystar) - s \pi_{\theta_\tau}(\ymed)\big)\\
        &\hspace{18mm} + \pi_{\theta_\tau}(\ymed) \advgt(\ymed; \theta_\tau) \big(s - \|\phi(\ymed)\|^2\pi_{\theta_\tau}(\ymed) - s \pi_{\theta_\tau}(\ystar) \big) \\
        & \hspace{18mm} - \sum\nolimits_{z \in \Ybad} \pi_{\theta_\tau}(z)^2 \advgt(z; \theta_\tau)\|\phi(z)\|^2 \Big]\\
        &\le \pi_{\theta_\tau}(\ystar) \Big[ \pi_{\theta_\tau}(\ystar) \advgt(\ystar; \theta_\tau) \big ( \|\phi(\ystar)\|^2 - \|\phi(\ystar)\|^2\pi_{\theta_\tau}(\ystar) - s \pi_{\theta_\tau}(\ymed) \big ) \\
        &\hspace{18mm} + \pi_{\theta_\tau}(\ymed) \advgt(\ymed; \theta_\tau) \big ( s - \|\phi(\ymed)\|^2\pi_{\theta_\tau}(\ymed) - s \pi_{\theta_\tau}(\ystar) \big) \\
        & \hspace{18mm} + \max\nolimits_{z \in \Ybad} \|\phi(z)\|^2 \big (\Delta_2 - \gamma^{\frac{13}{14} + \frac{h(s)}{2}} \big ) \sum\nolimits_{z \in \Ybad}\pi_{\theta_\tau}(z)^2  \Big]\\
        &\leq \pi_{\theta_\tau}(\ystar) \Big[ (\|\phi(\ystar)\|^2-s) \pi_{\theta_\tau}(\ystar) \advgt(\ystar; \theta_\tau) \big ( 1 - \pi_{\theta_\tau}(\ystar)\big ) \\
        &\hspace{18mm} - \|\phi(\ymed)\|^2\pi_{\theta_\tau}(\ymed)^2 \advgt(\ymed; \theta_\tau) \\
        &\hspace{18mm} + \max\nolimits_{z \in \Ybad} \|\phi(z)\|^2(\Delta_2 - \gamma^{\frac{13}{14} + \frac{h(s)}{2}}) \sum\nolimits_{z \in \Ybad} \pi_{\theta_\tau}(z)^2 \Big].
\end{align*}
The first inequality is from $-\advgt(z;\theta_\tau)=\Delta_2-\advgt(\ymed;\theta_\tau)\le \Delta_2-\gamma^{\frac{13}{14}+\frac{h(s)}{2}}$ for all $z\in\Ybad$ (this holds since $\tau\le t_\gamma$ and $C\ge 1$) and the second inequality is from $s \le 0$ and $\pi_{\theta_\tau}(\ymed) \le 1 - \pi_{\theta_\tau}(\ystar)$.
We continue upper bounding $\left.\frac{d}{dt} \pi_{\theta_t}(\ystar)\right|_{t=\tau}$ as follows:
\begin{align*}
        \left.\frac{d}{dt} \pi_{\theta_t}(\ystar)\right|_{t=\tau} &\leq \pi_{\theta_\tau}(\ystar) \Big[ (\|\phi(\ystar)\|^2-s) \pi_{\theta_{t_{\sqrt{\gamma}}}}(\ystar) \advgt(\ystar; \theta_0) \big(1 - \pi_{\theta_{t_{\sqrt{\gamma}}}}(\ystar)\big) \\
        &\hspace{18mm} - \|\phi(\ymed)\|^2\pi_{\theta_\tau}(\ymed)^2 \advgt(\ymed; \theta_\tau) \\
        &\hspace{18mm} + \max\nolimits_{z \in \Ybad} \|\phi(z)\|^2 \big ( \Delta_2 - \gamma^{\frac{13}{14} + \frac{h(s)}{2}} \big)  \pi_{\theta_\tau}(\Ybad)^2 \big]\\
        &\le \pi_{\theta_\tau}(\ystar) \Big[ (\|\phi(\ystar)\|^2-s) \pi_{\theta_{t_{\sqrt{\gamma}}}}(\ystar) \advgt(\ystar; \theta_0) \big(1 - \pi_{\theta_{t_{\sqrt{\gamma}}}}(\ystar)\big) \\
        & \hspace{18mm} - (\|\phi(\ymed)\|^2/8) \pi_{\theta_\tau}(\ymed)^2 \advgt(\ymed; \theta_\tau) \Big]\\
        &\le \pi_{\theta_\tau}(\ystar) \Big[ (\|\phi(\ystar)\|^2-s) \pi_{\theta_{t_{\sqrt{\gamma}}}}(\ystar) \advgt(\ystar; \theta_0) \big(1 - \pi_{\theta_{t_{\sqrt{\gamma}}}}(\ystar)\big) \\
        &\hspace{18mm} - (C\|\phi(\ymed)\|^2/8) \pi_{\theta_\tau}(\ymed)^2 \gamma^{\frac{13}{14} + \frac{h(s)}{2}} \Big].
\end{align*}
The first inequality is from $\advgt(\ystar; \theta_\tau) \le \advgt(\ystar; \theta_0)$, $\pi_{\theta_\tau}(\ystar) \le \pi_{\theta_{t_{\sqrt{\gamma}}}}(\ystar)$, $\pi_{\theta_\tau}(\ymed) \ge \pi_{\theta_{t_{\sqrt{\gamma}}}}(\ymed)$, and $\sum_{z \in \Ybad} \pi_{\theta_\tau}(z)^2 \le \pi_{\theta_\tau} (\Ybad)^2$;
the second inequality is from
\[
\begin{split}
\max\nolimits_{z \in \Ybad}\|\phi(z)\|^2(\Delta_2 - \gamma^{\frac{13}{14}+\frac{h(s)}{2}}) \pi_{\theta_\tau}(\Ybad)^2 & \le \max\nolimits_{z \in \Ybad}\|\phi(z)\|^2\Delta_2 \pi_{\theta_\tau}(\Ybad)^2 \\
& \le \frac{7\|\phi(\ymed)\|^2\pi_{\theta_\tau}(\ymed)^2 \advgt(\ymed;\theta_\tau)}{8} ,
\end{split}
\]
which is implied by \zcref{eq: ybad2_upper3}, and $\pi_{\theta_\tau}(\ymed)\ge\pi_{\theta_0}(\ymed)$;
and the last inequality is from $\advgt(\ymed; \theta_\tau) \ge \advgt(\ymed; \theta_{t_\gamma}) = C \gamma^{\frac{13}{14} + \frac{h(s)}{2}}$.

Therefore, using the upper bound on $\pi_{\theta_{t_{\sqrt{\gamma}}}} (\ystar)$ from \zcref{eq: ystar_upper_sqrtgamma_final_loglin}:
\[
\begin{split}
& \left.\frac{d}{dt} \pi_{\theta_t}(\ystar)\right|_{t=\tau} \\
&\le \pi_{\theta_\tau}(\ystar) \Big[ (\|\phi(\ystar)\|^2-s) \pi_{\theta_{t_{\sqrt{\gamma}}}}(\ystar) \advgt(\ystar; \theta_0) \big(1 - \pi_{\theta_{t_{\sqrt{\gamma}}}}(\ystar)\big) \\
&\hspace{18mm} - (C\|\phi(\ymed)\|^2/8) \pi_{\theta_\tau}(\ymed)^2 \gamma^{\frac{13}{14} + \frac{h(s)}{2}} \Big] \\
&< \pi_{\theta_\tau}(\ystar) \Big[ (C\|\phi(\ymed)\|^2/8) \pi_{\theta_0}(\ymed)^2\gamma^{\frac{13}{14}+\frac{h(s)}{2}} - (C\|\phi(\ymed)\|^2/8) \pi_{\theta_\tau}(\ymed)^2 \gamma^{\frac{13}{14}+\frac{h(s)}{2}} \Big] \\
&= \frac{C\pi_{\theta_\tau}(\ystar)\|\phi(\ymed)\|^2\gamma^{\frac{13}{14}+\frac{h(s)}{2}}}{8}\Big(\pi_{\theta_0}(\ymed)^2-\pi_{\theta_\tau}(\ymed)^2\Big) \\
& \le 0,
\end{split}
\]
where the last inequality is from $\pi_{\theta_\tau}(\ymed) \ge \pi_{\theta_0}(\ymed)$.
Thus, the strict inequality above implies $\left.\frac{d}{dt}\pi_{\theta_t}(\ystar)\right|_{t=\tau}<0$, which contradicts the fact that $\tau$ is the initial time at which $\left.\frac{d}{dt}\pi_{\theta_t}(\ystar)\right|_{t=\tau} \geq 0$.
Hence, it must be that $\frac{d}{dt} \pi_{\theta_t}(\ystar) < 0$ for all $t \in [0, t_\gamma]$.
Additionally, by \zcref{lemma: conditional_monotonicity_loglin_neg}, we also have $\frac{d}{dt} \pi_{\theta_t}(\ymed) > 0$ for all $t \in [0, t_\gamma]$.

To conclude, we showed that $\frac{d}{dt} \pi_{\theta_t} (\ystar) < 0$ and $\frac{d}{dt} \pi_{\theta_t}(\ymed) > 0$ for all $t \in [0, t_\gamma]$.
Thus, $\min\nolimits_{t \in [0, t_\gamma]} \pi_{\theta_t} (\ystar) = \pi_{\theta_{t_\gamma}}(\ystar)$.
Furthermore, since $t_{\sqrt{\gamma}} < t_\gamma$, by \zcref{eq: ystar_upper_sqrtgamma_final_loglin} it follows that
\[
\pi_{\theta_{t_\gamma}} (\ystar) < \pi_{\theta_{t_{\sqrt{\gamma}}}}(\ystar) \le \frac{C\|\phi(\ymed)\|^2\pi_{\theta_0}(\ymed)^2\gamma^{\frac{13}{14}+\frac{h(s)}{2}}}{8(\|\phi(\ystar)\|^2-s) \advgt(\ystar; \theta_0)} ,
\]
completing the proof of \zcref{prop: monotonical_loglin_neg}.
\end{proof}

\medskip

Based on \zcref{prop: monotonical_loglin_neg}, we can now complete the proof of Case I.

\begin{proof}[Proof of \zcref{thm:attraction_to_mediocre_outputs_beneficial_error_neg_inner_product}, Case I]
     By \zcref{prop: monotonical_loglin_neg}, as long as $\Vgt(\theta_t) \le \gtreward(\ymed) - C\gamma^{\frac{13}{14}+\frac{h(s)}{2}}$, we know that $\pi_{\theta_t}(\ystar)$ is monotonically decreasing and $\pi_{\theta_t}(\ymed)$ is monotonically increasing.
     Recall that $t_\gamma$ denotes the initial time at which $\Vgt(\theta_t) \ge \gtreward(\ymed) - C\gamma^{\frac{13}{14}+\frac{h(s)}{2}}$.
     Since $\Vgt (\theta_t)$ is continuous in $t$, at time $t_\gamma$ it must be that $\Vgt(\theta_{t_\gamma}) = \gtreward(\ymed) - C\gamma^{\frac{13}{14}+\frac{h(s)}{2}}$ and $t_\gamma$ is the initial time at which this equality holds.

     We first prove that at $t_\gamma$ the policy is highly concentrated on $\ymed$ by upper bounding $1 - \pi_{\theta_{t_\gamma}}(\ymed)$.
     This follows from
    \[
    \begin{split}
        \gtreward(\ymed) - C\gamma^{\frac{13}{14}+\frac{h(s)}{2}} & = \Vgt(\theta_{t_\gamma}) \\
        & = \gtreward(\ymed) \pi_{\theta_{t_\gamma}}(\ymed) + \pi_{\theta_{t_\gamma}}(\ystar) \gtreward(\ystar) + \sum\nolimits_{z \in \Ybad}\pi_{\theta_{t_\gamma}}(z) \gtreward(z)  ,
    \end{split}
    \]
    which implies:
    \be
    \begin{split}
        1 - \pi_{\theta_{t_\gamma}}(\ymed) &= \frac{1}{\gtreward(\ymed)} \brk*{ C\gamma^{\frac{13}{14}+\frac{h(s)}{2}} + \pi_{\theta_{t_\gamma}}(\ystar)\gtreward(\ystar) + \sum_{z \in \Ybad}\pi_{\theta_{t_\gamma}}(z) \gtreward(z) } \\
        &\le \frac{1}{\gtreward(\ymed)} \brk*{ C\gamma^{\frac{13}{14}+\frac{h(s)}{2}} + \frac{C\|\phi(\ymed)\|^2\pi_{\theta_0}(\ymed)^2\gamma^{\frac{13}{14}+\frac{h(s)}{2}}}{8(\|\phi(\ystar)\|^2-s)\advgt(\ystar;\theta_0)} } \\
        &= \frac{K\gamma^{\frac{13}{14}+\frac{h(s)}{2}}}{\gtreward(\ymed)} .
    \end{split}
    \label{eq: ymed_upper_tgamma_loglin}
    \ee
    The inequality is due to $\pi_{\theta_{t_\gamma}}(\ystar) \le \frac{C\|\phi(\ymed)\|^2\pi_{\theta_0}(\ymed)^2\gamma^{\frac{13}{14}+\frac{h(s)}{2}}}{8(\|\phi(\ystar)\|^2-s)\advgt(\ystar; \theta_0)}$ from \zcref{prop: monotonical_loglin_neg}, together with $\gtreward(\ystar)\le 1$ and $\sum_{z \in \Ybad}\pi_{\theta_{t_\gamma}}(z)\gtreward(z)\le 0$, and the last equality is by the definition of $K$ in \zcref{thm:attraction_to_mediocre_outputs_beneficial_error_neg_inner_product}.

    Now, the fact that the policy assigns high probability to $\ymed$ at $t_\gamma$ allows us to lower bound $\teps$---the initial time at which $\Vgt (\theta_t) \geq \gtreward (\ystar) - \epsilon$.
    For any time $T \ge t_\gamma$, we have:
    \be
    \begin{split}
        \|\theta_T - \theta_{t_\gamma}\| &= \norm*{ \int_{t_\gamma}^T \frac{d}{dt} \theta_t dt } \\
        &\le \int_{t_\gamma}^T \norm*{ \frac{d}{dt} \theta_t } dt\\
        &= \int_{t_\gamma}^T \norm*{ \nabla \Vgt(\theta_t) } dt\\
        &\le \int_{t_\gamma}^T 2B(\Delta_1+\Delta_2)(1 - \pi_{\theta_t}(\ymed)) dt , 
    \end{split}
    \label{eq: theta_upper_neg_loglin}
    \ee
    where the last inequality is by \zcref{lemma:l2_grad_bound}.
    We proceed by upper bounding the rate at which $1 - \pi_{\theta_t}(\ymed)$ increases (\ie, the rate at which $\pi_{\theta_t}(\ymed)$ decreases).
    From the chain rule we obtain:
    \begin{align*}
        \frac{d}{dt}(1 - \pi_{\theta_t}(\ymed)) &= \inprodBig{\nabla (1 - \pi_{\theta_t}(\ymed))}{\frac{d}{dt} \theta_t}\\
        &= \inprod{\nabla (1 - \pi_{\theta_t}(\ymed))}{\nabla \Vgt(\theta_t)}\\
        &= \inprod{\nabla(\pi_{\theta_t}(\ystar) + \sum\nolimits_{z \in \Ybad} \pi_{\theta_t}(z))}{\nabla \Vgt(\theta_t)}.
    \end{align*}
    For all $y \in \Y$, the gradient of $\pi_{\theta_t}(y)$ with respect to $\theta_t$ can be computed as follows:
    \begin{align*}
        \nabla \pi_{\theta_t}(y) = \pi_{\theta_t}(y) \brk*{\phi(y) - \sum\nolimits_{z \in \Y}\pi_{\theta_t}(z)\phi(z)} = \pi_{\theta_t}(y)\brk*{\phi(y) - \bar{\phi}_{\theta_t}} ,
    \end{align*}
    where $\bar{\phi}_{\theta_t} := \sum\nolimits_{z \in \Y}\pi_{\theta_t}(z)\phi(z)$.
    This yields:
    \begin{align*}
        \frac{d}{dt}(1 - \pi_{\theta_t}(\ymed)) &\le \sum\nolimits_{y \in \Y\setminus\{\ymed\}} \pi_{\theta_t}(y) \abs*{ \inprod{\phi(y) - \bar{\phi}_{\theta_t}}{\nabla \Vgt(\theta_t)} }\\
        &\le \sum\nolimits_{y \in \Y\setminus\{\ymed\}} \pi_{\theta_t}(y) \|\phi(y) - \bar{\phi}_{\theta_t}\| \norm*{ \nabla \Vgt(\theta_t) } \\
        &\le \max\nolimits_{y \in \Y\setminus\{\ymed\}} \|\phi(y) - \bar{\phi}_{\theta_t}\| \norm*{ \nabla \Vgt(\theta_t) } \cdot \sum\nolimits_{y \in \Y\setminus\{\ymed\}} \pi_{\theta_t}(y) \\
        &\le 4B^2 (\Delta_1+\Delta_2) (1 - \pi_{\theta_t}(\ymed))^2.
    \end{align*}
    Here, the second inequality is due to the Cauchy-Schwarz inequality; the third inequality is due to $\|\phi(z) - \bar{\phi}_{\theta_t}\| \leq \max\nolimits_{y \in \Y\setminus\{\ymed\}} \|\phi(y) - \bar{\phi}_{\theta_t}\|$ for all $z \in \Y \setminus \{ \ymed \}$; and the last inequality is from $\max\nolimits_{y \in \Y\setminus\{\ymed\}} \|\phi(y) - \bar{\phi}_{\theta_t}\| \leq 2B$ and \zcref{lemma:l2_grad_bound}.
    
    This implies that:
    \[
      \frac{d}{dt} \brk*{ \frac{1}{1 - \pi_{\theta_t} (\ymed)} }  = - \frac{ \frac{d}{dt} (1 - \pi_{\theta_t} (\ymed)) }{ (1 - \pi_{\theta_t} (\ymed)) ^2 } \geq -4B^2 (\Delta_1 + \Delta_2) .
    \]
    Integrating both sides from $t_\gamma$ to $t$, for any $t \geq t_\gamma$, leads to:
    \[
        \frac{1}{1 - \pi_{\theta_t} (\ymed)} - \frac{1}{1 - \pi_{\theta_{t_\gamma}} (\ymed)} \geq -4B^2 (\Delta_1 + \Delta_2) (t - t_\gamma) .
    \]
    Hence, for any $t \geq t_\gamma$, such that
    \[
        4B^2 (\Delta_1 + \Delta_2) (1 - \pi_{\theta_{t_\gamma}} (\ymed)) (t - t_\gamma) < 1 ,
    \]
    it holds that
    \[
    1 - \pi_{\theta_t}(\ymed) \le \frac{1-\pi_{\theta_{t_\gamma}}(\ymed)}{1 - 4B^2(\Delta_1 + \Delta_2)(1 - \pi_{\theta_{t_\gamma}}(\ymed)) (t - t_\gamma)} .
    \]
    Thus, for any $T \geq t_\gamma$ satisfying $4B^2 (\Delta_1 + \Delta_2) (1 - \pi_{\theta_{t_\gamma}} (\ymed)) (T - t_\gamma) < 1$, we can plug this upper bound into \zcref{eq: theta_upper_neg_loglin} to get:
    \be
    \begin{split}
        \|\theta_T - \theta_{t_\gamma}\| &\le 2B(\Delta_1 + \Delta_2) \int_{t_\gamma}^T (1 - \pi_{\theta_t}(\ymed))dt\\
        &= 2B(\Delta_1 + \Delta_2) \int_0^{T-t_\gamma} \frac{1 - \pi_{\theta_{t_\gamma}}(\ymed)}{1 - 4B^2(\Delta_1 + \Delta_2)(1 - \pi_{\theta_{t_\gamma}}(\ymed)) u} du\\
        &= \frac{1}{2B} \ln\bigg(\frac{1}{1-4B^2(\Delta_1 + \Delta_2)(1 - \pi_{\theta_{t_\gamma}}(\ymed)) (T - t_\gamma)}\bigg) .
    \end{split}
    \label{eq:bound_dist_V_T_V_t_gamma_neg}
    \ee
    By \zcref{lemma:vgt_lipschitzness}, $\Vgt$ is $B$-Lipschitz with respect to $\theta$, and so:
    \begin{align*}
        |\Vgt(\theta_T) - \Vgt(\theta_{t_{\gamma}})| \le B \|\theta_T - \theta_{t_{\gamma}}\| \le \frac{1}{2} \ln\bigg(\frac{1}{1-4B^2(\Delta_1 + \Delta_2)(1 - \pi_{\theta_{t_\gamma}}(\ymed)) (T - t_\gamma)}\bigg).
    \end{align*}
    If $\teps = \infty$, then the desired lower bound holds trivially.
    Hence, assume $\teps < \infty$.
    By the definition of $\teps$ and continuity of $\Vgt(\theta_t)$, $\Vgt(\theta_{\teps}) = \gtreward(\ystar) - \epsilon$.
    In particular, $\teps > t_\gamma$ since $\Vgt (\theta_{\teps}) = \gtreward (\ystar) - \epsilon > \gtreward (\ymed) > \Vgt (\theta_{t_\gamma})$.
    Therefore,
    \[
        \Vgt (\theta_{\teps}) - \Vgt(\theta_{t_{\gamma}}) = \gtreward (\ystar) - \epsilon - \gtreward (\ymed) + C\gamma^{\frac{13}{14}+\frac{h(s)}{2}} \ge \Delta_1 - \epsilon .
    \]
    If $4B^2 (\Delta_1 + \Delta_2) (1 - \pi_{\theta_{t_\gamma}} (\ymed)) (\teps - t_\gamma) \geq 1$, then directly
    \[
       \teps - t_\gamma \geq \frac{ 1 }{ 4B^2 (\Delta_1 + \Delta_2) (1 - \pi_{\theta_{t_\gamma}} (\ymed)) }  \geq  \frac{ 1 - e^{-2 (\Delta_1 - \epsilon)} }{ 4B^2 (\Delta_1 + \Delta_2) (1 - \pi_{\theta_{t_\gamma}} (\ymed)) } ,
    \]
    since $1 - e^{-2 (\Delta_1 - \epsilon)} \leq 1$.
    Otherwise, we apply the bound in \zcref{eq:bound_dist_V_T_V_t_gamma_neg} to get:
    \[
        \Delta_1 - \epsilon \le \Vgt (\theta_{\teps}) - \Vgt(\theta_{t_{\gamma}}) \le \frac{1}{2} \ln\bigg(\frac{1}{1-4B^2(\Delta_1 + \Delta_2)(1 - \pi_{\theta_{t_\gamma}}(\ymed)) (\teps - t_\gamma)}\bigg) ,
    \]
    which implies
    \[
        \teps - t_\gamma \ge \frac{1 - e^{-2(\Delta_1-\epsilon)}}{4B^2(\Delta_1 + \Delta_2) (1 - \pi_{\theta_{t_{\gamma}}}(\ymed))} .
    \]
    Therefore, in either case:
    \[
    \begin{split}
        \teps & \geq \teps - t_\gamma \\
        & \geq \frac{1 - e^{-2(\Delta_1-\epsilon)}}{4B^2(\Delta_1 + \Delta_2) (1 - \pi_{\theta_{t_{\gamma}}}(\ymed))} \\
        &\ge \frac{\gtreward(\ymed)(1 - e^{-2(\Delta_1-\epsilon)})}{4B^2(\Delta_1 + \Delta_2) K\gamma^{\frac{13}{14}+\frac{h(s)}{2}}} \\
        &= \frac{\gtreward(\ymed)M'^{14/13 + \alpha}(1 - e^{-2(\Delta_1-\epsilon)})}{4B^2(\Delta_1 + \Delta_2) K} \cdot \pi_{\theta_0} (\ystar)^{- 14 / 13 - \alpha } \\
        & = \Omega \brk*{ \pi_{\theta_0} (\ystar)^{-14 / 13 - \alpha} } .
    \end{split}
    \]
    where the third inequality follows from \zcref{eq: ymed_upper_tgamma_loglin}.
    This completes the proof for Case I of \zcref{thm:attraction_to_mediocre_outputs_beneficial_error_neg_inner_product}.
\end{proof}

\subsubsection{Proof of Case II}

Suppose that gradient flow is used to maximize the expected reward with respect to $\proxyreward$, which assigns $\ymed$ a low reward $\proxyreward (\ymed) = \min\nolimits_{y \in \Ybad} \gtreward (y)$.
Similarly to the proof of Case II in \zcref{thm:attraction_to_mediocre_outputs_beneficial_error_formal}, we show that in this case $\pi_{\theta_t} (\ystar)$ increases already from $t = 0$ and that $\tepsnomed = \OO (\pi_{\theta_0} (\ystar)^{-1})$, where recall that $\tepsnomed$ is the initial time at which $\Vgt (\theta_t) \geq \gtreward (\ystar) - \epsilon$ when maximizing $\proxyreward$.

\begin{proof}[Proof of \zcref{thm:attraction_to_mediocre_outputs_beneficial_error_neg_inner_product}, Case II]
Notice that $\Vgt(\theta) \ge \Vproxy(\theta)$ for all $\theta \in \R^D$ since $\gtreward(y) \ge \proxyreward(y)$ for all $y \in \Y$.
Thus, when $\Vproxy(\theta_t) \ge \gtreward(\ystar) - \epsilon$ it also holds that $\Vgt(\theta_t) \ge \gtreward(\ystar) - \epsilon$. 
We can therefore focus on $\Vproxy$ and consider the time it takes until it reaches a value of $\gtreward(\ystar) - \epsilon$.
In particular, denote by $\tepsnomedP$ the initial time at which $\Vproxy(\theta_t) \ge \gtreward(\ystar) - \epsilon$, \ie:
\[
\tepsnomedP := \min \brk[c]*{ t \geq 0 :\Vproxy (\theta_t) \ge \gtreward (\ystar) - \epsilon }.
\]
The discussion above implies that $\tepsnomed \le \tepsnomedP$, and so it suffices to upper bound $\tepsnomedP$.
For simplicity of notation, denote $\rho := \frac{\gtreward(\ystar) + 1 - \epsilon}{\gtreward(\ystar) + 1}$. 
Observe that when $\pi_{\theta_t}(\ystar) \ge \rho$, it holds that $\Vproxy (\theta_t) \ge \gtreward (\ystar) - \epsilon$. 
This is because:
\begin{align*}
    \Vproxy (\theta_t) &\ge \pi_{\theta_t}(\ystar) \proxyreward(\ystar) - (1 - \pi_{\theta_t}(\ystar))\\
    &= \pi_{\theta_t}(\ystar) (\gtreward(\ystar) + 1) -1\\
    &\ge \gtreward(\ystar) - \epsilon .
\end{align*}
Thus, $\pi_{\theta_t}(\ystar) \leq \rho$ for all $t \in [0, \tepsnomedP]$.

We prove that
\[
T:= \frac{1}{(1-\rho)^2 \advproxy(\ystar; \theta_0)\|\phi(\ystar)\|^2}\bigg(\frac{1}{\pi_{\theta_0}(\ystar)} - \frac{1}{\rho}\bigg) \geq \tepsnomedP ,
\]
from which the desired upper bound on $\tepsnomed$ immediately follows by $\tepsnomed \le \tepsnomedP$ and substituting the value of $\rho$.    

Assume by way of contradiction that \smash{$\tepsnomedP > T$}.
Let us lower bound the rate at which $\pi_{\theta_t}(\ystar)$ increases over \smash{$[0, \tepsnomedP]$}.
Notice first that $\Vproxy (\theta_0) > \proxyreward (y)$ for all $y \in \Y \setminus \{ \ystar\}$ since all such $y$ have the same proxy reward value, which is lower than $\proxyreward (\ystar) = \gtreward (\ystar)$.
Now, consider the time derivative of $\pi_{\theta_t}(\ystar)$, starting from the expression derived in \zcref{lem:probability_derivative_expression}:
\begin{align*}
        & \frac{d}{dt} \pi_{\theta_t}(\ystar) \\
        & = \pi_{\theta_t}(\ystar) \Big [ \pi_{\theta_t}(\ystar) \advproxy(\ystar; \theta_t) \big(\|\phi(\ystar)\|^2 - \|\phi(\ystar)\|^2\pi_{\theta_t}(\ystar) - s \pi_{\theta_t}(\ymed)\big)\\
        & \hspace{16mm} + \pi_{\theta_t}(\ymed) \advproxy(\ymed; \theta_t) \big(s - \|\phi(\ymed)\|^2\pi_{\theta_t}(\ymed) - s \pi_{\theta_t}(\ystar) \big) \\
        & \hspace{16mm} - \sum\nolimits_{z \in \Ybad} \pi_{\theta_t}(z)^2 \advproxy(z; \theta_t) \|\phi(z)\|^2 \Big]\\
        & \ge \pi_{\theta_t}(\ystar) \Big[ \pi_{\theta_t}(\ystar) \advproxy(\ystar; \theta_t) \big(\|\phi(\ystar)\|^2 - \|\phi(\ystar)\|^2\pi_{\theta_t}(\ystar) - s \pi_{\theta_t}(\ymed)\big) \\
        &\hspace{16mm} + \pi_{\theta_t}(\ymed) \advproxy(\ymed; \theta_t) \big(s - \|\phi(\ymed)\|^2\pi_{\theta_t}(\ymed) - s \pi_{\theta_t}(\ystar) \big) \Big]\\
        &\ge \pi_{\theta_t}(\ystar) \Big[ \pi_{\theta_t}(\ystar) \advproxy(\ystar; \theta_t) \big(\|\phi(\ystar)\|^2 - \|\phi(\ystar)\|^2\pi_{\theta_t}(\ystar)\big) \\
        &\hspace{16mm} + \pi_{\theta_t}(\ymed) \advproxy(\ymed; \theta_t) \big(s - \|\phi(\ymed)\|^2\pi_{\theta_t}(\ymed) - s \pi_{\theta_t}(\ystar) \big) \Big]\\
        & \ge \pi_{\theta_t}(\ystar)^2 \|\phi(\ystar)\|^2 \advproxy(\ystar; \theta_t) \big(1 -\pi_{\theta_t}(\ystar)\big) .
    \end{align*}
    Here, the first inequality is from $\advproxy(z; \theta_t) \le 0$ for all $ z \in \Ybad$; the second inequality is from $s \le 0$; and the last inequality is from $\advproxy(\ymed; \theta_t) \le \advproxy(\ymed; \theta_0) \le 0$ together with
    \[
        s - \|\phi(\ymed)\|^2\pi_{\theta_t}(\ymed) - s \pi_{\theta_t}(\ystar)
        = s(1-\pi_{\theta_t}(\ystar)) - \|\phi(\ymed)\|^2\pi_{\theta_t}(\ymed) \le 0 ,
    \]
    so the $\ymed$ term is non-negative.
    Next, we show that $\advproxy(\ystar; \theta_t) \ge (1-\rho)\advproxy(\ystar; \theta_0)$. 
    Observe that since $\proxyreward(\ystar) = \gtreward(\ystar)$ and $\Vproxy(\theta_0)\ge -1$,
    \[
        \frac{\gtreward(\ystar) - \epsilon - \Vproxy (\theta_0)}{\proxyreward(\ystar) - \Vproxy (\theta_0)}
        = 1 - \frac{\epsilon}{\proxyreward(\ystar) - \Vproxy (\theta_0)}
        \le 1 - \frac{\epsilon}{\proxyreward(\ystar) + 1}
        = \frac{\gtreward(\ystar) + 1 - \epsilon}{\gtreward(\ystar) + 1}
        = \rho .
    \]
    Using this fact, we have that:
\begin{align*}
    \frac{\advproxy(\ystar; \theta_t)}{\advproxy(\ystar; \theta_0)} = \frac{\proxyreward(\ystar) - \Vproxy(\theta_t)}{\proxyreward(\ystar) - \Vproxy(\theta_0)}=1 - \frac{\Vproxy(\theta_t) - \Vproxy(\theta_0)}{\proxyreward(\ystar) - \Vproxy(\theta_0)} \ge 1 - \frac{\gtreward(\ystar) - \epsilon - \Vproxy(\theta_0)}{\proxyreward(\ystar) - \Vproxy(\theta_0)} \geq 1-\rho.
\end{align*}
Therefore, along with $1 - \pi_{\theta_t} (\ystar) \geq 1 - \rho$ for $t \in [0, \tepsnomedP]$ (as proven above), this leads to:
\begin{align*}
    \frac{d}{dt}\pi_{\theta_t}(\ystar) \ge (1-\rho)^2 \advproxy(\ystar; \theta_0)\|\phi(\ystar)\|^2 \pi_{\theta_t}(\ystar)^2.
\end{align*}
Dividing both sides of the inequality above by $\pi_{\theta_t}(\ystar)^2$, integrating from $0$ to $t$, and rearranging terms results in the following lower bound:
\[
\pi_{\theta_t}(\ystar) \ge \frac{\pi_{\theta_0}(\ystar)}{1 - (1-\rho)^2\advproxy(\ystar; \theta_0)\|\phi(\ystar)\|^2\pi_{\theta_0}(\ystar) \cdot t} .
\]
At time $t = T$ for
\[
T = \frac{1}{(1-\rho)^2 \advproxy(\ystar; \theta_0)\|\phi(\ystar)\|^2}\bigg(\frac{1}{\pi_{\theta_0}(\ystar)} - \frac{1}{\rho}\bigg) ,
\]
this lower bound on $\pi_{\theta_t} (\ystar)$ implies that $\pi_{\theta_t}(\ystar) \ge \rho$, in contradiction to our assumption that $\tepsnomedP > T$.
Thus,
\[
\tepsnomed \le \tepsnomedP \le T \le \frac{ \brk*{ \gtreward (\ystar) + 1 }^2 }{\epsilon^2  \brk1{ \proxyreward (\ystar) - \Vproxy (\theta_0)  } \norm{ \phi(\ystar) }^2 } \cdot \pi_{\theta_0} (\ystar)^{-1} = \OO \brk*{ \pi_{\theta_0} (\ystar)^{-1} } .
\]
\end{proof}

\subsection{Proof of \zcref{prop:neg_inner_product_fail}}
\label{app:proofs:neg_inner_product_fail}

The core of the proof is to show that for all $t \geq 0$:
\begin{align}
    \frac{\pi_{\theta_t}(\ystar)}{\pi_{\theta_t}(\ybad)} < \zeta \leq \frac{\Deltatwo}{\Deltaone} .
    \label{eq: ratio_small_zeta}
\end{align}
That is, the policy never assigns substantially higher probability to $\ystar$ relative to $\ybad$.
We first establish that this bound implies the desired conclusion.
Towards this, we start by showing that $\Vgt(\theta_0) \le \gtreward(\ymed)$.
Indeed:
\begin{align}
    \gtreward(\ymed) - \Vgt(\theta_0) &= \gtreward(\ymed) - \pi_{\theta_0}(\ystar) \gtreward(\ystar) - \pi_{\theta_0}(\ymed) \gtreward(\ymed) - \sum_{z \in \Ybad} \pi_{\theta_0}(z) \gtreward(z)\nonumber\\
    &= \gtreward(\ymed) - \pi_{\theta_0}(\ystar) \gtreward(\ystar) - \pi_{\theta_0}(\ymed) \gtreward(\ymed) - \pi_{\theta_0}(\Ybad) \gtreward(\ybad)\nonumber\\
    &=(\pi_{\theta_0}(\ystar) + \pi_{\theta_0}(\Ybad))\gtreward(\ymed) - \pi_{\theta_0}(\ystar) \gtreward(\ystar) - \pi_{\theta_0}(\Ybad) \gtreward(\ybad)\nonumber\\
    &= -\Delta_1 \pi_{\theta_0}(\ystar) + \Deltatwo \pi_{\theta_0}(\Ybad)\nonumber\\
    &\ge 0,\label{eq: bar_r_small}
\end{align}
where the inequality is due to $\Deltaone \pi_{\theta_0} (\ystar) \leq \Deltatwo \pi_{\theta_0} (\Ybad)$ (see the second assumption in the proposition statement).
Next, we prove that $\zeta \le \Delta_2/\Delta_1$.
Since $\Vgt (\theta_0) \leq \gtreward (\ymed)$ (\zcref{eq: bar_r_small}), it holds that
\begin{align}
    \frac{\Vgt(\theta_0) - \gtreward(\ybad)}{\gtreward(\ystar) - \Vgt(\theta_0)}
    \le
    \frac{\gtreward(\ymed) - \gtreward(\ybad)}{\gtreward(\ystar) - \gtreward(\ymed)}
    =
    \frac{\Delta_2}{\Delta_1}.
    \label{eq: zeta_small_cond1}
\end{align}
Furthermore, because $\theta_0 = C (\phi(\ybad) - \phi(\ystar))$ for some $C > 0$ and $\inprod{ \phi(\ybad) }{ \phi (\ystar) } = 0$, we have that $\inprod{\phi (\ystar)}{ \theta_0} = - C \norm{\phi (\ystar)}^2  \leq 0 \leq C \norm{\phi(\ybad)}^2 = \inprod{\phi (\ybad)}{ \theta_0 }$.
That is, the logit of $\ystar$ is lower than that of $\ybad$ at initialization, and so $\pi_{\theta_0}(\ystar) \le \pi_{\theta_0}(\ybad)$.
Together with the first assumption in the proposition statement, which gives $\pi_{\theta_0}(\ymed) > \pi_{\theta_0}(\ybad)$, this implies
\begin{align}
\frac{\ln \pi_{\theta_0}(\ymed) - \ln \pi_{\theta_0}(\ybad)}{\ln \pi_{\theta_0}(\ymed) - \ln \pi_{\theta_0}(\ystar)}
\le 1.
\label{eq: zeta_small_cond2}
\end{align}
Combining \zcref{eq: zeta_small_cond1,eq: zeta_small_cond2} with the definition of $\zeta$, we then obtain
\begin{align*}
    \zeta = \frac{\ln \pi_{\theta_0}(\ymed) - \ln \pi_{\theta_0}(\ybad)}{\ln \pi_{\theta_0}(\ymed) - \ln \pi_{\theta_0}(\ystar)} \cdot \frac{\Vgt(\theta_0) - \gtreward(\ybad)}{\gtreward(\ystar)-\Vgt(\theta_0)} \le \frac{\Delta_2}{\Delta_1}.
\end{align*}
Now, assuming \zcref{eq: ratio_small_zeta} holds for all $t \geq 0$, we lower bound $\gtreward(\ystar) - \Vgt(\theta_t)$ as follows:
\begin{align*}
    \gtreward(\ystar) - \Vgt(\theta_t) &= \gtreward(\ystar)(1 - \pi_{\theta_t}(\ystar)) - \pi_{\theta_t}(\ymed) \gtreward(\ymed) - \sum\nolimits_{z \in \Ybad}\pi_{\theta_t}(z) \gtreward(z)\\
    & \ge \gtreward(\ystar)(1 - \pi_{\theta_t}(\ystar)) - \pi_{\theta_t}(\ymed) \gtreward(\ymed) - \pi_{\theta_t}(\ybad)\gtreward(\ybad)\\
    & \ge \gtreward(\ystar)(1 - \pi_{\theta_t}(\ystar)) - (1-\pi_{\theta_t}(\ystar) - \pi_{\theta_t}(\ybad))\gtreward(\ymed) \\
    & \hspace{4mm} - \pi_{\theta_t}(\ybad)\gtreward(\ybad)\\
    & = \Delta_1(1-\pi_{\theta_t}(\ystar)) + \Delta_2 \pi_{\theta_t}(\ybad)\\
    & > \Delta_1(1-\pi_{\theta_t}(\ystar)) + \Delta_1 \pi_{\theta_t}(\ystar) \\
    & = \Delta_1 .
\end{align*}
Here, the first inequality is from $\gtreward(z) \le 0$ for all $z \in \Ybad$; the second inequality is from $\pi_{\theta_t}(\ymed) \le 1 - \pi_{\theta_t}(\ystar) - \pi_{\theta_t}(\ybad)$; and the last inequality is from \zcref{eq: ratio_small_zeta}.
Thus, \zcref{eq: ratio_small_zeta} implies that for all $t \geq 0$,
\[
    \Vgt (\theta_t) < \gtreward(\ystar) - \Delta_1 = \gtreward(\ymed) .
\]

\medskip

It remains to prove \zcref{eq: ratio_small_zeta}.
Let $f_y(\theta) := \inprod{ \phi(y) }{ \theta }$ denote the logit of $y \in \Y$ under $\theta \in \R^D$.
At initialization,
    \begin{align*}
        \frac{\pi_{\theta_0}(\ystar)}{\pi_{\theta_0}(\ybad)} &= \exp\big(f_{\ystar}(\theta_0) - f_{\ybad}(\theta_0)\big)\\
        &= \exp\big(\inprod{\phi(\ystar) - \phi(\ybad)}{\theta_0}\big)\\
        &= \exp\big(-C\|\phi(\ybad) - \phi(\ystar)\|^2\big)\\
        &< \exp(\ln(\zeta)) = \zeta .
    \end{align*}
Both the third equality and the last inequality are due to the last assumption in the proposition statement, \ie, due to $\theta_0 = C \cdot \brk*{ \phi (\ybad) - \phi (\ystar) }$ for some $C > \max \brk[c]1{ 0, - \ln (\zeta) / \norm{ \phi (\ybad) - \phi (\ystar) }^2 }$.
Now, assume by way of contradiction that there exists a time $t'$ at which $\frac{\pi_{\theta_{t'}}(\ystar)}{\pi_{\theta_{t'}}(\ybad)} \geq \zeta$ and let $\tau$ denote the initial such time, \ie:
\[
\tau := \min \brk[c]*{ t \geq 0 : \frac{\pi_{\theta_t}(\ystar)}{\pi_{\theta_t}(\ybad)} \geq \zeta } .
\]
Then, for all $ t \in [0, \tau)$ it holds that $\pi_{\theta_t}(\ystar) / \pi_{\theta_t}(\ybad) < \zeta$.
We now show that $\pi_{\theta_\tau}(\ystar) / \pi_{\theta_\tau}(\ybad) < \zeta$, in contradiction to the definition of $\tau$ as the initial time at which this ratio is at least $\zeta$.
We will do so by proving that $\pi_{\theta_t}(\ystar) / \pi_{\theta_t}(\ybad)$ is decreasing for all $t \in [0, \tau)$, which, along with the fact that $\pi_{\theta_0}(\ystar) / \pi_{\theta_0}(\ybad) < \zeta$, implies that $\pi_{\theta_\tau}(\ystar) / \pi_{\theta_\tau}(\ybad) < \zeta$.

Notice that, since $\pi_{\theta_t}(\ystar) / \pi_{\theta_t}(\ybad) = \exp\big(f_{\ystar}(\theta_t) - f_{\ybad}(\theta_t)\big)$:
\[
\sign \brk*{ \frac{d}{dt} \brk*{ \frac{\pi_{\theta_t}(\ystar)}{\pi_{\theta_t}(\ybad)} } } = \sign \brk*{ \frac{d}{dt} \brk*{f_{\ystar}(\theta_t) - f_{\ybad}(\theta_t) } } .
\]
It therefore suffices to show that $\frac{d}{dt} \brk*{f_{\ystar}(\theta_t) - f_{\ybad}(\theta_t) } < 0$ for all $t \in [0, \tau )$.

The time derivative of $f_y (\theta_t)$ for an output $y \in \Y$ is given by:
\[
\frac{d}{dt} f_y(\theta_t) = \inprod{ \phi (y) }{ \tfrac{d}{dt} \theta_t } = \inprod{ \phi (y) }{ \nabla \Vgt (\theta_t) } = \sum\nolimits_{z \in \Y} \pi_{\theta_t}(z) \advgt(z; \theta_t) \inprod{\phi(y)}{\phi(z)} ,
\]
where the last equality is by \zcref{lem:gradient_expression}.
Thus, for all $ t \in [0, \tau)$ we have that:
\begin{align*}
    &\frac{d}{dt}\big(f_{\ystar}(\theta_t) - f_{\ybad}(\theta_t)\big)\\
    & = \sum\nolimits_{y \in \Y} \brk1{ \inprod{\phi(\ystar)}{\phi(y)} - \inprod{\phi(\ybad)}{\phi(y)} } \pi_{\theta_t}(y) \advgt(y; \theta_t)\\
    & = \sum\nolimits_{y \in \Y} \inprod{\phi(\ystar) - \phi(\ybad)}{\phi(y) - \phi(\ymed) + \phi(\ymed)}\pi_{\theta_t}(y) \advgt(y; \theta_t)\\
    & = \sum\nolimits_{y \in \Y} \inprod{\phi(\ystar) - \phi(\ybad)}{\phi(y) - \phi(\ymed)}\pi_{\theta_t}(y) \advgt(y; \theta_t) \\
    & \hspace{4mm} + \inprod{\phi(\ystar) - \phi(\ybad)}{\phi(\ymed)}\sum\nolimits_{y \in \Y}  \pi_{\theta_t}(y) \advgt(y; \theta_t) .
\end{align*}
Because
\begin{align*}
    \sum\nolimits_{y \in \Y}  \pi_{\theta_t}(y) \advgt(y; \theta_t) = \sum\nolimits_{y \in \Y} \pi_{\theta_t}(y)(\gtreward(y) - \Vgt(\theta_t)) = \Vgt (\theta_t) - \Vgt(\theta_t) = 0 ,
\end{align*}
we may write:
\begin{align}
    &\frac{d}{dt}\big(f_{\ystar}(\theta_t) - f_{\ybad}(\theta_t)\big) \nonumber\\
    &=
    \sum\nolimits_{y \in \Y}
    \inprod{\phi(\ystar) - \phi(\ybad)}{\phi(y) - \phi(\ymed)}
    \pi_{\theta_t}(y) \advgt(y; \theta_t) \nonumber\\
    &=
    \inprod{\phi(\ystar) - \phi(\ybad)}{\phi(\ystar) - \phi(\ymed)}
    \pi_{\theta_t}(\ystar) \advgt(\ystar; \theta_t) \nonumber\\
    &\hspace{4mm}
    + \inprod{\phi(\ystar) - \phi(\ybad)}{\phi(\ybad) - \phi(\ymed)}
    \pi_{\theta_t}(\ybad) \advgt(\ybad; \theta_t) \nonumber\\
    &\hspace{4mm}
    + \sum\nolimits_{z \in \Ybad \setminus \{\ybad\}}
    \inprod{\phi(\ystar) - \phi(\ybad)}{\phi(z) - \phi(\ymed)}
    \pi_{\theta_t}(z) \advgt(z; \theta_t).
    \label{eq:deriv_logit_diff_decomposition}
\end{align}
Notice that the contribution of outputs in $\Ybad \setminus \{\ybad\}$ is non-positive.
Indeed, for every $z \in \Ybad \setminus \{\ybad\}$, \zcref{assum:feature_structure_neg_inn_prod} implies $\inprod{\phi(\ystar) - \phi(\ybad)}{\phi(z) - \phi(\ymed)} = -\inprod{\phi(\ystar)}{\phi(\ymed)} > 0$,
and since all outputs in $\Ybad$ have the same minimal ground truth reward (\zcref{assum:reward_structure_orthonormal}), $\advgt(z;\theta_t) < 0$ for all $z \in \Ybad$.
Using \zcref{eq:deriv_logit_diff_decomposition}, we therefore get
\begin{align*}
    &\frac{d}{dt}\big(f_{\ystar}(\theta_t) - f_{\ybad}(\theta_t)\big) \nonumber\\
    &\le
    \inprod{\phi(\ystar) - \phi(\ybad)}{\phi(\ystar) - \phi(\ymed)}
    \pi_{\theta_t}(\ystar) \advgt(\ystar; \theta_t) \nonumber\\
    &\hspace{4mm}
    + \inprod{\phi(\ystar) - \phi(\ybad)}{\phi(\ybad) - \phi(\ymed)}
    \pi_{\theta_t}(\ybad) \advgt(\ybad; \theta_t).
\end{align*}

Next, using $\advgt(\ystar;\theta_t)=\gtreward(\ystar)-\Vgt(\theta_t)$ and $\advgt(\ybad;\theta_t)=\gtreward(\ybad)-\Vgt(\theta_t)$, the right-hand side of the equation above factors as
\begin{align}
    &\frac{d}{dt}\big(f_{\ystar}(\theta_t) - f_{\ybad}(\theta_t)\big) \nonumber\\
    &\le
    - \inprod{\phi(\ystar) - \phi(\ybad)}{\phi(\ybad) - \phi(\ymed)}
    \pi_{\theta_t}(\ybad) \brk*{\Vgt(\theta_t) - \gtreward(\ybad)} \nonumber\\
    &\hspace{4mm}
    \cdot
    \bigg(
    -\frac{\inprod{\phi(\ystar) - \phi(\ymed)}{\phi(\ystar) - \phi(\ybad)}}{\inprod{\phi(\ybad) - \phi(\ymed)}{\phi(\ystar) - \phi(\ybad)}}
    \cdot
    \frac{\pi_{\theta_t}(\ystar)}{\pi_{\theta_t}(\ybad)}
    \cdot
    \frac{\gtreward(\ystar) - \Vgt(\theta_t)}{\Vgt(\theta_t) - \gtreward(\ybad)}
    + 1
    \bigg).
    \label{eq:deriv_logit_diff_factored}
\end{align}

It remains to show that the right-hand side of \zcref{eq:deriv_logit_diff_factored} is negative.
By the first assumption in the proposition statement, $\pi_{\theta_0}(\ymed) > \pi_{\theta_0}(\ybad)$.
This implies that $f_{\ymed}(\theta_0) > f_{\ybad}(\theta_0)$, or equivalently, $\inprod{\phi(\ymed) - \phi(\ybad)}{ \theta_0 } > 0$.
Substituting $\theta_0 = C (\phi(\ybad) - \phi(\ystar))$ and dividing by $C > 0$ gives
\[
    \inprod{\phi(\ybad) - \phi(\ymed)}{\phi(\ystar) - \phi(\ybad)} = \inprod{\phi(\ybad) - \phi(\ystar)}{\phi(\ymed) - \phi(\ybad)} > 0 .
\]
Furthermore,
\begin{align*}
    \inprod{\phi(\ystar) - \phi(\ymed)}{\phi(\ystar) - \phi(\ybad)}
    &= \inprod{\phi(\ystar) - \phi(\ybad)}{\phi(\ystar) - \phi(\ybad)}\\
    &\hspace{4mm} + \inprod{\phi(\ybad) - \phi(\ymed)}{\phi(\ystar) - \phi(\ybad)}\\
    &\ge \inprod{\phi(\ybad) - \phi(\ymed)}{\phi(\ystar) - \phi(\ybad)}\\
    &> 0 .
\end{align*}
These are exactly the numerator and denominator of the inner product terms appearing in 
\[
  \zeta =
    \frac{\inprod{\phi(\ybad) - \phi(\ymed)}{\phi(\ystar)-\phi(\ybad)}}{\inprod{\phi(\ystar)-\phi(\ymed)}{\phi(\ystar)-\phi(\ybad)}}
    \cdot
    \frac{\Vgt(\theta_0) - \gtreward(\ybad)}{\gtreward(\ystar) - \Vgt(\theta_0)} ,
\]
and both are positive.
Thus, for $t \in [0,\tau)$, the bound $\pi_{\theta_t}(\ystar)/\pi_{\theta_t}(\ybad) < \zeta$ (\zcref{eq: ratio_small_zeta}) implies
\[
    \frac{\inprod{\phi(\ystar)-\phi(\ymed)}{\phi(\ystar)-\phi(\ybad)}}{\inprod{\phi(\ybad)-\phi(\ymed)}{\phi(\ystar)-\phi(\ybad)}}
    \cdot
    \frac{\pi_{\theta_t}(\ystar)}{\pi_{\theta_t}(\ybad)}
    <
    \frac{\Vgt(\theta_0) - \gtreward(\ybad)}{\gtreward(\ystar) - \Vgt(\theta_0)} .
\]
Since under gradient flow $\Vgt$ is monotonically non-decreasing, we have $\Vgt(\theta_0) \le \Vgt(\theta_t)$, and so
\[
    \frac{\gtreward(\ystar) - \Vgt(\theta_t)}{\Vgt(\theta_t) - \gtreward(\ybad)}
    \le
    \frac{\gtreward(\ystar) - \Vgt(\theta_0)}{\Vgt(\theta_0) - \gtreward(\ybad)} .
\]
Combining the last two inequalities shows that the product inside the parentheses in \zcref{eq:deriv_logit_diff_factored} is less than $1$.
Since $\inprod{\phi(\ystar) - \phi(\ybad)}{\phi(\ybad) - \phi(\ymed)} > 0$, $\pi_{\theta_t}(\ybad) > 0$, and $\Vgt(\theta_t) > \gtreward(\ybad)$, we conclude from \zcref{eq:deriv_logit_diff_factored} that
\begin{align}
    \frac{d}{dt}\big(f_{\ystar}(\theta_t) - f_{\ybad}(\theta_t)\big) < 0 .
    \label{eq:deriv_logit_diff_negative}
\end{align}
\zcref{eq:deriv_logit_diff_negative} implies that $f_{\ystar}(\theta_t) - f_{\ybad}(\theta_t)$ is non-increasing, and so $\pi_{\theta_t}(\ystar) / \pi_{\theta_t}(\ybad)$ is non-increasing, for all $ t \in [0, \tau)$. 
Thus,
\[
\frac{\pi_{\theta_\tau}(\ystar)}{\pi_{\theta_\tau}(\ybad)} \le \frac{\pi_{\theta_0}(\ystar)}{\pi_{\theta_0}(\ybad)} < \zeta ,
\]
in contradiction to our assumption that $\pi_{\theta_\tau}(\ystar) / \pi_{\theta_\tau}(\ybad) \geq \zeta$. 
This establishes \zcref{eq: ratio_small_zeta}, concluding the proof.
\qed

\subsection{Proof of \zcref{thm:attraction_to_mediocre_outputs_beneficial_error_pos_inner_product}}
\label{app:proofs:attraction_to_mediocre_outputs_beneficial_error_pos_inner_product}

For convenience of notation, throughout the proof we denote $s := \inprod{\phi(\ystar)}{\phi(\ymed)}$.

\subsubsection{Proof of Case I}

Suppose that gradient flow is used to maximize the expected reward with respect to $\gtreward$.
We begin by proving that $\pi_{\theta_t}(\Ybad)$ is monotonically
non-increasing (\zcref{lemma: ybad_decrease_nonorm}) and $\pi_{\theta_t}(\ystar)$ is monotonically non-decreasing (\zcref{lemma: ystar_increase_nonorm}) throughout optimization.
We then upper bound $\teps$---the initial time at which $\Vgt (\theta_t) \geq \gtreward (\ystar) - \epsilon$---by lower bounding the growth rate of $\pi_{\theta_t}(\ystar)$ over the interval $[0,\teps)$.

\begin{lemma}
\label{lemma: ybad_decrease_nonorm}
At any time $ t \ge 0$, it holds that $\frac{d}{dt}\pi_{\theta_t}(\Ybad) \le 0$.
\end{lemma}

\begin{proof}
Suppose first that $\Vgt(\theta_t)\le \gtreward(\ymed)$.
Under \zcref{assum:feature_structure_pos_inn_prod}, for every $z \in \Ybad$, the feature vector $\phi(z)$ is orthogonal to the feature vectors of all other outputs.
Therefore, by \zcref{lem:probability_derivative_expression} we have for all $ z \in \Ybad$:
\begin{align*}
        \frac{d}{dt}\pi_{\theta_t}(z)
        ={}& \pi_{\theta_t}(z)^2 \advgt(z; \theta_t) \|\phi(z)\|^2 \\
        & -  \pi_{\theta_t}(z) \Big(\sum\nolimits_{z' \in \Ybad} \pi_{\theta_t}(z')^2 \advgt(z'; \theta_t) \|\phi(z')\|^2 \\
        & \hspace{16mm} + \pi_{\theta_t}(\ymed) \advgt(\ymed; \theta_t) \big(\|\phi(\ymed)\|^2\pi_{\theta_t}(\ymed) + s\pi_{\theta_t}(\ystar) \big) \\
        & \hspace{16mm} + \pi_{\theta_t}(\ystar) \advgt(\ystar; \theta_t) \big(\|\phi(\ystar)\|^2 \pi_{\theta_t}(\ystar) + s \pi_{\theta_t}(\ymed) \big) \Big).
    \end{align*}
    Summing over all $z \in \Ybad$, we have:
    \begin{align*}
        \frac{d}{dt} \pi_{\theta_t}(\Ybad) & =  \underbrace{\Big(\sum\nolimits_{z \in \Ybad} \pi_{\theta_t}(z)^2 \advgt(z; \theta_t) \|\phi(z)\|^2\Big)\big(1 - \pi_{\theta_t}(\Ybad)\big)}_{I_1} \\
        & \hspace{4mm}- \pi_{\theta_t}(\Ybad) \Big(\pi_{\theta_t}(\ymed)\advgt(\ymed; \theta_t)\big ( \|\phi(\ymed)\|^2\pi_{\theta_t}(\ymed) + s\pi_{\theta_t}(\ystar) \big ) \\
        & \hspace{24mm} + \pi_{\theta_t}(\ystar) \advgt(\ystar; \theta_t)\big (\|\phi(\ystar)\|^2\pi_{\theta_t}(\ystar) + s \pi_{\theta_t}(\ymed) \big ) \Big),
    \end{align*}
    where we denote by $I_2$ the second term on the right-hand side of the equation above (excluding the minus sign).
    Since all outputs in $\Ybad$ have the minimal ground truth reward, we have $\advgt(z; \theta_t) \le 0$ for all $z \in \Ybad$.
    Together with $\advgt(\ymed; \theta_t) \geq 0$ and $\advgt(\ystar; \theta_t) \ge 0$, this gives $I_1 \le 0$ and $I_2 \ge 0$.
    This implies that $\frac{d}{dt}\pi_{\theta_t}(\Ybad) \le 0$ when $\Vgt(\theta_t) \le \gtreward(\ymed)$.

    \medskip

    Now, consider the case where $\Vgt(\theta_t) > \gtreward(\ymed)$. 
    In this case, $I_1 \le 0$ still holds. 
    Furthermore,
    \be
    \begin{split}
        \Vgt(\theta_t) - \gtreward(\ymed) &= (\pi_{\theta_t}(\ymed) - 1)\gtreward(\ymed) + \pi_{\theta_t}(\ystar) \gtreward(\ystar) + \sum\nolimits_{z \in \Ybad} \pi_{\theta_t}(z) \gtreward(z) \\
        & \le (\pi_{\theta_t}(\ymed) - 1)\gtreward(\ymed) + \pi_{\theta_t}(\ystar) \gtreward(\ystar) \\
        & \le -\pi_{\theta_t}(\ystar)\gtreward(\ymed) + \pi_{\theta_t}(\ystar)\gtreward(\ystar) \\
        & = \pi_{\theta_t}(\ystar)\Delta_1, 
    \end{split}
    \label{eq: det_ymed_upper}
    \ee
    where the first inequality is from $\gtreward(z) \le 0$ for all $z \in \Ybad$ and the second inequality is because $1 - \pi_{\theta_t}(\ymed) \ge \pi_{\theta_t}(\ystar)$ and $\gtreward(\ymed) > 0$. 
    Thus, $\abs{ \advgt(\ymed; \theta_t) } = \Vgt (\theta_t) - \gtreward (\ymed) \le \pi_{\theta_t}(\ystar) \Delta_1$.
    Moreover,
    \be
    \begin{split}
        \advgt(\ystar; \theta_t) &= \gtreward(\ystar) - \pi_{\theta_t}(\ystar) \gtreward(\ystar) - \pi_{\theta_t}(\ymed) \gtreward(\ymed) - \sum\nolimits_{z \in \Ybad}\pi_{\theta_t}(z) \gtreward(z) \\
        &\ge \gtreward(\ystar) - \pi_{\theta_t}(\ystar)\gtreward(\ystar) - (1 - \pi_{\theta_t}(\ystar))\gtreward(\ymed) \\
        &= (1 - \pi_{\theta_t}(\ystar))\Delta_1 .
    \end{split}
    \label{eq: det_ystar_lower}
    \ee
    We can therefore lower bound $I_2 / \pi_{\theta_t}(\Ybad)$ as:
    \begin{align*}
        &\pi_{\theta_t}(\ymed) \advgt(\ymed; \theta_t) \big (\|\phi(\ymed)\|^2 \pi_{\theta_t}(\ymed) + s \pi_{\theta_t}(\ystar) \big ) \\
        & \hspace{4mm} + \pi_{\theta_t}(\ystar) \advgt(\ystar; \theta_t) \big ( \|\phi(\ystar)\|^2 \pi_{\theta_t}(\ystar) + s \pi_{\theta_t}(\ymed) \big ) \\
        & \geq \pi_{\theta_t}(\ystar)(1 - \pi_{\theta_t}(\ystar))\Delta_1\big ( \|\phi(\ystar)\|^2\pi_{\theta_t}(\ystar) + s\pi_{\theta_t}(\ymed) \big ) \\
        &\hspace{4mm} - \pi_{\theta_t}(\ymed)\pi_{\theta_t}(\ystar)\Delta_1 \big (\|\phi(\ymed)\|^2\pi_{\theta_t} (\ymed) + s \pi_{\theta_t}(\ystar) \big ) \\
        & \geq \pi_{\theta_t}(\ystar)(1 - \pi_{\theta_t}(\ystar))\Delta_1\big((\|\phi(\ystar)\|^2-s)\pi_{\theta_t}(\ystar) + (s-\|\phi(\ymed)\|^2)\pi_{\theta_t}(\ymed)\big)\\
        &\geq 0 .
    \end{align*}
    Here, the first inequality is from $|\advgt(\ymed; \theta_t)| \le \pi_{\theta_t}(\ystar) \Delta_1$ and $\advgt(\ystar; \theta_t) \ge (1 - \pi_{\theta_t}(\ystar))\Delta_1$, along with the fact that the multipliers of these advantages are non-negative.
    The second inequality is from $1 - \pi_{\theta_t}(\ystar) \ge \pi_{\theta_t}(\ymed)$, combined with $s > 0$ and $\norm{\phi(\ymed)}^2 \ge 0$.
    The last inequality follows from the assumption in \zcref{thm:attraction_to_mediocre_outputs_beneficial_error_pos_inner_product} that $\norm{\phi(\ymed)}^2 < s < \norm{\phi(\ystar)}^2$.
    Thus, $I_2 \ge 0$, which means $\frac{d}{dt}\pi_{\theta_t}(\Ybad) \le 0$.

    Overall, we showed that $\frac{d}{dt}\pi_{\theta_t}(\Ybad) \le 0$ both when $\Vgt(\theta_t) \le \gtreward(\ymed)$ and when $\Vgt(\theta_t) > \gtreward(\ymed)$, and so $\pi_{\theta_t}(\Ybad)$ is non-increasing for all $t \ge 0$.
\end{proof}

\begin{lemma}
\label{lemma: ystar_increase_nonorm}
At any time $ t \ge 0$, it holds that $\frac{d}{dt}\pi_{\theta_t}(\ystar) \ge 0$.
\end{lemma}

\begin{proof}
Let $\ell(\theta) := \inprod{\phi (\ystar)}{\theta} - \inprod{\phi(\ymed)}{\theta} = \inprod{\phi (\ystar) - \phi (\ymed)}{\theta}$ be the logit difference between $\ystar$ and $\ymed$ under $\theta \in \R^D$.
We prove that $\ell(\theta_t)$ is non-decreasing for all $t \geq 0$.
We can write the time derivative of $\ell(\theta_t)$ as:
\begin{align*}
    \frac{d}{dt}\ell(\theta_t) ={}& \inprod{\phi (\ystar) - \phi (\ymed)}{\tfrac{d}{dt} \theta_t} \\
    ={}& \|\phi(\ystar)\|^2 \pi_{\theta_t}(\ystar) \advgt(\ystar; \theta_t) + s \pi_{\theta_t}(\ymed) \advgt(\ymed; \theta_t) - s \pi_{\theta_t}(\ystar) \advgt(\ystar; \theta_t)\\
    &- \|\phi(\ymed)\|^2 \pi_{\theta_t}(\ymed) \advgt(\ymed; \theta_t)\\
    ={}& (\|\phi(\ystar)\|^2-s) \pi_{\theta_t}(\ystar) \advgt(\ystar; \theta_t) + (s-\|\phi(\ymed)\|^2) \pi_{\theta_t}(\ymed) \advgt(\ymed; \theta_t).
\end{align*}
If $\Vgt (\theta_t) \le \gtreward (\ymed)$ (\ie, $\advgt (\ymed ; \theta_t) \geq 0$), since $\|\phi(\ystar)\|^2 > s > \|\phi(\ymed)\|^2$ and $\advgt (\ystar ; \theta_t) \geq 0$, it follows directly from the expression above that $\frac{d}{dt}\ell(\theta_t) \ge 0$. 
On the other hand, if $\Vgt(\theta_t) > \gtreward(\ymed)$, \zcref{eq: det_ymed_upper,eq: det_ystar_lower} hold, as established in the proof of \zcref{lemma: ybad_decrease_nonorm}. 
Thus, in this case it also holds that $\frac{d}{dt} \ell (\theta_t) \ge 0$ since:
    \begin{align*}
        \frac{d}{dt}\ell(\theta_t) &\ge (\|\phi(\ystar)\|^2-s) \pi_{\theta_t}(\ystar) (1 - \pi_{\theta_t}(\ystar)) \Delta_1 - (s-\|\phi(\ymed)\|^2) \pi_{\theta_t}(\ymed) \pi_{\theta_t}(\ystar)\Delta_1\\
        &\ge \pi_{\theta_t}(\ystar)(1 - \pi_{\theta_t}(\ystar))\Delta_1 (\|\phi(\ystar)\|^2- 2 s + \|\phi(\ymed)\|^2) \\
        & \ge 0,
    \end{align*}
    where the second inequality is from $\pi_{\theta_t}(\ymed) \le 1 - \pi_{\theta_t}(\ystar)$, and the last inequality is from $s \le \|\phi(\ystar)\|\|\phi(\ymed)\| \le (\|\phi(\ystar)\|^2 + \|\phi(\ymed)\|^2 ) / 2$.
    Thus, $\ell(\theta_t)$ is monotonically non-decreasing for all $t \ge 0$.

    Now, notice that
    \begin{align*}
        \pi_{\theta_t}(\ystar) = (1 - \pi_{\theta_t}(\Ybad)) \cdot \frac{\pi_{\theta_t}(\ystar)}{\pi_{\theta_t}(\ystar) + \pi_{\theta_t}(\ymed)} .
    \end{align*}
    The factor $1 - \pi_{\theta_t}(\Ybad)$ is monotonically non-decreasing by \zcref{lemma: ybad_decrease_nonorm}, since $\pi_{\theta_t}(\Ybad)$ is non-increasing.
    Furthermore, the factor $\pi_{\theta_t}(\ystar)/(\pi_{\theta_t}(\ystar) + \pi_{\theta_t}(\ymed))$ is also non-decreasing.
    To see this, let $r_t := \pi_{\theta_t}(\ystar) / \pi_{\theta_t}(\ymed)$.
    By the softmax parametrization of the policy, $\ell(\theta_t) = \ln r_t$, and so $r_t = \exp \brk{\ell(\theta_t)}$.
    Since both $\ell(\theta_t)$ and the exponential function are non-decreasing, $r_t$ is non-decreasing.
    Furthermore,
    \[
        \frac{\pi_{\theta_t}(\ystar)}{\pi_{\theta_t}(\ystar) + \pi_{\theta_t}(\ymed)}
        =
        \frac{r_t}{1+r_t},
    \]
    and the map $r \mapsto r/(1+r)$ is non-decreasing for $r \ge 0$.
    Thus, $\pi_{\theta_t}(\ystar)/(\pi_{\theta_t}(\ystar) + \pi_{\theta_t}(\ymed))$ is non-decreasing in $t$.
    Overall, the above implies that $\pi_{\theta_t}(\ystar)$ is non-decreasing, and so $\frac{d}{dt} \pi_{\theta_t} (\ystar) \geq 0$, for all $t \ge 0$.
\end{proof}

\medskip

We now complete the proof of Case I.

\begin{proof}[Proof of \zcref{thm:attraction_to_mediocre_outputs_beneficial_error_pos_inner_product}, Case I]
Denote 
\[
\rho := \frac{\gtreward(\ymed)+\Delta_1-\epsilon+1}{\gtreward(\ystar)+1} ,
\]
and notice that when $\pi_{\theta_t}(\ystar) \ge \rho$ it must be that $\Vgt(\theta_t) \ge \gtreward(\ymed) + \Delta_1 - \epsilon = \gtreward (\ystar) - \epsilon$ since:
\begin{align*}
\Vgt(\theta_t) &\ge \pi_{\theta_t}(\ystar) \gtreward(\ystar) - (1 - \pi_{\theta_t}(\ystar))\\
&= \pi_{\theta_t}(\ystar) (\gtreward(\ystar) + 1) -1\\
&\ge \gtreward(\ystar)-\epsilon ,
\end{align*}
where the first inequality is from $\gtreward(y) \ge -1$ for all $ y \in \Y$ and the second inequality follows from the definition of $\rho$. 
This implies that $\pi_{\theta_t}(\ystar) < \rho$ for all $ t \in [0, \teps)$.
In particular, \zcref{lem:mediocre_output_reward_larger_than_initial_expected_reward_neg_inner_product} implies that $\Vgt(\theta_0)<\gtreward(\ymed)<\gtreward(\ystar)-\epsilon$, and so $\pi_{\theta_0}(\ystar)<\rho$ and $\teps > 0$.

Next, we lower bound the time derivative of $\pi_{\theta_t}(\ystar)$ for $t \in [0, \teps)$, based on the expression derived in \zcref{lem:probability_derivative_expression}:
    \be
    \begin{split}
        \frac{d}{dt}\pi_{\theta_t}(\ystar) ={}& \pi_{\theta_t}(\ystar) \Big[ \pi_{\theta_t}(\ystar) \advgt(\ystar; \theta_t) \big ( \|\phi(\ystar)\|^2 - s \pi_{\theta_t}(\ymed) - \|\phi(\ystar)\|^2\pi_{\theta_t}(\ystar) \big )\\
        &\hspace{12mm} + \pi_{\theta_t}(\ymed) \advgt(\ymed; \theta_t) \big (s - \|\phi(\ymed)\|^2 \pi_{\theta_t}(\ymed) - s \pi_{\theta_t}(\ystar) \big )\\
        &\hspace{12mm} - \sum\nolimits_{z \in \Ybad} \pi_{\theta_t}(z)^2\advgt(z; \theta_t)\|\phi(z)\|^2\Big]\\
        \ge{}& \pi_{\theta_t}(\ystar) \Big[ \pi_{\theta_t}(\ystar) \advgt(\ystar; \theta_t)\big ( \|\phi(\ystar)\|^2 - s \pi_{\theta_t}(\ymed) - \|\phi(\ystar)\|^2\pi_{\theta_t}(\ystar) \big )\\
        &\hspace{12mm} + \pi_{\theta_t}(\ymed) \advgt(\ymed; \theta_t) \big (s - \|\phi(\ymed)\|^2 \pi_{\theta_t}(\ymed) - s \pi_{\theta_t}(\ystar) \big )\Big]\\
        ={}& \pi_{\theta_t}(\ystar) \Big[ \pi_{\theta_t}(\ystar) \advgt(\ystar; \theta_t) \big (\|\phi(\ystar)\|^2(1-\pi_{\theta_t}(\ystar))-s\pi_{\theta_t}(\ymed) \big)\\
        &\hspace{12mm} + \pi_{\theta_t}(\ymed) \advgt(\ymed; \theta_t) \big (s(1-\pi_{\theta_t}(\ystar))-\|\phi(\ymed)\|^2 \pi_{\theta_t}(\ymed) \big )\Big]\\
        \ge{}& \pi_{\theta_t}(\ystar)^2\Delta_1\Big[\|\phi(\ystar)\|^2(1-\pi_{\theta_t}(\ystar))^2 - 2s\pi_{\theta_t}(\ymed)(1-\pi_{\theta_t}(\ystar))\\
        & \hspace{18mm} + \|\phi(\ymed)\|^2\pi_{\theta_t}(\ymed)^2\Big] .
    \end{split}
    \label{eq:loglin_ystar_increase_lower_pos_case_I}
    \ee
    Here, the first inequality is from $\advgt(z; \theta_t) \le 0$ for all $z \in \Ybad$.
    The second inequality is from $\advgt(\ystar;\theta_t) \ge (1-\pi_{\theta_t}(\ystar))\Delta_1$ (as in \zcref{eq: det_ystar_lower}), $\advgt(\ymed;\theta_t) \ge -\pi_{\theta_t}(\ystar)\Delta_1$ (which follows from \zcref{eq: det_ymed_upper} when $\Vgt(\theta_t)>\gtreward(\ymed)$ and is trivial otherwise), and the fact that $\|\phi(\ystar)\|^2(1-\pi_{\theta_t}(\ystar))-s\pi_{\theta_t}(\ymed)$ and $s(1-\pi_{\theta_t}(\ystar))-\|\phi(\ymed)\|^2 \pi_{\theta_t}(\ymed)$ are non-negative, since $\pi_{\theta_t}(\ymed)\le 1-\pi_{\theta_t}(\ystar)$ and $\|\phi(\ymed)\|^2<s<\|\phi(\ystar)\|^2$.
    To simplify the bracketed term on the right-hand side of \zcref{eq:loglin_ystar_increase_lower_pos_case_I}, we can view the expression
    \[
        \|\phi(\ystar)\|^2(1-\pi_{\theta_t}(\ystar))^2 - 2s\pi_{\theta_t}(\ymed)(1-\pi_{\theta_t}(\ystar)) + \|\phi(\ymed)\|^2\pi_{\theta_t}(\ymed)^2
    \]
    as a function of $\pi_{\theta_t}(\ymed)$ over the interval $ [0, 1-\pi_{\theta_t}(\ystar)]$.
    Its derivative with respect to $\pi_{\theta_t}(\ymed)$ is $-2s(1-\pi_{\theta_t}(\ystar)) + 2\|\phi(\ymed)\|^2\pi_{\theta_t}(\ymed)$, which is at most $-2(s-\|\phi(\ymed)\|^2)(1-\pi_{\theta_t}(\ystar)) \le 0$, since $\pi_{\theta_t}(\ymed) \le 1-\pi_{\theta_t}(\ystar)$ and $s>\|\phi(\ymed)\|^2$.
    Thus, the expression is minimized at $\pi_{\theta_t}(\ymed)=1-\pi_{\theta_t}(\ystar)$.
    This implies that:
    \be
    \begin{split}
        \frac{d}{dt}\pi_{\theta_t}(\ystar)
        &\ge \pi_{\theta_t}(\ystar)^2(1-\pi_{\theta_t}(\ystar))^2\Delta_1(\|\phi(\ystar)\|^2+\|\phi(\ymed)\|^2-2s) \\
        &\ge (1-\rho)^2 \Delta_1(\|\phi(\ystar)\|^2+\|\phi(\ymed)\|^2-2s)\pi_{\theta_t}(\ystar)^2 \\
        &= (1-\rho)^2 \Delta_1 \norm{\phi (\ystar) - \phi (\ymed) }^2 \pi_{\theta_t}(\ystar)^2 ,
    \end{split}
    \label{eq: loglin_ystar_increase_lower}
    \ee
    where the second inequality is from $\pi_{\theta_t}(\ystar) < \rho$ for all $ t \in [0, \teps)$.

Dividing both sides of \zcref{eq: loglin_ystar_increase_lower} by $\pi_{\theta_t}(\ystar)^2$, integrating from $0$ to $t$, and rearranging terms~yields:
\[
\pi_{\theta_t}(\ystar) \ge \frac{\pi_{\theta_0}(\ystar)}{1 - (1-\rho)^2\Delta_1 \norm{\phi (\ystar) - \phi (\ymed) }^2 \pi_{\theta_0}(\ystar) \cdot t} .
\]
Since $\pi_{\theta_t}(\ystar)< \rho$ for all $t\in[0, \teps)$, we get that for all such $t$:
\[
t \le \frac{1}{(1-\rho)^2 \norm{\phi (\ystar) - \phi (\ymed) }^2 \Delta_1}\bigg(\frac{1}{\pi_{\theta_0}(\ystar)} - \frac{1}{\rho}\bigg) .
\]
Since this holds for all $t \in [0, \teps)$, we get
\[
    \teps \le \frac{1}{(1-\rho)^2 \norm{\phi (\ystar) - \phi (\ymed) }^2\Delta_1} \cdot \frac{1}{\pi_{\theta_0}(\ystar)} .
\]
Plugging in the definition of $\rho$ gives
\[
    \teps \le \frac{\brk*{\gtreward(\ystar)+1}^2}{\epsilon^2 \norm{\phi (\ystar) - \phi (\ymed) }^2\Delta_1} \cdot \frac{1}{\pi_{\theta_0}(\ystar)} ,
\]
which completes the proof of Case I.
\end{proof}

\subsubsection{Proof of Case II}

Suppose that gradient flow is used to maximize the expected reward with respect to $\proxyreward$, and that $\inprod{ \phi (\ystar)}{\phi(\ymed)}$ is sufficiently high, in the sense that
\be
    \inprod{ \phi (\ystar) }{ \phi (\ymed) } > \frac{ \pi_{\theta_0} (\ystar) \advproxy (\ystar ; \theta_0) \norm{ \phi (\ystar) }^2 - \pi_{\theta_0} (\ymed) \advproxy (\ymed ; \theta_0) \norm{ \phi (\ymed )}^2 }{ \pi_{\theta_0} (\ystar) \advproxy (\ystar ; \theta_0) - \pi_{\theta_0} (\ymed) \advproxy (\ymed ; \theta_0) } .
\label{eq:pos_inner_product_condition_case_II}
\ee

\begin{proof}[Proof of \zcref{thm:attraction_to_mediocre_outputs_beneficial_error_pos_inner_product}, Case II]
We prove that, due to the high inner product between $\phi (\ystar)$ and $\phi (\ymed)$, the logit difference $\inprod{ \phi (\ystar) }{ \theta_t} - \inprod{ \phi (\ymed) }{ \theta_t }$ is monotonically non-increasing with respect to $t$.
Consequently, the ratio $\pi_{\theta_t}(\ystar)/\pi_{\theta_t}(\ymed)$ is bounded by its initial value for all $t \geq 0$, and the policy never assigns $\ystar$ sufficiently high probability to reach an expected ground truth reward above the constant $\gtreward (\ystar) - \brk{ \gtreward (\ystar) - \gtreward (\ymed) } \pi_{\theta_0} (\ymed)$.

We first note that $\advproxy(\ymed;\theta_0)<0$ because $\ymed$ has a minimal proxy reward of
$\proxyreward(\ymed) \le \min\nolimits_{y \in \Ybad}\gtreward(y)$, while $\ystar$ has strictly larger proxy reward and $\pi_{\theta_0}$ has full support as a softmax policy.
Moreover, $\advproxy(\ystar;\theta_0) > 0$ since $\ystar$ has the unique maximal proxy reward.

Now, let $\ell(\theta) := \inprod{\phi (\ystar)}{\theta} - \inprod{\phi(\ymed)}{\theta} = \inprod{\phi (\ystar) - \phi (\ymed)}{\theta}$ be the logit difference between $\ystar$ and $\ymed$ under $\theta \in \R^D$.
We prove that $\ell(\theta_t)$ is non-increasing for all $t \geq 0$.
We can write the time derivative of $\ell(\theta_t)$ as:
\be
\begin{split}
    \frac{d}{dt}\ell(\theta_t) ={}& \inprod{\phi (\ystar) - \phi (\ymed)}{\tfrac{d}{dt} \theta_t} \\
    ={}& \|\phi(\ystar)\|^2 \pi_{\theta_t}(\ystar) \advproxy(\ystar; \theta_t) + s \pi_{\theta_t}(\ymed) \advproxy(\ymed; \theta_t) - s \pi_{\theta_t}(\ystar) \advproxy(\ystar; \theta_t) \\
    &- \|\phi(\ymed)\|^2 \pi_{\theta_t}(\ymed) \advproxy(\ymed; \theta_t)\\
    ={}& (\|\phi(\ystar)\|^2-s) \pi_{\theta_t}(\ystar) \advproxy(\ystar; \theta_t) + (s-\|\phi(\ymed)\|^2) \pi_{\theta_t}(\ymed) \advproxy(\ymed; \theta_t)\\
    \le{}& (\|\phi(\ystar)\|^2-s) \pi_{\theta_t}(\ystar) \advproxy(\ystar; \theta_0) + (s-\|\phi(\ymed)\|^2) \pi_{\theta_t}(\ymed) \advproxy(\ymed; \theta_0)\\
    ={}& \pi_{\theta_t}(\ymed)\big((\|\phi(\ystar)\|^2-s) \advproxy(\ystar; \theta_0) \exp \brk{\ell(\theta_t)} + (s-\|\phi(\ymed)\|^2)\advproxy(\ymed; \theta_0)\big) .
\end{split}
\label{eq:logit_derivative_upper_bound_pos_case_II}
\ee
Here, the first inequality is from $\Vproxy(\theta_t)$ being monotonically non-decreasing when maximizing it via gradient flow, together with the positivity of $\|\phi(\ystar)\|^2-s$ and $s-\|\phi(\ymed)\|^2$, and the last equality is by the definition of $\ell(\theta_t)$ since $\pi_{\theta_t}(\ystar) = \pi_{\theta_t}(\ymed) \cdot \exp \brk{\ell(\theta_t)}$.
Since $\advproxy(\ystar;\theta_0)>0>\advproxy(\ymed;\theta_0)$ and $\|\phi(\ymed)\|^2<s<\|\phi(\ystar)\|^2$, the denominator and numerator of the right-hand side below are both positive.
Therefore, when 
\be
\exp \brk{\ell(\theta_t)} < \frac{-(s-\|\phi(\ymed)\|^2)\advproxy(\ymed; \theta_0)}{(\|\phi(\ystar)\|^2-s) \advproxy(\ystar; \theta_0)} ,
\label{eq:condition_for_logit_decrease}
\ee
we get that $\frac{d}{dt} \ell(\theta_t) < 0$.

\zcref{eq:pos_inner_product_condition_case_II} implies \zcref{eq:condition_for_logit_decrease} at $t=0$.
Indeed, the denominator on the right-hand side of \zcref{eq:pos_inner_product_condition_case_II} is positive because $\advproxy(\ystar;\theta_0)>0>\advproxy(\ymed;\theta_0)$.
Therefore, multiplying both sides of \zcref{eq:pos_inner_product_condition_case_II} by this denominator and rearranging gives
\[
    \pi_{\theta_0}(\ystar)\advproxy(\ystar;\theta_0)
    \brk*{\|\phi(\ystar)\|^2-s}
    <
    -\pi_{\theta_0}(\ymed)\advproxy(\ymed;\theta_0)
    \brk*{s-\|\phi(\ymed)\|^2}.
\]
Since $\pi_{\theta_0}(\ymed)>0$, $\advproxy(\ystar;\theta_0)>0$, and $\|\phi(\ystar)\|^2-s>0$, we may divide both sides by
\[
    \pi_{\theta_0}(\ymed)\advproxy(\ystar;\theta_0)
    \brk*{\|\phi(\ystar)\|^2-s}
\]
to obtain
\[
    \frac{\pi_{\theta_0}(\ystar)}{\pi_{\theta_0}(\ymed)}
    <
    \frac{-(s-\|\phi(\ymed)\|^2)\advproxy(\ymed;\theta_0)}
    {(\|\phi(\ystar)\|^2-s)\advproxy(\ystar;\theta_0)} .
\]
Finally, $\pi_{\theta_0}(\ystar)/\pi_{\theta_0}(\ymed)=\exp\brk{\ell(\theta_0)}$, so \zcref{eq:condition_for_logit_decrease} holds at $t=0$.
We prove that \zcref{eq:condition_for_logit_decrease} holds for all $t \geq 0$ by contradiction.
Assume that there exists a time $t' > 0$ at which 
\[
\exp \brk{\ell(\theta_{t'})} \ge \frac{-(s-\|\phi(\ymed)\|^2)\advproxy(\ymed; \theta_0)}{(\|\phi(\ystar)\|^2-s) \advproxy(\ystar; \theta_0)}  ,
\]
and denote by $\tau$ the initial such time, \ie:
\[
    \tau := \min \brk[c]*{ t \geq 0 : \exp \brk{\ell(\theta_{t})} \ge \frac{-(s-\|\phi(\ymed)\|^2)\advproxy(\ymed; \theta_0)}{(\|\phi(\ystar)\|^2-s) \advproxy(\ystar; \theta_0)}   } .
\]
By the definition of $\tau$, for all $t \in [0, \tau)$ \zcref{eq:condition_for_logit_decrease} holds.
Thus, for all $t \in [0, \tau)$, by \zcref{eq:logit_derivative_upper_bound_pos_case_II} $\frac{d}{dt} \ell (\theta_t) < 0$, and so $\ell (\theta_\tau) < \ell (\theta_0)$.
This implies that:
\[
\exp \brk{\ell (\theta_\tau)} < \exp \brk{\ell(\theta_0)} < \frac{-(s-\|\phi(\ymed)\|^2)\advproxy(\ymed; \theta_0)}{(\|\phi(\ystar)\|^2-s) \advproxy(\ystar; \theta_0)} ,
\]
in contradiction to the definition of $\tau$.
Thus, for all $t \geq 0$, \zcref{eq:condition_for_logit_decrease} holds and $\frac{d}{dt} \ell(\theta_t) < 0$.

Now, examining $\pi_{\theta_t} (\ystar)$, we can see that:
\begin{align*}
    \pi_{\theta_t}(\ystar) &\le \frac{\pi_{\theta_t}(\ystar)}{\pi_{\theta_t}(\ymed) + \pi_{\theta_t}(\ystar)}\nonumber\\
    &= \frac{\pi_{\theta_t}(\ystar)/\pi_{\theta_t}(\ymed)}{1 + \pi_{\theta_t}(\ystar)/\pi_{\theta_t}(\ymed)}\nonumber\\
    &\le \frac{\pi_{\theta_0}(\ystar)/\pi_{\theta_0}(\ymed)}{1 + \pi_{\theta_0}(\ystar)/\pi_{\theta_0}(\ymed)} \\
    &= \frac{\pi_{\theta_0}(\ystar)}{\pi_{\theta_0}(\ystar)+\pi_{\theta_0}(\ymed)} ,
\end{align*}
where the second inequality is due to $\pi_{\theta_t}(\ystar)/\pi_{\theta_t}(\ymed) = \exp \brk{\ell(\theta_t)}$ and the fact that $\ell(\theta_t) \leq \ell(\theta_0)$ for all $t \geq 0$.
Thus,
\[
1 - \pi_{\theta_t}(\ystar) \geq \frac{\pi_{\theta_0}(\ymed)}{\pi_{\theta_0}(\ystar)+\pi_{\theta_0}(\ymed)} .
\]
Furthermore, since $\gtreward(z) < \gtreward(\ymed)$, for all $z \in \Ybad$, we know that
\begin{align*}
    \Vgt(\theta_t) &< \pi_{\theta_t}(\ystar)\gtreward(\ystar) + (1-\pi_{\theta_t}(\ystar))\gtreward(\ymed) .
\end{align*}
We may therefore conclude that for all $t \geq 0$:
\begin{align*}
    \gtreward(\ystar) - \Vgt(\theta_t) &> (1 - \pi_{\theta_t}(\ystar))\gtreward(\ystar) - (1 - \pi_{\theta_t}(\ystar))\gtreward(\ymed)\\
    &= (\gtreward(\ystar) - \gtreward(\ymed))(1 - \pi_{\theta_t}(\ystar))\\
    &\ge \frac{(\gtreward(\ystar) - \gtreward(\ymed))\pi_{\theta_0}(\ymed)}{\pi_{\theta_0}(\ystar) + \pi_{\theta_0}(\ymed)}\\
    &\ge (\gtreward(\ystar) - \gtreward(\ymed))\pi_{\theta_0}(\ymed),
\end{align*}
where the second inequality is from the lower bound on $1-\pi_{\theta_t}(\ystar)$ above, and the last inequality is from $\pi_{\theta_0}(\ystar) + \pi_{\theta_0}(\ymed) \le 1$.
\end{proof}

\subsection{Auxiliary Lemmas}
\label{app:proofs:auxiliary_lemmas}

In this appendix, we include auxiliary lemmas that are used throughout the proofs in \zcref{app:proofs}.

\begin{lemma}
\label{lem:gradient_expression}
Let $r : \Y \to [-1, 1]$.
For a linear softmax policy parameterized by $\theta \in \R^D$, denote the expected reward with respect to $r$ by $V (\theta) := \EE\nolimits_{y \sim \pi_\theta} \brk[s]{ r(y) }$.
Then:
\[
\nabla V (\theta) = \sum\nolimits_{y \in \Y} \pi_\theta (y) \adv (y ; \theta) \cdot \phi (y)
\text{\,,}
\]
where $\adv (y ; \theta) := r(y) - V (\theta)$ is the advantage of $y \in \Y$ under $r$ and $\pi_{\theta}$.
\end{lemma}

\begin{proof}
Since the probability of an output $y \in \Y$ under $\pi_{\theta}$ is given by
\[
\pi_\theta(y)=\frac{\exp(\langle \phi(y),\theta\rangle)}
{\sum\nolimits_{z \in \Y}\exp(\langle \phi (z),\theta\rangle)}
\text{\,,}
\]
the gradient of $\pi_\theta(y)$ with respect to $\theta$ can be written as:
\[
\nabla \pi_\theta(y)
=
\pi_\theta(y)\left(\phi(y)-\sum\nolimits_{z \in \Y}\pi_\theta(z) \cdot \phi(z)\right).
\]
Hence:
\begin{align*}
\nabla V (\theta) &=\sum\nolimits_{y \in \Y} r (y) \nabla \pi_\theta(y)\\
& =
\sum\nolimits_{y \in \Y}\pi_\theta(y)r (y) \cdot \phi(y)
-
\brk*{ \sum\nolimits_{y\in \Y}\pi_\theta(y) r (y) } \brk*{ \sum\nolimits_{z \in \Y}\pi_\theta(z) \cdot \phi(z) }\\
& =
\sum\nolimits_{y\in \Y}\pi_\theta(y)r (y) \cdot \phi(y)
-
V (\theta) \brk*{ \sum\nolimits_{z \in \Y}\pi_\theta(z) \cdot \phi(z) } \\
& = \sum\nolimits_{y \in \Y}\pi_\theta(y)\adv(y;\theta) \cdot \phi(y) .
\end{align*}
\end{proof}

\begin{lemma}
\label{lemma:l2_grad_bound}
For a linear softmax policy parameterized by $\theta \in \R^D$ and ground truth reward function $\gtreward$ upholding \zcref{assum:reward_structure_orthonormal}, it holds that:
\[
\norm*{ \nabla \Vgt(\theta) }
\leq
2B \brk*{ \Delta_1+\Delta_2 } \brk*{ 1-\pi_\theta(\ymed) }
\text{\,,}
\]
where $\Deltaone := \gtreward (\ystar) - \gtreward (\ymed)$, $\Deltatwo := \gtreward (\ymed) - \gtreward (\ybad)$ for some $\ybad \in \Ybad$, and $B := \max_{y \in \Y} \norm{ \phi (y) }$.
\end{lemma}

\begin{proof}
For a linear softmax policy $\pi_{\theta}$, by \zcref{lem:gradient_expression} we have that:
\[
\nabla \Vgt (\theta) = \sum\nolimits_{y \in \Y} \pi_\theta (y) \advgt (y ; \theta) \cdot \phi (y) .
\]
Since all outputs in $\Ybad$ have the same reward of $\gtreward(\ybad)$, we can write
\[
\begin{split}
\Vgt(\theta)
& = \pi_\theta(\ystar)\gtreward(\ystar) + \pi_\theta(\ymed)\gtreward(\ymed) + \pi_\theta(\Ybad)\gtreward(\ybad) \\
& = \gtreward(\ymed) + \pi_\theta(\ystar)\Delta_1 - \pi_\theta(\Ybad)\Delta_2 .
\end{split}
\]
Hence the three possible advantage values are:
\begin{align*}
\advgt(\ystar;\theta)
&= \gtreward(\ystar)-\Vgt(\theta)
= \brk*{ 1-\pi_\theta(\ystar) }\Delta_1 + \pi_\theta(\Ybad)\Delta_2,\\
\advgt(\ymed;\theta)
&= \gtreward(\ymed)-\Vgt(\theta)
= -\pi_\theta(\ystar)\Delta_1 + \pi_\theta(\Ybad)\Delta_2,\\
\advgt(z;\theta)
&= \gtreward(\ybad)-\Vgt(\theta)
= -\pi_\theta(\ystar)\Delta_1 - \brk*{ 1-\pi_\theta(\Ybad) }\Delta_2
, \qquad \text{for all } z \in \Ybad.
\end{align*}
Substituting these expressions into the gradient formula gives:
\[
\begin{split}
\nabla \Vgt(\theta)
& = \pi_\theta(\ystar)\advgt(\ystar;\theta)\phi(\ystar)
+ \pi_\theta(\ymed)\advgt(\ymed;\theta)\phi(\ymed) \\
& \hspace{4mm} + \sum\nolimits_{z \in \Ybad}\pi_\theta(z)\advgt(z;\theta)\phi(z) .
\end{split}
\]
Therefore, by the triangle inequality and the definition of $B$,
\begin{align*}
\norm{\nabla \Vgt(\theta)}
&\le B \brk*{
    \pi_\theta(\ystar)\advgt(\ystar;\theta)
    + \pi_\theta(\ymed)\abs{\advgt(\ymed;\theta)}
    + \sum\nolimits_{z \in \Ybad}\pi_\theta(z)\abs{\advgt(z;\theta)}
} \\
&= B \brk1{
    \pi_\theta(\ystar)\advgt(\ystar;\theta)
    + \pi_\theta(\ymed)\abs*{-\pi_\theta(\ystar)\Delta_1 + \pi_\theta(\Ybad)\Delta_2} } \\
    & \hspace{4mm} + B \brk1{ \pi_\theta(\Ybad)\brk*{ \pi_\theta(\ystar)\Delta_1 + \brk*{ 1-\pi_\theta(\Ybad) }\Delta_2 }  } . 
\end{align*}
Applying the triangle inequality, collecting the coefficients of $\Delta_1$ and $\Delta_2$, and using $\pi_\theta(\ystar)+\pi_\theta(\ymed)+\pi_\theta(\Ybad)=1$, we can further upper bound the above expression as:
\begin{align*}
\norm{\nabla \Vgt(\theta)}
&\le B \brk1{
    \pi_\theta(\ystar)\brk*{ \brk*{ 1-\pi_\theta(\ystar) }\Delta_1 + \pi_\theta(\Ybad)\Delta_2 }
    + \pi_\theta(\ymed)\brk*{ \pi_\theta(\ystar)\Delta_1 + \pi_\theta(\Ybad)\Delta_2 }
} \\
&\hspace{4mm} + B \brk1{\pi_\theta(\Ybad)\brk*{ \pi_\theta(\ystar)\Delta_1 + \brk*{ 1-\pi_\theta(\Ybad) }\Delta_2 }
} \\
&= B \brk1{
    2\pi_\theta(\ystar)\brk*{ 1-\pi_\theta(\ystar) }\Delta_1 + 2\pi_\theta(\Ybad)\brk*{ 1-\pi_\theta(\Ybad) }\Delta_2
} .
\end{align*}
Since $1-\pi_\theta(\ystar)\le 1$ and $1-\pi_\theta(\Ybad)\le 1$, we may conclude:
\begin{align*}
\norm{\nabla \Vgt(\theta)}
&\le 2B \brk*{ \pi_\theta(\ystar)\Delta_1 + \pi_\theta(\Ybad)\Delta_2 } \\
&\le 2B(\Delta_1+\Delta_2)\brk*{ \pi_\theta(\ystar)+\pi_\theta(\Ybad) } \\
&= 2B(\Delta_1+\Delta_2)\brk*{ 1-\pi_\theta(\ymed) } .
\end{align*}
\end{proof}

\begin{lemma}
\label{lemma:vgt_lipschitzness}
For a linear softmax policy parameterized by $\theta \in \R^D$ and a ground truth reward function $\gtreward : \Y \to [-1, 1]$, let $B := \max\nolimits_{y \in \Y} \norm{\phi (y)}$.
Then, the expected ground truth reward $\Vgt(\theta) = \EE\nolimits_{y \sim \pi_{\theta}} [\gtreward (y)]$ is $B$-Lipschitz with respect to the Euclidean norm, \ie, for any $\theta, \theta' \in \R^D$:
\[
\abs*{ \Vgt(\theta)-\Vgt(\theta') } \le B \norm*{ \theta-\theta' } .
\]
\end{lemma}

\begin{proof}
By \zcref{lem:gradient_expression}, for $r = \gtreward$ we have:
\[
\nabla \Vgt (\theta) = \sum\nolimits_{y \in \Y} \pi_\theta (y) \advgt (y ; \theta) \cdot \phi (y) .
\]
Since $\sum\nolimits_{y \in \Y}\pi_\theta(y)\advgt(y;\theta)=0$, we may rewrite this as:
\[
\nabla \Vgt (\theta) = \sum\nolimits_{y \in \Y} \pi_\theta (y) \advgt (y ; \theta) \cdot \brk*{ \phi(y) - \bar{\phi}_{\theta} } ,
\]
where $\bar{\phi}_{\theta} := \sum\nolimits_{z \in \Y}\pi_\theta(z)\phi(z)$.
Applying the triangle inequality followed by the Cauchy-Schwarz inequality gives:
\begin{align*}
    \norm*{ \nabla \Vgt (\theta) }
    &\le \sum\nolimits_{y \in \Y} \pi_\theta(y) \abs*{ \advgt (y ; \theta) } \norm*{ \phi(y) - \bar{\phi}_{\theta} } \\
    &\le \brk*{ \sum\nolimits_{y \in \Y} \pi_\theta(y) \advgt (y ; \theta)^2 }^{1/2}
    \brk*{ \sum\nolimits_{y \in \Y} \pi_\theta(y) \norm*{ \phi(y) - \bar{\phi}_{\theta} }^2 }^{1/2} \\
    &= \var\nolimits_{y \sim \pi_\theta}\brk[s]*{ \gtreward(y) }^{1/2}
    \brk*{ \EE\nolimits_{y \sim \pi_\theta}\brk[s]1{ \norm*{ \phi(y) }^2 } - \norm*{ \bar{\phi}_{\theta} }^2 }^{1/2} \\
    &\le B ,
\end{align*}
where the last inequality follows from the fact that $\gtreward(y) \in [-1,1]$ implies $\var_{y \sim \pi_\theta}\brk[s]*{ \gtreward(y) } \le 1$, and from $\EE_{y \sim \pi_\theta}\brk[s]1{ \norm{ \phi(y) }^2 } - \norm{ \bar{\phi}_{\theta} }^2 \le \EE_{y \sim \pi_\theta}\brk[s]1{ \norm{ \phi(y) }^2 } \le B^2$.
Therefore, by the mean value theorem,
\[
|\Vgt(\theta)-\Vgt(\theta')|
\le \sup\nolimits_{\xi \in [\theta,\theta']}\|\nabla \Vgt(\xi)\|\,\|\theta-\theta'\|
\le B\|\theta-\theta'\| .
\]
\end{proof}

\begin{lemma}
\label{lem:probability_derivative_expression}
Let $r : \Y \to [-1, 1]$.
For a linear softmax policy parameterized by $\theta \in \R^D$, denote the expected reward with respect to $r$ by $V(\theta) := \EE\nolimits_{y \sim \pi_\theta}\brk[s]{r(y)}$.
Suppose that gradient flow is used to maximize the expected reward with respect to $r$ (\zcref{eq:gf} with $V$ in place of $\Vproxy$).
Then, for any output $y \in \Y$,
\[
\frac{d}{dt}\pi_{\theta_t}(y)
=
\pi_{\theta_t}(y)
\inprod{
\phi(y) - \sum\nolimits_{z \in \Y}\pi_{\theta_t}(z)\phi(z)
}{
\sum\nolimits_{z \in \Y}\pi_{\theta_t}(z)\adv(z; \theta_t)\phi(z)
}
\text{\,,}
\]
where $\adv(z;\theta_t) := r(z) - V(\theta_t)$ is the advantage of $z \in \Y$ under $r$ and $\pi_{\theta_t}$.
Moreover, if the feature vectors $\{ \phi(z) \}_{z \in \Y}$ are orthonormal, then
\[
\frac{d}{dt}\pi_{\theta_t}(y)
=
\pi_{\theta_t}(y)\brk*{
\pi_{\theta_t}(y)\adv(y;\theta_t)
-
\sum\nolimits_{z \in \Y}\pi_{\theta_t}(z)^2\adv(z;\theta_t)
}
\text{\,.}
\]
\end{lemma}

\begin{proof}
By the chain rule and the gradient flow dynamics,
\[
\frac{d}{dt}\pi_{\theta_t}(y)
=
\inprod{\nabla \pi_{\theta_t}(y)}{\tfrac{d}{dt}\theta_t}
=
\inprod{\nabla \pi_{\theta_t}(y)}{\nabla V(\theta_t)} .
\]
Since $\pi_\theta$ is a linear softmax policy, we have
\[
\nabla \pi_{\theta_t}(y)
=
\pi_{\theta_t}(y)\brk*{\phi(y) - \sum\nolimits_{z \in \Y}\pi_{\theta_t}(z)\phi(z)} .
\]
Moreover, by \zcref{lem:gradient_expression},
\[
\nabla V(\theta_t)
=
\sum\nolimits_{z \in \Y}\pi_{\theta_t}(z)\adv(z ; \theta_t)\phi(z) .
\]
Substituting these two expressions into the derivative identity above gives the first claim.
If, in addition, the feature vectors are orthonormal, then $\inprod{\phi(y)}{\phi(z)} = \indc{y = z}$ for all $y, z \in \Y$, and therefore
\[
\inprod{
\phi(y) - \sum\nolimits_{z \in \Y}\pi_{\theta_t}(z)\phi(z)
}{
\sum\nolimits_{u \in \Y}\pi_{\theta_t}(u)\adv(u; \theta_t)\phi(u)
}
=
\pi_{\theta_t}(y)\adv(y;\theta_t)
-
\sum_{z \in \Y}\pi_{\theta_t}(z)^2\adv(z;\theta_t) .
\]
Plugging this back into the general expression yields the orthonormal-features formula.
\end{proof}

\begin{lemma}
\label{lem:mediocre_output_reward_larger_than_initial_expected_reward}
Under \zcref{assum:reward_structure_orthonormal,assum:initial_probs_orthonormal}, it holds that $\Vgt (\theta_0) = \EE\nolimits_{y \sim \pi_{\theta_0}} \brk[s]{ \gtreward (y) } < \gtreward (\ymed) - \sqrt{\gamma}$, for $\gamma:= \pi_{\theta_0}(\ystar)^{14 / 13} M^{-14/13}$ with $M$ defined as in \zcref{assum:initial_probs_orthonormal}.
\end{lemma}

\begin{proof}
Fix some $\ybad \in \Ybad$.
Under \zcref{assum:reward_structure_orthonormal,assum:initial_probs_orthonormal}, \zcref{assump: multi_s} establishes bounds on $\gamma$, $\pi_{\theta_0} (\ystar)$, and $\pi_{\theta_0} (\Ybad)$.
Based on these bounds, we can show that $\Vgt (\theta_0) < \gtreward (\ymed) - \sqrt{\gamma}$ as follows:
\begin{align*}
        \Vgt (\theta_0) &= \gtreward(\ymed)\pi_{\theta_0}(\ymed) + \gtreward(\ystar)\pi_{\theta_0}(\ystar) + \sum\nolimits_{z \in \Ybad}\gtreward(z)\pi_{\theta_0}(z)\nonumber\\
        &\le (1 - \pi_{\theta_0}(\Ybad))\gtreward(\ymed) + \gamma^{13/14} + \pi_{\theta_0}(\Ybad)\gtreward(\ybad)\nonumber\\
        &= \gtreward(\ymed) - \pi_{\theta_0}(\Ybad)\Delta_2  + \gamma^{13/14}\nonumber\\
        &\le \gtreward(\ymed) - 2\sqrt{\gamma} + \gamma^{13/14}\nonumber\\
        &< \gtreward(\ymed) - \sqrt{\gamma} ,
\end{align*}
The first inequality is from $\gtreward(\ystar) \le 1$ and $\pi_{\theta_0}(\ystar) \le \gamma^{13/14}$ (\zcref{assumpt: multi_ystar} in \zcref{assump: multi_s}); the second inequality is from $\pi_{\theta_0}(\Ybad) \ge 2\sqrt{\gamma}/\Delta_2$ (\zcref{assumpt: multi_ybad} in  \zcref{assump: multi_s}); and the last inequality is due to $\sqrt{\gamma} > \gamma^{13/14}$ for $\gamma \le 0.01$ (\zcref{assumpt: multi_gamma} in \zcref{assump: multi_s}).
\end{proof}

\begin{lemma}
\label{lem:mediocre_output_reward_larger_than_initial_expected_reward_neg_inner_product}
Under \zcref{assum:reward_structure_orthonormal,assum:initial_probs_neg_inner_product}, it holds that $\Vgt (\theta_0) = \EE\nolimits_{y \sim \pi_{\theta_0}} \brk[s]{ \gtreward (y) } < \gtreward (\ymed) - \sqrt{\gamma}$, for $\gamma := \pi_{\theta_0}(\ystar)^{14 / 13} M'^{-14/13}$ with $M'$ defined as in \zcref{assum:initial_probs_neg_inner_product}.
\end{lemma}

\begin{proof}
Fix some $\ybad \in \Ybad$.
Under \zcref{assum:reward_structure_orthonormal,assum:initial_probs_neg_inner_product}, \zcref{assump: loglin_neg_minor} establishes bounds on $\gamma$, $\pi_{\theta_0} (\ystar)$, and $\pi_{\theta_0} (\Ybad)$.
The proof of this lemma is analogous to that of \zcref{lem:mediocre_output_reward_larger_than_initial_expected_reward}, using the bounds provided by \zcref{assump: loglin_neg_minor} in place of those from \zcref{assump: multi_s}.
Based on these bounds, we can show that $\Vgt (\theta_0) < \gtreward (\ymed) - \sqrt{\gamma}$ as follows:
\begin{align*}
        \Vgt (\theta_0) &= \gtreward(\ymed)\pi_{\theta_0}(\ymed) + \gtreward(\ystar)\pi_{\theta_0}(\ystar) + \sum\nolimits_{z \in \Ybad}\gtreward(z)\pi_{\theta_0}(z)\nonumber\\
        &\le (1 - \pi_{\theta_0}(\Ybad))\gtreward(\ymed) + \gamma^{13/14} + \pi_{\theta_0}(\Ybad)\gtreward(\ybad)\nonumber\\
        &= \gtreward(\ymed) - \pi_{\theta_0}(\Ybad)\Delta_2  + \gamma^{13/14}\nonumber\\
        &\le \gtreward(\ymed) - 2\sqrt{\gamma} + \gamma^{13/14}\nonumber\\
        &< \gtreward(\ymed) - \sqrt{\gamma} .
\end{align*}
The first inequality is from $\gtreward(\ystar) \le 1$ and $\pi_{\theta_0}(\ystar) \le \gamma^{13/14}$ (\zcref{assump: loglin_neg_ystar} in \zcref{assump: loglin_neg_minor}); the second inequality is from $\pi_{\theta_0}(\Ybad) \ge 2\sqrt{\gamma}/\Delta_2$ (\zcref{assumpt: loglin_neg_ybad} in \zcref{assump: loglin_neg_minor}); and the last inequality is due to $\sqrt{\gamma} > \gamma^{13/14}$ for $\gamma \le 0.01$ (\zcref{assump: gamma_loglin_neg} in \zcref{assump: loglin_neg_minor}).
\end{proof}

\begin{lemma}
\label{lemma: max_prob_y_prime_s}
    Under \zcref{assum:reward_structure_orthonormal}, for all $\gamma > 0$, if $\Vgt (\theta_t) \le \gtreward(\ymed)-\gamma$, then $\pi_{\theta_t}(\ymed) \le 1 - \frac{\gamma}{1+\gtreward(\ymed)}$.
\end{lemma}
\begin{proof}
We will proceed by contradiction. Suppose to the contrary that $\pi_{\theta_t}(\ymed) > 1 - \frac{\gamma}{1 + \gtreward(\ymed)}$. Then,
\begin{align*}
    \Vgt (\theta_t) &= \pi_{\theta_t}(\ymed) \gtreward(\ymed) + \pi_{\theta_t}(\ystar) \gtreward(\ystar) + \sum_{z \in \Ybad} \pi_{\theta_t}(z) \gtreward(z)\\
    &> \bigg(1 - \frac{\gamma}{1+\gtreward(\ymed)}\bigg) \gtreward(\ymed) - \pi_{\theta_t}(\ystar)- \pi_{\theta_t}(\Ybad)\\
    &> \bigg(1 - \frac{\gamma}{1+\gtreward(\ymed)}\bigg) \gtreward(\ymed) - \frac{\gamma}{1+\gtreward(\ymed)}\\
    &= \gtreward(\ymed)-\gamma.
\end{align*}
Here, the first inequality is from our assumption on $\pi_{\theta_t}(\ymed)$, $\gtreward(\ymed)>0$ (\zcref{assum:reward_structure_orthonormal}), and $\gtreward(y) \ge -1$ for all $y \in \Y$; the second inequality is due to $\sum_{y \in \Y\setminus\{\ymed\}}\pi_{\theta_t}(y) < \frac{\gamma}{1 + \gtreward(\ymed)}$ also by assumption. 
This contradicts the premise that $\Vgt (\theta_t) \le \gtreward(\ymed) - \gamma$.
Hence, $\pi_{\theta_t}(\ymed) \le 1 - \frac{\gamma}{1+\gtreward(\ymed)}$, as desired.
\end{proof}

\begin{lemma}
\label{lemma: upperbound_ode}
For $T \in (0, \infty]$, let $g:[0,T)\to(0,\infty)$ be a continuously differentiable function satisfying
\[
\frac{d}{dt}g(t)\le c \cdot g(t)^p
\]
for all $t\in[0,T)$, where $c\neq0$ and $p>1$ are constants. Then it holds that
\[
g(t)\le 
\frac{g(0)}
{\big(1-(p-1)c\,g(0)^{p-1} \cdot t \big)^{\frac{1}{p-1}}}
\]
for all $t\in[0,T)$ such that the denominator is positive.
Similarly, if
\[
\frac{d}{dt}g(t) \ge c \cdot g(t)^p,
\]
then it holds that 
\[
g(t) \ge 
\frac{g(0)}
{\big(1-(p-1)c\,g(0)^{p-1} \cdot t \big)^{\frac{1}{p-1}}}
\]
for all $t\in[0,T)$ such that the denominator is positive, \ie:
\[
t<T \quad \text{and} \quad t<
\begin{cases}
\displaystyle \frac{1}{(p-1)c\,g(0)^{p-1}}, & c>0,\\[1em]
+\infty, & c<0.
\end{cases}
\]
\end{lemma}

\begin{proof}
Fix $t\in[0,T)$.
Since $g(s)>0$ for all $s\in[0,T)$, we may divide the inequality $g'(s)\le c\,g(s)^p$ by $g(s)^p$ and integrate over $[0,t]$ to obtain
\[
\int_0^t \frac{g'(s)}{g(s)^p}\,ds
\le c \cdot t .
\]
Since $\frac{d}{ds}g(s)^{-(p-1)} = -(p-1)g'(s)/g(s)^p$, we have that
\[
-\frac{1}{p-1}\brk*{g(t)^{-(p-1)} - g(0)^{-(p-1)}} \le c \cdot t .
\]
Rearranging this inequality, we then get
\[
g(t)^{-(p-1)} \ge g(0)^{-(p-1)}-(p-1)c \cdot t .
\]
Thus,
\[
g(t)\le \frac{g(0)}
{\big(1-(p-1)c\,g(0)^{p-1} \cdot t\big)^{1/(p-1)}}.
\]
If $c>0$, this requires
$t<1/((p-1)c\,g(0)^{p-1})$.
When $c<0$, the denominator is always positive, so the bound holds for all $t\in[0,T)$.
The proof of the lower bound from $g'(t) \ge c \cdot g(t)^p$ is analogous.
\end{proof}

	% FURTHER EXPERIMENTS AND IMPLEMENTATION DETAILS
	\section{Additional Experimental Results}
\label{app:experiments:further}

\subsection{Empirical Demonstration of Theoretical Results With Linear Softmax Policies}
\label{app:experiments:further:empirical_demonstration}

\zcref{fig:loglin_exps_sample_based_gradients} corroborates \zcref{fig:loglin_exps} by presenting the results of an identical experiment, except that the policy is trained using REINFORCE \citep{williams1992simple}, \ie, sample-based gradients of the expected proxy reward, instead of exact gradients.
The results show that outputs with mediocre proxy reward can impede optimization regardless of whether one uses exact or sample-based gradients.

\subsection{Comparison of Ranking Accuracy Variants}
\label{app:experiments:further:modified_acc}

Listed below are experimental results omitted from \zcref{sec:modified_acc}.

\begin{itemize}[leftmargin=6mm]
    \item \zcref{fig:rm_accuracy_with_gt_exps_test_rewards} supplements \zcref{fig:rm_accuracy_with_gt_exps} by reporting the correlation with ground truth reward increase and regret of ranking accuracy variants, when measuring reward increase on test examples from the UltraFeedback dataset as opposed to examples from the policy gradient training set.
    The results are analogous to \zcref{fig:rm_accuracy_with_gt_exps} since, as can be seen in \zcref{tab:rm-selection-llama-3b-instruct-ultrafeedback,tab:rm-selection-llama-1b-instruct-ultrafeedback,tab:rm-selection-olmo-1b-sft-ultrafeedback,tab:rm-selection-qwen-1-7b-base-ultrafeedback}, the reward increase on test examples is nearly identical to that on training examples.

    \item \zcref{fig:rm_accuracy_with_gt_exps_rb2} supplements \zcref{fig:rm_accuracy_with_gt_exps} by considering ranking accuracy variants computed on RewardBench2 instead of on UltraFeedback examples from the policy gradient training set.
    Across language models, accuracy values computed on RewardBench2 are substantially less predictive of which reward model leads to better language model performance, as indicated by the negative correlation and high regret.
    This highlights the importance of evaluating reward models on prompts close in distribution to those used for policy gradient training.

    \item \zcref{fig:rm_accuracy_with_gt_exps_diff_dataset_diff_gt} supports the experiments of \zcref{fig:rm_accuracy_with_gt_exps} by showing that similar trends persist when using: \emph{(i)} a different dataset for policy gradient training and reward model evaluation (WildChat-IF-On-Policy-8B \citep{lambert2024tulu} instead of UltraFeedback); or \emph{(ii)} a different ground truth reward model (Skywork-Reward-V2-Llama-3.1-8B \citep{liu2025skywork} instead of ArmoRM).

    \item \zcref{fig:rm_accuracy_with_winrates} supports the experiments of \zcref{fig:rm_accuracy_with_gt_exps} by showing that similar trends persist when evaluating language model performance through win-rate against the initial language model, according to a frontier (GPT) judge model, instead of through ground truth reward increase.

    \item Tables~\ref{tab:rm-selection-llama-3b-instruct-ultrafeedback} to \ref{tab:rm-selection-qwen-1-7b-base-ultrafeedback-win-rate} include the numerical results based on which \zcref{fig:rm_accuracy_with_gt_exps,fig:rm_accuracy_with_gt_exps_test_rewards,fig:rm_accuracy_with_gt_exps_rb2,fig:rm_accuracy_with_gt_exps_diff_dataset_diff_gt,fig:rm_accuracy_with_winrates} were generated.
    Specifically, for each setting (\ie, language model, reward model, and dataset), they report the increase in ground truth reward due to policy gradient training (or win-rate against the initial language model), as well as reward model accuracy values.
\end{itemize}

\subsection{When Should Partially Correct Outputs Be Rewarded?}
\label{app:experiments:further:reward_design}

Listed below are experimental results omitted from \zcref{sec:reward_design}.

\begin{itemize}[leftmargin=6mm]
    \item \zcref{fig:rlvr_partial_rewards_llama,fig:rlvr_partial_rewards_olmo} present analogous results to \zcref{fig:rlvr_partial_rewards_qwen} for the Llama-3.2-3B-Instruct and OLMo-2-1B-Instruct language models (instead of Qwen3-1.7B), respectively.
    
    \item \zcref{fig:rlvr_partial_rewards_qwen_gap_prob_but_partial_works} demonstrates that the probability of initially satisfying a constraint is not the sole factor determining its ease of learnability.
    In particular, it shows that when one constraint is initially satisfied with substantially higher probability than the other, it is possible for both constraints to be learned quickly under partial rewards.
    In such cases, under binary rewards (\ie, rewarding only fully correct outputs), the constraints can either also be learned quickly or fail to be learned, especially when the initial probability of satisfying both constraints is near-zero.

    \item \zcref{tab:lm-constraints-summary} specifies the IFBench constraint pairs used in each of the experiments corresponding to \zcref{fig:rlvr_partial_rewards_qwen,fig:rlvr_partial_rewards_llama,fig:rlvr_partial_rewards_olmo,fig:rlvr_partial_rewards_qwen_gap_prob_but_partial_works}.
\end{itemize}

\section{Additional Implementation Details}
\label{app:experiments:details}

In this appendix, we provide implementation details omitted from the main text and \zcref{app:experiments:further}.
Code for reproducing our results, based on the PyTorch~\citep{paszke2017automatic}, Hugging Face TRL~\citep{wolf2019huggingface}, and open-instruct~\citep{lambert2024tulu,olmo2025olmo} libraries,\ifdefined\CAMREADY
~can be found at \url{https://github.com/princeton-pli/imperfect-rewards}.
\else
~will be made publicly available.
\fi

\subsection{Empirical Demonstration of Theoretical Results With Linear Softmax Policies}
\label{app:experiments:details:empirical_demonstration}

\textbf{Output features.}
The feature vectors are of dimension $D = 5$ and the number of outputs is $5$ in all experiments.
For \zcref{fig:loglin_exps,fig:loglin_exps_sample_based_gradients}, we use standard basis vectors as output features.
Specifically, $\phi(\ystar) = (1,0,0,0,0)$, $\phi(\ymed) = (0,1,0,0,0)$, and the remaining three outputs have feature vectors $(0,0,1,0,0)$, $(0,0,0,1,0)$, and $(0,0,0,0,1)$.
For \zcref{fig:loglin_exps_features_sim}, the feature vectors of outputs other than $\ystar$ and $\ymed$ are the same as in the experiments of \zcref{fig:loglin_exps,fig:loglin_exps_sample_based_gradients}.
The feature vectors of $\ystar$ and $\ymed$ are set as follows.
\begin{itemize}[leftmargin=6mm]
    \item In the case of $\inprod{ \phi (\ystar) }{ \phi (\ymed) } = 0$:
    \[
        \phi (\ystar) = (3 / 2,0,0,0,0) \quad , \quad \phi (\ymed) = (0,1,0,0,0)
        \text{\,.}
    \]

    \item In the case of $\inprod{ \phi (\ystar) }{ \phi (\ymed)} < 0$:
    \[
        \phi (\ystar) = (3 / 2,0,0,0,0) \quad , \quad \phi (\ymed) = (- 1 / \sqrt{2}, 1 / \sqrt{2}, 0, 0,0)
        \text{\,.}
    \]
    
    \item In the case of $\inprod{ \phi (\ystar) }{ \phi (\ymed)} > 0$:
    \[
        \phi (\ystar) = (3/2 ,0,0,0,0) \quad , \quad \phi (\ymed) = (1 / \sqrt{2}, 1 / \sqrt{2}, 0, 0,0)
        \text{\,.}
    \]
\end{itemize}
Note that the norm of $\phi (\ystar)$ is taken to be larger than that of $\phi (\ymed)$ to accord with the conditions of \zcref{thm:attraction_to_mediocre_outputs_beneficial_error_pos_inner_product}.
We observed trends analogous to those in \zcref{fig:loglin_exps_features_sim} when setting the first coordinate of $\phi (\ystar)$ to $1$ instead of $3 /2$, in which case all feature vectors have unit norm.

\textbf{Initial policy.}
For the experiments of \zcref{fig:loglin_exps,fig:loglin_exps_sample_based_gradients}, we set the initial policy weights $\theta_0$ such that $\pi_{\theta_0} (\ystar)$ is either $0.05$, $0.1$, or $0.15$, $\pi_{\theta_0} (\ymed) = 0.5$, and the remaining probability mass is uniformly divided between the remaining three outputs.
Since the feature vectors in these experiments are standard basis vectors, we can achieve this by setting the coordinate of $\theta_0$ associated with each output $y \in \Y$ to the logarithm of the desired probability.
For the experiments of \zcref{fig:loglin_exps_features_sim}, we set $\theta_0$ such that $\pi_{\theta_0} (\ystar) = 0.05$, $\pi_{\theta_0} (\ymed) = 0.5$, and the remaining probability mass is uniformly divided between the remaining three outputs.
This is done by using a linear equations solver to find a $\theta_0$ satisfying $\inprod{ \phi (y) }{ \theta_0 } = \ln \pi (y)$ for all $y \in \Y$, where $\pi (y)$ is the desired initial probability of $y$.

\textbf{Policy gradient optimization.}
In the experiments of \zcref{fig:loglin_exps,fig:loglin_exps_features_sim}, we train linear softmax policies by running gradient ascent over the expected proxy or ground truth rewards using exact gradients.
In \zcref{fig:loglin_exps_sample_based_gradients}, we use REINFORCE \citep{williams1992simple}.
Namely, at each training step, we sample an output $y$ from the current policy and update the parameters via $\theta_{t + 1} = \theta_t + \eta \cdot \proxyreward(y) \cdot \nabla_\theta \ln \pi_{\theta_t}(y)$, where $\proxyreward(y) \cdot \nabla_\theta \ln \pi_{\theta_t}(y)$ is an unbiased estimate of $\nabla \Vproxy (\theta_t)$ and $\eta > 0$ is the learning rate.
We set the learning rate to $0.1$ in all experiments.
Each figure caption specifies the reward assignments in the corresponding experiments.

\textbf{Hardware.}
Standard laptop CPU.

\subsection{Comparison of Ranking Accuracy Variants}
\label{app:experiments:details:modified_acc}

\textbf{Language models.}
We consider as initial policies four language models that cover different model families and types (\ie, pretrained, supervised finetuned, and instruction-tuned): Llama-3.2-3B-Instruct, Llama-3.2-1B-Instruct, OLMo-2-1B-SFT, and Qwen3-1.7B-Base.
When processing inputs to the language models we use their default chat templates.

\textbf{Reward models.}
\zcref{tab:rms} lists the reward models used in the experiments of \zcref{sec:modified_acc,app:experiments:further:modified_acc}.
The reward models were chosen so that they cover a range of RewardBench2 scores \citep{malik2025rewardbench}.
When using Skywork-Reward-V2-Llama-3.1-8B as the ground truth reward model (\zcref{fig:rm_accuracy_with_gt_exps_diff_dataset_diff_gt}), we exclude it from the reward models used for training and include instead the ArmoRM \citep{wang2024interpretable} model.

\begin{table}[H]
    \vspace{0mm}
	\caption{
        Listed are the reward models used in the experiments of \zcref{sec:modified_acc,app:experiments:further:modified_acc}.
        The RewardBench2 scores are taken from the official leaderboard \citep{malik2025rewardbench}.
        The two models without a RewardBench2 score were chosen and positioned in the table based on their original RewardBench score \citep{lambert2025rewardbench}.
	}
	\vspace{0mm}
    \begin{center}
    \fontsize{8.5}{9.5}\selectfont
    \begin{tabular}{lc}
        \toprule
        Reward Model & RewardBench2 Score \\
        \midrule
        Skywork-Reward-V2-Llama-3.1-8B & 84.1 \\
        Skywork-Reward-V2-Qwen3-8B & 78.4 \\
        Skywork-Reward-V2-Qwen3-4B & 75.5 \\
        Llama-3.1-8B-Instruct-RM-RB2 & 72.8 \\
        Skywork-Reward-V2-Qwen3-1.7B & 68.2 \\
        GRM-Llama3.2-3B-rewardmodel-ft & -- \\
        Llama-3-OffsetBias-RM-8B & 64.8 \\
        Skywork-Reward-V2-Llama-3.2-1B & 64.4 \\
        Skywork-Reward-V2-Qwen3-0.6B & 61.2 \\
        RM-Mistral-7B & 59.6 \\
        internlm2-1\_8b-reward & 39.0 \\
        llama-3-tulu-2-8b-uf-mean-rm & -- \\
        RM-Gemma-2B & 30.6 \\
        \bottomrule
    \end{tabular}
    \end{center}
    \label{tab:rms}
\end{table}

\textbf{Reward normalization.}
To ensure fair comparison across reward models, which can differ in the scale of rewards they produce, we normalize rewards in each batch of training by subtracting the mean and dividing by the standard deviation.

\textbf{Data.}
We take the binarized version of UltraFeedback\footnote{
    \url{https://huggingface.co/datasets/HuggingFaceH4/ultrafeedback_binarized}
} \citep{cui2024ultrafeedback} and filter out examples in which the prompt or one of the outputs exceeds 512 tokens.
We then select a subset of 10000 examples from the training set (\texttt{train\_prefs} split), relabel output preferences in the selected subset and the test set (\texttt{test\_prefs} split) using the ground truth reward model (ArmoRM or Skywork-Reward-V2-Llama-3.1-8B, depending on the experiment), and filter out examples where both outputs have the same ground truth reward.
For the ArmoRM ground truth reward model, the resulting training and test sets contain 9808 and 1406 examples, respectively, while for the Skywork-Reward-V2-Llama-3.1-8B ground truth reward model, the resulting training and test sets contain 9930 and 1422 examples, respectively.
For experiments on the WildChat-IF-On-Policy-8B dataset\footnote{
    \url{https://huggingface.co/datasets/allenai/tulu-3-wildchat-if-on-policy-8b}
} (\zcref{fig:rm_accuracy_with_gt_exps_diff_dataset_diff_gt}), which does not have a test set, we follow the same procedure using 3000 randomly selected examples from the dataset and split the examples remaining after filtering into a training set of 1969 examples and a test set of 985 examples.

\textbf{Policy gradient optimization.}
Our RLOO implementation is based on the RLOOTrainer class from the TRL framework, which uses the Adam optimizer.
We set the learning rate to 1e-7, batch size to 32, KL regularization coefficient to 0.01, and the \texttt{num\_mini\_batches} hyperparameter to 1.
Optimization is carried out for two full passes over the prompts in the training set, where for each prompt in a batch we sample four outputs.

\textbf{Generation hyperparameters.}
For both training and evaluation, we generate outputs from the policies (\ie, language models) using a temperature of 1 and a maximum output length of 512 tokens.

\textbf{Ground truth reward increase policy evaluation.}
We evaluate the ground truth reward achieved by a policy based on fixed (randomly selected) subsets of 500 prompts from the policy gradient training and test sets.
We sample 10 outputs from the policy for each prompt, compute their ground truth rewards, and then average the rewards across all outputs and prompts.
As noted in \zcref{sec:modified_acc}, we mainly consider the ground truth reward increase on training examples as it directly reflects policy gradient optimization, which is the focus of our work, without conflating it with generalization.
Nonetheless, as reported in \zcref{app:experiments:further}, we found the reward increase on test examples to be nearly identical.
For each combination of reward model and initial policy, the reported ground truth reward increase is the mean across three separate runs with different random seeds.

\textbf{Win-rate policy evaluation.} 
For \zcref{fig:rm_accuracy_with_winrates}, instead of using a ground truth reward model for policy evaluation, we adopt the AlpacaEval framework \citep{alpaca_eval} to compute win-rates against the initial policy according to a GPT judge model.
The win-rates are computed for a single policy gradient run per combination of reward model and initial policy, based on the same 500 prompts used for evaluation via ground truth reward increase.
We sample one output from the initial policy and one from the final policy, for each prompt, and ask the judge model to select the better output using the format provided by AlpacaEval.
For all models, except Qwen3-1.7B-Base, \texttt{gpt-4-1106-preview} served as the judge model.
For Qwen3-1.7B-Base we used \texttt{gpt-4.1-2025-04-14} as the judge since the win-rates for this model were computed after the deprecation of \texttt{gpt-4-1106-preview}.

\textbf{Reward model evaluation.}
For each reward model and initial policy in \zcref{fig:rm_accuracy_with_gt_exps,fig:rm_accuracy_with_gt_exps_test_rewards,fig:rm_accuracy_with_gt_exps_diff_dataset_diff_gt,fig:rm_accuracy_with_winrates}, we compute all ranking-accuracy variants (Acc, Acc-W, HAcc, and HAcc-W) on the same 500 training prompts used for policy evaluation, where output preferences are relabeled using the ground truth reward model.
In \zcref{fig:rm_accuracy_with_gt_exps_rb2}, accuracy values are computed on RewardBench2, after excluding examples from the ``Ties'' subset and  relabeling preferences using the ground truth reward model.
Lastly, to compute HAcc and HAcc-W, we estimate the mean proxy reward achieved by an initial policy for each prompt using 10 sampled outputs.

\textbf{Hardware.}
All experiments ran on Nvidia H100 GPUs with 80GB memory.
For the Llama-3.2-1B-Instruct and OLMo-2-1B-SFT models, we used two GPUs per policy gradient run.
For the Llama-3.2-3B-Instruct and Qwen3-1.7B-Base models, we used four GPUs per policy gradient run.
Evaluations were carried out on a single GPU.

\subsection{When Should Partially Correct Outputs Be Rewarded?}
\label{app:experiments:details:reward_design}

\textbf{Language models.}
We consider as initial policies three instruction-tuned language models: Qwen3-1.7B (thinking disabled), OLMo-2-1B-Instruct, and Llama-3.2-3B-Instruct.
When processing inputs to the language models we use their default chat templates.

\textbf{Data.}
We take the binarized version of UltraFeedback and randomly sample 4000 prompts that do not exceed 512 tokens from the \texttt{train\_prefs} split.
Then, we construct instruction following datasets by appending to each prompt a pair of constraints from IFBench \citep{pyatkin2025generalizing} using the following template:
\begin{quote}
\begin{center}
\texttt{\{original prompt\} \{constraint 1\} \{constraint 2\}}
\end{center}
\end{quote}
All prompts within a dataset share the same pair of constraints and the datasets differ in which constraints they contain.
For each initial policy, we experimented with multiple pairs of constraints and report the results for pairs exhibiting representative behaviors.

\textbf{Reward design.}
We consider two reward schemes.
The “partial rewards'' scheme corresponds to giving a reward of 0.5 for each constraint satisfied.
Thus, an output receives reward 0, 0.5, or 1 depending on whether it satisfies none, exactly one, or both constraints. 
The “binary (full-correctness) rewards'' scheme corresponds to assigning a reward of 1 to fully correct outputs, satisfying both constraints, and 0 to all other outputs.

\textbf{Policy gradient optimization.}
We rely on the open-instruct framework and train in bfloat16 mixed precision using its default DAPO objective implementation \citep{yu2025dapo}. 
Under our configuration, which adopts symmetric clipping (with $\varepsilon = 0.2$), disables dynamic sampling, sets \texttt{non\_stop\_penalty} to false, and does not use KL regularization, the objective reduces to the GRPO objective with token-level normalization. 
Each batch consists of 16 unique prompts.
The number of sampled outputs per prompt varies across language models: 96 for Qwen3-1.7B, 64 for Llama-3.2-3B-Instruct, and 32 for OLMo-2-1B-Instruct.
Optimization is carried out for four full passes over the training prompts using the Adam optimizer and a linear decay learning rate schedule with warmup ratio 0.03 and base learning rate 1e-6.
We chose these relatively large numbers of outputs per prompt together with this learning rate schedule to promote stable training.
Note that for the learning rate decay schedule we set the number of epochs to eight, but ended up stopping early after four epochs (1000 training steps) since we observed that the probabilities of success tend to plateau by that point.

\textbf{Policy evaluation.}
The probability of satisfying a constraint (or both constraints jointly) is estimated at each training step by the fraction of outputs in a batch that satisfy the constraint (or both constraints).

\textbf{Generation hyperparameters.}
We generate outputs from the policies (\ie, language models) using a temperature of 1 and a maximum output length of 512 tokens.

\textbf{Hardware.}
Each experiment ran on two Nvidia H100 GPUs with 80GB memory.

\textbf{Effect of randomness on the success of binary (full-correctness) rewards.}
When the initial probability of satisfying both constraints is low, success under binary rewards can vary across runs since it depends on fully correct outputs being sampled often enough during training.
This phenomenon persists even when fixing the random seed due to non-determinism in GPU computation.
In particular, in the setting of \zcref{fig:rlvr_partial_rewards_qwen}, binary rewards led to successful learning of both constraints in roughly 50\% of our runs, whereas in the remaining runs the policy failed to learn either constraint.
When the initial probability of satisfying both constraints is not low, we found the outcome under binary rewards to be more consistent across runs.

\newpage

\begin{figure}[H]
	\vspace{0mm}
	\begin{center}
        \includegraphics[width=1\textwidth]{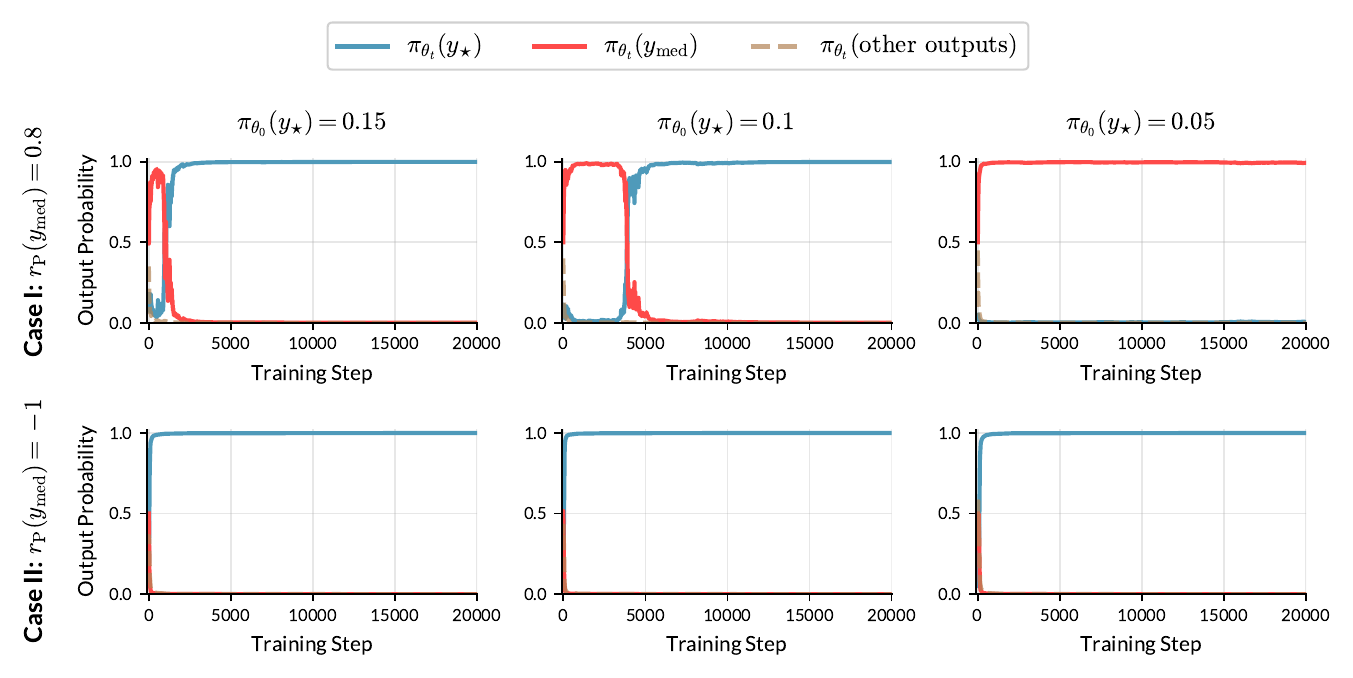}
	\end{center}
	\vspace{-2.25mm}
	\caption{
       	\textbf{Attraction to outputs with mediocre proxy reward can impede policy gradient optimization (with sample-based gradients).}
        This figure presents the results of an experiment identical to that of \zcref{fig:loglin_exps}, except that the policy is trained using REINFORCE \citep{williams1992simple}, \ie, sample-based gradients of the expected proxy reward, instead of exact gradients.
        As when training with exact gradients of the expected proxy reward, outputs with mediocre proxy reward can impede optimization.
        See \zcref{app:experiments:details} for additional implementation details.
	}
	\label{fig:loglin_exps_sample_based_gradients}
\end{figure}

\begin{figure}[H]
	\vspace{0mm}
	\begin{center}
        \includegraphics[width=1\textwidth]{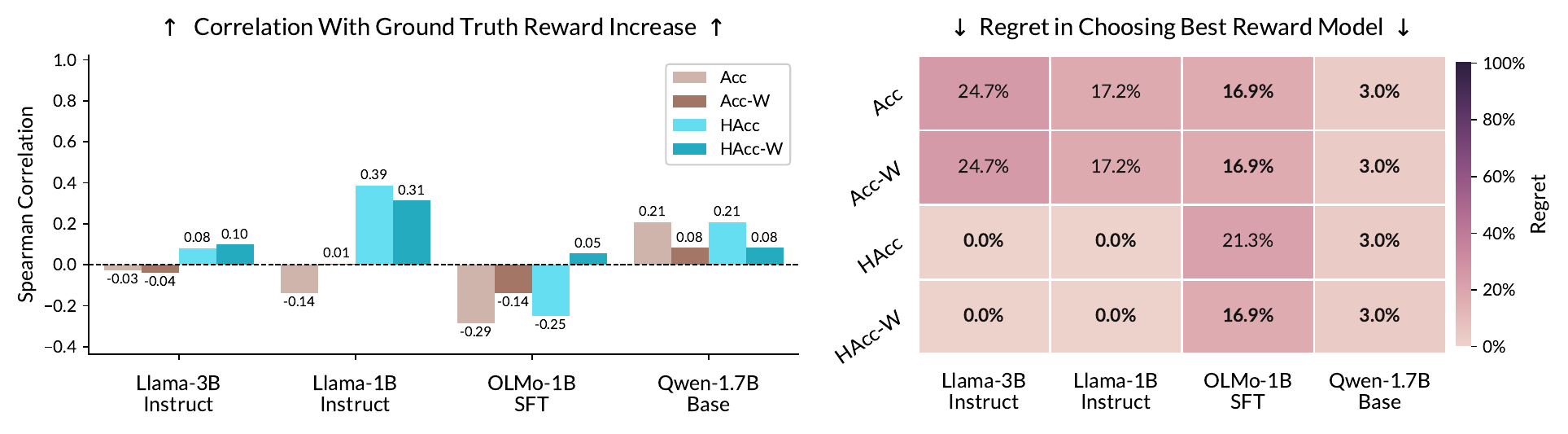}
	\end{center}
	\vspace{-2.25mm}
	\caption{
        \textbf{Harm-aware ranking accuracy variants are more predictive of which reward model leads to better language model performance.}
        This figure supplements \zcref{fig:rm_accuracy_with_gt_exps} by reporting the same metrics, but based on the reward increase over test examples from the UltraFeedback dataset.
        The results are analogous to \zcref{fig:rm_accuracy_with_gt_exps}, where reward increase is measured on UltraFeedback examples used for policy gradient training.
        Specifically, as can be seen in \zcref{tab:rm-selection-llama-3b-instruct-ultrafeedback,tab:rm-selection-llama-1b-instruct-ultrafeedback,tab:rm-selection-olmo-1b-sft-ultrafeedback,tab:rm-selection-qwen-1-7b-base-ultrafeedback}, the reward increase on test examples is nearly identical to that on training examples.
        Additional implementation details are provided in \zcref{app:experiments:details}.
    }
	\label{fig:rm_accuracy_with_gt_exps_test_rewards}
\end{figure}

\begin{figure}[H]
	\vspace{0mm}
	\begin{center}
        \includegraphics[width=1\textwidth]{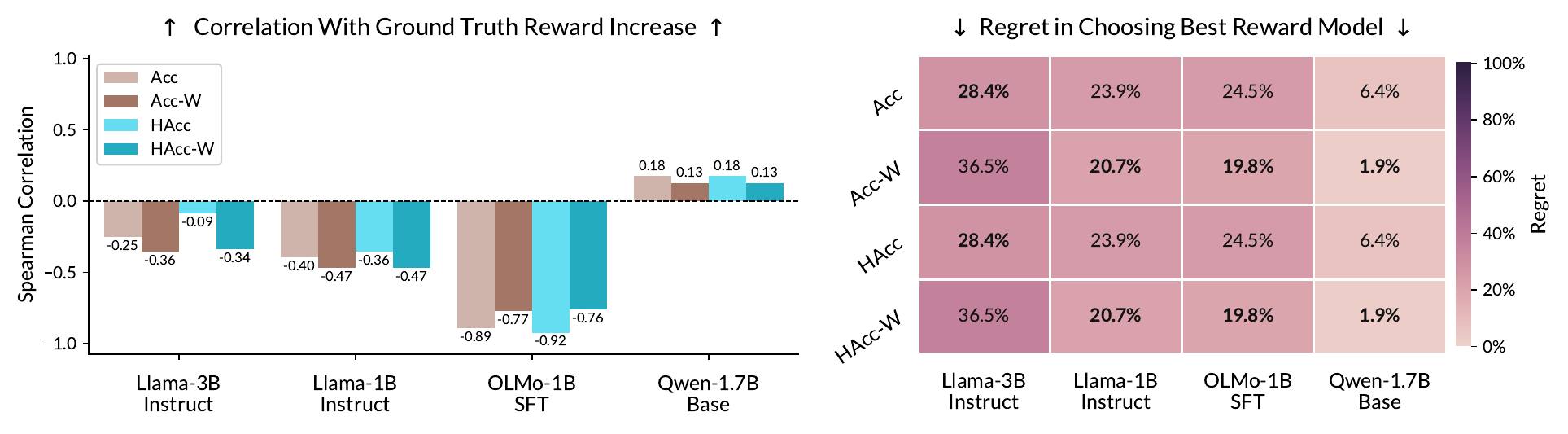}
	\end{center}
	\vspace{-2.25mm}
	\caption{
        \textbf{Ranking accuracy variants computed on prompts differing from those used for policy gradient training are less predictive of which reward model leads to better language model performance.}
        This figure supplements \zcref{fig:rm_accuracy_with_gt_exps} by considering ranking accuracy variants computed on RewardBench2 instead of on UltraFeedback examples from the policy gradient training set.
        Across language models, accuracy values computed on RewardBench2 are substantially less predictive of which reward model leads to better language model performance, as indicated by the negative correlation and high regret.
        This highlights a pitfall of evaluating reward models based on prompts that differ from those used for training the language model.
        See \zcref{app:experiments:details} for additional implementation details.
	}
	\label{fig:rm_accuracy_with_gt_exps_rb2}
\end{figure}

\begin{figure}[H]
	\vspace{0mm}
	\begin{center}
        \includegraphics[width=0.9\textwidth]{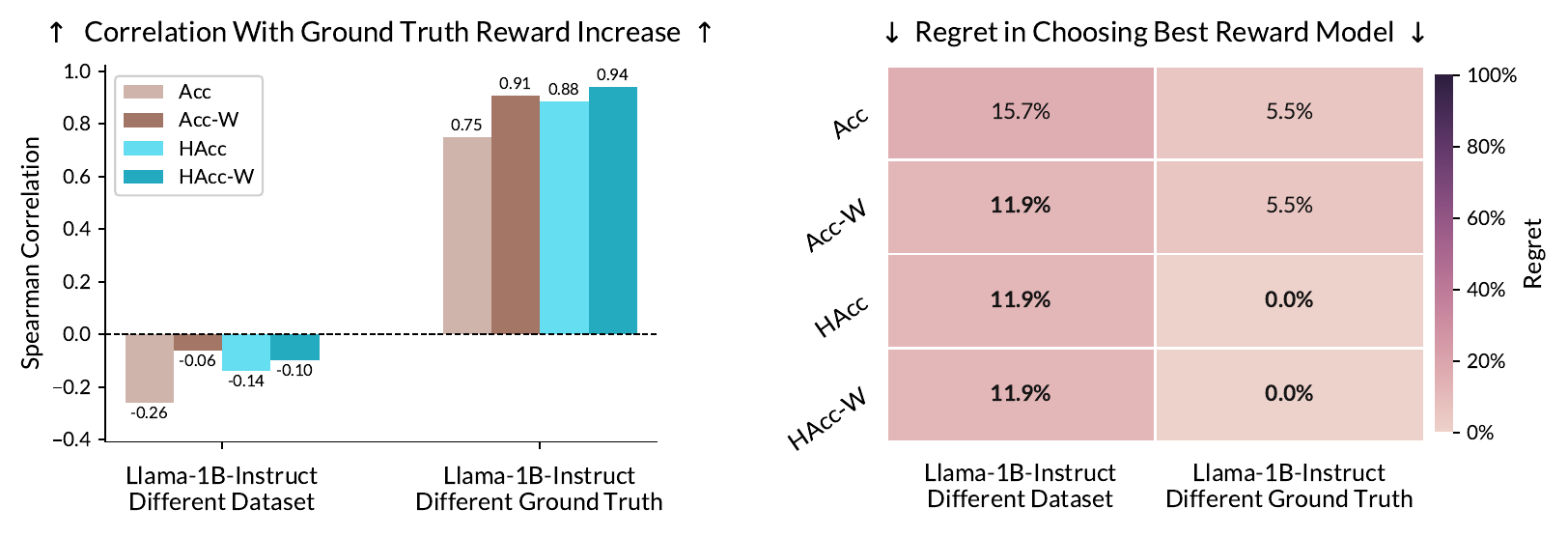}
	\end{center}
	\vspace{-2.25mm}
	\caption{
        \textbf{Harm-aware ranking accuracy variants are more predictive of which reward model leads to better language model performance.}
        This figure supports the experiments of \zcref{fig:rm_accuracy_with_gt_exps} by showing the same trends persist when using: \emph{(i)} a different dataset for policy gradient training and reward model evaluation (WildChat-IF-On-Policy-8B \citep{lambert2024tulu} instead of UltraFeedback); or \emph{(ii)} a different ground truth reward model for evaluating language model performance and accuracy values (Skywork-Reward-V2-Llama-3.1-8B \citep{liu2025skywork} instead of ArmoRM).
        Both sets of experiments were conducted with the Llama-3.2-1B-Instruct language model.
        See \zcref{app:experiments:details} for additional implementation details.
	}
	\label{fig:rm_accuracy_with_gt_exps_diff_dataset_diff_gt}
\end{figure}

\begin{figure}[H]
	\vspace{0mm}
	\begin{center}
        \includegraphics[width=1\textwidth]{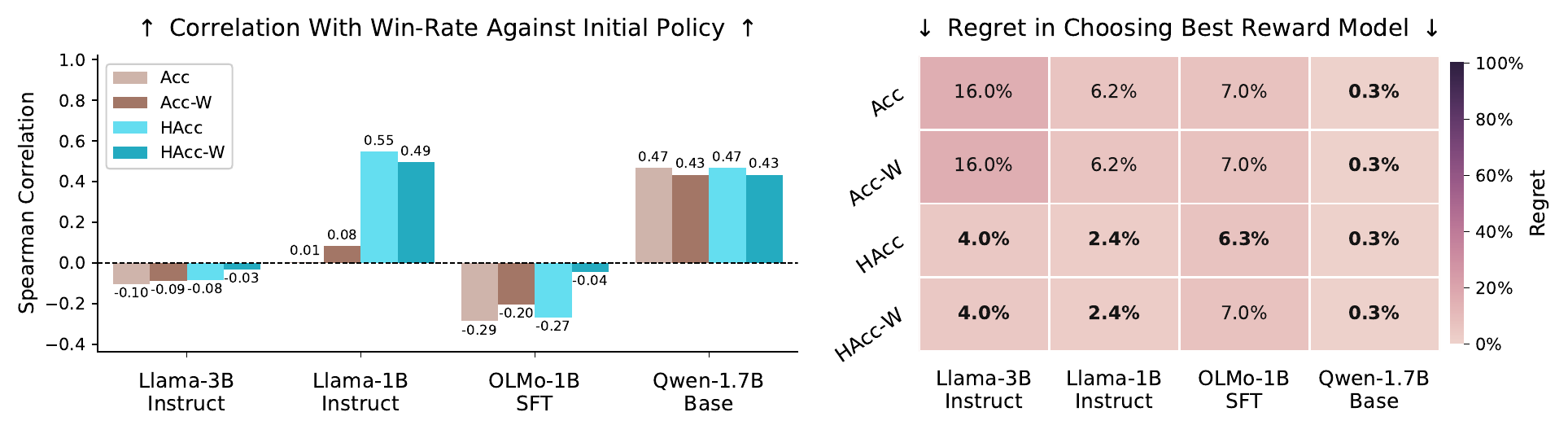}
	\end{center}
	\vspace{-2.25mm}
	\caption{
        \textbf{Harm-aware ranking accuracy variants are more predictive of which reward model leads to better language model performance.}
        This figure supports \zcref{fig:rm_accuracy_with_gt_exps} by showing that similar trends persist when evaluating language model performance through win-rate against the initial language model, according to a frontier (GPT) judge model, instead of through ground truth reward increase.
        See \zcref{app:experiments:details} for additional implementation details.
	}
	\label{fig:rm_accuracy_with_winrates}
\end{figure}

\begin{figure}[H]
	\vspace{0mm}
	\begin{center}
        \includegraphics[width=1\textwidth]{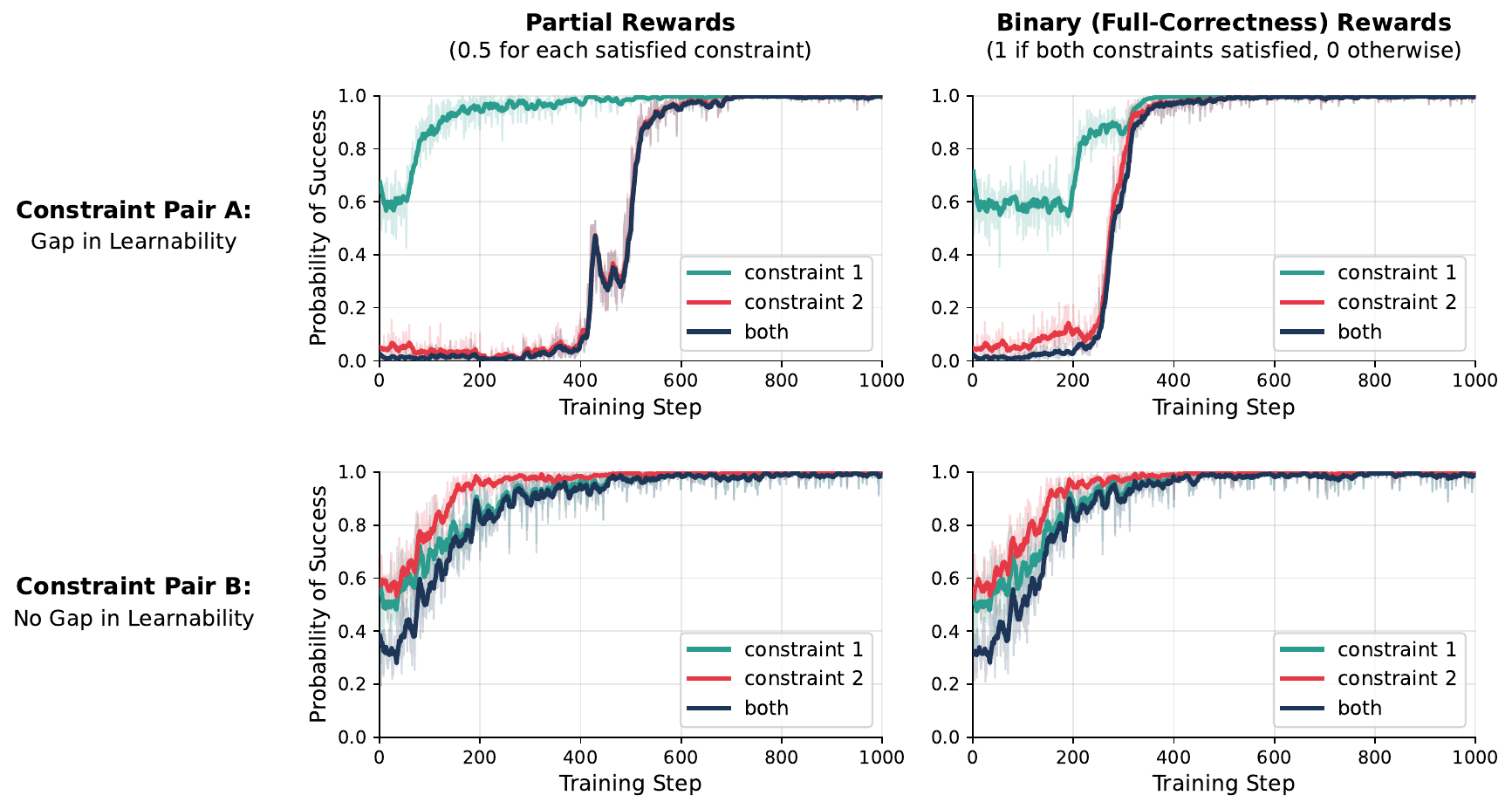}
	\end{center}
	\vspace{-2.25mm}
	\caption{
        \textbf{Rewarding partially correct outputs can impede policy gradient optimization.}
        This figure presents the results of an experiment analogous to that of \zcref{fig:rlvr_partial_rewards_qwen}, where the language model is Llama-3.2-3B-Instruct instead of Qwen3-1.7B.
        Similarly to \zcref{fig:rlvr_partial_rewards_qwen}, when the initial probability of satisfying one constraint is noticeably higher than the probability of satisfying the other, rewarding partial correctness can cause the policy to stall on satisfying only the easier constraint.
        In this case, rewarding only fully correct outputs can lead to faster learning of both constraints.
        However, when both constraints are initially satisfied with similar probability, both reward designs tend to work well.
        See \zcref{app:experiments:details} for further implementation details.
    }
	\label{fig:rlvr_partial_rewards_llama}
\end{figure}

\begin{figure}[H]
	\vspace{0mm}
	\begin{center}
        \includegraphics[width=1\textwidth]{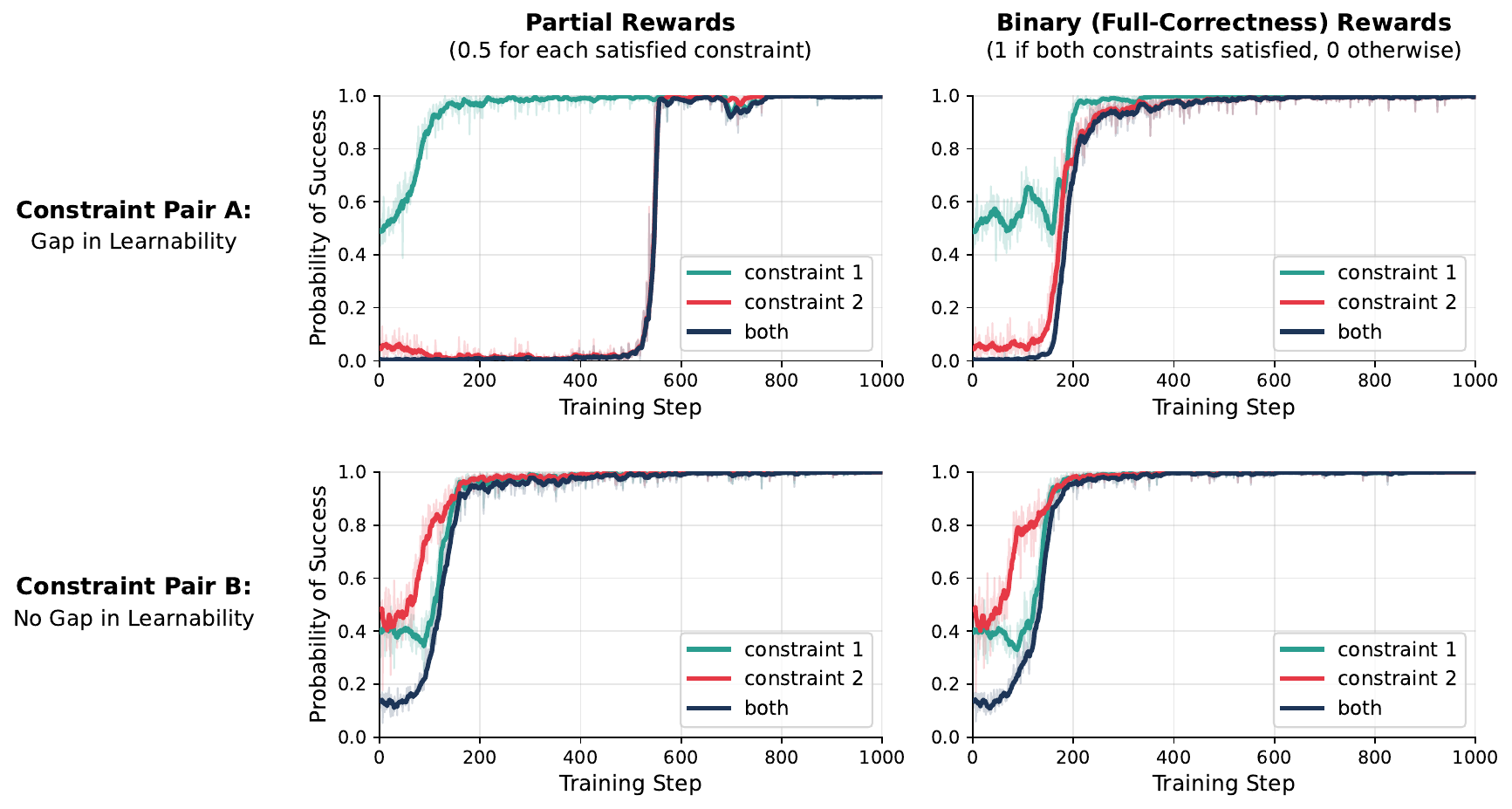}
	\end{center}
	\vspace{-2.25mm}
	\caption{
        \textbf{Rewarding partially correct outputs can impede policy gradient optimization.}
        This figure presents the results of an experiment analogous to that of \zcref{fig:rlvr_partial_rewards_qwen}, where the language model is OLMo-2-1B-Instruct instead of Qwen3-1.7B.
        Similarly to \zcref{fig:rlvr_partial_rewards_qwen}, when the initial probability of satisfying one constraint is noticeably higher than the probability of satisfying the other, rewarding partial correctness can cause the policy to stall on satisfying only the easier constraint.
        In this case, rewarding only fully correct outputs can lead to faster learning of both constraints.
        However, when both constraints are initially satisfied with similar probability, both reward designs tend to work well.
        See \zcref{app:experiments:details} for further implementation details.
    }
	\label{fig:rlvr_partial_rewards_olmo}
\end{figure}

\begin{figure}[H]
	\vspace{0mm}
	\begin{center}
        \includegraphics[width=1\textwidth]{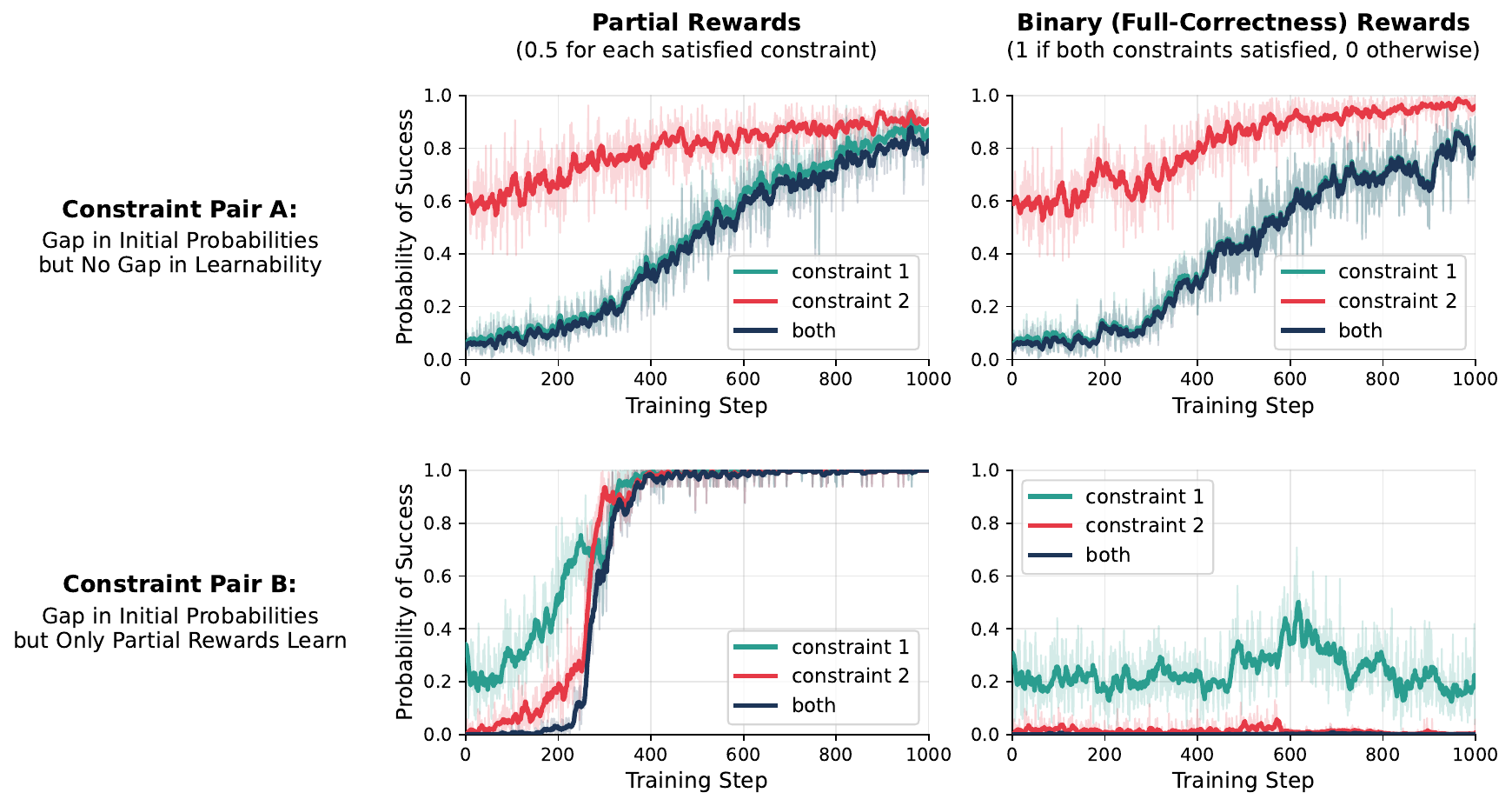}
	\end{center}
	\vspace{-2.25mm}
	\caption{
        \textbf{The probability of initially satisfying a constraint is not the sole factor determining its ease of learnability.}
        This figure supplements \zcref{fig:rlvr_partial_rewards_qwen} by demonstrating that when one constraint is initially satisfied with substantially higher probability than the other, it is possible for both constraints to be learned at a similar rate under partial rewards.
        In such cases, under binary rewards, we find that the constraints can either also be learned relatively in tandem (top row) or fail to be learned (bottom row).
        The latter often occurs when the initial probability of satisfying both constraints is near-zero.
        These experiments were conducted with the Qwen3-1.7B language model and the same training procedure as in \zcref{fig:rlvr_partial_rewards_qwen}, but used different constraint pairs.
    }
	\label{fig:rlvr_partial_rewards_qwen_gap_prob_but_partial_works}
\end{figure}

% Llama-3.2-3B-Instruct | UltraFeedback
\begin{table}[H]
\caption{
    Numerical results underlying \zcref{fig:rm_accuracy_with_gt_exps,fig:rm_accuracy_with_gt_exps_test_rewards} for the Llama-3.2-3B-Instruct language model.
    Specifically, reported is the ground truth reward increase (GT Increase) on train and test prompts (mean and standard deviation over three separate runs), along with accuracy values for all considered reward models.
    We measure accuracy values on examples from the policy gradient training set, with preference labels provided by the ArmoRM ground truth reward model.
}
\begin{center}
\fontsize{7.5}{9}\selectfont
\begin{tabular}{lcccccc}
\toprule
\addlinespace[1.5mm]
\multicolumn{7}{c}{\textbf{Language Model:} Llama-3.2-3B-Instruct} \\
\addlinespace[1.1mm]
\toprule
\addlinespace[1.5mm]
 & \multicolumn{6}{c}{Dataset: UltraFeedback} \\
\cmidrule(lr){2-7}
Reward Model (used for training) & GT Increase (train) & GT Increase (test) & Acc & Acc-W & HAcc & HAcc-W \\
\midrule
GRM-Llama3.2-3B-rewardmodel-ft & $\mathbf{0.374}\,{\scriptstyle \pm 0.011}$ & $\mathbf{0.369}\,{\scriptstyle \pm 0.002}$ & $0.850$ & $0.842$ & $\mathbf{0.930}$ & $\mathbf{0.925}$ \\
internlm2-1\_8b-reward & $0.220\,{\scriptstyle \pm 0.018}$ & $0.223\,{\scriptstyle \pm 0.015}$ & $0.828$ & $0.802$ & $0.892$ & $0.859$ \\
Llama-3.1-8B-Instruct-RM-RB2 & $0.268\,{\scriptstyle \pm 0.005}$ & $0.270\,{\scriptstyle \pm 0.004}$ & $0.810$ & $0.791$ & $0.880$ & $0.853$ \\
Llama-3-OffsetBias-RM-8B & $0.278\,{\scriptstyle \pm 0.007}$ & $0.278\,{\scriptstyle \pm 0.003}$ & $\mathbf{0.864}$ & $\mathbf{0.850}$ & $0.926$ & $0.914$ \\
llama-3-tulu-2-8b-uf-mean-rm & $0.273\,{\scriptstyle \pm 0.014}$ & $0.293\,{\scriptstyle \pm 0.008}$ & $0.792$ & $0.765$ & $0.856$ & $0.842$ \\
RM-Gemma-2B & $0.329\,{\scriptstyle \pm 0.013}$ & $0.334\,{\scriptstyle \pm 0.006}$ & $0.778$ & $0.764$ & $0.878$ & $0.867$ \\
RM-Mistral-7B & $0.290\,{\scriptstyle \pm 0.016}$ & $0.293\,{\scriptstyle \pm 0.008}$ & $0.840$ & $0.794$ & $0.890$ & $0.847$ \\
Skywork-Reward-V2-Llama-3.1-8B & $0.238\,{\scriptstyle \pm 0.010}$ & $0.243\,{\scriptstyle \pm 0.007}$ & $0.796$ & $0.739$ & $0.890$ & $0.874$ \\
Skywork-Reward-V2-Llama-3.2-1B & $0.286\,{\scriptstyle \pm 0.001}$ & $0.292\,{\scriptstyle \pm 0.005}$ & $0.806$ & $0.779$ & $0.920$ & $0.883$ \\
Skywork-Reward-V2-Qwen3-0.6B & $0.277\,{\scriptstyle \pm 0.002}$ & $0.281\,{\scriptstyle \pm 0.007}$ & $0.802$ & $0.766$ & $0.904$ & $0.879$ \\
Skywork-Reward-V2-Qwen3-1.7B & $0.272\,{\scriptstyle \pm 0.004}$ & $0.276\,{\scriptstyle \pm 0.005}$ & $0.826$ & $0.793$ & $0.906$ & $0.872$ \\
Skywork-Reward-V2-Qwen3-4B & $0.268\,{\scriptstyle \pm 0.008}$ & $0.274\,{\scriptstyle \pm 0.008}$ & $0.844$ & $0.809$ & $0.912$ & $0.885$ \\
Skywork-Reward-V2-Qwen3-8B & $0.281\,{\scriptstyle \pm 0.002}$ & $0.285\,{\scriptstyle \pm 0.004}$ & $0.834$ & $0.800$ & $0.924$ & $0.894$ \\
\bottomrule
\end{tabular}
\end{center}
\label{tab:rm-selection-llama-3b-instruct-ultrafeedback}
\end{table}

% Llama-3.2-3B-Instruct | RewardBench2
\begin{table}[H]
\caption{
    For the Llama-3.2-3B-Instruct language model, this table reports reward model accuracy values computed on examples from RewardBench2, with preference labels provided by the ArmoRM ground truth reward model.
    These results were used for producing \zcref{fig:rm_accuracy_with_gt_exps_rb2}.
}
\begin{center}
\fontsize{7.5}{9}\selectfont
\begin{tabular}{lcccc}
\toprule
\addlinespace[1.5mm]
\multicolumn{5}{c}{\textbf{Language Model:} Llama-3.2-3B-Instruct} \\
\addlinespace[1.1mm]
\toprule
\addlinespace[1.5mm]
 & \multicolumn{4}{c}{Dataset: RewardBench2} \\
\cmidrule(lr){2-5}
Reward Model (used for training) & Acc & Acc-W & HAcc & HAcc-W \\
\midrule
GRM-Llama3.2-3B-rewardmodel-ft & $0.698$ & $0.646$ & $0.753$ & $0.684$ \\
internlm2-1\_8b-reward & $0.467$ & $0.467$ & $0.555$ & $0.571$ \\
Llama-3.1-8B-Instruct-RM-RB2 & $\mathbf{0.739}$ & $0.714$ & $\mathbf{0.767}$ & $0.733$ \\
Llama-3-OffsetBias-RM-8B & $0.704$ & $0.693$ & $0.735$ & $0.721$ \\
llama-3-tulu-2-8b-uf-mean-rm & $0.602$ & $0.623$ & $0.626$ & $0.633$ \\
RM-Gemma-2B & $0.387$ & $0.408$ & $0.442$ & $0.439$ \\
RM-Mistral-7B & $0.653$ & $0.630$ & $0.684$ & $0.650$ \\
Skywork-Reward-V2-Llama-3.1-8B & $0.720$ & $\mathbf{0.725}$ & $0.751$ & $\mathbf{0.742}$ \\
Skywork-Reward-V2-Llama-3.2-1B & $0.672$ & $0.638$ & $0.718$ & $0.665$ \\
Skywork-Reward-V2-Qwen3-0.6B & $0.647$ & $0.644$ & $0.681$ & $0.665$ \\
Skywork-Reward-V2-Qwen3-1.7B & $0.690$ & $0.672$ & $0.721$ & $0.688$ \\
Skywork-Reward-V2-Qwen3-4B & $0.712$ & $0.700$ & $0.732$ & $0.717$ \\
Skywork-Reward-V2-Qwen3-8B & $0.720$ & $0.707$ & $0.744$ & $0.724$ \\
\bottomrule
\end{tabular}
\end{center}
\label{tab:rm-selection-llama-3b-instruct-rewardbench2}
\end{table}

% Llama-3.2-1B-Instruct | UltraFeedback
\begin{table}[H]
\caption{
    Numerical results underlying \zcref{fig:rm_accuracy_with_gt_exps,fig:rm_accuracy_with_gt_exps_test_rewards} for the Llama-3.2-1B-Instruct language model.
    Specifically, reported is the ground truth reward increase (GT Increase) on train and test prompts (mean and standard deviation over three separate runs), along with accuracy values for all considered reward models.
    We measure accuracy values on examples from the policy gradient training set, with preference labels provided by the ArmoRM ground truth reward model.
}
\begin{center}
\fontsize{7.5}{9}\selectfont
\begin{tabular}{lcccccc}
\toprule
\addlinespace[1.5mm]
\multicolumn{7}{c}{\textbf{Language Model:} Llama-3.2-1B-Instruct} \\
\addlinespace[1.1mm]
\toprule
\addlinespace[1.5mm]
 & \multicolumn{6}{c}{Dataset: UltraFeedback} \\
\cmidrule(lr){2-7}
Reward Model (used for training) & GT Increase (train) & GT Increase (test) & Acc & Acc-W & HAcc & HAcc-W \\
\midrule
GRM-Llama3.2-3B-rewardmodel-ft & $\mathbf{0.457}\,{\scriptstyle \pm 0.015}$ & $\mathbf{0.465}\,{\scriptstyle \pm 0.014}$ & $0.850$ & $0.835$ & $\mathbf{0.896}$ & $\mathbf{0.895}$ \\
internlm2-1\_8b-reward & $-1.945\,{\scriptstyle \pm 0.100}$ & $-1.960\,{\scriptstyle \pm 0.105}$ & $0.828$ & $0.801$ & $0.850$ & $0.820$ \\
Llama-3.1-8B-Instruct-RM-RB2 & $0.348\,{\scriptstyle \pm 0.010}$ & $0.359\,{\scriptstyle \pm 0.012}$ & $0.810$ & $0.796$ & $0.838$ & $0.822$ \\
Llama-3-OffsetBias-RM-8B & $0.376\,{\scriptstyle \pm 0.006}$ & $0.385\,{\scriptstyle \pm 0.010}$ & $\mathbf{0.864}$ & $\mathbf{0.852}$ & $0.888$ & $0.876$ \\
llama-3-tulu-2-8b-uf-mean-rm & $0.421\,{\scriptstyle \pm 0.012}$ & $0.433\,{\scriptstyle \pm 0.021}$ & $0.792$ & $0.777$ & $0.826$ & $0.809$ \\
RM-Gemma-2B & $0.403\,{\scriptstyle \pm 0.008}$ & $0.425\,{\scriptstyle \pm 0.012}$ & $0.778$ & $0.775$ & $0.864$ & $0.874$ \\
RM-Mistral-7B & $0.402\,{\scriptstyle \pm 0.011}$ & $0.418\,{\scriptstyle \pm 0.009}$ & $0.840$ & $0.819$ & $0.864$ & $0.854$ \\
Skywork-Reward-V2-Llama-3.1-8B & $0.362\,{\scriptstyle \pm 0.005}$ & $0.364\,{\scriptstyle \pm 0.009}$ & $0.796$ & $0.742$ & $0.822$ & $0.773$ \\
Skywork-Reward-V2-Llama-3.2-1B & $0.390\,{\scriptstyle \pm 0.004}$ & $0.398\,{\scriptstyle \pm 0.004}$ & $0.806$ & $0.779$ & $0.876$ & $0.837$ \\
Skywork-Reward-V2-Qwen3-0.6B & $0.384\,{\scriptstyle \pm 0.008}$ & $0.396\,{\scriptstyle \pm 0.008}$ & $0.802$ & $0.773$ & $0.840$ & $0.807$ \\
Skywork-Reward-V2-Qwen3-1.7B & $0.390\,{\scriptstyle \pm 0.010}$ & $0.400\,{\scriptstyle \pm 0.010}$ & $0.826$ & $0.796$ & $0.852$ & $0.814$ \\
Skywork-Reward-V2-Qwen3-4B & $0.383\,{\scriptstyle \pm 0.010}$ & $0.388\,{\scriptstyle \pm 0.014}$ & $0.844$ & $0.823$ & $0.862$ & $0.835$ \\
Skywork-Reward-V2-Qwen3-8B & $0.389\,{\scriptstyle \pm 0.014}$ & $0.393\,{\scriptstyle \pm 0.005}$ & $0.834$ & $0.809$ & $0.850$ & $0.820$ \\
\bottomrule
\end{tabular}
\end{center}
\label{tab:rm-selection-llama-1b-instruct-ultrafeedback}
\end{table}

% Llama-3.2-1B-Instruct | RewardBench2
\begin{table}[H]
\caption{
    For the Llama-3.2-1B-Instruct language model, this table reports reward model accuracy values computed on examples from RewardBench2, with preference labels provided by the ArmoRM ground truth reward model.
    These results were used for producing \zcref{fig:rm_accuracy_with_gt_exps_rb2}.
}
\begin{center}
\fontsize{7.5}{9}\selectfont
\begin{tabular}{lcccc}
\toprule
\addlinespace[1.5mm]
\multicolumn{5}{c}{\textbf{Language Model:} Llama-3.2-1B-Instruct} \\
\addlinespace[1.1mm]
\toprule
\addlinespace[1.5mm]
 & \multicolumn{4}{c}{Dataset: RewardBench2} \\
\cmidrule(lr){2-5}
Reward Model (used for training) & Acc & Acc-W & HAcc & HAcc-W \\
\midrule
GRM-Llama3.2-3B-rewardmodel-ft & $0.696$ & $0.642$ & $0.714$ & $0.654$ \\
internlm2-1\_8b-reward & $0.467$ & $0.462$ & $0.523$ & $0.530$ \\
Llama-3.1-8B-Instruct-RM-RB2 & $\mathbf{0.739}$ & $0.711$ & $\mathbf{0.746}$ & $0.716$ \\
Llama-3-OffsetBias-RM-8B & $0.704$ & $0.693$ & $0.710$ & $0.699$ \\
llama-3-tulu-2-8b-uf-mean-rm & $0.602$ & $0.621$ & $0.612$ & $0.623$ \\
RM-Gemma-2B & $0.387$ & $0.417$ & $0.427$ & $0.445$ \\
RM-Mistral-7B & $0.653$ & $0.631$ & $0.663$ & $0.637$ \\
Skywork-Reward-V2-Llama-3.1-8B & $0.720$ & $\mathbf{0.727}$ & $0.729$ & $\mathbf{0.730}$ \\
Skywork-Reward-V2-Llama-3.2-1B & $0.672$ & $0.634$ & $0.693$ & $0.646$ \\
Skywork-Reward-V2-Qwen3-0.6B & $0.647$ & $0.649$ & $0.662$ & $0.657$ \\
Skywork-Reward-V2-Qwen3-1.7B & $0.690$ & $0.682$ & $0.698$ & $0.685$ \\
Skywork-Reward-V2-Qwen3-4B & $0.712$ & $0.707$ & $0.718$ & $0.709$ \\
Skywork-Reward-V2-Qwen3-8B & $0.720$ & $0.716$ & $0.726$ & $0.719$ \\
\bottomrule
\end{tabular}
\end{center}
\label{tab:rm-selection-llama-1b-instruct-rewardbench2}
\end{table}

% OLMo-2-1B-SFT | UltraFeedback
\begin{table}[H]
\caption{
    Numerical results underlying \zcref{fig:rm_accuracy_with_gt_exps,fig:rm_accuracy_with_gt_exps_test_rewards} for the OLMo-2-1B-SFT language model.
    Specifically, reported is the ground truth reward increase (GT Increase) on train and test prompts (mean and standard deviation over three separate runs), along with accuracy values for all considered reward models.
    We measure accuracy values on examples from the policy gradient training set, with preference labels provided by the ArmoRM ground truth reward model.
}
\begin{center}
\fontsize{7.5}{9}\selectfont
\begin{tabular}{lcccccc}
\toprule
\addlinespace[1.5mm]
\multicolumn{7}{c}{\textbf{Language Model:} OLMo-2-1B-SFT} \\
\addlinespace[1.1mm]
\toprule
\addlinespace[1.5mm]
 & \multicolumn{6}{c}{Dataset: UltraFeedback} \\
\cmidrule(lr){2-7}
Reward Model (used for training) & GT Increase (train) & GT Increase (test) & Acc & Acc-W & HAcc & HAcc-W \\
\midrule
GRM-Llama3.2-3B-rewardmodel-ft & $0.218\,{\scriptstyle \pm 0.011}$ & $0.198\,{\scriptstyle \pm 0.014}$ & $0.850$ & $0.827$ & $\mathbf{0.880}$ & $0.861$ \\
internlm2-1\_8b-reward & $\mathbf{0.281}\,{\scriptstyle \pm 0.006}$ & $\mathbf{0.251}\,{\scriptstyle \pm 0.011}$ & $0.828$ & $0.793$ & $0.854$ & $0.809$ \\
Llama-3.1-8B-Instruct-RM-RB2 & $0.212\,{\scriptstyle \pm 0.004}$ & $0.191\,{\scriptstyle \pm 0.006}$ & $0.810$ & $0.778$ & $0.828$ & $0.789$ \\
Llama-3-OffsetBias-RM-8B & $0.236\,{\scriptstyle \pm 0.009}$ & $0.209\,{\scriptstyle \pm 0.009}$ & $\mathbf{0.864}$ & $\mathbf{0.856}$ & $0.878$ & $\mathbf{0.863}$ \\
llama-3-tulu-2-8b-uf-mean-rm & $0.260\,{\scriptstyle \pm 0.006}$ & $0.234\,{\scriptstyle \pm 0.002}$ & $0.792$ & $0.759$ & $0.814$ & $0.782$ \\
RM-Gemma-2B & $0.275\,{\scriptstyle \pm 0.004}$ & $0.242\,{\scriptstyle \pm 0.006}$ & $0.778$ & $0.760$ & $0.814$ & $0.800$ \\
RM-Mistral-7B & $0.266\,{\scriptstyle \pm 0.006}$ & $0.233\,{\scriptstyle \pm 0.004}$ & $0.840$ & $0.798$ & $0.846$ & $0.809$ \\
Skywork-Reward-V2-Llama-3.1-8B & $0.213\,{\scriptstyle \pm 0.002}$ & $0.192\,{\scriptstyle \pm 0.007}$ & $0.796$ & $0.748$ & $0.818$ & $0.769$ \\
Skywork-Reward-V2-Llama-3.2-1B & $0.234\,{\scriptstyle \pm 0.006}$ & $0.212\,{\scriptstyle \pm 0.004}$ & $0.806$ & $0.763$ & $0.826$ & $0.776$ \\
Skywork-Reward-V2-Qwen3-0.6B & $0.238\,{\scriptstyle \pm 0.001}$ & $0.216\,{\scriptstyle \pm 0.005}$ & $0.802$ & $0.757$ & $0.824$ & $0.771$ \\
Skywork-Reward-V2-Qwen3-1.7B & $0.241\,{\scriptstyle \pm 0.001}$ & $0.219\,{\scriptstyle \pm 0.009}$ & $0.826$ & $0.784$ & $0.840$ & $0.794$ \\
Skywork-Reward-V2-Qwen3-4B & $0.225\,{\scriptstyle \pm 0.009}$ & $0.206\,{\scriptstyle \pm 0.002}$ & $0.844$ & $0.807$ & $0.856$ & $0.814$ \\
Skywork-Reward-V2-Qwen3-8B & $0.226\,{\scriptstyle \pm 0.008}$ & $0.204\,{\scriptstyle \pm 0.006}$ & $0.834$ & $0.794$ & $0.846$ & $0.802$ \\
\bottomrule
\end{tabular}
\end{center}
\label{tab:rm-selection-olmo-1b-sft-ultrafeedback}
\end{table}

% OLMo-2-1B-SFT | RewardBench2
\begin{table}[H]
\caption{
    For the OLMo-2-1B-SFT language model, this table reports reward model accuracy values computed on examples from RewardBench2, with preference labels provided by the ArmoRM ground truth reward model.
    These results were used for producing \zcref{fig:rm_accuracy_with_gt_exps_rb2}.    
}
\begin{center}
\fontsize{7.5}{9}\selectfont
\begin{tabular}{lcccc}
\toprule
\addlinespace[1.5mm]
\multicolumn{5}{c}{\textbf{Language Model:} OLMo-2-1B-SFT} \\
\addlinespace[1.1mm]
\toprule
\addlinespace[1.5mm]
 & \multicolumn{4}{c}{Dataset: RewardBench2} \\
\cmidrule(lr){2-5}
Reward Model (used for training) & Acc & Acc-W & HAcc & HAcc-W \\
\midrule
GRM-Llama3.2-3B-rewardmodel-ft & $0.691$ & $0.651$ & $0.708$ & $0.658$ \\
internlm2-1\_8b-reward & $0.467$ & $0.465$ & $0.526$ & $0.527$ \\
Llama-3.1-8B-Instruct-RM-RB2 & $\mathbf{0.739}$ & $0.715$ & $\mathbf{0.750}$ & $0.721$ \\
Llama-3-OffsetBias-RM-8B & $0.704$ & $0.703$ & $0.708$ & $0.705$ \\
llama-3-tulu-2-8b-uf-mean-rm & $0.602$ & $0.622$ & $0.615$ & $0.624$ \\
RM-Gemma-2B & $0.387$ & $0.414$ & $0.402$ & $0.418$ \\
RM-Mistral-7B & $0.653$ & $0.645$ & $0.659$ & $0.647$ \\
Skywork-Reward-V2-Llama-3.1-8B & $0.720$ & $0.735$ & $0.733$ & $0.742$ \\
Skywork-Reward-V2-Llama-3.2-1B & $0.672$ & $0.653$ & $0.683$ & $0.658$ \\
Skywork-Reward-V2-Qwen3-0.6B & $0.647$ & $0.667$ & $0.666$ & $0.675$ \\
Skywork-Reward-V2-Qwen3-1.7B & $0.690$ & $0.703$ & $0.704$ & $0.708$ \\
Skywork-Reward-V2-Qwen3-4B & $0.712$ & $0.723$ & $0.727$ & $0.730$ \\
Skywork-Reward-V2-Qwen3-8B & $0.720$ & $\mathbf{0.738}$ & $0.733$ & $\mathbf{0.744}$ \\
\bottomrule
\end{tabular}
\end{center}
\label{tab:rm-selection-olmo-1b-sft-rewardbench2}
\end{table}

% Qwen3-1.7B-Base | UltraFeedback
\begin{table}[H]
\caption{
    Numerical results underlying \zcref{fig:rm_accuracy_with_gt_exps,fig:rm_accuracy_with_gt_exps_test_rewards} for the Qwen3-1.7B-Base language model.
    Specifically, reported is the ground truth reward increase (GT Increase) on train and test prompts (mean and standard deviation over three separate runs), along with accuracy values for all considered reward models.
    We measure accuracy values on examples from the policy gradient training set, with preference labels provided by the ArmoRM ground truth reward model.
}
\begin{center}
\fontsize{7.5}{9}\selectfont
\begin{tabular}{lcccccc}
\toprule
\addlinespace[1.5mm]
\multicolumn{7}{c}{\textbf{Language Model:} Qwen3-1.7B-Base} \\
\addlinespace[1.1mm]
\toprule
\addlinespace[1.5mm]
 & \multicolumn{6}{c}{Dataset: UltraFeedback} \\
\cmidrule(lr){2-7}
Reward Model (used for training) & GT Increase (train) & GT Increase (test) & Acc & Acc-W & HAcc & HAcc-W \\
\midrule
GRM-Llama3.2-3B-rewardmodel-ft & $\mathbf{2.257}\,{\scriptstyle \pm 0.009}$ & $\mathbf{2.264}\,{\scriptstyle \pm 0.006}$ & $0.850$ & $0.859$ & $0.850$ & $0.859$ \\
internlm2-1\_8b-reward & $-0.969\,{\scriptstyle \pm 0.038}$ & $-0.964\,{\scriptstyle \pm 0.052}$ & $0.828$ & $0.807$ & $0.832$ & $0.807$ \\
Llama-3.1-8B-Instruct-RM-RB2 & $2.113\,{\scriptstyle \pm 0.007}$ & $2.106\,{\scriptstyle \pm 0.006}$ & $0.810$ & $0.817$ & $0.810$ & $0.817$ \\
Llama-3-OffsetBias-RM-8B & $2.187\,{\scriptstyle \pm 0.050}$ & $2.197\,{\scriptstyle \pm 0.056}$ & $\mathbf{0.864}$ & $\mathbf{0.865}$ & $\mathbf{0.866}$ & $\mathbf{0.868}$ \\
llama-3-tulu-2-8b-uf-mean-rm & $2.206\,{\scriptstyle \pm 0.006}$ & $2.207\,{\scriptstyle \pm 0.004}$ & $0.792$ & $0.794$ & $0.794$ & $0.795$ \\
RM-Gemma-2B & $2.011\,{\scriptstyle \pm 0.074}$ & $2.026\,{\scriptstyle \pm 0.065}$ & $0.778$ & $0.790$ & $0.778$ & $0.790$ \\
RM-Mistral-7B & $2.207\,{\scriptstyle \pm 0.006}$ & $2.205\,{\scriptstyle \pm 0.004}$ & $0.840$ & $0.843$ & $0.840$ & $0.843$ \\
Skywork-Reward-V2-Llama-3.1-8B & $2.179\,{\scriptstyle \pm 0.008}$ & $2.177\,{\scriptstyle \pm 0.007}$ & $0.796$ & $0.756$ & $0.796$ & $0.756$ \\
Skywork-Reward-V2-Llama-3.2-1B & $2.213\,{\scriptstyle \pm 0.004}$ & $2.216\,{\scriptstyle \pm 0.004}$ & $0.806$ & $0.793$ & $0.806$ & $0.793$ \\
Skywork-Reward-V2-Qwen3-0.6B & $2.200\,{\scriptstyle \pm 0.007}$ & $2.213\,{\scriptstyle \pm 0.008}$ & $0.802$ & $0.768$ & $0.802$ & $0.768$ \\
Skywork-Reward-V2-Qwen3-1.7B & $2.203\,{\scriptstyle \pm 0.009}$ & $2.207\,{\scriptstyle \pm 0.005}$ & $0.826$ & $0.803$ & $0.826$ & $0.803$ \\
Skywork-Reward-V2-Qwen3-4B & $2.201\,{\scriptstyle \pm 0.008}$ & $2.199\,{\scriptstyle \pm 0.003}$ & $0.844$ & $0.821$ & $0.844$ & $0.821$ \\
Skywork-Reward-V2-Qwen3-8B & $2.214\,{\scriptstyle \pm 0.008}$ & $2.216\,{\scriptstyle \pm 0.008}$ & $0.834$ & $0.814$ & $0.834$ & $0.814$ \\
\bottomrule
\end{tabular}
\end{center}
\label{tab:rm-selection-qwen-1-7b-base-ultrafeedback}
\end{table}

% Qwen3-1.7B-Base | RewardBench2
\begin{table}[H]
\caption{
    For the Qwen3-1.7B-Base language model, this table reports reward model accuracy values computed on examples from RewardBench2, with preference labels provided by the ArmoRM ground truth reward model.
    These results were used for producing \zcref{fig:rm_accuracy_with_gt_exps_rb2}.
}
\begin{center}
\fontsize{7.5}{9}\selectfont
\begin{tabular}{lcccc}
\toprule
\addlinespace[1.5mm]
\multicolumn{5}{c}{\textbf{Language Model:} Qwen3-1.7B-Base} \\
\addlinespace[1.1mm]
\toprule
\addlinespace[1.5mm]
 & \multicolumn{4}{c}{Dataset: RewardBench2} \\
\cmidrule(lr){2-5}
Reward Model (used for training) & Acc & Acc-W & HAcc & HAcc-W \\
\midrule
GRM-Llama3.2-3B-rewardmodel-ft & $0.691$ & $0.648$ & $0.691$ & $0.648$ \\
internlm2-1\_8b-reward & $0.467$ & $0.470$ & $0.469$ & $0.470$ \\
Llama-3.1-8B-Instruct-RM-RB2 & $\mathbf{0.739}$ & $0.714$ & $\mathbf{0.739}$ & $0.714$ \\
Llama-3-OffsetBias-RM-8B & $0.704$ & $0.697$ & $0.705$ & $0.697$ \\
llama-3-tulu-2-8b-uf-mean-rm & $0.602$ & $0.652$ & $0.602$ & $0.652$ \\
RM-Gemma-2B & $0.387$ & $0.421$ & $0.387$ & $0.421$ \\
RM-Mistral-7B & $0.653$ & $0.652$ & $0.653$ & $0.652$ \\
Skywork-Reward-V2-Llama-3.1-8B & $0.720$ & $0.729$ & $0.720$ & $0.729$ \\
Skywork-Reward-V2-Llama-3.2-1B & $0.672$ & $0.655$ & $0.672$ & $0.655$ \\
Skywork-Reward-V2-Qwen3-0.6B & $0.647$ & $0.669$ & $0.647$ & $0.669$ \\
Skywork-Reward-V2-Qwen3-1.7B & $0.690$ & $0.700$ & $0.690$ & $0.700$ \\
Skywork-Reward-V2-Qwen3-4B & $0.712$ & $0.723$ & $0.712$ & $0.723$ \\
Skywork-Reward-V2-Qwen3-8B & $0.720$ & $\mathbf{0.732}$ & $0.720$ & $\mathbf{0.732}$ \\
\bottomrule
\end{tabular}
\end{center}
\label{tab:rm-selection-qwen-1-7b-base-rewardbench2}
\end{table}

% Llama-1B-Instruct Different Dataset | WildChat-IF
\begin{table}[H]
\caption{
    Numerical results underlying \zcref{fig:rm_accuracy_with_gt_exps_diff_dataset_diff_gt} for the Llama-3.2-1B-Instruct language model trained on the WildChat-IF-On-Policy-8B dataset.
    Specifically, reported is the ground truth reward increase (GT Increase) on train and test prompts (mean and standard deviation over three separate runs), along with accuracy values for all considered reward models.
    We measure accuracy values on examples from the policy gradient training set, with preference labels provided by the ArmoRM ground truth reward model.
}
\begin{center}
\fontsize{7.5}{9}\selectfont
\begin{tabular}{lcccccc}
\toprule
\addlinespace[1.5mm]
\multicolumn{7}{c}{\textbf{Language Model:} Llama-3.2-1B-Instruct} \\
\addlinespace[1.1mm]
\toprule
\addlinespace[1.5mm]
 & \multicolumn{6}{c}{Dataset: WildChat-IF-On-Policy-8B} \\
\cmidrule(lr){2-7}
Reward Model (used for training) & GT Increase (train) & GT Increase (test) & Acc & Acc-W & HAcc & HAcc-W \\
\midrule
GRM-Llama3.2-3B-rewardmodel-ft & $0.361\,{\scriptstyle \pm 0.033}$ & $0.338\,{\scriptstyle \pm 0.035}$ & $0.838$ & $0.841$ & $0.852$ & $0.847$ \\
internlm2-1\_8b-reward & $-0.689\,{\scriptstyle \pm 0.924}$ & $-0.693\,{\scriptstyle \pm 0.915}$ & $0.740$ & $0.732$ & $0.762$ & $0.748$ \\
Llama-3.1-8B-Instruct-RM-RB2 & $0.297\,{\scriptstyle \pm 0.049}$ & $0.277\,{\scriptstyle \pm 0.049}$ & $0.848$ & $0.826$ & $0.858$ & $0.832$ \\
Llama-3-OffsetBias-RM-8B & $0.360\,{\scriptstyle \pm 0.049}$ & $0.329\,{\scriptstyle \pm 0.035}$ & $0.846$ & $0.824$ & $0.850$ & $0.826$ \\
llama-3-tulu-2-8b-uf-mean-rm & $0.380\,{\scriptstyle \pm 0.050}$ & $0.350\,{\scriptstyle \pm 0.041}$ & $0.782$ & $0.777$ & $0.792$ & $0.779$ \\
RM-Gemma-2B & $\mathbf{0.391}\,{\scriptstyle \pm 0.057}$ & $\mathbf{0.369}\,{\scriptstyle \pm 0.051}$ & $0.738$ & $0.735$ & $0.776$ & $0.773$ \\
RM-Mistral-7B & $0.356\,{\scriptstyle \pm 0.045}$ & $0.324\,{\scriptstyle \pm 0.041}$ & $0.838$ & $0.827$ & $0.848$ & $0.834$ \\
Skywork-Reward-V2-Llama-3.1-8B & $0.306\,{\scriptstyle \pm 0.054}$ & $0.278\,{\scriptstyle \pm 0.048}$ & $0.856$ & $0.840$ & $0.868$ & $0.842$ \\
Skywork-Reward-V2-Llama-3.2-1B & $0.329\,{\scriptstyle \pm 0.054}$ & $0.318\,{\scriptstyle \pm 0.041}$ & $0.854$ & $0.845$ & $0.872$ & $0.863$ \\
Skywork-Reward-V2-Qwen3-0.6B & $0.321\,{\scriptstyle \pm 0.064}$ & $0.302\,{\scriptstyle \pm 0.051}$ & $0.818$ & $0.832$ & $0.836$ & $0.850$ \\
Skywork-Reward-V2-Qwen3-1.7B & $0.329\,{\scriptstyle \pm 0.058}$ & $0.299\,{\scriptstyle \pm 0.044}$ & $0.844$ & $0.836$ & $0.850$ & $0.838$ \\
Skywork-Reward-V2-Qwen3-4B & $0.330\,{\scriptstyle \pm 0.039}$ & $0.297\,{\scriptstyle \pm 0.035}$ & $\mathbf{0.876}$ & $0.864$ & $0.878$ & $0.864$ \\
Skywork-Reward-V2-Qwen3-8B & $0.344\,{\scriptstyle \pm 0.042}$ & $0.315\,{\scriptstyle \pm 0.036}$ & $0.874$ & $\mathbf{0.869}$ & $\mathbf{0.880}$ & $\mathbf{0.870}$ \\
\bottomrule
\end{tabular}
\end{center}
\label{tab:rm-selection-llama-1b-instruct-different-dataset-wildchat-if}
\end{table}

% Llama-1B-Instruct Different Ground Truth | UltraFeedback (Skywork Reward V2 Llama 3 1 8B)
\begin{table}[H]
\caption{
    Numerical results underlying \zcref{fig:rm_accuracy_with_gt_exps_diff_dataset_diff_gt} for the Llama-3.2-1B-Instruct language model.
    Specifically, reported is the ground truth reward increase (GT Increase) on train and test prompts (mean and standard deviation over three separate runs), along with accuracy values for all considered reward models.
    We measure accuracy values on examples from the policy gradient training set, with preference labels provided by the Skywork-Reward-V2-Llama-3.1-8B ground truth reward model.
}
\begin{center}
\fontsize{7.5}{9}\selectfont
\begin{tabular}{lcccccc}
\toprule
\addlinespace[1.5mm]
\multicolumn{7}{c}{\textbf{Language Model:} Llama-3.2-1B-Instruct} \\
\addlinespace[1.1mm]
\toprule
\addlinespace[1.5mm]
 & \multicolumn{6}{c}{Dataset: UltraFeedback (Ground Truth: Skywork-Reward-V2-Llama-3.1-8B)} \\
\cmidrule(lr){2-7}
Reward Model (used for training) & GT Increase (train) & GT Increase (test) & Acc & Acc-W & HAcc & HAcc-W \\
\midrule
ArmoRM-Llama3-8B-v0-1 & $0.420\,{\scriptstyle \pm 0.105}$ & $0.389\,{\scriptstyle \pm 0.107}$ & $0.810$ & $0.779$ & $0.834$ & $0.811$ \\
GRM-Llama3.2-3B-rewardmodel-ft & $0.466\,{\scriptstyle \pm 0.084}$ & $0.462\,{\scriptstyle \pm 0.103}$ & $0.808$ & $0.805$ & $0.842$ & $0.837$ \\
internlm2-1\_8b-reward & $-1.768\,{\scriptstyle \pm 0.221}$ & $-1.824\,{\scriptstyle \pm 0.247}$ & $0.774$ & $0.721$ & $0.802$ & $0.745$ \\
Llama-3.1-8B-Instruct-RM-RB2 & $0.416\,{\scriptstyle \pm 0.082}$ & $0.406\,{\scriptstyle \pm 0.078}$ & $0.862$ & $0.804$ & $0.878$ & $0.822$ \\
Llama-3-OffsetBias-RM-8B & $0.440\,{\scriptstyle \pm 0.062}$ & $0.417\,{\scriptstyle \pm 0.072}$ & $0.826$ & $0.814$ & $0.852$ & $0.836$ \\
llama-3-tulu-2-8b-uf-mean-rm & $0.414\,{\scriptstyle \pm 0.146}$ & $0.397\,{\scriptstyle \pm 0.152}$ & $0.810$ & $0.767$ & $0.826$ & $0.776$ \\
RM-Gemma-2B & $0.330\,{\scriptstyle \pm 0.176}$ & $0.338\,{\scriptstyle \pm 0.178}$ & $0.740$ & $0.703$ & $0.832$ & $0.785$ \\
RM-Mistral-7B & $0.471\,{\scriptstyle \pm 0.116}$ & $0.453\,{\scriptstyle \pm 0.119}$ & $0.824$ & $0.801$ & $0.848$ & $0.816$ \\
Skywork-Reward-V2-Llama-3.2-1B & $\mathbf{0.533}\,{\scriptstyle \pm 0.095}$ & $\mathbf{0.522}\,{\scriptstyle \pm 0.106}$ & $0.878$ & $0.886$ & $\mathbf{0.920}$ & $\mathbf{0.914}$ \\
Skywork-Reward-V2-Qwen3-0.6B & $0.503\,{\scriptstyle \pm 0.101}$ & $0.498\,{\scriptstyle \pm 0.101}$ & $0.852$ & $0.850$ & $0.898$ & $0.890$ \\
Skywork-Reward-V2-Qwen3-1.7B & $0.510\,{\scriptstyle \pm 0.101}$ & $0.489\,{\scriptstyle \pm 0.097}$ & $0.856$ & $0.864$ & $0.904$ & $0.907$ \\
Skywork-Reward-V2-Qwen3-4B & $0.503\,{\scriptstyle \pm 0.078}$ & $0.479\,{\scriptstyle \pm 0.086}$ & $0.894$ & $0.881$ & $0.912$ & $0.899$ \\
Skywork-Reward-V2-Qwen3-8B & $0.504\,{\scriptstyle \pm 0.080}$ & $0.488\,{\scriptstyle \pm 0.078}$ & $\mathbf{0.896}$ & $\mathbf{0.893}$ & $0.908$ & $0.908$ \\
\bottomrule
\end{tabular}
\end{center}
\label{tab:rm-selection-llama-1b-instruct-different-ground-truth-ultrafeedback-skywork-reward-v2-Llama-3.1-8b}
\end{table}

% Llama-3.2-3B-Instruct | UltraFeedback Win-Rate
\begin{table}[H]
\caption{
    Numerical results underlying \zcref{fig:rm_accuracy_with_winrates} for the Llama-3.2-3B-Instruct language model.
    Specifically, for each reward model, the table reports the win-rate of the language model after policy gradient against the initial language model, as judged by a frontier GPT model, along with the accuracy values for the reward model.
    We measure accuracy values on examples from the policy gradient training set, with preference labels provided by the ArmoRM ground truth reward model.
}
\begin{center}
\fontsize{7.5}{9}\selectfont
\begin{tabular}{lccccc}
\toprule
\addlinespace[1.5mm]
\multicolumn{6}{c}{\textbf{Language Model:} Llama-3.2-3B-Instruct} \\
\addlinespace[1.1mm]
\toprule
 & \multicolumn{5}{c}{Dataset: UltraFeedback} \\
\cmidrule(lr){2-6}
Reward Model (used for training) & Win-Rate & Acc & Acc-W & HAcc & HAcc-W \\
\midrule
GRM-Llama3.2-3B-rewardmodel-ft & $0.738$ & $0.850$ & $0.842$ & $\mathbf{0.930}$ & $\mathbf{0.925}$ \\
internlm2-1\_8b-reward & $0.704$ & $0.828$ & $0.802$ & $0.892$ & $0.859$ \\
Llama-3.1-8B-Instruct-RM-RB2 & $0.715$ & $0.810$ & $0.791$ & $0.880$ & $0.853$ \\
Llama-3-OffsetBias-RM-8B & $0.646$ & $\mathbf{0.864}$ & $\mathbf{0.850}$ & $0.926$ & $0.914$ \\
llama-3-tulu-2-8b-uf-mean-rm & $0.702$ & $0.792$ & $0.765$ & $0.856$ & $0.842$ \\
RM-Gemma-2B & $\mathbf{0.769}$ & $0.778$ & $0.764$ & $0.878$ & $0.867$ \\
RM-Mistral-7B & $0.726$ & $0.840$ & $0.794$ & $0.890$ & $0.847$ \\
Skywork-Reward-V2-Llama-3.1-8B & $0.677$ & $0.796$ & $0.739$ & $0.890$ & $0.874$ \\
Skywork-Reward-V2-Llama-3.2-1B & $0.720$ & $0.806$ & $0.779$ & $0.920$ & $0.883$ \\
Skywork-Reward-V2-Qwen3-0.6B & $0.735$ & $0.802$ & $0.766$ & $0.904$ & $0.879$ \\
Skywork-Reward-V2-Qwen3-1.7B & $0.713$ & $0.826$ & $0.793$ & $0.906$ & $0.872$ \\
Skywork-Reward-V2-Qwen3-4B & $0.724$ & $0.844$ & $0.809$ & $0.912$ & $0.885$ \\
Skywork-Reward-V2-Qwen3-8B & $0.702$ & $0.834$ & $0.800$ & $0.924$ & $0.894$ \\
\bottomrule
\end{tabular}
\end{center}
\label{tab:rm-selection-llama-3b-instruct-ultrafeedback-win-rate}
\end{table}

% Llama-3.2-1B-Instruct | UltraFeedback Win-Rate
\begin{table}[H]
\caption{
    Numerical results underlying \zcref{fig:rm_accuracy_with_winrates} for the Llama-3.2-1B-Instruct language model.
    Specifically, for each reward model, the table reports the win-rate of the language model after policy gradient against the initial language model, as judged by a frontier GPT model, along with the accuracy values for the reward model.
    We measure accuracy values on examples from the policy gradient training set, with preference labels provided by the ArmoRM ground truth reward model.
}
\begin{center}
\fontsize{7.5}{9}\selectfont
\begin{tabular}{lccccc}
\toprule
\addlinespace[1.5mm]
\multicolumn{6}{c}{\textbf{Language Model:} Llama-3.2-1B-Instruct} \\
\addlinespace[1.1mm]    
\toprule
 & \multicolumn{5}{c}{Dataset: UltraFeedback} \\
\cmidrule(lr){2-6}
Reward Model (used for training) & Win-Rate & Acc & Acc-W & HAcc & HAcc-W \\
\midrule
GRM-Llama3.2-3B-rewardmodel-ft & $0.737$ & $0.850$ & $0.835$ & $\mathbf{0.896}$ & $\mathbf{0.895}$ \\
internlm2-1\_8b-reward & $0.009$ & $0.828$ & $0.801$ & $0.850$ & $0.820$ \\
Llama-3.1-8B-Instruct-RM-RB2 & $0.701$ & $0.810$ & $0.796$ & $0.838$ & $0.822$ \\
Llama-3-OffsetBias-RM-8B & $0.708$ & $\mathbf{0.864}$ & $\mathbf{0.852}$ & $0.888$ & $0.876$ \\
llama-3-tulu-2-8b-uf-mean-rm & $0.722$ & $0.792$ & $0.777$ & $0.826$ & $0.809$ \\
RM-Gemma-2B & $0.733$ & $0.778$ & $0.775$ & $0.864$ & $0.874$ \\
RM-Mistral-7B & $\mathbf{0.755}$ & $0.840$ & $0.819$ & $0.864$ & $0.854$ \\
Skywork-Reward-V2-Llama-3.1-8B & $0.695$ & $0.796$ & $0.742$ & $0.822$ & $0.773$ \\
Skywork-Reward-V2-Llama-3.2-1B & $0.741$ & $0.806$ & $0.779$ & $0.876$ & $0.837$ \\
Skywork-Reward-V2-Qwen3-0.6B & $0.725$ & $0.802$ & $0.773$ & $0.840$ & $0.807$ \\
Skywork-Reward-V2-Qwen3-1.7B & $0.704$ & $0.826$ & $0.796$ & $0.852$ & $0.814$ \\
Skywork-Reward-V2-Qwen3-4B & $0.712$ & $0.844$ & $0.823$ & $0.862$ & $0.835$ \\
Skywork-Reward-V2-Qwen3-8B & $0.716$ & $0.834$ & $0.809$ & $0.850$ & $0.820$ \\
\bottomrule
\end{tabular}
\end{center}
\label{tab:rm-selection-llama-1b-instruct-ultrafeedback-win-rate}
\end{table}

% OLMo-2-1B-SFT | UltraFeedback Win-Rate
\begin{table}[H]
\caption{
    Numerical results underlying \zcref{fig:rm_accuracy_with_winrates} for the OLMo-2-1B-SFT language model.
    Specifically, for each reward model, the table reports the win-rate of the language model after policy gradient against the initial language model, as judged by a frontier GPT model, along with the accuracy values for the reward model.
    We measure accuracy values on examples from the policy gradient training set, with preference labels provided by the ArmoRM ground truth reward model.
}
\begin{center}
\fontsize{7.5}{9}\selectfont
\begin{tabular}{lccccc}
\toprule
\addlinespace[1.5mm]
\multicolumn{6}{c}{\textbf{Language Model:} OLMo-2-1B-SFT} \\
\addlinespace[1.1mm]
\toprule
 & \multicolumn{5}{c}{Dataset: UltraFeedback} \\
\cmidrule(lr){2-6}
Reward Model (used for training) & Win-Rate & Acc & Acc-W & HAcc & HAcc-W \\
\midrule
GRM-Llama3.2-3B-rewardmodel-ft & $0.616$ & $0.850$ & $0.827$ & $\mathbf{0.880}$ & $0.861$ \\
internlm2-1\_8b-reward & $0.628$ & $0.828$ & $0.793$ & $0.854$ & $0.809$ \\
Llama-3.1-8B-Instruct-RM-RB2 & $0.607$ & $0.810$ & $0.778$ & $0.828$ & $0.789$ \\
Llama-3-OffsetBias-RM-8B & $0.611$ & $\mathbf{0.864}$ & $\mathbf{0.856}$ & $0.878$ & $\mathbf{0.863}$ \\
llama-3-tulu-2-8b-uf-mean-rm & $0.627$ & $0.792$ & $0.759$ & $0.814$ & $0.782$ \\
RM-Gemma-2B & $0.642$ & $0.778$ & $0.760$ & $0.814$ & $0.800$ \\
RM-Mistral-7B & $\mathbf{0.657}$ & $0.840$ & $0.798$ & $0.846$ & $0.809$ \\
Skywork-Reward-V2-Llama-3.1-8B & $0.610$ & $0.796$ & $0.748$ & $0.818$ & $0.769$ \\
Skywork-Reward-V2-Llama-3.2-1B & $0.611$ & $0.806$ & $0.763$ & $0.826$ & $0.776$ \\
Skywork-Reward-V2-Qwen3-0.6B & $0.634$ & $0.802$ & $0.757$ & $0.824$ & $0.771$ \\
Skywork-Reward-V2-Qwen3-1.7B & $0.624$ & $0.826$ & $0.784$ & $0.840$ & $0.794$ \\
Skywork-Reward-V2-Qwen3-4B & $0.610$ & $0.844$ & $0.807$ & $0.856$ & $0.814$ \\
Skywork-Reward-V2-Qwen3-8B & $0.586$ & $0.834$ & $0.794$ & $0.846$ & $0.802$ \\
\bottomrule
\end{tabular}
\end{center}
\label{tab:rm-selection-olmo-1b-sft-ultrafeedback-win-rate}
\end{table}

% Qwen3-1.7B-Base | UltraFeedback Win-Rate
\begin{table}[H]
\caption{
        Numerical results underlying \zcref{fig:rm_accuracy_with_winrates} for the Qwen3-1.7B-Base language model.
        Specifically, for each reward model, the table reports the win-rate of the language model after policy gradient against the initial language model, as judged by a frontier GPT model, along with the accuracy values for the reward model.
        We measure accuracy values on examples from the policy gradient training set, with preference labels provided by the ArmoRM ground truth reward model.
     }
\begin{center}
\fontsize{7.5}{9}\selectfont
\begin{tabular}{lccccc}
\toprule
\addlinespace[1.5mm]
\multicolumn{6}{c}{\textbf{Language Model:} Qwen3-1.7B-Base} \\
\addlinespace[1.1mm]    
\toprule
 & \multicolumn{5}{c}{Dataset: UltraFeedback} \\
\cmidrule(lr){2-6}
Reward Model (used for training) & Win-Rate & Acc & Acc-W & HAcc & HAcc-W \\
\midrule
GRM-Llama3.2-3B-rewardmodel-ft & $0.957$ & $0.850$ & $0.859$ & $0.850$ & $0.859$ \\
internlm2-1\_8b-reward & $0.191$ & $0.828$ & $0.807$ & $0.832$ & $0.807$ \\
Llama-3.1-8B-Instruct-RM-RB2 & $0.963$ & $0.810$ & $0.817$ & $0.810$ & $0.817$ \\
Llama-3-OffsetBias-RM-8B & $0.962$ & $\mathbf{0.864}$ & $\mathbf{0.865}$ & $\mathbf{0.866}$ & $\mathbf{0.868}$ \\
llama-3-tulu-2-8b-uf-mean-rm & $0.936$ & $0.792$ & $0.794$ & $0.794$ & $0.795$ \\
RM-Gemma-2B & $0.919$ & $0.778$ & $0.790$ & $0.778$ & $0.790$ \\
RM-Mistral-7B & $0.955$ & $0.840$ & $0.843$ & $0.840$ & $0.843$ \\
Skywork-Reward-V2-Llama-3.1-8B & $0.957$ & $0.796$ & $0.756$ & $0.796$ & $0.756$ \\
Skywork-Reward-V2-Llama-3.2-1B & $0.961$ & $0.806$ & $0.793$ & $0.806$ & $0.793$ \\
Skywork-Reward-V2-Qwen3-0.6B & $0.951$ & $0.802$ & $0.768$ & $0.802$ & $0.768$ \\
Skywork-Reward-V2-Qwen3-1.7B & $0.946$ & $0.826$ & $0.803$ & $0.826$ & $0.803$ \\
Skywork-Reward-V2-Qwen3-4B & $0.959$ & $0.844$ & $0.821$ & $0.844$ & $0.821$ \\
Skywork-Reward-V2-Qwen3-8B & $\mathbf{0.965}$ & $0.834$ & $0.814$ & $0.834$ & $0.814$ \\
\bottomrule
\end{tabular}
\end{center}
\label{tab:rm-selection-qwen-1-7b-base-ultrafeedback-win-rate}
\end{table}

\begin{table}[H]
\caption{
This table specifies the IFBench constraint pairs used in each of the experiments corresponding to \zcref{fig:rlvr_partial_rewards_qwen,fig:rlvr_partial_rewards_llama,fig:rlvr_partial_rewards_olmo,fig:rlvr_partial_rewards_qwen_gap_prob_but_partial_works}.
We include the constraint name according to the IFBench taxonomy (\eg, \texttt{words:start\_verb}), followed by the instruction that is appended to the prompt.
}
\begin{center}
\fontsize{7.5}{9}\selectfont
\setlength{\tabcolsep}{4pt}
\begin{tabular}{c c p{2.9cm} p{2.9cm}}
\toprule
\textbf{Language Model} &
\textbf{Constraint Pair} &
\multicolumn{1}{c}{\textbf{Constraint 1}} &
\multicolumn{1}{c}{\textbf{Constraint 2}} \\
\midrule
\noalign{\vskip 2pt}
\multirow{2}{*}[-28pt]{\shortstack[c]{Qwen3-1.7B\\ \zcref{fig:rlvr_partial_rewards_qwen}}}
& Pair A
& \texttt{words:start\_verb}\newline The response must start with a verb.
& \texttt{words:alphabet}\newline Each word must start with the next letter of the alphabet, looping back to `A' after `Z'.
\\
\noalign{\vskip 4pt}
\cline{2-4}
\noalign{\vskip 4pt}
& Pair B
& \texttt{count:numbers}\newline Include exactly 13 numbers in the response.
& \texttt{format:newline}\newline Write each word on a new line.
\\
\noalign{\vskip 4pt}
\midrule
\noalign{\vskip 2pt}
\multirow{2}{*}[-28pt]{\shortstack[c]{Llama-3.2-3B-Instruct\\ \zcref{fig:rlvr_partial_rewards_llama}}}
& Pair A
& \texttt{count:pronouns}\newline The response should include at least 5 pronouns.
& \texttt{words:last\_first}\newline The last word of each sentence must become the first word of the next sentence.
\\
\noalign{\vskip 4pt}
\cline{2-4}
\noalign{\vskip 4pt}
& Pair B
& \texttt{words:start\_verb}\newline The response must start with a verb.
& \texttt{count:pronouns}\newline The response should include at least 5 pronouns.
\\
\noalign{\vskip 4pt}
\midrule
\noalign{\vskip 2pt}
\multirow{2}{*}[-28pt]{\shortstack[c]{OLMo-2-1B-Instruct\\ \zcref{fig:rlvr_partial_rewards_olmo}}}
& Pair A
& \texttt{count:pronouns}\newline The response should include at least 5 pronouns.
& \texttt{sentence:increment}\newline Each sentence must contain exactly 5 more words than the previous one.
\\
\noalign{\vskip 4pt}
\cline{2-4}
\noalign{\vskip 4pt}
& Pair B
& \texttt{format:quotes}\newline Include quotes within quotes within quotes, at least 3 levels deep, alternating between double quotes and single quotes.
& \texttt{words:repeats}\newline The response should not repeat any word more than 10 times.
\\
\noalign{\vskip 4pt}
\midrule
\noalign{\vskip 2pt}
\multirow{2}{*}[-28pt]{\shortstack[c]{Qwen3-1.7B\\ \zcref{fig:rlvr_partial_rewards_qwen_gap_prob_but_partial_works}}}
& Pair A
& \texttt{count:punctuation}\newline Use every standard punctuation mark at least once, including semicolons, colons, and the interrobang (?!).
& \texttt{count:person\_names}\newline Mention at least 7 different person names in the response, from this list of person names: Emma, Liam, Sophia...
\\
\noalign{\vskip 4pt}
\cline{2-4}
\noalign{\vskip 4pt}
& Pair B
& \texttt{words:start\_verb}\newline The response must start with a verb.
& \texttt{custom:multiples}\newline Count from 10 to 50 but only print multiples of 7.
\\
\noalign{\vskip 2pt}
\bottomrule
\end{tabular}
\end{center}
\label{tab:lm-constraints-summary}
\end{table}

	% % CHECKLIST
	% \newpage
	% \input{checklist}

\end{document}